\documentclass[notitlepage,12pt]{article}
\usepackage{tablefootnote}
\usepackage{grffile}
\usepackage{enumerate}
\usepackage[OT1]{fontenc}
\usepackage{multirow}
\usepackage{amssymb,mathrsfs}
\usepackage{graphicx}
\usepackage{mathtools}
\usepackage[usenames]{color}
\usepackage[dvipsnames]{xcolor}
\usepackage{dsfont,makecell}
\usepackage{smile}
\usepackage{multirow,stfloats,booktabs}
\usepackage{caption}
\usepackage{subcaption}
\usepackage[colorlinks,
            linkcolor=red,
            linktoc=all,
            anchorcolor=red,
            citecolor=blue
            ]{hyperref}
\usepackage[toc,page]{appendix}
\usepackage{setspace}
\usepackage{natbib}
\setcitestyle{authoryear}

\everymath{\displaystyle}
\usepackage{chngcntr}
\counterwithout{equation}{section}

\usepackage{tabularx,array}
\newcolumntype{M}{>{\centering\arraybackslash}m{0.17\textwidth}}

\newcommand{\tsinv}{\widetilde \bSigma^{-1}}
\newcommand{\tSig}{\widetilde \bSigma}

\usepackage{hyperref}
\def\##1\#{\begin{align}#1\end{align}}
\def\$#1\${\begin{align*}#1\end{align*}}
\usepackage{enumitem}

\theoremstyle{mytheoremstyle}
\newtheorem{innercustomass}{Assumption}

\def\spacingset#1{\renewcommand{\baselinestretch}%
{#1}\small\normalsize} \spacingset{1}

\newcommand{\blue}[1]{{\leavevmode\color{black}{#1}}}
\newcommand{\cyan}[1]{{\leavevmode\color{black}{#1}}}

\usepackage{xr}
\makeatletter

\begin{document}
\title{\bf Neural networks for
geospatial data}
\author{Wentao Zhan, Abhirup Datta\footnote{Department of Biostatistics, Johns Hopkins University, Email: abhidatta@jhu.edu}.
\hspace{.2cm}\\
Department of Biostatistics, Johns Hopkins University
}
\date{}

\maketitle
\begin{abstract}
Analysis of geospatial data has traditionally been model-based, with a mean model, customarily specified as a linear regression on the covariates, and a Gaussian process covariance model, encoding the spatial dependence. While non-linear machine learning algorithms like neural networks are increasingly being used for spatial analysis, current approaches depart from the model-based setup and cannot explicitly incorporate spatial covariance. 
\blue{We propose {\em NN-GLS}, 
 embedding neural networks 
 directly within the traditional Gaussian process (GP) geostatistical model} to accommodate non-linear mean functions while retaining all other advantages of GP, like explicit modeling of the spatial covariance 
 and predicting at new locations via kriging. \blue{In NN-GLS, estimation of the neural network parameters for the non-linear mean of the Gaussian Process explicitly accounts for the spatial covariance through use of the generalized least squares (GLS) loss, thus extending the linear case.} We show that NN-GLS admits a representation as a special type of graph neural network (GNN). This connection facilitates the use of standard neural network computational techniques for irregular geospatial data, enabling novel and scalable mini-batching, backpropagation, and kriging schemes. \blue{We provide methodology to obtain uncertainty bounds for estimation and predictions from NN-GLS.} Theoretically, we show that NN-GLS will be consistent for irregularly observed spatially correlated data processes. \blue{We also provide a finite sample concentration rates which quantifies the need to accurately model the spatial covariance in neural networks for dependent data.} To our knowledge, these are the first large-sample  results for any neural network algorithm for irregular spatial data. We demonstrate the methodology through numerous simulations and an application to air pollution modeling. We develop a software implementation of NN-GLS in the Python package \href{https://pypi.org/project/geospaNN/}{geospaNN}.
\end{abstract}
\noindent%
{\it Keywords:}  geostatistics, Gaussian process, neural networks, graph neural networks, machine learning, kriging, consistency. 
\vfill

\spacingset{1.4} 
\section{Introduction}\label{sec-intro}

{\em Geostatistics}, the analysis of geocoded data, is traditionally based on stochastic process models, which offer a coherent way to model data at any finite collection of locations while ensuring the generalizability of inference to the entire region. 
{\em Gaussian processes (GP),} with a mean function capturing effects of covariates and a covariance function encoding the spatial dependence, is a staple for geostatistical analysis \citep{stein1999interpolation,banerjee2014hierarchical,cressie2015statistics}. 
The mean function of a Gaussian process is often modeled as a linear regression on the covariates. 
The growing popularity and accessibility of machine learning algorithms such as random forest \citep{breiman2001random}, gradient boosting \citep{freund1999short}, and neural networks \citep{goodfellow2016deep}, capable of modeling complex non-linear relationships has heralded a paradigm shift. Practitioners are increasingly shunning models with parametric assumptions like linearity in favor of machine-learning approaches that can capture non-linearity and high-order interactions. 
In particular, deep neural networks 
have seen considerable recent adoption and adaptation for geospatial data \citep[see][for a comprehensive review]{wikle2023statistical}. 

Many of the machine-learning based regression approaches assume independent observations, implicit in the choice of an additive loss (e.g., {\em ordinary least squares} or {\em OLS} loss) as the objective function used in estimating the algorithm parameters. Explicit encoding of spatial dependency via a covariance function, as is common in process-based geospatial models, is challenging within these algorithms. Current renditions of neural networks for spatial data circumvent this by using spatial co-ordinates or some transformations (like distances or basis functions) as additional covariates 
\citep{gray2022use,chen2020deepkriging,wang2019nearest}. 
These \blue{{\em added-spatial-features} neural networks incorporate all spatial information  into the mean, thus assuming that the error terms are independent. Consequently, they leave the GP model framework, abandoning its many advantages. They can suffer from curse of dimensionality on account of adding many basis functions and cannot provide a separate estimate of the covariate effect $E(Y\given \bX)$ (see Section \ref{sec-mtd-NN}).} 

We propose a novel algorithm to estimate non-linear means \blue{with neural networks} within the traditional Gaussian process models while   explicitly accounting for the spatial dependence encoded in the GP covariance matrix. The core motivation comes from the extension of the ordinary least squares (OLS) to generalized least squares (GLS) for linear models with dependent errors. For correlated data, GLS is more efficient than OLS according to the Gauss-Markov theorem. 
\blue{For a neural network embedded as the mean of a GP, the GLS loss naturally arises as the negative log-likelihood.} 
We thus refer to our algorithm as {\em NN-GLS}. We retain all advantages of the model-based framework, including separation of the covariate and spatial effects thereby allowing inference on both, parsimonious modeling of the spatial effect through the GP covariance function circumventing the need to create and curate spatial features, and seamlessly offering predictions at new locations via kriging. NN-GLS is compatible with any neural network architecture for the mean function 
and with any family of GP covariance functions. 

We note that the philosophy of GLS has been adopted for non-linear spatial analysis before. \blue{\cite{nandy2017additive} propose {\em GAM-GLS}, using GLS loss based on Gaussian Process covariance for estimating parameters of generalized additive models (GAM). GAM-GLS improves over GAM for non-linear additive function estimation by explicitly accounting for the spatial covariance. However, like GAM, GAM-GLS} \cyan{does} {not account for interactions among covariates. 
Recently, tree-based machine learning algorithms like  
Random Forests \citep[{\em RF-GLS},][]{saha2023random} 
and boosted trees \citep[{\em GP-boost},][]{sigrist2022gaussian,iranzad2022gradient} has been extended to use GLS loss. Forest and tree estimators use a brute force search to iteratively grow the regression trees, requiring multiple evaluations of the GLS loss within each step. This severely reduces the scalability of these approaches. RF-GLS also requires pre-estimation of the spatial parameters, which are then kept fixed during the random forest estimation. 

NN-GLS avoids both the issues \blue{of lack of scalability and prefixed spatial covariance parameters of existing GLS-based non-linear methods} by offering a representation of the algorithm as a special type of {\em graph neural network (GNN)}. We show that NN-GLS using any neural network architecture for the mean and a {\em Nearest Neighbor Gaussian Process} \citep[NNGP,][]{datta2016hierarchical} for the covariance is a GNN with two graph-convolution layers based on the nearest-neighbor graph \blue{where the graph-convolution weights are naturally derived as kriging weights.} 
Leveraging this representation of the model as a GNN, we can exploit the various computing techniques used to expedite neural networks. This includes using {\em mini-batching} or {\em stochastic gradients} to run each iteration of the estimation on only a subset of the data 
\blue{and representing kriging predictions using GNN convolution and deconvolution. The spatial covariance parameters are now simply weights parameters in the network and are updated throughout the training.}  \blue{Finally, we provide a spatial bootstrap-based approach to construct interval estimates for the non-linear mean function.}

We provide a comprehensive theoretical study of neural networks when the observations have spatially correlated errors arising from a stochastic process. Our theoretical contribution distinguishes itself from the current body of theoretical work on neural networks in two main ways. The existing asymptotic theory of neural networks (see Section \ref{sec:threv} for a brief review) does not consider spatial dependence in either the error process generating the data (often being restricted to i.i.d. errors) or introduce any procedural modifications to the neural network algorithm to explicitly encode spatial covariance. 

Our theory of NN-GLS accounts for spatial dependence in both the data generation and the estimation algorithm. We prove general results on the existence and the consistency 
of NN-GLS which subsumes special cases of interest including \blue {irregular spatial data designs}, popular GP models like the Mat\'{e}rn Gaussian process, and using NNGP for the GLS loss. \blue{We also provide finite-sample error rates for the mean function estimation from NN-GLS. The error bound 
provides novel insights on the necessity of modeling the spatial covariance in neural networks, showing that use of 
OLS loss in NN, ignoring spatial dependence, 
will lead to 
larger error rates. 
To our knowledge, these are the first theoretical results for neural networks for irregularly sampled, spatially dependent data.}



The rest of the manuscript is organized as follows. In Section \ref{sec-mtd}, we review process-based geostatistical models and existing spatial neural networks. Section \ref{sec-mtd-NNGP} describes the idea of combining the two perspectives. 
Section \ref{sec:gnn} formally proposes the algorithm NN-GLS, depicts its connection with graph neural networks (GNN), and provides scalable estimation, prediction \blue{and inference} algorithms. 
Section \ref{sec:th} presents the theoretical results. 
Section \ref{sec-sim} and \ref{sec-real}  respectively demonstrate the performance of NN-GLS in simulated and real data. 
\section{Preliminaries}\label{sec-mtd}

\subsection{Process-based geostatistical modeling}\label{sec-splmm}
Consider spatial data collected at locations $s_i$, $i=1,\ldots,n$, 
comprising a 
covariate vector  $\bX_i := \bX(s_i) \in \mathbb R^d$ and a response $Y_i:=Y(s_i) \in \mathbb R$. 
Defining $\bY=(Y_1,\ldots,Y_n)'$ and the $n \times d$ covariate matrix $\Xb$ similarly, the spatial linear model is given by
\begin{equation}\label{eq:splm}
    \bY \sim \mathcal{N}(\Xb\bm{\beta},\bSigma(\btheta))
\end{equation}

\noindent where $\bSigma:=\bSigma(\btheta)$ is a $n \times n$ covariance matrix. 
A central objective in spatial analysis is to extrapolate inference beyond just the data locations 
to the entire continuous spatial domain. Stochastic processes are natural candidates for such domain-wide models. 
In particular, (\ref{eq:splm}) can be viewed as a finite sample realization of a Gaussian process (GP) 
\begin{equation}\label{eq:stoc}
Y(s) = \bX(s)^\top\bm{\beta} + \epsilon(s), \  \epsilon(\cdot) \sim GP(0,\Sigma(\cdot,\cdot)).
\end{equation}
where $\epsilon(s)$ is a zero-mean Gaussian process modeling the spatial dependence via the covariance function $\Sigma(\cdot,\cdot)$ such that $\Sigma(s_i,s_j)=Cov(Y(s_i), Y(s_j))$. 
Often, $\epsilon(\cdot)$ can be decomposed into a latent spatial GP and a non-spatial (random noise) process. 
This 
results in the variance decomposition $\bSigma=\Cb + \tau^2 \Ib$ where $\Cb$ is the covariance matrix corresponding to the latent spatial GP and $\tau^2$ is the variance for the noise process. One can impose plausible modeling assumptions on the nature of spatial dependence like stationarity ($C(s_i,s_j)=C(s_i-s_j)$) or isotropy ($C(s_i,s_j)=C(\|s_i-s_j\|)$) to induce parsimony, thereby requiring a very-low dimensional parameter $\btheta$ to specify the covariance function. 

\blue{{\em Nearest Neighbor Gaussian Processes} \citep[NNGP,][]{datta2016hierarchical} are used to model the spatial error process when the number of spatial locations $n$ is large as evaluation of the traditional GP likelihood (\ref{eq:splm}), requiring $O(n^3)$ computations, becomes infeasible. 
NNGP provides a sparse approximation $\tilde{\mathbf{\Sigma}}^{-1}$ to the dense full GP precision matrix $\bSigma^{-1}$. 
This is constructed using a {\em directed acyclic graph (DAG)} based on pairwise distances among the $n$ data locations, such that each node (location) has almost $m \ll n$ directed (nearest) neighbors. 
Letting $N(i)$ be the set of neighbors of location $s_i$ in the DAG, we have 
\begin{equation}\label{eq:nngp_bf}
\begin{split}
\tilde{\mathbf{\Sigma}}^{-1} &= (\Ib - \Bb)^\top\mathbf{F}^{-1}(\Ib - \Bb), \mbox{ where }\\
\Bb_{i,N(i)} &= \mathbf{\Sigma}\big(i,N(i)\big)\mathbf{\Sigma}\big(N(i), N(i)\big)^{-1},\, B_{ij}=0 \mbox{ elsewhere}, \mbox{ and }\\ 
\mathbf{F}_{ii} &= \mathbf{\Sigma}_{ii} - \mathbf{\Sigma}\big(i, N(i)\big)\mathbf{\Sigma}\big(N(i), N(i)\big)^{-1}\mathbf{\Sigma}\big(N(i), i\big).
\end{split}
\end{equation}
Here $\Bb$ is a strictly lower triangular matrix and $\Fb$ is a diagonal matrix.
NNGP precision matrices only require 
 inversion of $n$ small matrices of size $m\times m$. As $m \ll n$, evaluation of NNGP likelihood requires total $O(n)$ time and storage.}

\subsection{Neural networks}\label{sec-mtd-NN}
Artificial neural networks (ANN) or, simply, neural networks (NN) are widely used to model non-parametric regression $f$ where $E(Y_i) = f(\bX_i)$. 
Mathematically, an $L$-layer {\em feed-forward neural networks} or {\em multi-layer perceptron (MLP)} can be described as,
\begin{equation}\label{def-NN}
\begin{split}
\bA^{(0)}_i = \bX_i, &\quad 
\bZ^{(l)}_i = \Wb_{(l)}^{\top} \bA^{(l-1)}_i,\, \bA^{(l)}_i = \bg_l(\bZ_i^{(l)}),\,   l=1,\ldots,L\\
O_i &= \Wb_{(L+1)}^{\top}\bA_i^{(L)},\, f(\bX_i) = O_i,\, i=1, \ldots,n
\end{split}    
\end{equation}
where for layer $l$, $\bA^{(l)}$ represents the $d_l$ nodes, $\Wb_{(l)}$'s are the weight matrix, $\bZ^{(l)}_i$'s are a linear combination of nodes with unknown weights $\Wb_{(l)}$'s, and $\bg_{(l)}(\bZ_i^{(l)})$ denotes the known non-linear {\em activation} (link) functions $g_l(\cdot)$ (e.g., sigmoid function, ReLU function) applied to each component of $\bZ_i^{(l)}$. The final layer $O_i$ is called the {\em output layer} and gives the modeled mean of the response, i.e., $O_i=f(\bX_i)=E(Y_i)$. 
For regression, the unknown weights are estimated using {\em backpropagation} based on the 
ordinary least squares (OLS) loss
\begin{equation}\label{eq:mse}
    \sum_{i=1}^n (Y_i - f(\bX_i))^2.
\end{equation}

Estimation is expedited by mini-batching where the data are split into smaller and disjoint mini-batches and at each iteration the loss (\ref{eq:mse}) is approximated by restricting to one of the mini-batches, and cycling among the mini-batches over iterations. 

\blue{Current extensions of neural networks for spatial data have mostly adopted the {\em added-spatial-features} strategy, adding spatial features like the (geographical coordinates, spatial distance or basis functions) as extra covariates in the neural network \citep{wang2019nearest,chen2020deepkriging}. 
Formally, they model $Y_i = g(\bX,\bB(s_i))$ where $\bB(s)$ denotes a set of spatial basis functions, and $g$ denotes a neural network on the joint space of the covariates $\bX$ and the basis functions $\bB(s)$. These methods thus depart the GP model framework, and do not explicitly model spatial covariance as all the spatial structure is put into the mean, implicitly assuming that there is no residual spatial dependence. They 
cannot provide a separate estimate of the {\em covariate effect} $E(Y_i \given \bX_i)$ as it models $Y_i$ jointly as a function of $\bX_i$ and $\bB(s_i)$. 
Also, 
the number of added basis functions carries a tradeoff with more basis functions improving accuracy at the cost of increasing parameter dimensionality.}


\section{Neural networks for Gaussian process models}\label{sec-mtd-NNGP}
The two paradigms reviewed in Section \ref{sec-mtd} are complimentary in their scope. The popularity of the geospatial linear models (Section \ref{sec-splmm}) is owed to their simplicity, interpretability, and parsimony -- separating the covariate effect and the spatial effects, modeling the former through a linear mean, and the latter \blue{parsimoniously} through the GP covariance, \blue{ specified typically using only 2-3 interpretable parameters by encoding stationarity or isotropy}.
However, it relies on the strong assumption of a linear covariate effect. Neural networks can estimate arbitrary non-linear covariate effects. However, implicit in the usage of the OLS loss (\ref{eq:mse}) for NN is the assumption that the data units are independent. This is violated for spatial data, where the error process is a dependent stochastic process as in (\ref{eq:stoc}). \blue{To our knowledge, there has been no previous work in directly incorporating spatial covariances in the neural network estimation process, with existing spatial neural networks mostly adopting the added-spatial-features strategy reviewed in Section \ref{sec-mtd-NN}}.


We bridge the paradigms \blue{of traditional geosptatial modeling and neural networks by embedding neural networks directly into the GP model}, allowing the mean of the GP to be non-linear and \blue{modeling it with NN. We propose estimating the NN parameters using a novel GLS loss with the GP covariance matrix, which arises naturally from the log-likelihood of the GP model. We retain all advantages of GP based modeling framework, including separating the covariate and spatial effects into the mean and covariance structures, and obtaining predictions at new locations simply through kriging. In Section \ref{sec:gnn}, we show that when using Nearest neighbor Gaussian processes for the spatial covariance, one can make a novel and principled connection between graph neural networks and geospatial models for irregular data, facilitating a scalable algorithm for estimation and prediction.}

\subsection{NN-GLS: Neural networks using GLS loss}\label{sec:gls}
We extend (\ref{eq:stoc}) to 
\begin{equation}\label{model-sp-nonlinear}
    Y(s) = f(\bX(s)) + \epsilon(s); \ \epsilon(\cdot) \sim GP(0,\Sigma(\cdot,\cdot)).
\end{equation}
Estimating $f$ in (\ref{model-sp-nonlinear}) using neural networks (\ref{def-NN}) needs to account for the spatial dependence modeled via the GP covariance. Using the OLS loss (\ref{eq:mse}) ignores this covariance. We now extend NN to explicitly accommodate the spatial covariance in its estimation. 
Let $\bm{f}(\bX)=\big(f(\bX_1), f(\bX_2), \cdots, f(\bX_n)\big)^\top$. Then the data likelihood for the model (\ref{model-sp-nonlinear}) is
\begin{equation}\label{eq:spnlm}
    \bY \sim \mathcal{N}(\cyan{\bm{f}}(\bX), \bSigma).
\end{equation}
This is a non-linear generalization of the spatial linear model (\ref{eq:splm}). \blue{For parameter estimation in linear models like (\ref{eq:splm}) for dependent data, the OLS loss is replaced by a {\em generalized least squares (GLS)} loss using the covariance matrix $\bSigma$, as it is more efficient according to the Gauss-Markov theorem. }
Similarly, to estimate $f$ using NN, we propose 
using the GLS loss 
\begin{equation}\label{def-NN-loss-3}
\mathcal{L}_n(f) = \frac{1}{n}(\bY - \bm{f(\bX)})^{\top}\Qb(\bY - \bm{f(\bX)}), 
\end{equation}
which accounts for the spatial dependency via the {\em working precision matrix} $\Qb$, which equals $\bSigma^{-1}$ or, more practically, an estimate of it. 

We refer to the neural network estimation using the GLS loss (\ref{def-NN-loss-3}) as {\em NN-GLS}. Conceptually, generalizing NN to NN-GLS is well-principled as minimizing the GLS loss (\ref{def-NN-loss-3}) with $\Qb = \bSigma^{-1}$ is equivalent to obtaining a maximum likelihood estimate of $f$ in (\ref{eq:spnlm}). 
In practice, however, for spatial dependence modeled using GP, the GLS loss ushers in multiple computational issues for both mini-batching and backpropagation, techniques fundamental to the success of NN. \blue{As the GLS loss (\ref{def-NN-loss-3}) is not additive over the data units, minibatching cannot be deployed as for the OLS loss,} and back-propagation will involve computing an inverse of the dense $n \times n$ matrix $\bSigma$ 
which requires $O(n^2)$ storage and $O(n^3)$ time \blue{for each iteration}. 
These computing needs are infeasible for even moderate $n$. 

Next, we develop an algorithm {\em NN-GLS} with a specific class of GLS loss that mitigates these issues and offers a pragmatic approach to using NN for GP models.   

\section{NN-GLS as Graph Neural Network}\label{sec:gnn}
We offer a representation of NN-GLS as a special graph neural network (GNN). This connection allows \blue{use of OLS loss with transformed (graph-convoluted) data leading to} scalable mini-batching and backpropagation algorithms for NN-GLS. 
We propose choosing $\Qb$ as the precision matrix from a Nearest Neighbor Gaussian Process (see Section \ref{sec-splmm}) \cyan{, i.e., we optimize the objective function (\ref{def-NN-loss-3}) with $\Qb = \tSig^{-1}$ as provided in (\ref{eq:nngp_bf}).} 
Basis functions derived from NNGP have been used as added-spatial-features in neural networks \citep{wang2019nearest}. This differs from our approach of using NNGP to directly model the spatial covariance which is akin to the practice in spatial linear models. 


A GLS loss can be viewed as an OLS loss for the decorrelated response $\bY^*=\Qb^{\frac{1}{2}}\bY$, where $\Qb^{\frac{1}{2}}$ is the Cholesky factor of $\Qb = \Qb^{\frac{\top}{2}}\Qb^{\frac{1}{2}}$. Hence, decorrelation is simply a linear operation. A convenience of choosing $\Qb=\tsinv$, the NNGP precision matrix, is that decorrelation becomes a convolution in the nearest neighbor DAG. 
To elucidate, note that $\Bb_{i,N(i)}$ defined in (\ref{eq:nngp_bf}) 
denotes the kriging weights for predicting $Y_i$ based on its directed nearest neighbors $\bY_{N(i)}$ using a GP with covariance $\Sigma(\cdot,\cdot)$. Similarly, $\mathbf{F}_{ii}$ in (\ref{eq:nngp_bf}) is the corresponding nearest neighbor kriging variance. Letting $N^*[i]=N(i) \cup \{i\}$ denote the graph neighborhood for the $i^{th}$ node and defining weights 
\begin{equation}\label{eq:graph_weights}
    \bv_i^\top = \frac 1{\sqrt{\cyan{\Fb_{ii}}}} (1, - \Bb_{i,N(i)}),
\end{equation}
we can write $Y^*_i = \bv_i^\top \bY_{N^*[i]}$. \cyan{In this paper, index sets $N(i)$ or $N^*[i]$ are used in the subscript to subset a matrix of a vector.} Thus, the decorrelated responses $Y^*_i$ are simply convolution over the DAG used in NNGP with the graph convolution weights $\bv_i$ defined using kriging. Similarly, one can define the decorrelated output layer $O^*_i = \bv_i^\top \bO_{N^*[i]}$ using the same graph convolution, where $O_i$ is the output layer of the neural network $f$ (see (\ref{def-NN})). 

\begin{figure}[!t]
 \centering
\includegraphics[scale=0.5]{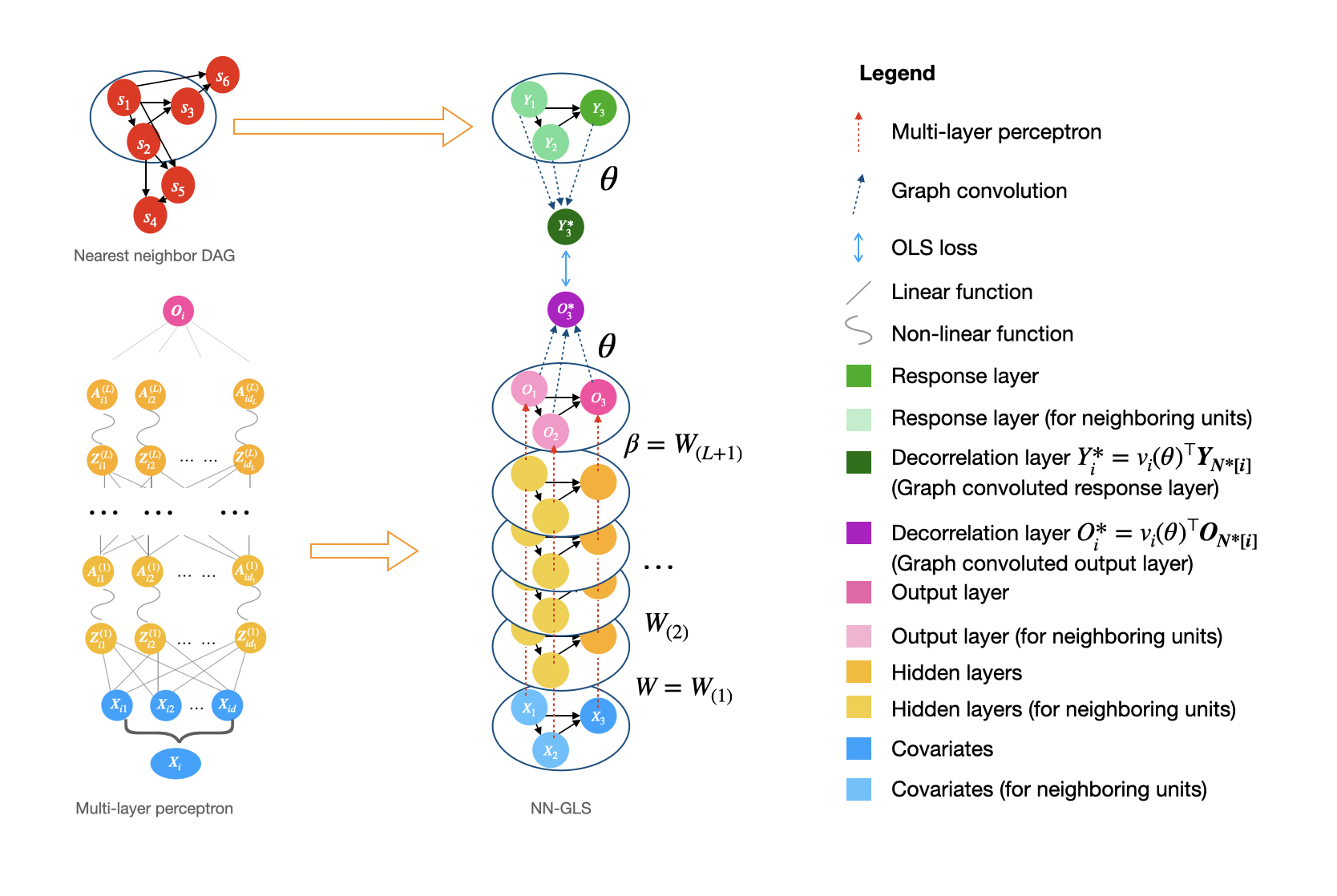}
\caption{NN-GLS as a graph neural network with two graph convolution layers}
\centering
\label{fig-dec-GNN}
\end{figure}

 The decorrelation step makes NN-GLS a special type of graph neural network (GNN) as depicted in Figure \ref{fig-dec-GNN}. In GNN, typically, the input observations are represented on a graph and the locality information is aggregated using convolution layers based on the graph structure (graph convolution). 
For NN-GLS, both the inputs $\bX_i$ and the responses $Y_i$ are graph-valued objects as they both correspond to the locations $s_i$, which are the nodes of the nearest neighbor DAG. 
First, the input $\bX_i$ is passed through the feed-forward NN (or multi-layer perceptron) to produce the respective output layer of $O_i=f(\bX_i)$. This is a within-node operation, and any architecture can be used (number of layers, number of nodes within each layer, sparsity of connections, choice of activation functions). Subsequently, the output layer $O_i$'s from the MLP is passed through an additional graph-convolution to create the decorrelated output layer of $O^*_i$'s using the weights $\bv_i$ (\ref{eq:graph_weights}). This layer is matched,  using the OLS loss, to the decorrelated response layer of $Y_i^*$'s, created from the $Y_i$'s using the same graph convolution. \cyan{So the objective function can be expressed as 
\begin{equation}\label{eq:olsgnn}
 \sum_i  (Y^*_i - O^*_i)^2 \mbox{ where } Y^*_i = \bv_i^\top \bY_{N^*[i]} \mbox{ and } O^*_i = \bv_i^\top \bO_{N^*[i]} .
\end{equation}} 
Thus fitting a GLS loss to an NN is simply fitting an OLS loss to a new NN with two additional decorrelation layers at the end of the MLP. \cyan{From the form of $\bv_i$ in (\ref{eq:graph_weights}) and of $\bB_{i,N(i)}$ and $\mathbf{F}_{ii}$ in (\ref{eq:nngp_bf}), it is clear that $\bv_i$'s can be calculated using matrices of dimension at most $m$, The NNGP precision matrix $\tSig^{-1}$ in (\ref{eq:nngp_bf}) or the covariance matrix $\tSig$ is never actually computed. Thus the loss function (\ref{eq:olsgnn}) is evaluated without any large matrix computation.}

To summarize, there are two ingredients of NN-GLS: a feed-forward NN or MLP (using any architecture) for intra-node operations; and a sparse DAG among the locations for incorporation of spatial correlation via inter-node graph convolutions. Information about the mutual distances among the irregular set of locations is naturally incorporated in 
the kriging-based convolution weights. We now discuss how this formulation of NN-GLS as a GNN with OLS loss helps leverage the traditional strategies to scale computing for NN.

\subsection{Mini-batching} 
Leveraging the additivity of the OLS loss \eqref{eq:olsgnn}, we can write $\mathcal L_n = \sum_{b=1}^B \mathcal L_{b,n}$ where $\mathcal{L}_{b,n} = \sum_{i \in S_b} (Y^*_i - O^*_i)^2$, $S_1, \ldots, S_B$ being a partition of the data-locations each of size $K$. The $Y_i^*$'s are uncorrelated and identically distributed (exactly under NNGP distribution, and approximately if the true distribution is a full GP), so the loss $L_{b,n}$ corresponding to the mini-batch $S_b$ are approximately i.i.d. for $b=1,\ldots, B$. Hence, parameter estimation via mini-batching can proceed like the i.i.d. case. 
The only additional computation we incur is during the graph convolution as obtaining all $O_i^*$ for the mini-batch $S_b$ involves calculating $O_i$'s for the neighbors of all units $i$ included in $S_b$.




\subsection{Back-propagation}\label{sec:backprop} Gradient descent or back-propagation steps for NN-GLS can be obtained without any large matrix inversions. We provide the back-propagation equations for a single-layer network with 
 \blue{$O_i = \bA_{i}^\top \bbeta$}, \cyan{$\delta_i=-2(Y_i - O_i)$, where $\bbeta$ and $\bA_{i} = \bg(\bZ_i)$ are $d_1$ dimensional, $d_1$ being the number of hidden nodes, $g$ is the known activation (link) function. Here $\bZ_i=\Wb^\top\bX_i$ is the hidden layer created using $d \times d_1$ weight-matrix $\Wb$ and $d \times 1$ input vector $\bX_i$}. 
\cyan{$O_i$, $\delta_i$ are concatenated into vectors $\bO$, $\bdelta$, and $\bX_{i}$, $\bZ_{i}$, $\bA_{i}$ are concatenated into matrices $\Xb$, $\Zb$, $\Ab$. 
With $\bv_i$ defined in (\ref{eq:graph_weights}), we introduce scalars 
\blue{$O^*_i = \bv_i^\top \bO_{N^*[i]} = \bv_i^\top \Ab_{\cdot,N^*[i]}^{\top}\bbeta$, $a^*_{ir}=\bv^\top_i \Ab^\top_{r,N^*[i]}$}
}, and $\delta^*_i = \bv_i^\top\bdelta_{N^*[i]}$. 
 \cyan{Using the loss function \eqref{eq:olsgnn} we derive} the following customized back-propagation updates: 
\begin{equation}\label{eq:backprop}
\begin{split}
\beta_r^{(t+1)} =\;& \beta_r^{(t)} - \gamma_t \sum_{i \in S_{b(t)}} \delta^*_i a^*_{ir} \\
w_{rj}^{(t+1)} =\;& w_{rj}^{(t)} - \gamma_t \beta_{r} \sum_{i \in S_{b(t)}} \delta^*_i \left(\bv_i^\top (\bg'(\Zb_{N^*[i],r}) \odot \Xb_{N^*[i],j}) \right) ,\\
\btheta^{(t+1)} =\;& \btheta^{(t)} + \frac{\gamma_t}{2} \sum_{i \in S_{b(t)}} \delta^*_i \left( \left(\frac{\partial \bv_i}{\partial \btheta}\right)^\top\bdelta_{N^*[i]}\right) 
\end{split}
\end{equation}
where $\btheta$ are parameters for the graph convolution weights $\bv_i$. Here $\gamma_t$ is the learning rate, and $S_{b(t)}$ is the mini-batch for the $t^{th}$ iteration and $g'$ is the derivative of $g$. \cyan{Similar equations can be established for networks with more than one-layer, although the derivations will be tedious. Instead, for multi-layer networks,} scalable gradient descent steps for NN-GLS can be conducted using off-the-shelf software. One can evaluate the mini-batch loss \cyan{$\mathcal L_{b,n}$ scalably, as the $\bv_i$'s can be computed without any large matrix computation,} and obtain scalable gradient descent updates using numerical, automatic, or symbolic differentiation. 

Note that in this GNN representation, \blue{the convolution weights $\bv_i$ are parametrized by spatial covariance parameter $\btheta$ (an example of $\btheta$ is the one for Mat\'{e}rn covariance family (Definition \ref{def-Matern}) where $\btheta = (\sigma^2, \phi, \nu)$)}. Hence, the updates for $\btheta$ can be just absorbed as a back-propagation step as shown in (\ref{eq:backprop}). Alternatively, the spatial parameters $\btheta$ can also be updated, given the current estimate of $f$, by writing \cyan{$Y_i^*(\btheta) := Y_i^*$ and $O_i^*(\btheta) := O_i^*$ to note the dependence on $\btheta$} and maximizing the NNGP log-likelihood  \begin{equation}\label{eq:nngploglik}
\sum_i  \Big( (Y^*_i(\btheta) - O^*_i(\btheta))^2 + \log \mathbf{F}_{ii}(\btheta) \Big).
\end{equation} 
\cyan{We only use full optimization of \eqref{eq:nngploglik} to update the spatial parameters $\btheta$, replacing its gradient descent update. 
The network parameters are always updated using gradient descent.}  

\subsection{Kriging}\label{sec-mtd-krig}
\blue{Subsequent to estimating $\hat f(\cdot)$, 
NN-GLS leverages the GP model-based formulation to provide both point and interval predictions at a new location simply via kriging. Exploiting the GNN representation of NN-GLS with a NNGP working covariance matrix, kriging can be conducted entirely with the neural architecture, as we show below.

Given a new location $s_{0}$ and covariates $\bX_{0} := \bX(s_{0})$, first $\bX_{0}$ is passed through the trained feed-forward part (MLP) of the GNN  to obtain the output $O_{0}=\hat f (\bX_{0})$ (see Figure \ref{fig-dec-GNN}). Next, the new location $s_{0}$ is added as a new node to the nearest-neighbor DAG. Let $N(0)$ be its set of $m$ neighbors of $s_0$ on the DAG, $N^*[0]=\{s_{0}\} \cup N(0)$, and define the graph weights $\bv_{0}$ similar to (\ref{eq:graph_weights}). Then using the graph convolution step, we obtain the decorrelated output 
$O^*_{0}=\bv_{0}^\top\bO_{N^*[0]}$. As the decorrelated output layer $O^*$ is the model for the decorrelated response layer $Y^*$ in the GNN, we have $\hat Y^*_{0} = O^*_{0}$. Finally, as $Y^*_{0}=\bv_{0}^\top\bY_{N^*[0]}$ is the graph-convoluted version of $Y_{0}$, we need to deconvolve $\hat Y^*_{0}$ over the DAG to obtain $\hat Y_{0}$. This leads to the final prediction equation 
\begin{equation}\label{eq:gnnkrig}
    \hat Y_{0} =  \sqrt{\cyan{\mathbf{F}_{00}}}\hat Y^*_{0} +  \Bb_{0,N(0)}^\top \bY_{N(0)}.
\end{equation}}
It is easy to verify that the prediction (\ref{eq:gnnkrig}) is exactly same as the $m$-nearest neighbor kriging predictor for the spatial non-linear model (\ref{eq:spnlm}), i.e.,

\begin{equation}\label{eq:krig}
   \hat Y_{0} = \hat{f}(\bX_{0}) + \mathbf{\Sigma}(s_{0}, N(0))\bSigma(N(0), N(0))^{-1}(\bY_{N(0)} - \hat {\bm{f}}_{N(0)}).
\end{equation}

\blue{Additionally, prediction variance is simply 
$\sigma^2_0 =  \mathbf{F}_{00}$. 
As NN-GLS is embedded within the Gaussian Process framework, predictive distributions conditional on the data and parameters are Gaussian. So the prediction interval (PI) can be obtained as $[\hat{f}(\bX_0) + Z_{2.5}\sigma_0, \hat{f}(\bX_0) + Z_{97.5}\sigma_0]$, where $Z_{q}$ is the $q$-quantile of a standard normal distribution.} 

Thus, the GNN connection for NN-GLS offers a simple and coherent way to obtain kriging predictions \blue{entirely within the neural architecture.} 
 We summarize the \blue{estimation and prediction} 
from {\em NN-GLS} in Algorithm \ref{alg-NN-GLS-main} in Section \ref{sec:alg} of the Supplement. \blue{Section \ref{sec-mtd-CI} presents a spatial bootstrap for obtaining pointwise confidence intervals of the mean.}

\section{Theory}\label{sec:th}

\blue{
We review some relevant theory of neural network regression in Section \ref{sec:threv}. Our theoretical contributions have two main differences from the existing theory. 
First, the overwhelming majority of the theoretical developments focus on the setting of i.i.d. data units. To our knowledge, asymptotic theory for neural networks under spatial dependence among observations from an irregular set of locations has not been developed. 

Second, the existing theory has almost exclusively considered neural network methods like the vanilla NN (using the OLS loss) that do not explicitly accommodate spatial covariance among the data units. We provide general results on existence (Theorem \ref{Thm-exist-2}) and asymptotic consistency (Theorem \ref{Thm-main-1}) of NN-GLS that explicitly encodes the spatial covariance via the GLS loss. In Propositions \ref{prop-main-2} and \ref{prop-main-1}, we show that consistency of NN-GLS holds for spatial data generated from the popular Mat\'ern GP under arbitrarily irregular designs. We also derive the finite-sample error rate of NN-GLS (Theorem \ref{thm-rate}) which shows how a poor choice of the working covariance matrix in the GLS loss  will lead to large error rates, highlighting the necessity to accurately model spatial covariance in neural networks.}

\subsection{Notations and Assumptions}\label{sec:assume} 
Let $\mathbb{R}$ and $\mathbb{N}$ denote the set of real numbers and natural numbers respectively. $\|\cdot\|_{p}$ denotes the $\ell_p$ norm for vectors or matrices,  
 $0 < p \leq \infty$.  
Given the covariates $\bX_1, \bX_2, \cdots, \bX_n$, for any function $f$, we define the norm 
$\|f\|^2_n = \frac{1}{n}\sum_{i=1}^nf^2(X_i)$. 
Given a $n\times n$ matrix $\Ab$, $\lambda(\Ab)$ denotes its eigenspace, $\lambda_{\max} = \sup\{\lambda(\Ab)\}$, and $\lambda_{\min} = \inf\{\lambda(\Ab)\}$. 
A sequence of numbers $\{a_n\}_{n \in \mathbb{N}}$ is $O(b_n)$ ($o(b_n)$) if the sequence $\{|a_n/b_n|\}_{n \in \mathbb{N}}$ is bounded from above (goes to zero) as $n \to \infty$. 
Random variables (distributions) $X \sim Y$ means $X$ and $Y$ have the same distribution. 
We first specify the assumption of the data generation process. 
\begin{innercustomass}[Data generation process]\label{Asmp-1}
The data $Y_i:=Y(s_i), i=1,\dots,n$ is generated from a GP with a non-linear mean, i.e., $Y_i = f_0(\bX_i) +\epsilon_i$, where $f_0(\cdot)$ is a continuous function, \blue{$\{\bX_i\}_{i = 1, \cdots, n}$ are fixed covariate vectors in a compact subset in $\mathbb{R}^d$,} and the error process $\{\epsilon_i\}$ is a 
GP such that \blue{the maximum (minimum) eigenvalue of the covariance matrix $\mathbf{\Sigma}=Cov(\bY)$ is uniformly (in $n$) upper-bounded (lower-bounded) by $\Psi_{high}(\Psi_{low})$.}
\end{innercustomass}

Assumption \ref{Asmp-1} imposes minimal restrictions on the data generation process. The mean $f_0$ is allowed to be any continuous function. 
The restriction on the eigenvalues of the GP covariance matrix is tied to the spatial design. We show in Propositions \ref{prop-main-2} and \ref{prop-main-1} that this is satisfied for common GP covariance choices for any irregular set of locations in $\mathbb R^2$ \blue{under the {\em increasing domain} asymptotics where locations} are separated by a minimum distance. 

We next state assumptions on the analysis model, i.e., the neural network family and the GLS working precision matrix. We consider a one-layer neural network class:
\begin{equation}\label{def-NN-class}
\mathcal{F}_0:=\left\{\alpha_0+\sum_{i=1}^n\alpha_i\sigma(\Wb_i^\top\bX + \bw_{0i}), \ \Wb_i\in \mathbb{R}^{d\times p}, \ \bw_{0i}\in\mathbb{R}^d, \  \alpha_i \in \mathbb{R}\right\}.
\end{equation}
This formulation is equivalent to setting $L = 1$, $g_1(\cdot)$ as the sigmoid function $\sigma(\cdot)$ and $g_2(\cdot)$ as the identity function in (\ref{def-NN}). 
It is easy to see that $\mathcal{F}_0$ can control the approximation error, as one-layer NN are universal approximators. However, this class of functions $\mathcal{F}_0$ can be too rich to control the estimation error. 
A common strategy to circumvent this is to construct a sequence of increasing function classes, also known as seive, to approximate 
$\mathcal{F}_0$, i.e.,
\[
\mathcal{F}_1 \subseteq \mathcal{F}_2 \subseteq \cdots\subseteq \mathcal{F}_n\subseteq \mathcal{F}_{n+1} \subseteq  \cdots \subseteq \mathcal{F}_0.
\]
With careful tradeoff of the complexity of the function classes, it's possible to control the estimation error (in terms of the covering number of $\mathcal F_n$) using some 
suitable {\em uniform-law-of-large-number (ULLN)} while still being able to keep the approximation error in check. 
Following \cite{shen2023asymptotic} we consider the seive given below. 
\begin{innercustomass}[Function class]\label{Asmp-2} The mean function $f$ is modeled to be in the NN class
\begin{equation}\label{def-F_n} 
\begin{split}
  \mathcal{F}_n = &\Bigg\{ \alpha_0+\sum_{j=1}^{r_n}\alpha_j\sigma(\bm{\gamma}_j^\top\bX + \gamma_{0,j}): \bm{\gamma}_j \in \mathbb{R}^d, \alpha_j, \gamma_{0,j} \in \mathbb{R},\\
  &\sum_{j=0}^{r_n}|\alpha_j| \leq V_n \text{ for some } V_n > \frac{1}{L_{lip}},\, \max_{1\leq j\leq r_n}\sum_{i,j = 0}^d|\gamma_{i,j}|\leq M_n \text{ for some } M_n > 0\Bigg\},
\end{split}
\end{equation}
where $r_n, V_n, M_n \to \infty$ as $n \to \infty$ and satisfies the scaling
\begin{equation}\label{eq:scaling}
    r_nV_n^2\log V_nr_n = o(n).
\end{equation}
Here the activation function $\sigma(\cdot)$ is a Lipschitz function on $\mathbb{R}$ with range $[-r_b, r_b]$ and Lipschitz constant \blue{$L_{lip}$}. (For the sigmoid function, $r_b = 1$ and \blue{$L_{lip} = 1/4$}).
\end{innercustomass}

\cite{hornik1989multilayer} has shown that $\cup_n\mathcal{F}_n$ is dense in the continuous function class $\mathcal{F}_0$ 
which will 
control the approximation error. The estimation error will depend on the covering number for this class which can be controlled under the scaling rate (\ref{eq:scaling}). 

Finally, to guarantee the regularity of the GLS loss (\ref{def-NN-loss-3}) used for estimating the NN function $f$, 
we require 
conditions on the working precision matrix $\Qb$. \blue{Instead of directly imposing these conditions on $\Qb$, we look at the {\em discrepancy matrix} 
\begin{equation}\label{eq:disc_mat}
    \Eb = \mathbf{\Sigma}^{\frac{\top}{2}}\Qb\mathbf{\Sigma}^{\frac{1}{2}}
\end{equation} 
which measures the discrepancy between the true covariance matrix $\mathbf{\Sigma}$ and the working covariance matrix $\Qb^{-1}$. This is because we will see later in Theorem \ref{thm-rate} that the finite sample error rates of NN-GLS will depend on $\Qb$ through the spectral interval of $\Eb$. First, without loss of generality, we can assume $\lambda_{\min}(\Eb) \leq 1 \leq \lambda_{\max}(\Eb)$. This is because the optimization in NN-GLS only depends on $\Qb$ up to a scalar multiplier. The following assumption states that additionally the spectral interval of $\Eb$ needs to be uniformly bounded in $n$.}

\begin{innercustomass}[\blue{Spectral interval of the discrepancy matrix.}]\label{Asmp-3}
For all $n$, all eigenvalues $\lambda$ of \blue{$\Eb = \mathbf{\Sigma}^{\frac{\top}{2}}\Qb\mathbf{\Sigma}^{\frac{1}{2}}$} lie in $(\Lambda_{low}, \Lambda_{high})$ for universal constants \blue{$0 <\; \Lambda_{low} \leq 1 \leq \;\Lambda_{high} < \infty$.}
\end{innercustomass}

\blue{
Uniform upper and lower bounds of $\Eb$ are used for ensuring the continuity of the GLS loss function and the consistency of the loss function (Lemma \ref{lemma-main-1}) using empirical process results.} We show in Propositions \ref{prop-main-2} and \ref{prop-main-1} how this Assumption is satisfied when $\Qb$ is either the true GP precision matrix $\bSigma^{-1}$, its NNGP approximation. \cyan{Of course, $\Qb=\Ib$, which can be viewed as NNGP precision matrix with $m=0$, trivially satisfies Assumption \ref{Asmp-3} if Assumption \ref{Asmp-1} holds as $\Eb$ is simply $\bSigma$.}  




\subsection{Main results}\label{sec:thgen}
We provide general results on the existence and consistency of neural network estimators minimizing a GLS loss for dependent data. 
The expected value of the GLS loss (\ref{def-NN-loss-3}) is:
\begin{equation}\label{def-L_n}
\begin{split}
    L_n(f) = \EE\big[\mathcal{L}_n(f)\big] &= \EE\left[\frac{1}{n}(\by - \bm{f}(x))^\top\Qb(\by - \bm{f}(x))\right] \nonumber \\
     & = \frac{1}{n}\EE\big[\bm{\epsilon}^\top\Qb\bm{\epsilon}\big] + \frac{1}{n}(\bm{f}_0(x) - \bm{f}(x))^\top\Qb(\bm{f}_0(x) - \bm{f}(x)) \nonumber
\end{split}
\end{equation}

It is evident from above that 
$f_0$ naturally minimizes $L_n(f)$, while NN-GLS tries to minimize $\mathcal{L}_n(f)$. 
We first show that such a minimizer exists in the seive class $\mathcal F_n$.
\begin{theorem}[Existence of seive estimator]\label{Thm-exist-2}
Given data $(Y_i, \bX_i, s_i), i = 1, \cdots, n$ generated from (\ref{model-sp-nonlinear}) \blue{under Assumption \ref{Asmp-1}}, and a working precision matrix $\Qb$ \blue{satisfying Assumption \ref{Asmp-3},} 
with the function classes $\mathcal F_n$ defined in (\ref{def-F_n}), there exists a seive estimator $\hat{f}_n$ such that 
\begin{equation}\label{eq:opt}
\hat{f}_n = \argmin\{\mathcal{L}_n(f): f\in \mathcal{F}_n \}.  
\end{equation}
\end{theorem}
All proofs are \blue{in Section \ref{Append-prof}} of Supplementary materials. 
The existence result ensures that a seive estimator in the class of neural networks that minimizes the GLS loss is well-defined. It is then natural to study its asymptotic consistency, as we do in the next result. 

\begin{theorem}[Consistency]\label{Thm-main-1}
Under Assumptions \ref{Asmp-1}, \ref{Asmp-2}, and \ref{Asmp-3}, the NN-GLS estimate $\hat{f}_n$ (\ref{eq:opt}) minimizing the GLS loss $\mathcal{L}_n(f)$ in (\ref{def-NN-loss-3}) 
is consistent in the sense
$\|\hat{f}_n - f_0\|_n \xrightarrow{p} 0$.
\end{theorem}

To our knowledge, this is the first result on the consistency of neural networks \blue{for estimating the mean function of a spatially dependent process observed at an irregular set of locations}. We \blue{do not impose any assumption on the true mean function $f_0$ beyond continuity} and rely on very mild assumptions on the function class, and the covariance matrices of the data generation and analysis models. In Section \ref{sec:thspatial} we show that these assumptions are satisfied for typical GP covariances and irregular spatial data designs. Also, note that this general result does not restrict the nature of the dependence to be spatial. Hence, while spatial applications are the focus of this manuscript, Theorem \ref{Thm-main-1} can be used to establish consistency of neural networks for time-series, spatio-temporal, network, or other types of structured dependence. \blue{Section \ref{sec-sim-large} presents a simulation study empirically demonstrating this consistency.}

Theorem \ref{Thm-main-1} generalizes the analogous consistency result of \cite{shen2023asymptotic} from i.i.d. data and OLS loss to dependent error processes and the use of a GLS loss. Consequently, the proof addresses challenges that do not arise in the i.i.d. case. 
The spatial dependency makes the standard Rademacher randomization fail and prevents using the standard Uniform law of large number (ULLN) result. We overcome this via construction of a different normed functional space equipped with a new Orlicz norm to adjust for data dependence and use of the GLS loss. This enables applying a ULLN for our dependent setting by showing that the empirical process is well-behaved with respect to this Orlicz norm. 

\subsection{Mat\'ern Gaussian processes}\label{sec:thspatial}
We now establish consistency of NN-GLS for common GP covariance families, \blue{spatial data designs,} and choice of working precision matrices. The main task in applying the general consistency result (Theorem \ref{Thm-main-1}) for these specific settings is verifying compliance to the regularity assumptions -- i.e., the spectral bounds on the true Gaussian process covariance (Assumption \ref{Asmp-1}) and on the working precision matrix (Assumptions \ref{Asmp-3}). 

We provide consistency results of NN-GLS for spatial data generated from the Mat\'ern Gaussian process (\blue{see \ref{def-Matern}}). 
This is the predominant choice of covariance family in geostatistical models due to the interpretability of its parameters with a  marginal spatial variance $\sigma^2$, decay of spatial correlation $\phi$, smoothness $\nu$ of the underlying process \citep{stein1999interpolation}. 
Our first result considers data generated from a class of GP that contains the Mat\'ern family and where the working precision matrix is the true GP precision matrix. 


\begin{proposition}\label{prop-main-2}
Consider data generated from a spatial process $Y_i = f_0(\bX_i) + 
\epsilon(s_i)$ at locations $s_1,\ldots,s_n$ in $\mathbb R^2$, where $f_0$ is continuous, $\epsilon(s_i)$ is a Gaussian process with covariance function $\Sigma(s_i, s_j)=C(s_i, s_j)+\tau^2 I(s_i = s_j)$, 
 and $C(s_i, s_j) = C(\| s_i - s_j\|)$ is a covariance of a stationary spatial GP. 
Suppose the data locations are separated by a minimum distance $h > 0$, i.e.,  $\|s_i - s_j\| \geq h, \ \forall i\neq j$. 
Let $\mathbf{\Sigma}=\Cb + \tau^2\Ib$ denote the  covariance matrix of $\bY=(Y(s_1),\ldots,Y(s_n))^\top$, where $\Cb=(C(s_i - s_j))$. 
\blue{Then NN-GLS using $\Qb = \bSigma_n^{-1}$ is consistent , i.e., $\|\hat{f}_n - f_0\|_n \xrightarrow{p} 0$, 
if either (a) 
    $\tau^2 > 0$ and $C(u) = o\big(u^{-(2+\kappa)}\big)$ for some $\kappa >0$, or (b) 
    $\tau^2 \geq 0$ and $C(u)$ is a Mat\'{e}rn covariance function with parameters $\btheta = (\sigma^2, \phi, \nu)$.
}
\end{proposition}

The decay rate $C(u) = o\big(u^{-(2+\kappa)}\big)$ is satisfied by the Mat\'ern family \citep{abramowitz1948handbook}, so
 Proposition \ref{prop-main-2} proves the consistency of NN-GLS for Mat\'ern GP \blue{both with and without a nugget.} 
\blue{The result holds for any irregular spatial design in $\mathbb R^2$ being separated by a minimum distance. As the sample size $n$ grows this is equivalent to considering the increasing domain paradigm that is commonly adopted as Mat\'ern GP parameters are not identifiable if data are collected densely in a fixed spatial domain \citep{zhang2004inconsistent}. }

Propositions \ref{prop-main-2} describes the case where true covariance structure is known. In that case, it's possible to directly use the inverse of the covariance matrix as the working precision matrix in the GLS loss. However, this is often infeasible for multiple reasons. First, the true covariance parameters are usually unknown, and the working covariance matrix will typically use different (estimated) parameter values. Computationally, GLS loss using the full Mat\'ern GP covariance matrix will require $O(n^3)$ time and $O(n^2)$ storage, which is not available even for moderate $n$. The next proposition introduces a more pragmatic result proving the consistency of NN-GLS for data generated from Mat\'ern GP but when using a working precision matrix derived from NNGP (as described in Section \ref{sec:gnn}) and with parameter values different 
 from the truth. 
 

\begin{proposition}\label{prop-main-1}
Consider data generated as in Proposition \ref{prop-main-2} from a Mat\'ern GP with parameters $\btheta_0 = (\sigma_0^2, \phi_0, \nu_0,\tau_0^2)$ at locations separated by a \blue{minimum} distance $h > 0$. 
\blue{Let $\Qb$ be the NNGP precision matrix 
based on a Mat\'{e}rn covariance with parameters $\btheta = (\sigma^2, \phi, \nu, \tau^2)$, i.e. $C(h | \sigma^2, \phi, \nu) + \tau^2 I(h=0)$}, and using neighbor sets of maximum size $m$ with each location appearing in at most $M$ many neighbor sets. 
Then NN-GLS  
using $\Qb$ is consistent, i.e., $\|\hat{f}_n - f_0\|_n \xrightarrow{p} 0$ 
\blue{for any choice of $\phi > \frac 1h C^{*-1}\left(\min(1,\frac{\sigma^2 + \tau^2}{(2m+M\sqrt m)\sigma^2})\right)$, 
 where $C^*$ is the Mat\'ern covariance function with unit variance, unit spatial decay, and smoothness $\nu$, and $C^{*-1}$ is its inverse function.} 
\end{proposition}

Proposition \ref{prop-main-1} provides consistency of NN-GLS for Mat\'ern GP when using NNGP working precision matrices. This is the actual choice of $\Qb$ used in our algorithm as it can be represented as a GNN, thereby facilitating a scalable implementation, as described in Section \ref{sec:gnn}. 
The result allows the spatial parameter used in the working covariance to be different from the truth. \blue{The spatial decay parameter $\phi$ for the working precision matrix $\Qb$ needs to be sufficiently large to ensure $\Qb$ is not close to being singular which will lead to numerical instability. We note that this is not a restriction on the data generation process but just on the working precision matrix and can be chosen by the user. This is not enforced in practice as these parameters are estimated as outlined in Section \ref{sec:backprop}. In Section \ref{Append-parameter}, we provide the theortical intuition and empirical evidence that NN-GLS estimates the spatial parameters consistently.} The restriction on each location appearing in at most a fixed number of neighbor sets of NNGP is usually satisfied in all but very pathological designs.

To our knowledge, Propositions \ref{prop-main-2} and \ref{prop-main-1} on the consistency of NN-GLS are the first examples of consistency of any machine-learning-based approach to estimating a non-linear mean of a Mat\'ern GP for irregular spatial designs. The only similar result in the literature is the consistency of RF-GLS, the GLS-based random forest approach of \cite{saha2023random}. However, their result relies on a one-dimensional regular lattice design and restricts the true Mat\'ern process smoothness to be half-integers. 
Our result is valid for any irregular spatial designs in the two-dimensional space -- the most typical setting of spatial data collection. The result also holds for any true parameter values of the Mat\'ern process. 
\vskip-1cm\blue{\subsection{Finite-sample error rates}}\label{sec:rate}
\blue{We also obtain the finite-sample error rate for NN-GLS that informs about the importance of using the GLS loss in neural networks for spatial data.}
\blue{\begin{theorem}[Convergence rate]\label{thm-rate}
Let $\pi_{r_n}$ be the projection of $f_0$ in $\mathcal F_n$. 
Under Assumptions \ref{Asmp-1},  \ref{Asmp-2}, and \ref{Asmp-3}, the NN-GLS estimate $\hat{f}_n$ (\ref{eq:opt}) satisfies 
\[
\begin{split}
\|\hat{f}_n - f_0\|_n = O_p\left(\frac{\Lambda_{high}}{\Lambda_{low}}\|\pi_{r_n}f_0 - f_0\|_n \vee \frac{\Lambda_{high}}{\Lambda_{low}}\sqrt{\frac{r_n\log n}{n}}\right).
\end{split}
\]
\end{theorem}
}
\blue{There are two main takeaways from this rate result. First, the rate for NN-GLS is, up to a scaling factor, same as that for OLS neural networks in the i.i.d. case \citep{shen2023asymptotic}. We do not make any assumption on the class of the regression function $f_0$ beyond continuity and refer to the readers to \citep{shen2023asymptotic} for discussion on specific choices of classes of the true function $f_0$, $r_n$ and $d$ where this rate is sharp up to logarithmic terms. 

The second and novel insight is the scaling of the error rate by a factor $\lambda=\Lambda_{high}/\Lambda_{low} \geq 1$ --- the ratio of the largest and smallest eigenvalues of the discrepancy matrix $\Eb$ (\ref{eq:disc_mat}). In general, $\lambda$ is close to $1$ if the discrepancy matrix is closer to $\Ib$. Proposition \ref{prop-E} in Appendix \ref{Append-E} shows that the Kullback-Leibler distance $\mbox{KLD}(\Ib,\Eb(m))$ between the identity matrix $\Ib$ and the discrepancy matrix $\Eb(m)$ when using an NNGP working covariance matrix with $m$-nearest neighbors is a decreasing function of $m$.
The best case scenario is when using the true covariance as the working covariance, i.e., $\Qb=\bSigma^{-1}$ yielding $\Eb=\Ib$ and $\lambda=1$. This corresponds to using NNGP with $m=n$. At the other extreme, is the vanilla neural network using the OLS loss, corresponds to using a working covariance matrix of $\Ib$ in NN-GLS. This corresponds to NNGP with $m=0$ and is thus the worst approximation according to Proposition \ref{prop-E}. We verify this empirically plotting in Figure \ref{fig-E} both the KLD and the spectral width of $\Eb$ as a function of $m$. We see a very large spectral width for $m=0$, i.e., using OLS loss, implying a large error rate and inefficient estimation when using OLS loss in neural networks for spatial data. This is manifested in our results detailed later. The spectral width for NNGP approximation with $m \geq 20$ is really narrow yielding a near-optimal scaling factor $\lambda$ close to $1$. This demonstrates how modeling the spatial covariance in the neural network estiamtion via the GLS loss improves the error rates \cyan{and guides the choice of the number of nearest neighbor $m$ in NN-GLS (see Section \ref{sec:choosem}).}}


\section{Simulation study}\label{sec-sim}

We conduct extensive simulation experiments to study the advantages of NN-GLS over existing methods in terms of both prediction and estimation. The data are simulated from the spatial GP model (\ref{model-sp-nonlinear}) 
with two choices for the non-linear mean function 
$f_1(x) = 10 \sin(\pi x)$
and 
$f_2(x) = \frac 16 (10\sin(\pi x_1x_2)+20(x_3 - 0.5)^2+10x_4+5x_5)$
  \citep[Friedman function,][]{friedman1991multivariate}. \blue{Section \ref{sec:sim_methods} provides all parameter values and implementation details.} 

\blue{We consider a total of $10$ methods that represent both linear and nonlinear, spatial and non-spatial, statistical and machine learning paradigms. We include $3$ candidate neural network approaches for comparison to NN-GLS: NN without any spatial information ({\em NN-nonspatial}), NN using the spatial coordinates as two additional inputs ({\em NN-latlon}), and NN using spline basis as additional inputs ({\em NN-splines}).} \cyan{The NN-splines method is a rendition of {\em Deepkriging} \citep{chen2020deepkriging}. It uses the same basis functions as deepKriging but to study the effect of additional covariates terms, we keep the numbers of layers and nodes same for all the neural network based methods.} Among these, \blue{only} NN-GLS and NN-nonspatial offer both estimation and prediction, NN-latlon, and NN-splines only offer prediction \blue{and do not estimate of the covariate effect (see Section \ref{sec-mtd-NN})}. 

\blue{Other than neural networks, we also include 5 other popular non-linear methods, including {\em Generalized Additive Models (GAM)}, {\em Random Forests (RF)}, and their counterparts under spatial settings, GAM using the spatial coordinates as two additional inputs ({\em GAM-latlon}), {\em GAM-GLS} \citep{nandy2017additive}, and {\em RF-GLS} \citep{saha2023random}. RF-GLS offers both estimation and spatial predictions. GAM-latlon is used for predictions and the other three methods are only considered for estimation. We also include the {\em spatial linear GP model}, implemented using NNGP through the BRISC R-package \citep{saha2018brisc}. 
 Section \ref{sec:sim_methods} provides details of all the methods and Table \ref{tab:methods} summarizes their estimation and prediction capabilites and scalability for large datasets.} 

\begin{figure}[htbp]
\centering
\begin{subfigure}{.5\textwidth}
  \centering
  \includegraphics[width=0.85\linewidth]{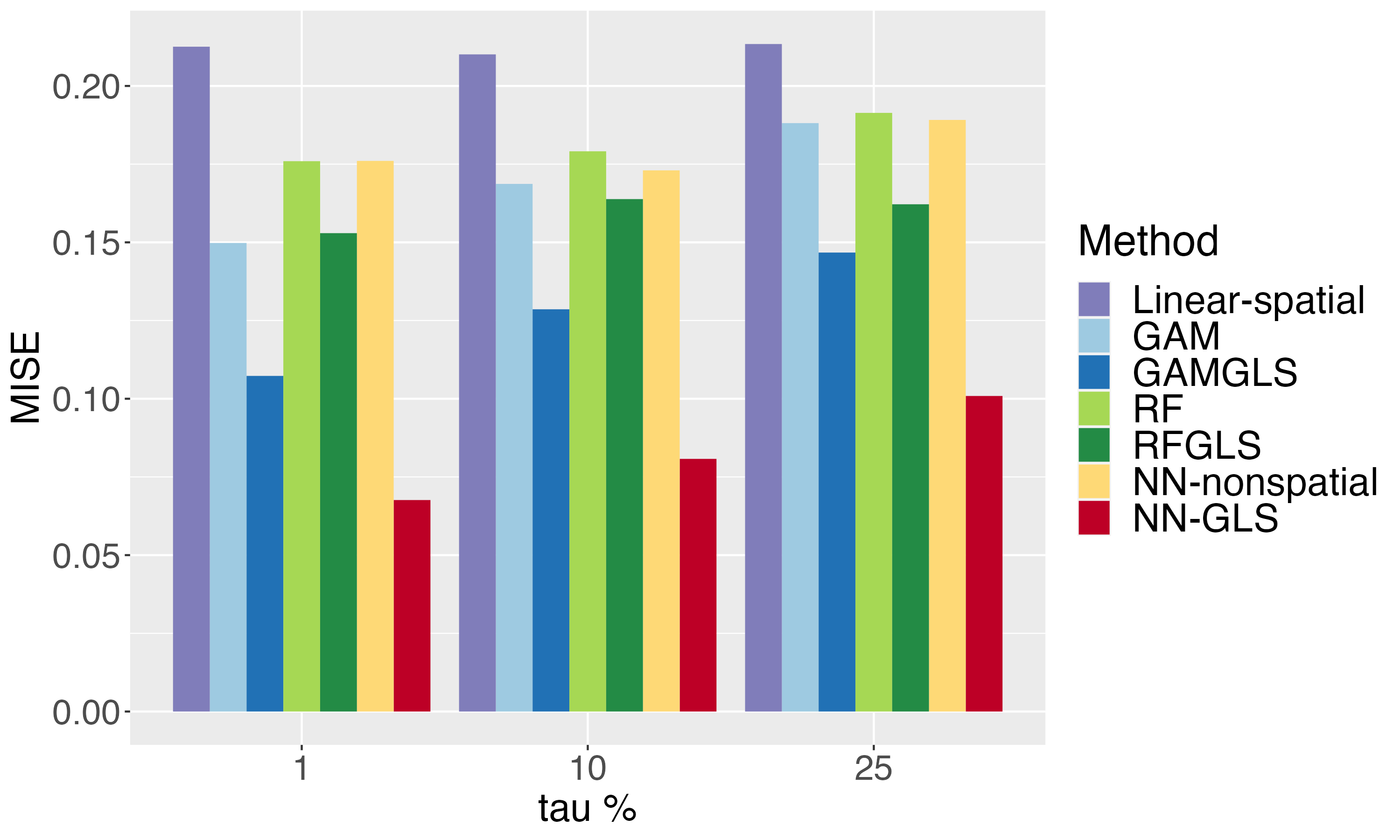}
  \caption{Estimation performance of all methods.}
\end{subfigure}%
\begin{subfigure}{.5\textwidth}
  \centering
  \includegraphics[width=0.85\linewidth]
  {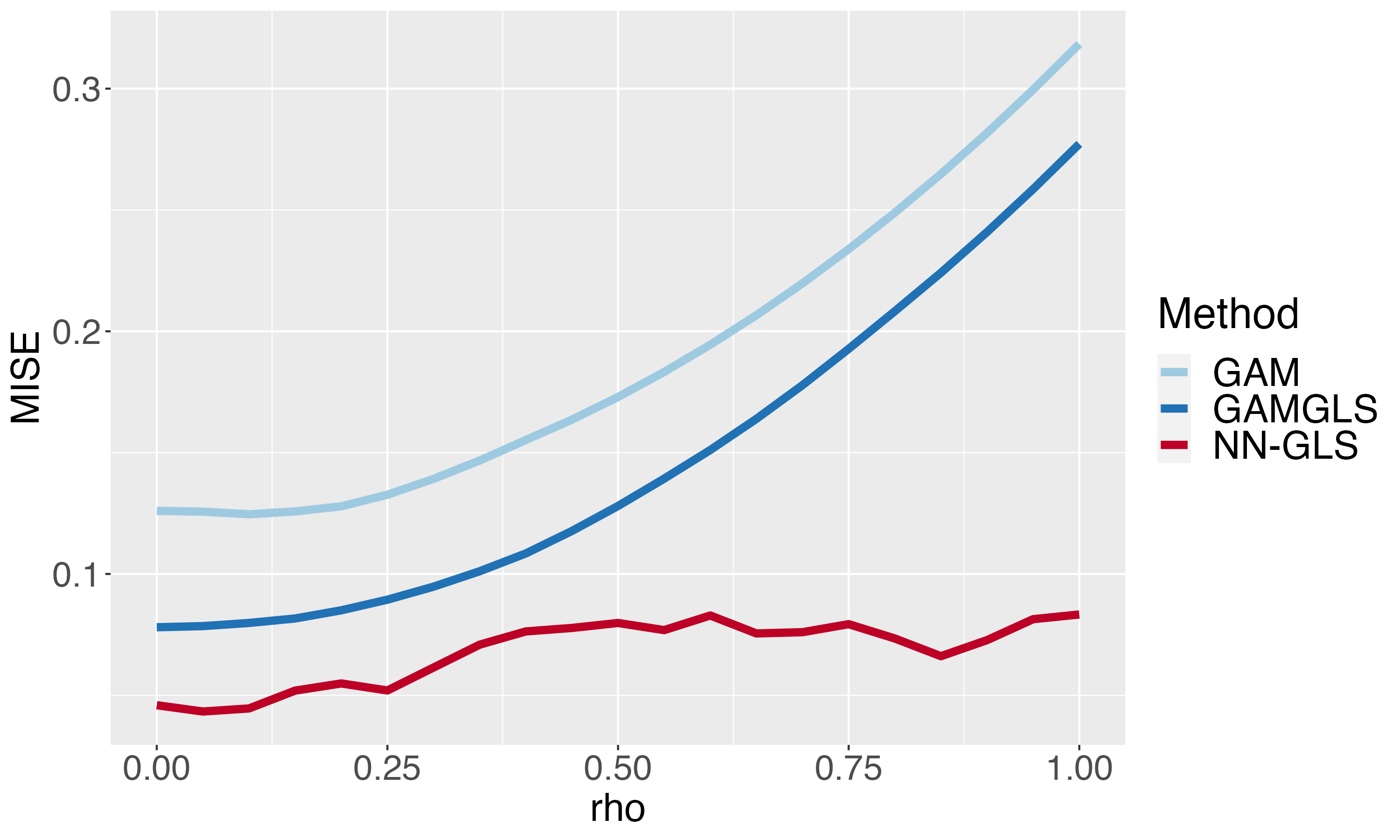}
  \caption{GAM-GLS vs. NN-GLS on estimation.}
\end{subfigure}
\\
\begin{subfigure}{.5\textwidth}
  \centering
  \includegraphics[width=0.85\linewidth]{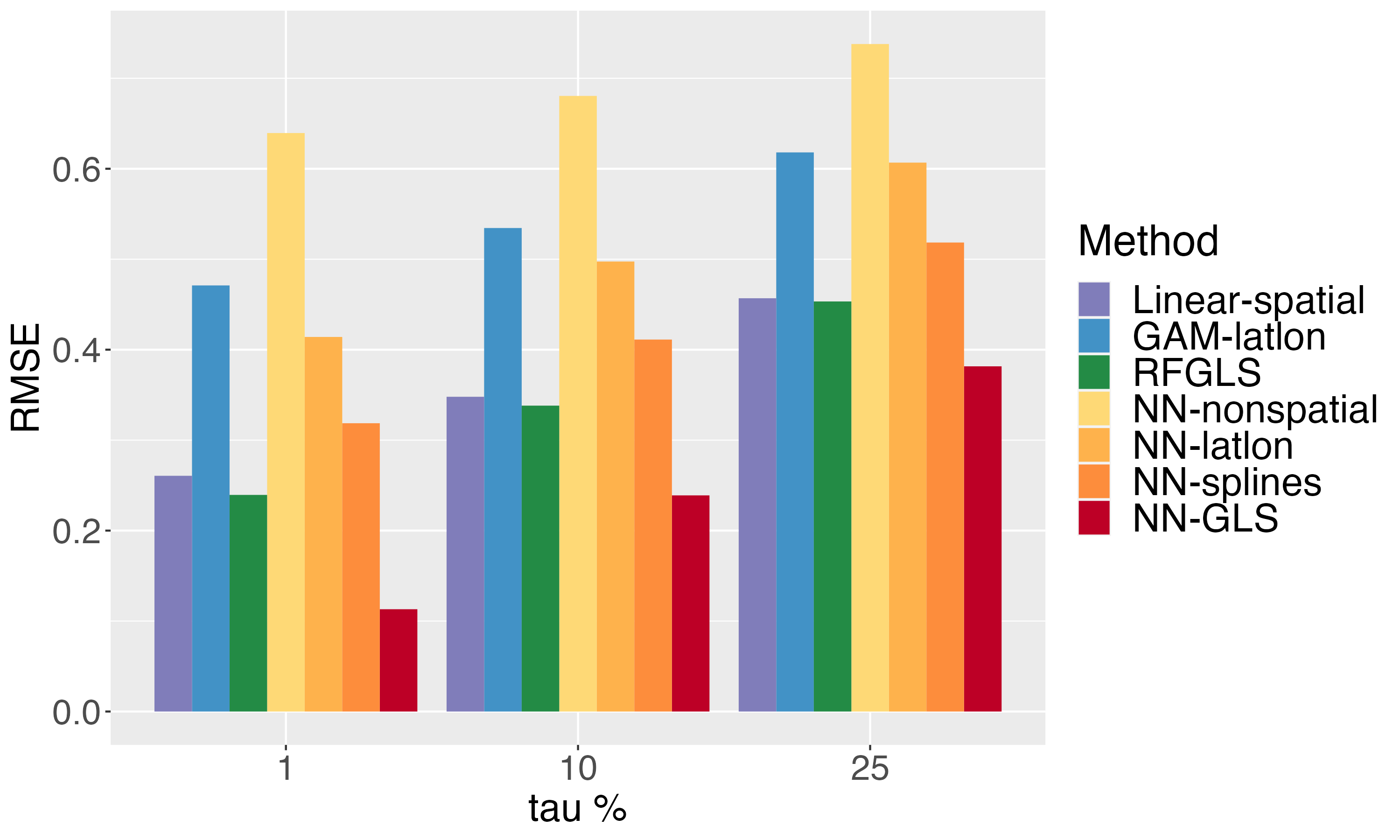}
  \caption{\blue{Prediction performance for all methods.}}
\end{subfigure}%
\begin{subfigure}{.5\textwidth}
  \centering
  \includegraphics[width=0.85\linewidth]
  {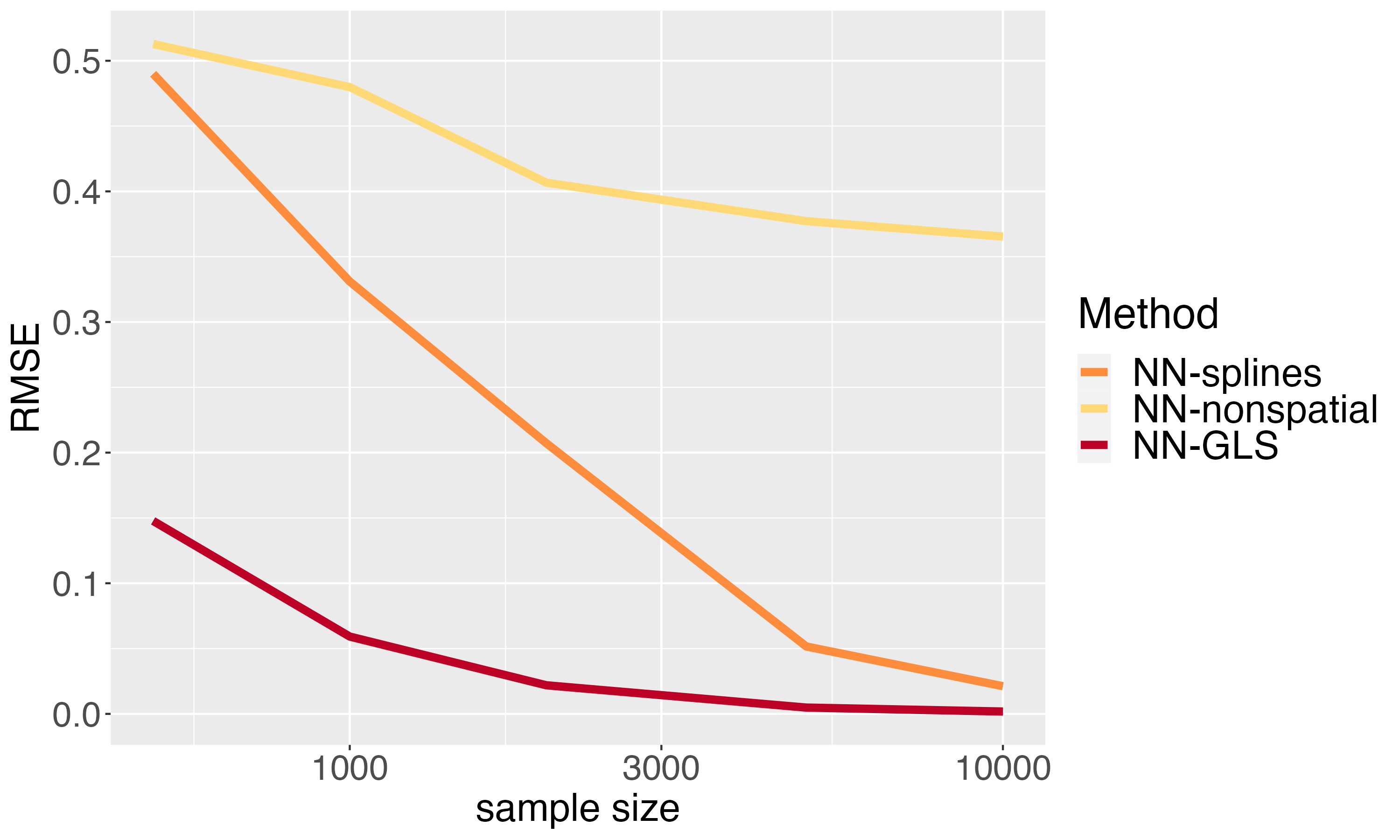}
  \caption{Added-spatial-features NN vs NN-GLS.}
\end{subfigure}
\\
\begin{subfigure}{.5\textwidth}
  \centering
  \includegraphics[width=0.85\linewidth]{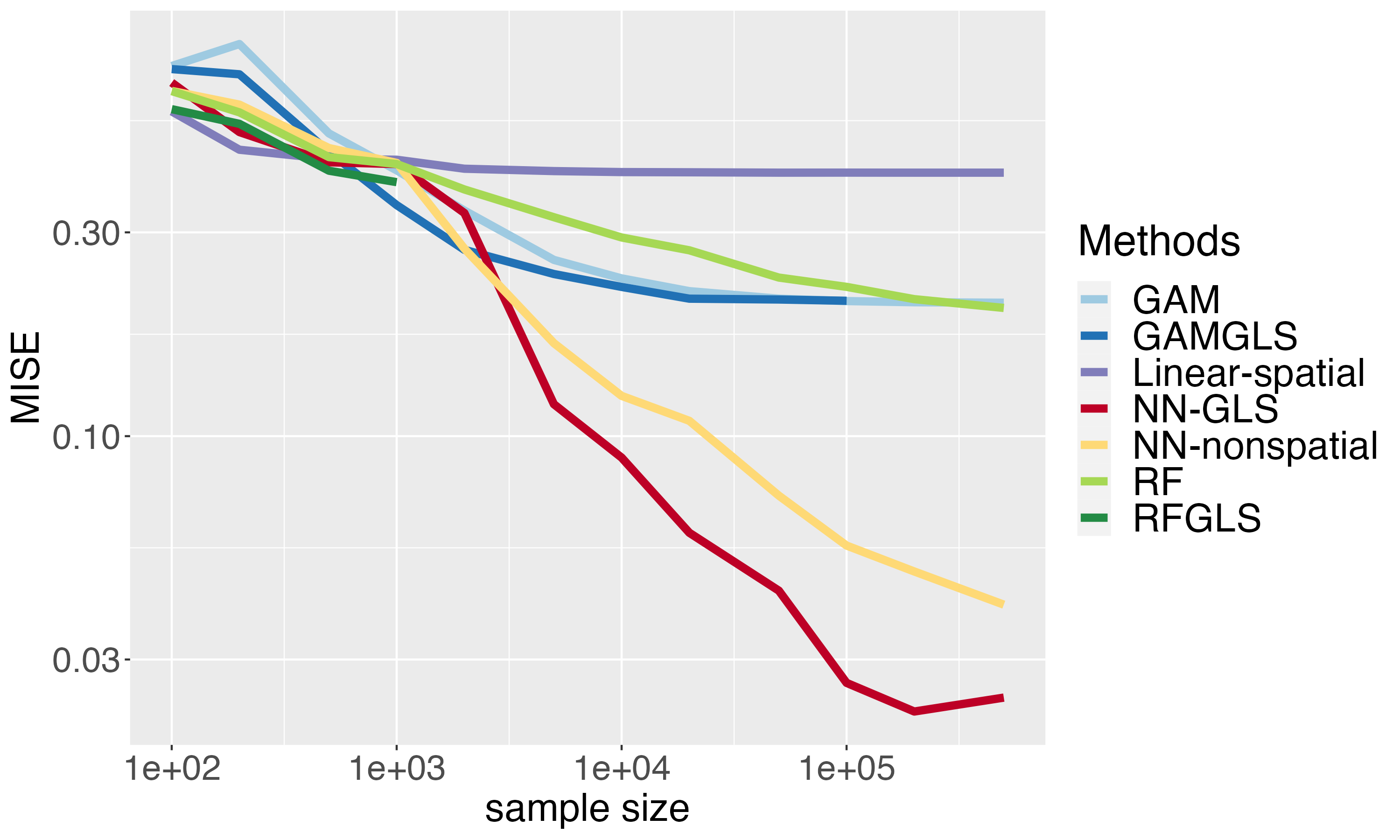}
  \caption{\blue{Consistency of estimation.}}
\end{subfigure}%
\begin{subfigure}{.5\textwidth}
  \centering
  \includegraphics[width=0.85\linewidth]
  {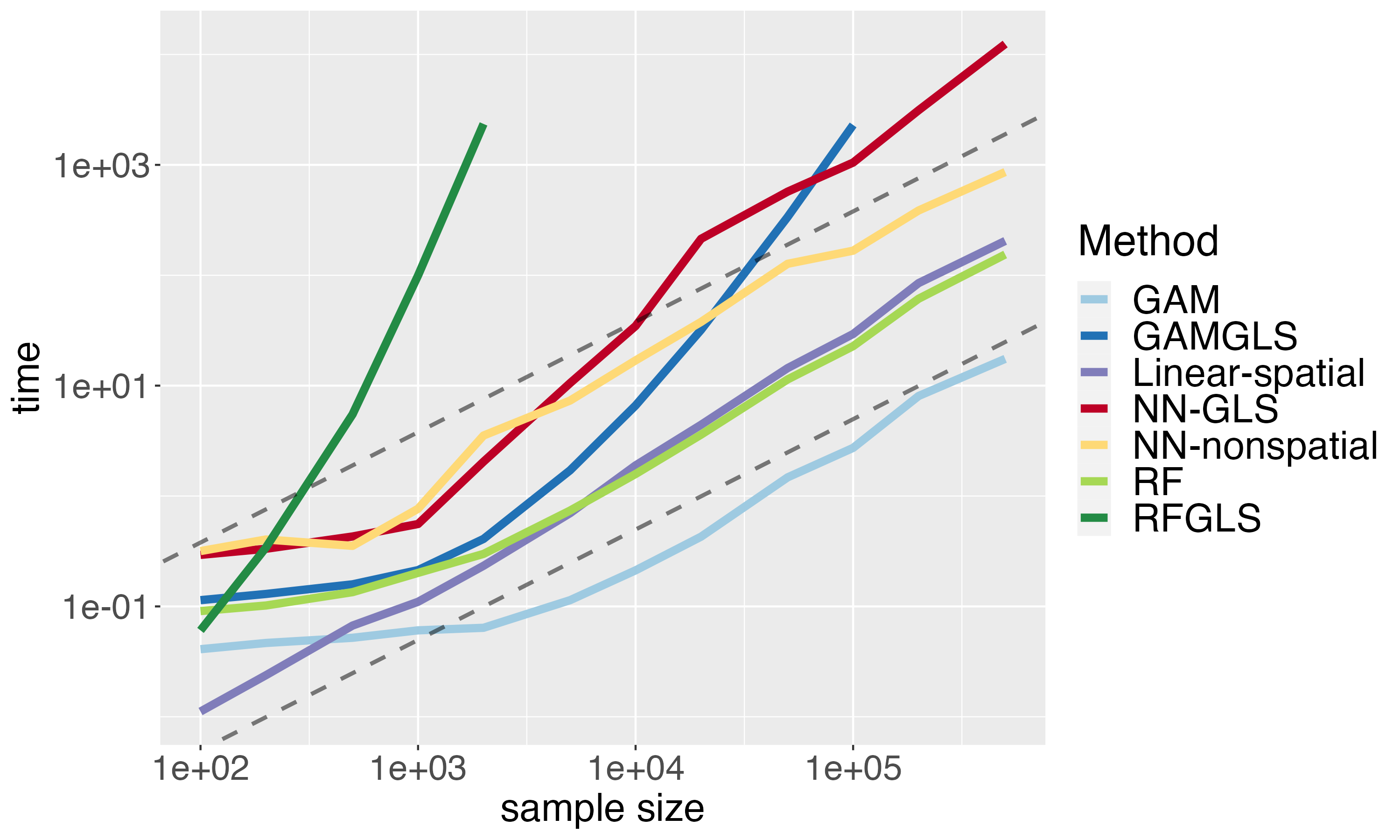}
  \caption{Running time.}
\end{subfigure}
\caption{\blue{(a): The estimation performance comparison with $f_0 = f_2$; (b): (Section \ref{sec-sim-friedman}) Estimation performance comparison among GAM, GAMGLS and NN-GLS against $rho$, i.e. the interaction strength; (c): The prediction performance comparison with $f_0 = f_2$; (d): (Section \ref{sec-sim-NN}) Prediction performance comparison among non-spatial NN, NN-GLS and NN-splines against sample size; (e): (Section \ref{sec-sim-large}) The consistency of estimation. (f): (Section \ref{sec-sim-large}) The running time for estimation.}}
\label{fig-sim-main}
\end{figure}

\blue{A snapshot of the results from the different experiments are provided in Figure \ref{fig-sim-main}. We first evaluate the estimation performance of the methods that offer an estimate of the mean function --- linear-spatial model, GAM, GAM-GLS, RF, RF-GLS, NN-nonspatial, and NN-GLS. 
We use the Mean Integrated Square Error (MISE) for the estimate $\hat f$.  
Figures \ref{fig-sim-2}(a) and \ref{fig-sim-1}(a) in Sections \ref{sec-sim-5} and \ref{sec-sim-noise} respectively present the estimation results for all scenarios. A representative result is provided in Figure \ref{fig-sim-main} (a) which presents the MISE for all methods when $f_0=f_2$ and for $3$ choices of the nugget variance $\tau^2$. NN-GLS consistently outperforms the other methods. The non-spatial neural network, which uses OLS loss, performs poorly for estimation even though it uses the same function class as NN-GLS for the mean. This shows the importance of explicitly accounting for the spatial covariance in the neural network estimation process. We further compare the two in Figure \ref{fig-sim-5-phi} which shows that} NN-GLS has a more significant advantage over non-spatial neural networks when 
the spatial decay $\phi$ is small. This is expected as for small $\phi$, there is strong spatial correlation in the data, so the performance of NN-nonspatial suffers on account of ignoring this spatial information. The deterioration in the performance of NN-nonspatial over NN-GLS is lesser for large $\phi$, as there is only a weak spatial correlation in the data.

\blue{In Figure \ref{fig-sim-main}(b) we present focussed comparisons between NN-GLS and methods from the GAM family. We consider variants of the Fridman function $f_2$ for the mean controlling the weight of the interaction term with a parameter $\rho \in [0,1]$. Neither GAM or GAM-GLS can model interaction, and this is reflected in their performance. When $\rho$ is small (weak interaction effect), their MISE are close to that from NN-GLS. However, when $\rho$ is large (strong interaction effect), the MISE of the GAM methods are considerably worse. This shows the advantage of the neural network family over GAM for non-linear regression in the presence of interaction terms (see Section \ref{sec-sim-friedman} for details).}

We compare prediction performances using 
 the Relative
Mean Squared Error (RMSE) on the test set, 
obtained by standardizing the MSE by the variance of the response so that the quantity can be compared across different experiments. \blue{Figures \ref{fig-sim-2}(b) and \ref{fig-sim-1}(b) in Sections \ref{sec-sim-5} and \ref{sec-sim-noise} present the prediction results from all scenarios. 
When $f_0=f_1$, i.e., the sine function (Figure \ref{fig-sim-2}(b)), the spatial linear model, unsurprisingly, offers the worst prediction performance by missing the strong non-linearity. NN-GLS is consistently the best, on account of using spatial information both in estimation and prediction, and the performances of other methods lie between NN-GLS and the linear model. 
In Figure \ref{fig-sim-main}(c), we present the more interesting scenario when $f_0 = f_2$. 
It's surprising to notice that the linear-spatial model outperforms all methods except NN-GLS in terms of prediction performance. The reason is the partial-linearity of Friedman function and the power of the GP in the linear model, NN-GLS offers the lowest RMSE. In fact, all the GP-based approaches (spatial linear model, RF-GLS, NN-GLS) outperform the added-spatial-features approaches (GAM-latlon,NN-latlon,NN-splines).} 
\blue{
This demonstrates the limitation of the added-spatial-features approaches compared to NN-GLS which embeds NN in the GP model.} The choice and the number of the added basis functions may need to be tuned carefully to optimize performance for specific applications. \blue{Also the sample sizes may need to be much larger to fully unlock the non-parametric ability of the added-spatial-featrures approaches. We see this in Figure \ref{fig-sim-main}(d) which shows the prediction performance against sample size $n$ under fixed-domain sampling. For large $n$, the performance of NN-splines becomes similar to NN-GLS, although NN-GLS still has lower RMSE. NN-GLS circumvents  basis functions by parsimoniously modeling the spatial dependence directly through the GP covariance and performs well for both small and large sample sizes and fixed- and increasing-domain sampling paradigms (see Section \ref{sec-sim-NN} for all the scenarios).}

\blue{
Figures \ref{fig-sim-main}(e) and (f) present the estimation performance and running times for different methods up to sample sizes of $500,000$. The large sample runs demonstrate how the approximation error for NN-GLS goes to zero while for others like the GAM family and linear models stay away from zero as they cannot approximate the Friedman function. This empirically verifies the consistency result. NN-GLS also scales linearly in $n$ due to the innovations of Section \ref{sec:gnn} and is much faster than the other non-linear GLS approaches like RF-GLS and GAM-GLS. Details are in 
Section \ref{sec-sim-large}.}

Section \ref{Append-sim} of the Supplementary materials presents all the other simulation results. 
\blue{
The confidence interval for $\hat{f}$ and the prediction interval for $\hat{y}$ are evaluated in Sections \ref{sec-sim-CI} and \ref{sec-sim-PI} respectively. Both types of intervals provided by NN-GLS provide near-nominal coverage as well as superior interval scores against the competing methods. To study impact of sampling design, a denser sampling is considered in Section \ref{sec-sim-5000} keeping the domain fixed. This is to supplement the theoretical results which are only for the increasing domain.} 
In section \ref{sec-sim-orc} we investigate if the estimation of the spatial parameters has an impact on the performance of NN-GLS by comparing it to {\em NN-GLS (oracle)} which uses the true spatial parameters. We observe that NN-GLS's performance is quite similar to that of NN-GLS (oracle) since it provides an accurate estimation of spatial parameters. NN-GLS also performs well for a higher dimensional mean function (of $15$ covariates) (Section \ref{sec-sim-hd}). We finally assess the robustness of NN-GLS to model misspecification, including misspecification of the GP covariance (Section \ref{sec-sim-mis1}), and complete misspecification of the spatial dependence  (Section  \ref{sec-sim-mis2}). For both estimation and prediction, NN-GLS performs the best or comparably with the best method for both cases of misspecification.


\section{Real data example}\label{sec-real}
The real data example considered here is the spatial analysis of PM$_{2.5}$ (fine particulate matter) data in the continental United States. Following a similar analysis in  \cite{chen2020deepkriging}, we consider the daily PM$_{2.5}$ data from $719$ locations on \blue{June 18th, 2022} and regress PM$_{2.5}$ on six meteorological variables --- precipitation accumulation, air temperature, pressure, relative humidity, west wind (U-wind) speed, and north wind (V-wind) speed. \blue{Section \ref{sec:data} has more details about the data.} 

\begin{figure}[!h]
\centering
\begin{subfigure}{.54\textwidth}
  \centering
  \includegraphics[width=0.95\linewidth]{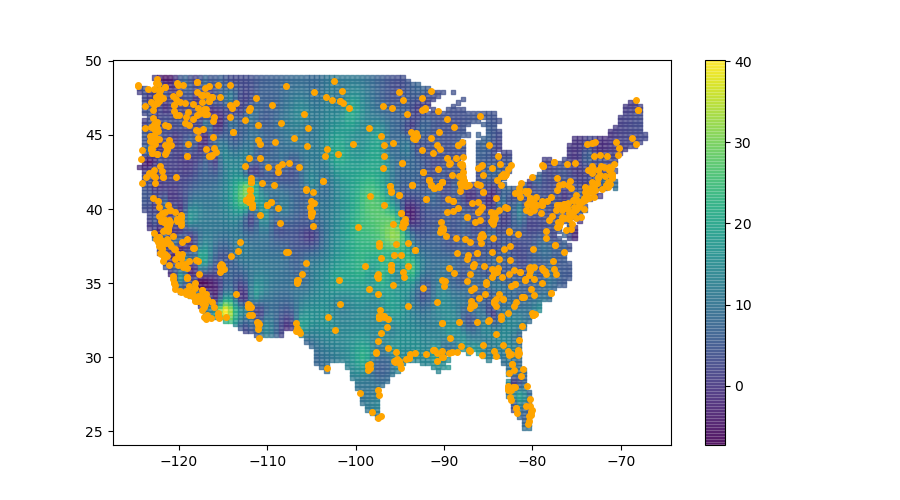}
  \caption{Interpolated PM$_{2.5}$ data on \blue{June 18th, 2022} along with the data locations}
\end{subfigure}
\begin{subfigure}{.44\textwidth}
  \centering
  \includegraphics[width=0.78\linewidth]{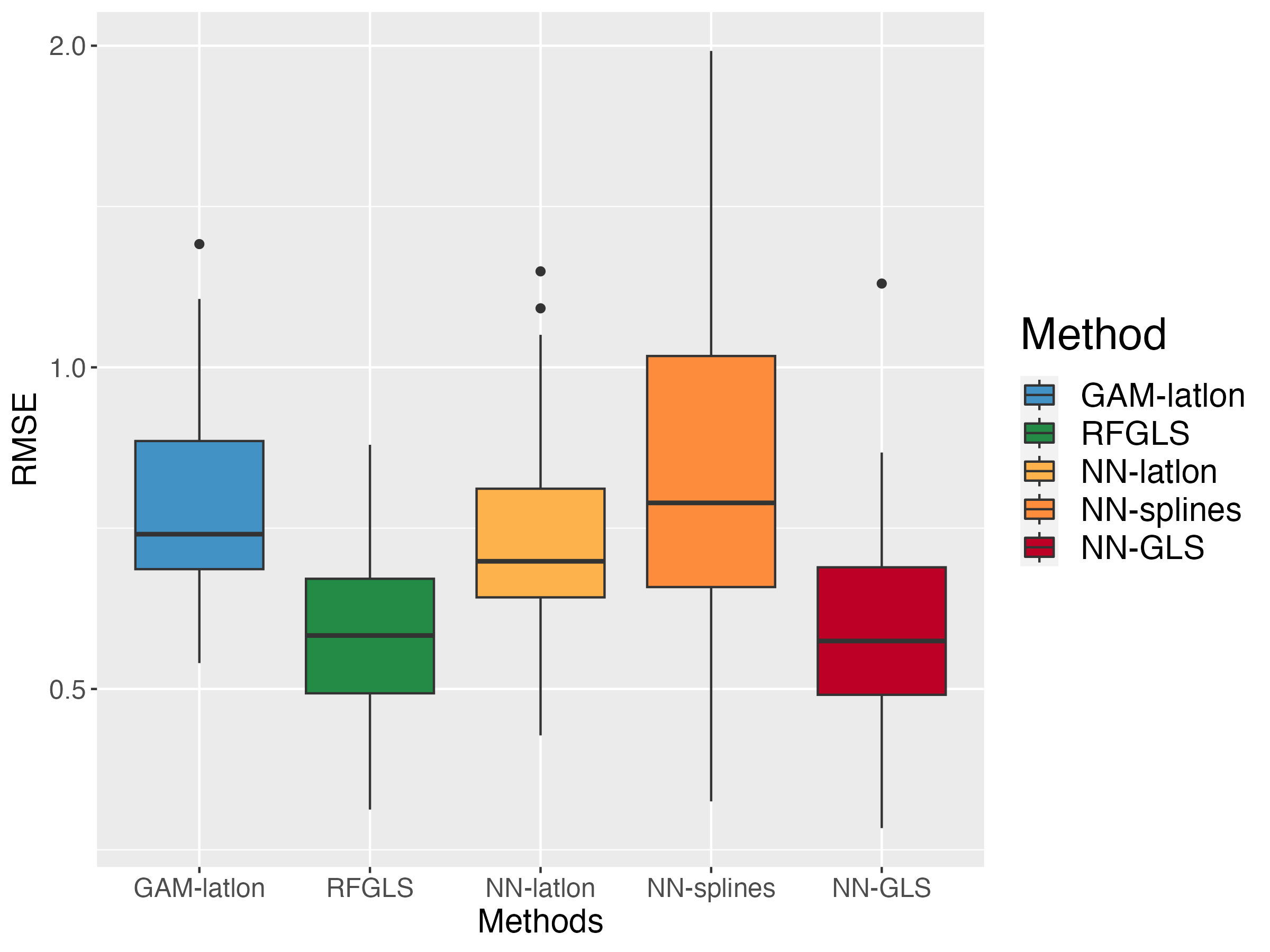}
  \caption{\blue{Prediction performance}}
\end{subfigure}%
\caption{PM$_{2.5}$ data analysis.}
\label{fig-real-1}
\end{figure}

Figure \ref{fig-real-1} (a) demonstrates the PM$_{2.5}$ distribution on the date (smoothed by inverse distance weighting interpolation), as well as the nationwide EPA monitor's distribution (the orange spots). The spatial nature of PM$_{2.5}$ is evident from the map. \blue{We consider the following 5 methods previously introduced in Section \ref{sec-sim}: GAM-latlon, RF-GLS, NN-latlon, NN-splines, and NN-GLS. To evaluate the prediction performance, we randomly take $83\%$ ($5/6^{th}$) of the data as a training set and the rest part as a testing set using block-random data splitting strategy, which is closer to a real-world scenario (see Appendix \ref{sec-block} for details)}, train the model and calculate the RMSE of the prediction on the testing set. The procedure is repeated for $100$ times. 
The performance of each method is shown in Figure \ref{fig-real-1}(b). 
\blue{We find that NN-GLS and RF-GLS have comparably the lowest average RMSE, while GAM-latlon, NN-latlon, and NN-splines are outperformed.
This trend is consistent for other choices of dates and for different ways of data splitting, except for some cases where the spatial pattern is not clear and NN-spline may do similarly (see Section \ref{sec:real_days} of the Supplementary materials). Sensitivity to the choice of hold-out data is presented in Section \ref{sec-block}. Section \ref{sec-real-PI} evaluates the performance of prediction intervals on hold-out data. We see that prediction intervals from NN-GLS offers near nominal coverage. Model adequacy checks are performed in Section \ref{sec:asmp-check} and reveal that data from different days adhere to the modeling assumptions of Gaussianity and exponentially decaying spatial correlation to varying degrees. As NN-GLS performs well on all days, it gives confidence in the applicability of NN-GLS in a wide variety of scenarios.} 

NN-GLS also provides a direct estimate of the effect of the meteorological covariates on PM$_{2.5}$, specified through the mean function $f(\bX)$. 
However, $\bX$ is six-dimensional in this application, precluding any convenient visualization of the function $f(\bX)$. 
Hence, we use partial dependency plots (PDP) -- a common tool in the machine learning world for visualizing the marginal effect of one or two features on the response \blue{(see Section \ref{sec:PDP} for a detailed introduction)}. In Figure \ref{fig-PDP-main}, we present the PDP of PM$_{2.5}$ on temperature and wind. While the PM$_{2.5}$ level is linearly affected by temperature, we see clear non-linear effects in the wind, thereby demonstrating the need to move beyond the traditional linear spatial models and consider geostatistical models with non-linear means, like NN-GLS.

\begin{figure}[!h]
\centering
\begin{subfigure}{.49\textwidth}
  \centering
  \includegraphics[width=0.85\linewidth]{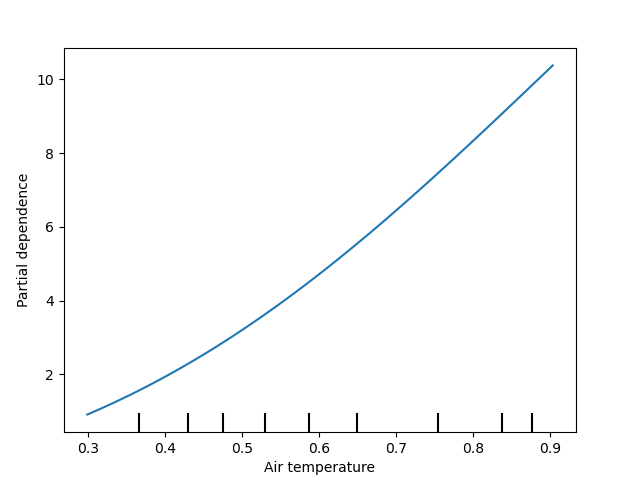}
  \caption{Air temperature}
\end{subfigure}
\begin{subfigure}{.49\textwidth}
  \centering
  \includegraphics[width=0.85\linewidth]{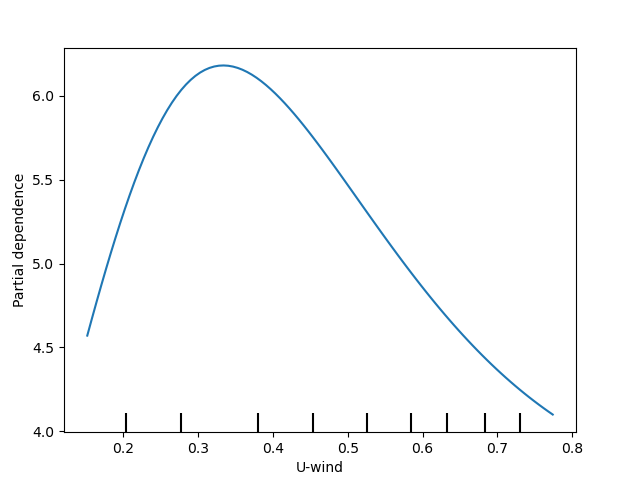}
  \caption{U-wind}
\end{subfigure}%
\caption{Partial dependency plots showing the marginal effects of temperature and west wind (U-wind).}
\label{fig-PDP-main}
\end{figure}

\section{Discussion}\label{sec:disc}
In this work, we showed how neural networks can be embedded directly within traditional geostatistical Gaussian process models, creating a hybrid machine learning statistical modeling approach that balances parsimony and flexibility. Compared to existing renditions of NN for spatial data which need to create and curate many spatial basis functions as additional features, NN-GLS simply encodes spatial information explicitly and parsimoniously through the GP covariance. It reaps all the benefits of 
of the model-based framework including the separation of the non-spatial and spatial structures into the mean and the covariance, the use of GLS loss to account for spatial dependence while estimating the non-linear mean using neural networks, and the prediction at new locations via kriging.


We show that NN-GLS can be represented as a  graph neural network. The resulting GLS loss is equivalent to adding a graph convolution layer to the output layer of an OLS-style neural network. 
This ensures the computational techniques used in standard NN optimization like mini-batching and backpropagation can be adapted for NN-GLS, resulting in a linear time algorithm. Also,  kriging predictions can be obtained entirely using the GNN architecture using graph convolution and deconvolution.

Our connection of NN-GLS to GNN is of independent importance, as it demonstrates 
the connection of GNN to traditional geostatistics. Current adaptations of GNN to spatial data \citep{tonks2022forecasting,fan2022gnn} mostly use simple graph convolution choices like equal-weighted average of neighboring observations. For irregular spatial data, this disregards the unequal distances among the observations, which dictate the strength of spatial correlations in traditional geospatial models. Our representation result shows how graph-convolution weights arise naturally as nearest neighbor kriging weights which accounts for mutual distances among data locations. Future work will extend this GNN framework and consider more general graph deconvolution layers, and applicable 
 to other types of irregular spatial data where graphical models are already in use, like multivariate spatial outcomes using inter-variable graphs \citep{dey2022graphical} or areal data graphs \citep{datta2019spatial}.

We 
prove general theory on the existence and consistency of neural networks using GLS loss for spatial processes, including Mat\'ern GP, observed over an irregular set of locations. \blue{We establish finite sample error rates of NN-GLS, which improves when the working covariance matrix used in the GLS loss is close to the true data covariance matrix. It proves that ignoring the spatial covariance by simply using the OLS loss leads to large error rates.}
To the best of our knowledge, we presented the first theoretical results for any neural network approach under spatial dependency.   

There is a gap between the theory and the actual implementation of NN-GLS. The theory relies on a restricted class of neural networks, and 
does not consider steps like mini-batching and backpropagation used in practice. However, even in a non-spatial context, this optimization error gap between the practice and theory of NN is yet to be bridged. \blue{Our theory is also not complexity-adaptive as we do not assume any structure on the true mean function beyond continuity, It will be interesting to pursue adaptive rates for NN-GLS akin to recent works of \cite{schmidthieber2020} and others in a non-spatial context (see Section \ref{sec:threv}). Extension to high-dimensional covariates as considered in \cite{fan2023} is also a future direction. Theoretically justified methodology for pointwise inference (interval estimates) for the mean function from neural networks remains open even for i.i.d. case.}



 We primarily focussed on modeling the mean as a function of the covariates using a rich non-linear family, i.e., the neural network class, while using stationary covariances to model the spatial structure. However, non-stationarity can be easily accommodated in NN-GLS either by including a few basis functions in the mean or using the GLS loss with a non-stationary covariance function. For example, \cite{zammit2021deep} proposed rich classes of non-stationary covariance functions using transformations of the space modeled with neural networks. 

\section*{\cyan{Supplementary materials and Software}}
\cyan{The supplement contains proofs, additional theoretical insights or methodological details, and results from more numerical experiments. A Python implementation of NN-GLS which is publicly available in the geospaNN package at \url{https://pypi.org/project/geospaNN/}.
}

\section*{\cyan{Acknowledgement and Disclosures}}
\cyan{This work is supported by National Institute of Environmental Health Sciences grant R01ES033739 and National Science Foundation (NSF) Division
	of Mathematical Sciences grant DMS-1915803. 
The authors report there are no competing interests to declare.}

\begin{singlespace}
{\footnotesize
\bibliography{zwtBib}}
\end{singlespace}
\newpage

\setcounter{section}{0}
\setcounter{theorem}{0}
\setcounter{equation}{0}
\setcounter{lemma}{0}
\setcounter{proposition}{0}

\renewcommand{\thesection}{S\arabic{section}}
\renewcommand{\thesubsection}{S\arabic{section}.\arabic{subsection}}
\renewcommand{\thesubsubsection}{S\arabic{section}.\arabic{subsection}.\arabic{subsubsection}}
\renewcommand{\thealgorithm}{S\arabic{algorithm}}
\renewcommand{\thefigure}{S\arabic{figure}}
\renewcommand{\thetable}{S\arabic{table}}
\renewcommand{\theequation}{S\arabic{equation}}
\renewcommand{\thetheorem}{S\arabic{theorem}}
\renewcommand{\thelemma}{S\arabic{lemma}}
\renewcommand{\theproposition}{S\arabic{proposition}}

\section{NN-GLS algorithm}\label{sec:alg}

\begin{algorithm}[!h]
\caption{NN-GLS algorithm using GNN representation}\label{alg-NN-GLS-main}
\begin{algorithmic}
\REQUIRE{Data $\mathbf{X}, \bY, \bm{S}=\{s_1,\ldots,s_n\}$, 
neighborhood size $m$, number of mini-batches $B$, update frequency $I$.}\\

\STATE \textbf{Initialization}: 
\STATE Create the $m$ nearest-neighbor DAG for the locations $\bS$.
\STATE Split $\bS$ 
randomly into $B$ mini-batches  $(S_1, \cdots, S_B)$. 
\STATE Obtain initial estimate $\hat{f}_{init}(\cdot)$ using a non-spatial NN (\ref{def-NN}). 
\STATE Estimate the spatial parameters $\hat{\theta}$ from (\ref{eq:nngploglik}) with $\bY$ and $\hat{\bm{f}}_{init}(\Xb)$.
\STATE Obtain the graph-convolution weights $\bv_i$ from (\ref{eq:graph_weights})

 \textbf{Estimation:} 
 \FOR{$epoch = 1: \text{max epochs}$}
  \FOR{$b = 1:B$}
  \STATE Compute $\mathcal L_{b,n} = \sum_{i \in S_b} (Y^*_i - O^*_i)^2$ with $Y^*_i = \bv_i^\top \bY_{N^*[i]}$ and, $O^*_i = \bv_i^\top \bO_{N^*[i]}$.
  \STATE Update $\hat{f}(\cdot)$ by updating the NN weights via gradient-descent 
   (\ref{eq:backprop}).
  \ENDFOR
    \IF{$\text{epoch}\mod I = 0$}
    \STATE Update $\hat{\btheta}$ from (\ref{eq:nngploglik}) with $\bY$ and the current $\hat{\bm{f}}(\Xb)$.
    \STATE Update the graph-convolution weights $\bv_i$ from (\ref{eq:graph_weights}).
    \ENDIF
    \IF{Early stopping rule is met}
    \STATE Break.
    \ENDIF
  \ENDFOR
\STATE \textbf{Prediction} 
at new location $s_{0}$ with covariates $\bX_{0}$. 
\STATE Add $s_{0}$ as a new node in the DAG connected to $m$ nearest neighbors $N(0)$.
\STATE Predict $\hat Y_{0}$ using (\ref{eq:gnnkrig}). 
\STATE \textbf{Output:} $\hat{f}(\cdot)$, $\hat{\theta}$, and $\hat{Y}_{0}$. 
\end{algorithmic}
\end{algorithm}

\newpage
\blue{\section{Uncertainty quantification (UQ) for the estimated mean function}\label{sec-mtd-CI}
\subsection{Review of current theory and methods for UQ in neural networks}}
\blue{Pointwise uncertainty quantification for the estimate $\hat f(\bx)$ from a neural network at a specified input value $\bx$ remains challenging, despite the surge in recent theoretical advances to understanding NN estimators. Most of the state-of-the-art theory (see Section \ref{sec:threv} for a review) have only studied properties of the point estimate. Two exceptions to this are asymptotic normality results provided in \cite{shen2023asymptotic} and \cite{farrell2021deep}. However, these asymptotic results are for averages of the regression function, i.e., they provide asymptotic inference for $\frac 1n \sum_i \hat f(\bx_i)$. These averages are useful for studying causal effects, which will be a difference of these averages between treatment and control groups, as studied in \cite{farrell2021deep}.

However, the goal of this paper is not causal inference. We are interested in the typical objectives of geospatial analysis -- estimation of the mean function $f(\bx)$ and spatial predictions at locations of interest, not of averages over the entire population. If a causal estimate was of interest, we conjecture that the result of \cite{shen2023asymptotic} can be extended to our dependent setting and for NN-GLS with GLS loss (instead of vanilla NN with OLS loss) for certain spatial designs.  We do not pursue this here. 

Neither \cite{shen2023asymptotic} nor \cite{farrell2021deep} provide any asymptotic result for pointwise estimate of the mean function $\hat f(\bx)$. In fact, even in the setting of iid data, we could not find any existing theory that can guide pointwise uncertainty quantification of the regression function estimate from neural networks. Hence, developing a theoretical way to obtain pointwise inference on the mean functions from NN-GLS for irregular spatially dependent data is beyond the scope of the paper, as the problem remains open even in the iid setting for vanilla NN. 

Alternative approaches to uncertainty quantification would be full or variational Bayes implementation of the method or resampling/bootstrapping approaches. While a full Bayesian formulation of NN-GLS is conceptually straightforward, running an MCMC for the high-dimensional set of network parameters is infeasible for even moderate sample sizes. Also, full Bayes will yield valid interval estimates only under the correct specification of the error distribution class. Variation Bayes can be a more pragmatic choice. \cite{gal2016dropout} demonstrated that {\em dropout}, the technique of randomly making some weights of the NN to be 0, can be used for uncertainty quantification. The justification lies in demonstrating that the ensemble of estimates using random dropouts approximately corresponds to a variational posterior under binary variational families for the weight parameters. Approximating continuous posterior distributions of the weight parameters with binary distributions is crude, prohibiting valid UQ from this approach. While a similar dropout approach can be implemented for NN-GLS in the future, it is likely to inherit the same issues on account of the crude variational approximation.

The most common empirical approach to obtain pointwise uncertainty estimates for neural network regression has been using resampling/bootstrap methods which can leverage parallel computing. However, na\"ive bootstrap cannot be applied for geospatial data as data units are not independent. Instead, in the next Section, we propose a spatial bootstrap approach for providing uncertainty estimates of the mean function.  

\subsection{Spatial bootstrap for NN-GLS}\label{sec:boot}
We propose a spatial bootstrap to obtain the CI for our estimated mean function extending the idea in \cite{saha2018brisc} to non-linear means. The key idea is that, under the NN-GLS model (\ref{eq:spnlm}) $\bY \sim \mathcal{N}(\mathbf{f}(\bX), \bSigma)$, we have $\bSigma^{-1/2}(\bY - \mathbf{f}(\bX)) \sim \mathcal{N}(0, \Ib)$. These decorrelated residuals are iid and can be resampled. Following the practice of parametric bootstrap, we will replace $f$ and $\bSigma$ with the NN-GLS point estimates. Also, as we use NNGP covariance matrix $\tilde \bSigma$, the precision matrix $\Qb=\tilde \bSigma^{-1}$ as well as the Cholesky factor $\Qb^{1/2}$ are readily available, facilitating a scalable algorithm. 

The details of the spatial bootstrap are provided in the algorithm below:
\begin{enumerate}[noitemsep]
    \item Obtain the NN-GLS point estimates $\hat f$ and the spatial covariance parameters $\hat{\btheta}$ (see Section \ref{sec:backprop}).
    \item Obtain the NNGP Cholesky factor $\hat{\Qb}^{\frac{1}{2}}(\hat \btheta) = \hat{\Qb}^{\frac{1}{2}}=\hat\Fb^{-1/2}(\Ib - \hat\Bb)$  where $\Bb$ and $\Fb$ are specified in (\ref{eq:nngp_bf}).
    \item Decorrelate the residuals to obtain $\bm{\omega} = \hat{\Qb}^{\frac{1}{2}}\big(\bm{Y} - \hat{\bm{f}}(\bX)\big)$.
    \item Generate $B$ bootstrap samples $\{\bm{\omega}^*_b\}_{b = 1, \cdots, B}$ by resampling without replacement the entries of $\bm\omega$.
    \item Correlate back to generate bootstrapped datasets $\bm{Y}^*_b = \hat{\bm{f}}(\bm{X}) + \hat{\Qb}^{-\frac{1}{2}}\bm{\omega}^*_b$
    \item Run NN-GLS on $(\bm{Y}^*_b, \bX)$, the estimates $\big\{\hat{\bm{f}}_b(\bX_{new})\big\}_{b = 1, \cdots, B}$ is the bootstrap sample for $\hat{\bm{f}}(\bX_{new})$, from which the confidence interval is obtained.
\end{enumerate}
Under the sparsity obtained from the NNGP process, both the decorrelation step (Step 3.) and the 'correlate back' steps (Step 2.) are only of $O(n)$ complexity (see \cite{datta2022nearest} Section 3.1), making the CI estimation scalable to large data sets. Section \ref{sec-sim-CI} evaluates the effectiveness of our approach under different settings.}
\newpage

\section{Brief review of current theory for neural networks}\label{sec:threv}
\blue{
There is now a large body of theoretical literature studying the statistical properties of neural networks. We briefly highlight some notable relevant work on the theory of neural network regression, and outline how our theoretical contribution, focussing on spatial data, distinguishes itself from the existing literature. Theory of neural networks has been discussed for a long time \citep{hornik1989multilayer} with the early series of work by \cite{barron1994approximation}, \cite{makovoz1996random}, \cite{chen1998sieve}, \cite{chen1999improved} developing the asymptotic theory for the single-layer neural networks. 
\cite{farrell2021deep} provides sharp finite-sample error rates for deep neural networks. However, this work assumes bounded errors, which is not applicable here as Gaussian Processes have unbounded support. Recent work \citep{bauer2019,schmidthieber2020,kohler2021} has provided complexity-adaptive error rates for non-parametric regression using deep neural networks, demonstrating that if the true function admits a lower-complexity form through compositions, then neural networks overcome the curse of dimensionality and achieve tighter rates than other common non-parametric regression approaches. 
\cite{fan2023} considers the setting of high-dimensional covariates and develops new neural networks exploiting a {\em low-rank-plus-sparse} representation of the covariates. They propose using  {\em diversified projections} of the covariates \citep{fan2022learning} as a proxy for the low-rank factors and use them as features in the neural network, and show that this {\em factor-augmented} network can achieve sharp and complexity-adaptive error rates under the high-dimensional setting. Estimation of the diversified projections needs independent split of the data which is challenging for spatial settings as all data are dependent. More generally, the complexity-adaptive theory for neural networks assume a true regression function that can be expressed as a hierarchy of compositions, which may be hard to verify in practical settings.  \cite{shen2023asymptotic} provides results on the existence, asymptotic consistency, and finite-sample error rates of one-layer neural networks.   
We adopt their theoretical framework here, thereby not assuming finite errors or any structure on the function beyond continuity. 

Our theoretical contributions have two main differences from the existing theory. 
First, the overwhelming majority of the theoretical developments reviewed above focus on the setting of iid data units. \cite{chen1998sieve} and \cite{chen1999improved} considered weakly ($\beta$-mixing) dependent data in the time series context. They did not consider any spatial setting and it is unclear if their conditions on the decay rate of the $\beta$-mixing coefficients hold for common Gaussian processes under irregular spatial designs. 
To our knowledge, the theory for neural networks under spatial dependence among observations from an irregular set of locations has not been developed. 

Second, and more importantly, the existing theory has almost exclusively considered neural network methods like the vanilla NN (using the OLS loss) that do not make any procedural modifications to explicitly accommodate spatial covariance among the data units. There is no asymptotic theory for neural networks modified to account for spatial structure (like the added-spatial-features strategy reviewed in Section \ref{sec-mtd-NN}).
Our theory establishes existence (Theorem \ref{Thm-exist-2}) and asymptotic consistency (Theorem \ref{Thm-main-1}) of the estimate of mean function from NN-GLS, which minimizes the GLS loss with a working covariance matrix that explicitly encodes the spatial covariance. The vanilla NN is subsumed as a special case with the identity matrix as the working covariance matrix, and hence consistency of vanilla NN under spatial dependence, also a novel result, comes as a corollary. However, our finite-sample error rate result (Theorem \ref{thm-rate}) shows how a poor choice of the working covariance matrix, like the identity matrix, will lead to large error rates, shedding light on the perils of using vanilla NN for spatial data and highlighting the necessity to explicitly model spatial covariance in neural networks.}




\section{Proofs}\label{Append-prof}
\subsection{Proof of existence}
In order to obtain the minimizer, we frame the process of training NN-GLS as a mapping from observations to $\mathcal{F}_n$ by minimizing the GLS loss function. Following \citep{shen2023asymptotic}, we leverage a general result of \cite{white1991some} where it suffices to prove the compactness of the domain and continuity of the mapping to show that the optimum can be achieved, thereby proving the existence of the minimizer. 
\begin{proof}[Proof of Theorem \ref{Thm-exist-2}]
We first state a general result for proving the existence of optimizers. 
\begin{theorem}[Theorem 2.2 in \cite{white1991some}]\label{Thm-exist-1}
Let $(\Omega, \mathcal{A}, \mathbb{P})$ be a complete probability space and let $(\Theta, \rho)$ be a pseudo-metric space. Let $\{\Theta_n\}$ be a sequence of compact subsets of $\Theta$. Let $\mathbb{Q}_n: \Omega\times\Theta_n \to \mathbb{R}$ be $\mathcal{A}\times\mathcal{B}(\Theta_n)/\mathcal{B}(\mathbb{R})$-measurable, and suppose that for each $\omega \in \Omega$, $\mathbb{Q}_n(\omega,\cdot)$ is lower semicontinuous on $\Theta_n$, $n = 1, 2, \cdots$. Then for each $n = 1, 2, \cdots$, there exists $\hat{\theta}_n: \Omega \to \Theta_n$, $\mathcal{A}/\mathcal{B}(\Theta_n)$-measurable such that for each $\omega \in \Omega$, 
\[
\mathbb{Q}_n(\omega, \hat{\theta}_n(\omega)) = \inf_{\theta\in \Theta_n}\mathbb{Q}_n(\omega, \theta).
\]
\end{theorem}


To apply Theorem \ref{Thm-exist-1}, we need to check the following things: the compactness of the subclasses $\{\Theta_n\}$, i.e. $\mathcal{F}_n$; the measurability and continuity of the \blue{GLS} loss function $\mathbb{Q}_n$ i.e., $\mathcal{L}_n$ defined in (\ref{def-NN-loss-3}). For a fixed $n$, $\bX$, and $\bm{\epsilon}$, and letting $\mathbf{f}=f(\bX_1),\ldots,f(\bX_n)$, we have 
\[
\begin{split}
    \big|\mathcal{L}_n(f) - \mathcal{L}_n(g)\big| 
    &= \frac{1}{n}\Big|(\by - \bm{f})^\top\Qb(\by - \bm{f}) - (\by - \bm{g})^\top\Qb(\by - \bm{g})\Big| \\
    &= \frac{1}{n}\Big|(\by - \bm{f})^\top\blue{\mathbf{\Sigma}^{-\frac{\top}{2}}\Eb\mathbf{\Sigma}^{-\frac{1}{2}}}(\by - \bm{f}) - (\by - \bm{g})^\top\blue{\mathbf{\Sigma}^{-\frac{\top}{2}}\Eb\mathbf{\Sigma}^{-\frac{1}{2}}}(\by - \bm{g})\Big| \\
    &\leq \frac{1}{n}\left(2\Big|\by^\top\blue{\mathbf{\Sigma}^{-\frac{\top}{2}}\Eb\mathbf{\Sigma}^{-\frac{1}{2}}}(\bm{f} - \bm{g})\Big| + \Big|\bm{f}^\top\blue{\mathbf{\Sigma}^{-\frac{\top}{2}}\Eb\mathbf{\Sigma}^{-\frac{1}{2}}}\bm{f} - \bg^\top\blue{\mathbf{\Sigma}^{-\frac{\top}{2}}\Eb\mathbf{\Sigma}^{-\frac{1}{2}}}\bg\Big| \right)\\
    &\leq \frac{1}{n}\Big(2\Big|\by^\top\mathbf{\Sigma}^{-\frac{\top}{2}}\Eb\mathbf{\Sigma}^{-\frac{1}{2}}(\bm{f} - \bm{g})\Big| + \Big|\bm{f}^\top\mathbf{\Sigma}^{-\frac{\top}{2}}\Eb\mathbf{\Sigma}^{-\frac{1}{2}}(\bm{f}-\bm{g})\Big| \\
    &+ \Big|(\bm{f}-\bg)^\top\mathbf{\Sigma}^{-\frac{\top}{2}}\Eb\mathbf{\Sigma}^{-\frac{1}{2}}\bg\Big| \Big)\\
    &\leq \big(2\|\by\|_n+\|f\|_n+\|g\|_n\big)\|\Eb\|_2\|\mathbf{\Sigma}^{-1}\|_2\|f - g\|_n.
\end{split}
\]
For $f(\cdot)$ and $g(\cdot)$ from $\mathcal{F}_n$, by Definition (\ref{def-F_n}), $\|f\|_{\infty}$, $\|g\|_{\infty}\leq V_n$, and $\|f\|_{n}, \|g\|_n \leq V_n$. As a result, \blue{using Assumption \ref{Asmp-3}}
\[
\begin{split}
\big|\mathcal{L}_n(f) - \mathcal{L}_n(g)\big| &\leq 2(\|\by\|_n + V_n)\|\Eb\|_2\|\mathbf{\Sigma}^{-1}\|_2\|f-g\|_n\\
&\leq \frac{2\Lambda_{high}}{\Psi_{low}}(\|\by\|_n + V_n)\|f-g\|_n,
\end{split}
\]
which implies that $\mathcal{L}_n(\by, \cdot)$ is a continuous function on $\mathcal{F}_n$ (we write $\mathcal{L}_n(\by, \cdot)$ instead of $\mathcal{L}_n(\cdot)$ in order to apply Theorem \ref{Thm-exist-1}). By Lemma 2.1 in \cite{shen2023asymptotic}, $\mathcal{F}_n$'s are compact. Applying Theorem \ref{Thm-exist-1} where $(\Omega, \mathcal{A}, \mathbb{P})$ where the covariates, the spatial errors  are all measurable random variables,  
 with $\mathbb{Q}_n$ replaced by the loss function $\mathcal{L}_n(\by, f)$, and $\Theta_n$ replaced by $\mathcal{F}_n$, we can find  $\hat{\theta}_n: \Omega \to \Theta_n$, such that for each $\omega \in \Omega$, 
\[
\mathbb{Q}_n(\omega, \hat{\theta}_n(\omega)) = \inf_{\theta\in \Theta_n}\mathbb{Q}_n(\omega, \theta).
\]
That is, for any combinations of the covariates $\bX$, and spatial random effect $\bm{\epsilon}$, there exists $\hat{f}_n\in \mathcal{F}_n$ such that $\mathcal{L}_n(\hat{f}_n) = \min_{f\in \mathcal{F}_n}\mathcal{L}_n(f)$, which is \blue{the NN-GLS estimator}.
\end{proof}

\subsection{Proof of consistency}\label{sec-prf}

\begin{lemma}\label{lemma-wlln}
Under Assumptions \ref{Asmp-1} and \ref{Asmp-3}, \blue{we have $\EE[\bm{\epsilon}^\top\Qb\bm{\epsilon}]=O(n)$ and} 
\begin{equation*}
    \Big|\frac{1}{n}\big(\bm{\epsilon}^\top\Qb\bm{\epsilon} - \EE[\bm{\epsilon}^\top\Qb\bm{\epsilon}]\big)\Big|\ \xrightarrow{p}\  0
\end{equation*}
\end{lemma}
\begin{proof}[Proof of lemma \ref{lemma-wlln}]
We note that the following results hold:
\[
\EE[\bm{\epsilon}^\top\Qb\bm{\epsilon}] = \EE[\bm{\epsilon}^\top\blue{\mathbf{\Sigma}^{-\frac{\top}{2}}\Eb\mathbf{\Sigma}^{-\frac{1}{2}}}\bm{\epsilon}] = \blue{\text{tr}(\Eb)}, \ \  \blue{\text{var}(\bm{\epsilon}^\top\Qb\bm{\epsilon}) = 2\text{tr}(\Eb^2)},
\]
as $\mathbf{\Sigma}$ is the covariance matrix of the Gaussian error $\bm{\epsilon} = (\epsilon_1, \epsilon_2, \cdots, \epsilon_n)^\top$. We show that both terms have a rate $O(n)$ 
by the following calculation:
\[
\begin{split}
    \EE[\bm{\epsilon}^\top\Qb\bm{\epsilon}] = \blue{\text{tr}(\Eb)} &= \sum_{i = 1}^n\lambda_i(\Eb) \leq n\|\Eb\|_2,
\end{split}
\]
where $\{\lambda_i(\Eb), i = 1, \cdots, n\}$ are the eigenvalues of $\Eb$. By Assumption \ref{Asmp-3}, $\|\Eb\|_2$ is uniformly bounded, which means that $\frac{1}{n}\EE[\bm{\epsilon}^\top\Qb\bm{\epsilon}]$ is uniformly bounded by $\Lambda_{high}$. This intermediate result will be used in the proof of Lemma \ref{lemma-main-1}. To further bound the variance term, we have
\[
\begin{split}
\text{var}(\bm{\epsilon}^\top\blue{\mathbf{\Sigma}^{-\frac{\top}{2}}\Eb\mathbf{\Sigma}^{-\frac{1}{2}}}\bm{\epsilon}) = \blue{2\text{tr}(\Eb^2)} &= 2\text{tr}(\Eb^2) \leq 2n\|\Eb^2\|_2 \leq 2n\Lambda^2_{high}.
\end{split}
\]
So $\text{var}(\bm{\epsilon}^\top\Qb\bm{\epsilon}) = O(n)$. By Markov's inequality, we have 
\begin{align*}
\PP\bigg(\Big|\frac{1}{n}\big(\bm{\epsilon}^\top\Qb\bm{\epsilon} - \EE[\bm{\epsilon}^\top\Qb\bm{\epsilon}]\big)\Big|>t\bigg) 
&\leq \frac{\text{var}(\bm{\epsilon}^\top\Qb\bm{\epsilon})}{n^2t^2} = o(1),
\end{align*}
which proves the statement of the lemma.

\end{proof}
\begin{lemma}\label{lemma-main-1}
Under Assumptions same as Theorem \ref{Thm-main-1}, 
\begin{equation}
    \sup_{f\in\mathcal{F}_n}\big|\mathcal{L}_n(f) - L_n(f)\big|\ \xrightarrow{p}\ 0
\end{equation}
\end{lemma}
\begin{proof}[Proof of Lemma \ref{lemma-main-1}]
By definitions of $\mathcal{L}_n(f)$ and $L_n(f)$,
\begin{equation}\label{lemma-main-eq-1}
\begin{split}
      \sup_{f\in\mathcal{F}_n}\big|\mathcal{L}_n(f) - L_n(f)\big| &=   \sup_{f\in\mathcal{F}_n}\Big|\frac{1}{n}\big(\bm{\epsilon}^\top\blue{\mathbf{\Sigma}^{-\frac{\top}{2}}\Eb\mathbf{\Sigma}^{-\frac{1}{2}}}\bm{\epsilon} - \EE[\bm{\epsilon}^\top\blue{\mathbf{\Sigma}^{-\frac{\top}{2}}\Eb\mathbf{\Sigma}^{-\frac{1}{2}}}\bm{\epsilon}]\big) - 2\frac{1}{n}\bm{\epsilon}^\top\blue{\mathbf{\Sigma}^{-\frac{\top}{2}}\Eb\mathbf{\Sigma}^{-\frac{1}{2}}}(\bm{f} - \bm{f}_0)\Big| \\
      &\leq \Big|\frac{1}{n}\big(\bm{\epsilon}^\top\blue{\mathbf{\Sigma}^{-\frac{\top}{2}}\Eb\mathbf{\Sigma}^{-\frac{1}{2}}}\bm{\epsilon} - \EE[\bm{\epsilon}^\top\blue{\mathbf{\Sigma}^{-\frac{\top}{2}}\Eb\mathbf{\Sigma}^{-\frac{1}{2}}}\bm{\epsilon}]\big)\Big| + 2\sup_{f\in\mathcal{F}_n}\Big|\frac{1}{n}\bm{\epsilon}^\top\blue{\mathbf{\Sigma}^{-\frac{\top}{2}}\Eb\mathbf{\Sigma}^{-\frac{1}{2}}}(\bm{f} - \bm{f}_0)\Big|\\
      &:= \text{term}1 + \text{term}2.
\end{split}
\end{equation}
By Lemma \ref{lemma-wlln}, term1 converges to $0$ in probability, it suffices to show that 
\begin{equation}\label{lemma-main-eq-2}
\sup_{f\in\mathcal{F}_n}\Big|\frac{1}{n}\bm{\epsilon}^\top\blue{\mathbf{\Sigma}^{-\frac{\top}{2}}\Eb\mathbf{\Sigma}^{-\frac{1}{2}}}(\bm{f} - \bm{f}_0)\Big| \ \xrightarrow{p}\ 0.
\end{equation}
Here, we prove a stronger result:
\begin{equation}\label{lemma-main-eq-3}
    \EE\bigg[\sup_{f\in\mathcal{F}_n}\Big|\frac{1}{n}\bm{\epsilon}^\top\blue{\mathbf{\Sigma}^{-\frac{\top}{2}}\Eb\mathbf{\Sigma}^{-\frac{1}{2}}}(\bm{f} - \bm{f}_0)\Big|\bigg] \longrightarrow 0,
\end{equation}
which makes (\ref{lemma-main-eq-3}) hold by applying Markov's inequality. 
Given a fixed $n$, define a random process $X_f$ indexed by $f\in \mathcal{F}_n$ as
\begin{equation}\label{Thm-main-eq-rp-def}
    X_f = \frac{1}{\sqrt{n}}\bm{\epsilon}^\top\blue{\mathbf{\Sigma}^{-\frac{\top}{2}}\Eb\mathbf{\Sigma}^{-\frac{1}{2}}}(\bm{f} - \bm{f}_0).
\end{equation}
We note that $X_f$ is a separable process (see \cite{sen2018gentle} Definition 4.4). This can be easily checked by choosing functions in $\mathcal{F}_n$ with rational coefficients as the countable dense subset of $\mathcal{F}_n$, \blue{as both $\Eb$ and $\bSigma^{-1}$ have bounded spectral norms by Assumption \ref{Asmp-3}.} By introducing a metric $d(f,g) = \|f - g\|_{\infty}$ to $\mathcal{F}_n$, 
a metric space is created, where the randomness from process $X_f$ can be controlled by the covering number of the space. In detail, for any $f^*_n \in \mathcal{F}_n$, the right-hand side of this asymptotic formula can be decomposed as:
\begin{equation}\label{lemma-main-eq-4}
\begin{split}
    \EE\bigg[\sup_{f\in\mathcal{F}_n}\Big|\frac{1}{n}\bm{\epsilon}^\top\blue{\mathbf{\Sigma}^{-\frac{\top}{2}}\Eb\mathbf{\Sigma}^{-\frac{1}{2}}}(\bm{f} - \bm{f}_0)\Big|\bigg] &= \EE\bigg[\frac{1}{\sqrt{n}}\sup_{f\in\mathcal{F}_n}\Big|\frac{1}{\sqrt{n}}\bm{\epsilon}^\top\blue{\mathbf{\Sigma}^{-\frac{\top}{2}}\Eb\mathbf{\Sigma}^{-\frac{1}{2}}}(\bm{f} - \bm{f}_0)\Big|\bigg]\\
    &\leq \EE\bigg[\Big|\frac{1}{n}\bm{\epsilon}^\top\blue{\mathbf{\Sigma}^{-\frac{\top}{2}}\Eb\mathbf{\Sigma}^{-\frac{1}{2}}}(\bm{f}^*_n - \bm{f}_0)\Big|\bigg] \\
    &+ \EE\bigg[\frac{1}{\sqrt{n}}\sup_{f\in\mathcal{F}_n}\Big|\frac{1}{\sqrt{n}}\bm{\epsilon}^\top\blue{\mathbf{\Sigma}^{-\frac{\top}{2}}\Eb\mathbf{\Sigma}^{-\frac{1}{2}}}(\bm{f} - \bm{f}^*_n)\Big|\bigg]\\
    &= \EE\bigg[\Big|\frac{1}{n}\bm{\epsilon}^\top\blue{\mathbf{\Sigma}^{-\frac{\top}{2}}\Eb\mathbf{\Sigma}^{-\frac{1}{2}}}(\bm{f}^*_n - \bm{f}_0)\Big|\bigg] \\
    &+
    \frac{1}{\sqrt{n}}\EE\Big[\sup_{f\in\mathcal{F}_n}\big|X_f - X_{f^*}\big|\Big]
\end{split}
\end{equation}
For the first term, 
\[
\begin{split}
\EE\bigg[\Big|\frac{1}{n}\bm{\epsilon}^\top\blue{\mathbf{\Sigma}^{-\frac{\top}{2}}\Eb\mathbf{\Sigma}^{-\frac{1}{2}}}(\bm{f}^*_n - \bm{f}_0)\Big|\bigg] &\leq
\frac{1}{n}\EE\Big[\big(\bm{\epsilon}^\top\blue{\mathbf{\Sigma}^{-\frac{\top}{2}}\Eb\mathbf{\Sigma}^{-\frac{1}{2}}}\bm{\epsilon}\big)^{\frac{1}{2}}\Big]\big((\bm{f}^*_n - \bm{f}_0)^\top\blue{\mathbf{\Sigma}^{-\frac{\top}{2}}\Eb\mathbf{\Sigma}^{-\frac{1}{2}}}(\bm{f}^*_n - \bm{f}_0)\big)^{\frac{1}{2}}\\
&\leq \left(\frac{1}{n}\EE\big[\bm{\epsilon}^\top\blue{\mathbf{\Sigma}^{-\frac{\top}{2}}\Eb\mathbf{\Sigma}^{-\frac{1}{2}}}\bm{\epsilon}\big]\right)^{\frac{1}{2}}\left(\frac{\|\Eb\|_2\|\mathbf{\Sigma}^{-1}\|_2\|f^*_n - f_0\|_n^2}{n}\right)^{\frac{1}{2}}\\
&\leq \left(\frac{1}{n}\EE\big[\bm{\epsilon}^\top\blue{\mathbf{\Sigma}^{-\frac{\top}{2}}\Eb\mathbf{\Sigma}^{-\frac{1}{2}}}\bm{\epsilon}\big]\right)^{\frac{1}{2}}\left(\frac{\Lambda_{high}}{\Psi_{low}}\right)\|f^*_n - f_0\|_{\infty}.
\end{split}
\]
\blue{From Lemma \ref{lemma-wlln}, we know that} $\frac{1}{n}\EE\big[\bm{\epsilon}^\top\blue{\mathbf{\Sigma}^{-\frac{\top}{2}}\Eb\mathbf{\Sigma}^{-\frac{1}{2}}}\bm{\epsilon}\big]$ is uniformly upper-bounded 
, and so is the term $\frac{\Lambda_{high}}{\Psi_{low}}$ by Assumptions \ref{Asmp-1} and \ref{Asmp-3}. According to the universal approximation theorem introduced in \cite{hornik1989multilayer}, 
\[
\sup_{x\in\mathcal{X}}|\pi_{r_n}f_0(x) - f_0(x)| \ \to \ 0 \text{ as } n\ \to\ 0,
\]
where $\pi_{r_n}f_0$ is the projection of $f_0$ onto space $\mathcal{F}_n$. By choosing $f^*_n = \pi_{r_n}f_0$, we can guarantee that
\[
\EE\bigg[\Big|\frac{1}{n}\bm{\epsilon}^\top\blue{\mathbf{\Sigma}^{-\frac{\top}{2}}\Eb\mathbf{\Sigma}^{-\frac{1}{2}}}(\bm{f}^*_n - \bm{f}_0)\Big|\bigg] \ \longrightarrow \ 0.
\]
Now it suffices to prove the second term on the right-hand side of (\ref{lemma-main-eq-4}) vanishes as $n$ goes to infinity. By Proposition \ref{prop-thm-main-1} \blue{(which is stated and proved after this lemma)}, on semi metric space \blue{$(\mathcal{F}_n, \|\cdot\|_n)$}, there exists a constant \blue{$C_{\Lambda} = 2(\sqrt{2}+1)^{\frac{1}{2}}\Lambda_{high}\Psi^{-\frac{1}{2}}_{low}$} such that:
\[
\|X_f - X_g\|_{\psi_2} \leq C_{\Lambda}\cdot d(f,g),
\]
where $\psi_2(x) = e^{x^2}-1$ and $\|\cdot\|_{\psi_2}$ refers to the Orlicz norm defined in definition \ref{def-Orlicz} \blue{after this proof.} Then, according to Theorem 2.2.4 and Corollary 2.2.8 in \cite{van1996weak} 
, there exists a constant \blue{$K := K(\psi_2)$} (it's replaced by $K$ in the following part since $\psi$ is uniformly constant),  such that 
\begin{equation}\label{lemma-main-eq-CN}
\begin{split}
\frac{1}{\sqrt{n}}\EE\Big[\sup_{f\in\mathcal{F}_n}\big|X_f - X_{f^*}\big|\Big] &\leq \frac{1}{\sqrt{n}}K\blue{C_{\Lambda}}\int_{0}^{\infty}\sqrt{\log N\Big(\frac{\eta}{2}, \mathcal{F}_n, \|\cdot\|_{n}\Big)} d\eta.\\
&\blue{\leq \frac{1}{\sqrt{n}}KC_{\Lambda}\int_{0}^{\infty}\sqrt{\log N\Big(\frac{\eta}{2}, \mathcal{F}_n, \|\cdot\|_{\infty}\Big)} d\eta},\\
& = K\blue{C_{\Lambda}}\int_{0}^{2V_n}\sqrt{\frac{\log N\Big(\frac{\eta}{2}, \mathcal{F}_n, \|\cdot\|_{\infty}\Big)}{n}} d\eta
\end{split}
\end{equation}
Here $N$ refers to the covering number, and $N\Big(\frac{\eta}{2}, \mathcal{F}_n, \|\cdot\|_{\infty}\Big)$ is defined as the minimum number of balls of radius $\frac{\eta}{2}$ to cover $(\mathcal{F}_n, \|\cdot\|_{\infty})$. The second inequality holds because $\|f\|_n \leq \|f\|_{\infty}$. The \blue{third} equality holds due to Definition (\ref{def-F_n}). Literature exists in bounding the covering number of such well-restricted function space and here we use the bound from Theorem 14.5 in \cite{anthony2009neural}, i.e.,
\begin{equation}\label{lemma-main-eq-coveringnum}
\begin{split}
    N\left(\frac{\eta}{2}, \mathcal{F}_n, \|\cdot\|_{\infty}\right) &\leq \left(\frac{8er_b\big[r_n(d+2)+1\big]L_{lip}^2V_n^2}{\eta(L_{lip}V_n-1)}\right)^{r_n(d+2)+1}
\end{split}
\end{equation}
\blue{where $r_b$ and $L_{lip}$ are universal constants defined below Assumption \ref{Asmp-2}} 
 and
\[
\begin{split}
\log N\Big(\frac{\eta}{2}, \mathcal{F}_n, \|\cdot\|_{\infty}\Big) &\leq \big(r_n(d+2)+1\big)\bigg(\log\Big[8r_b\big(r_n(d+2)+1\big)\blue{\frac{L_{lip}^2V_n^2}{L_{lip}V_n-1}}\Big]+1+\log\frac{1}{\eta}\bigg)\\
&\leq \big(r_n(d+2)+1\big)\bigg(\log\Big[8r_b\big(r_n(d+2)+1\big)\blue{\frac{L_{lip}^2V_n^2}{L_{lip}V_n-1}}\Big]+\frac{1}{\eta}\bigg)\\
&\leq \big(r_n(d+2)+1\big)\bigg(\log\Big[8r_b\big(r_n(d+2)+1\big)\blue{\frac{L_{lip}^2V_n^2}{L_{lip}V_n-1}}\Big]\bigg)\bigg(1+\frac{1}{\eta}\bigg),\\
&:= B(d, r_n, V_n)\bigg(1+\frac{1}{\eta}\bigg),
\end{split}
\]
the third inequality is due to the fact that $r_n, V_n \to \infty$ as $n\to\infty$. With this result, we are able to bound the integral in (\ref{lemma-main-eq-CN}).
\[
\begin{split}
K\blue{C_{\Lambda}}\int_{0}^{2V_n}\sqrt{\frac{\log N\Big(\frac{\eta}{2}, \mathcal{F}_n, \|\cdot\|_{\infty}\Big)}{n}} d\eta &\leq K\blue{C_{\Lambda}} \frac{1}{\sqrt{n}}B^{\frac{1}{2}}(d, r_n, V_n)\int_0^{2V_n}\bigg(1+\frac{1}{\eta}\bigg)^{\frac{1}{2}}d\eta\\
&= K\blue{C_{\Lambda}}\frac{1}{\sqrt{n}}B^{\frac{1}{2}}(d, r_n, V_n)\Bigg[\int_0^{1}\bigg(1+\frac{1}{\eta}\bigg)^{\frac{1}{2}}d\eta + \int_1^{2V_n}\bigg(1+\frac{1}{\eta}\bigg)^{\frac{1}{2}}d\eta\Bigg]\\
&\leq K\blue{C_{\Lambda}}\frac{1}{\sqrt{n}}B^{\frac{1}{2}}(d, r_n, V_n)\bigg[\sqrt{2}\int_0^1\eta^{-\frac{1}{2}}d\eta +\sqrt{2}(2V_n-1)\bigg]\\
&\leq 3K\blue{C_{\Lambda}}\sqrt{\frac{2}{n}}B^{\frac{1}{2}}(d, r_n, V_n)V_n.
\end{split}
\]
This implies that (\ref{lemma-main-eq-CN}) can be bounded with rate
\[
\begin{split}
\frac{1}{\sqrt{n}}\EE\Big[\sup_{f\in\mathcal{F}_n}\big|X_f - X_{f^*}\big|\Big] &\leq K\blue{C_{\Lambda}}\int_{0}^{2V_n}\sqrt{\frac{\log N\Big(\frac{\eta}{2}, \mathcal{F}_n, \|\cdot\|_{\infty}\Big)}{n}} d\eta \\
&\leq 3K\blue{C_{\Lambda}}\sqrt{\frac{2\big(r_n(d+2)+1\big)V_n^2\Big(\log8r_b\big(r_n(d+2)+1\big)\blue{\frac{L_{lip}^2V_n^2}{L_{lip}V_n-1}}\Big)}{n}} \\
&\sim \sqrt{\frac{\big(r_n(d+2)+1\big)V_n^2\log\Big( 8r_b\big(r_n(d+2)+1\big)\blue{V_n}\Big)}{n}}
\end{split}
\]
\blue{as $V_n \to \infty$ and $L_{lip}$ is fixed.} 
By the Assumption \ref{Asmp-2} on $r_n$ and $V_n$'s increasing rate, 
\[
\frac{1}{\sqrt{n}}\EE\Big[\sup_{f\in\mathcal{F}_n}\big|X_f - X_{f^*}\big|\Big] \rightarrow 0,
\]
which brings us back to (\ref{lemma-main-eq-4}), (\ref{lemma-main-eq-3}). Then we have
\[
\sup_{f\in\mathcal{F}_n}\Big|\frac{1}{n}\bm{\epsilon}^\top\blue{\mathbf{\Sigma}^{-\frac{\top}{2}}\Eb\mathbf{\Sigma}^{-\frac{1}{2}}}(\bm{f} - \bm{f}_0)\Big| \ \xrightarrow{p}\ 0
\]
by Markov's inequality. Returning to (\ref{lemma-main-eq-1}) completes the proof.
\end{proof}

In order to use tools from the empirical process, we introduce the Orlicz norm:
\begin{definition}[Orlicz norm]\label{def-Orlicz}
Suppose $X$ is a random variable and $\psi$ is a non-decreasing, convex function with $\psi(0) = 0$. Then, the Orlicz norm $\|X\|_{\psi}$ is defined as 
\[
\|X\|_{\psi} = \inf\left\{C>0: \EE\ \psi\left(\frac{|X|}{C}\right) \leq 1\right\}.
\]
In addition, some of Orlicz norms are of special interest, where $\psi_p(x) = e^{x^p} - 1$, and the corresponding Orlicz norms are denoted as $\|X\|_{\psi_p}$. 
\end{definition}
\begin{proposition}\label{prop-thm-main-1}
Let the random process $X_f$ be defined as (\ref{Thm-main-eq-rp-def}) in the proof of Lemma \ref{lemma-main-1}. In the semi-metric space $(\mathcal{F}_n, d)$ where \blue{$d(f,g) = \|f-g\|_{n}$}, 
\[
\|X_f - X_g\|_{\psi_2} \leq C_{\Lambda}\cdot d(f,g) = \blue{C_{\Lambda}\cdot\|f-g\|_{n}},
\]
where \blue{$C_{\Lambda} = 2(\sqrt{2}+1)^{\frac{1}{2}}\frac{\Lambda_{high}}{\sqrt{\Psi_{low}}}$} is a constant independent of $n$.
\end{proposition}

\begin{proof}[Proof of Proposition \ref{prop-thm-main-1}]
By Assumption \ref{Asmp-1},
\[
X_f - X_g = \frac{1}{\sqrt{n}}\bm{\epsilon}^\top\blue{\mathbf{\Sigma}^{-\frac{\top}{2}}\Eb\mathbf{\Sigma}^{-\frac{1}{2}}}(\bm{f} - \bm{g}) \sim \mathcal{N}(0, \frac{1}{n}(\bm{f} - \bm{g})^\top\blue{\mathbf{\Sigma}^{-\frac{\top}{2}}\Eb^\top\Eb\mathbf{\Sigma}^{-\frac{1}{2}}}(\bm{f} - \bm{g})) := \mathcal{N}(0,S_n).
\]

The tail bound of a normal random variable can be obtained by Mill's inequality, i.e., 
\[
\Pb\big(|Z| > t\big) \leq \sqrt{\frac{2}{\pi}}\frac{e^{-\frac{t^2}{2}}}{t},
\]
where $Z \sim \mathcal{N}(0,1)$. For $X_f - X_g$, 
\[
\Pb\big(|X_f - X_g| > t\big) \leq \sqrt{\frac{2S_n}{\pi}}\frac{\exp(-\frac{t^2}{2S_n})}{t}.
\]
We then focus on calculating the $\psi_2$ norm. In general, for a random variable $X$, if 
\[
\Pb\big(|X| > t\big) \leq \kappa \frac{\exp(-Ht^2)}{t},
\]
then for any $D\in (0, H]$,
\[
\begin{split}
    \EE\left[e^{D|X|^2} - 1\right] & = \EE\int_0^{|X|^2}De^{Ds}ds\\
    &= \int_0^\infty \Pb\big(|X| > \sqrt{s}\big)De^{Ds}ds\\
    &\leq \int_0^{\infty}\kappa D\cdot s^{-1/2}e^{(D-H)s}ds\\
    &= 2\kappa D\int_0^{\infty}e^{-(H-D)x^2}dx\\
    &= \frac{\sqrt{\pi}\kappa D}{\sqrt{H-D}}
\end{split}
\]

Let $D_0$ be the solution to $\frac{\sqrt{\pi}\kappa D}{\sqrt{H-D}} = 1$. 
By definition \ref{def-Orlicz}, 
\[
\|X\|_{\psi_2} := \inf\left\{C>0: \EE\psi\left(\frac{|X|}{C}\right) \leq 1\right\} \leq D_0^{-\frac{1}{2}}
\]
After solving a quadratic equation, and plugging in $\kappa  = \sqrt{\frac{2S_n}{\pi}}, H = \frac{1}{2S_n}$ for our case on $X_f - X_g$, we get
\[
D_0 = \frac{-1+\sqrt{1+\pi \kappa ^2H}}{2\pi \kappa ^2}  = \frac{\sqrt{2}-1}{4S_n},
\]
and 
\[ 
\begin{split}
\|X_f-X_g\|_{\psi_2} &\leq D_0^{-\frac{1}{2}} = 2(\sqrt{2}+1)^{\frac{1}{2}} S_n^{\frac{1}{2}}\\
&= 2(\sqrt{2}+1)^{\frac{1}{2}}\left(\frac{1}{n}(\bm{f} - \bm{g})^\top\mathbf{\Sigma}^{-\frac{\top}{2}}\Eb^\top\Eb\mathbf{\Sigma}^{-\frac{1}{2}}(\bm{f} - \bm{g})\right)^{\frac{1}{2}}\\
&\leq 2(\sqrt{2}+1)^{\frac{1}{2}}\left(\|f-g\|_n^2\|\Eb\|_2^2\|\mathbf{\Sigma}^{-1}\|_2\right)^{\frac{1}{2}}\\
&\leq \left(\frac{4(\sqrt{2}+1)}{\Psi_{low}}\right)^{\frac{1}{2}}\Lambda_{high}\blue{\|f-g\|_{n}}.
\end{split}
\]
Assumptions \ref{Asmp-1} and \ref{Asmp-3} provide us the eigenbounds, and the proof is completed.

\end{proof}

\begin{proof}[Proof of Theorem \ref{Thm-main-1}]
Recall the definition of $L_n$ (\ref{def-L_n}):
\[
 L_n(f) = \EE\big[\mathcal{L}_n(f)\big] = \EE\left[\frac{1}{n}(\by - \bm{f}(x))^\top\blue{\mathbf{\Sigma}^{-\frac{\top}{2}}\Eb\mathbf{\Sigma}^{-\frac{1}{2}}}(\by - \bm{f}(x))\right].
\]
By Assumption \ref{Asmp-3}, 
\[
\begin{split}
L_n(f) - L_n(f_0) &= \frac{1}{n}(\bm{f}_0(x) - \bm{f}(x))^\top\blue{\mathbf{\Sigma}^{-\frac{\top}{2}}\Eb\mathbf{\Sigma}^{-\frac{1}{2}}}(\bm{f}_0(x) - \bm{f}(x))\\ &\geq \frac{1}{n}\lambda_{\min}(\blue{\mathbf{\Sigma}^{-\frac{\top}{2}}\Eb\mathbf{\Sigma}^{-\frac{1}{2}}})\|(\bm{f}_0(x) - \bm{f}(x))\|_2^2\\ 
&\geq \lambda_{\min}(\Eb)\lambda_{\min}(\mathbf{\Sigma}^{-1})\|f-f_0\|_n^2  = \frac{\Lambda_{low}}{\|\mathbf{\Sigma}\|_2}\|f-f_0\|_n^2 
\end{split}
\]
The inequality implies that for any $\epsilon > 0$,
\[
\inf_{f:\|f - f_0\|_n \geq \epsilon}L_n(f) - L_n(f_0) \geq \frac{\Lambda_{low}}{\|\mathbf{\Sigma}\|_2}\epsilon^2 >0,
\]
with the probabilistic statement:
\begin{equation}\label{Thm-main-eq-2}
\PP\left(\|\hat{\bm{f}}_n - \bm{f}_0\| \geq \epsilon\right) \leq \PP\left(L_n(\hat{f}_n) - L_n(f_0) \geq \frac{\Lambda_{low}}{\|\mathbf{\Sigma}\|_2}\epsilon^2\right).
\end{equation}
Now, in order to prove the theorem, what we need is $\PP\left(L_n(\hat{f}_n) - L_n(f_0) \geq \frac{\Lambda_{low}}{\|\mathbf{\Sigma}\|_2}\epsilon^2\right) \to 0$. First,
\begin{equation}\label{Thm-main-eq-3}
\begin{split}
\PP\left(L_n(\hat{f}_n) - L_n(f_0) \geq \frac{\Lambda_{low}}{\|\mathbf{\Sigma}\|_2}\epsilon^2\right) &\leq \PP\left(L_n(\hat{f}_n) - \mathcal{L}_n(\hat{f}_n) \geq \frac{\Lambda_{low}}{2\|\mathbf{\Sigma}\|_2}\epsilon^2\right)\\ 
&+ \PP\left(\mathcal{L}_n(\hat{f}_n) - L_n(f_0) \geq \frac{\Lambda_{low}}{2\|\mathbf{\Sigma}\|_2}\epsilon^2\right).
\end{split}
\end{equation}
Since $\hat{f}_n = \argmin_{f\in \mathcal{F}_n}\mathcal{L}_n(f) \in \mathcal F_n$, the first term $\PP\left(\big|L_n(\hat{f}_n) - \mathcal{L}_n(\hat{f}_n)\big| \geq \frac{\Lambda_{low}}{2\|\mathbf{\Sigma}\|_2}\epsilon^2\right) \to 0$ according to Lemma \ref{lemma-main-1}. The second term can be further decomposed as:
\begin{equation}\label{Thm-main-eq-4} 
\begin{split}
\PP\left(\mathcal{L}_n(\hat{f}_n) - L_n(f_0)\geq \frac{\Lambda_{low}}{2\|\mathbf{\Sigma}\|_2}\epsilon^2\right) &\leq \PP\left(\mathcal{L}_n(\hat{f}_n) - \mathcal{L}_n(\pi_{r_n}f_0)\geq 0\right)\\ 
&+ \PP\left(\mathcal{L}_n(\pi_{r_n}f_0) - L_n(\pi_{r_n}f_0)\geq \frac{\Lambda_{low}}{4\|\mathbf{\Sigma}\|_2}\epsilon^2\right) \\
&+ \PP\left(L_n(\pi_{r_n}f_0) - L_n(f_0)\geq \frac{\Lambda_{low}}{4\|\mathbf{\Sigma}\|_2}\epsilon^2\right),  
\end{split}
\end{equation}
where $\pi_{r_n}f_0$ is the projection of $f_0$ onto $\mathcal{F}_n$ with respect to the metric $\|\cdot\|_n$. Since $\hat{f}_n$ achieves the minimum, $\PP\left(\mathcal{L}_n(\hat{f}_n) - \mathcal{L}_n(\pi_{r_n}f_0)\geq 0\right) = 0$, and $$\PP\left(\mathcal{L}_n(\pi_{r_n}f_0) - L_n(\pi_{r_n}f_0)\geq \frac{\Lambda_{low}}{4\|\mathbf{\Sigma}\|_2}\epsilon^2\right) \to 0$$ also by Lemma \ref{lemma-main-1}. For the last term in (\ref{Thm-main-eq-4}), 
\[
\begin{split}
L_n(f) - L_n(f_0) &= \frac{1}{n}(\bm{f}_0(x) - \bm{f}(x))^\top\blue{\mathbf{\Sigma}^{-\frac{\top}{2}}\Eb\mathbf{\Sigma}^{-\frac{1}{2}}}(\bm{f}_0(x) - \bm{f}(x))\\
&\leq \frac{1}{n}\lambda_{\max}(\blue{\mathbf{\Sigma}^{-\frac{\top}{2}}\Eb\mathbf{\Sigma}^{-\frac{1}{2}}})\|(\bm{f}_0(x) - \bm{f}(x))\|_2^2\\
&\leq \|\Eb\|_2\|\mathbf{\Sigma}^{-1}\|_2\|f - f_0\|_{\infty}^2 \leq \frac{\Lambda_{high}}{\Psi_{low}}\|f - f_0\|_{\infty}^2.
\end{split}
\]
By the universal approximation theorem (theorem 2.1) introduced
by \cite{hornik1989multilayer}, $\|\pi_{r_n}f_0 - f_0\|_{\infty}$ goes to $0$. 
With all these arguments plugged into (\ref{Thm-main-eq-2}) - (\ref{Thm-main-eq-4}), $$\PP\left(\mathcal{L}_n(\hat{f}_n) - L_n(f_0)\geq \frac{\Lambda_{low}}{2\|\mathbf{\Sigma}\|_2}\epsilon^2\right) \to 0, \quad \PP\left(L_n(\hat{f}_n) - L_n(f_0) \geq \frac{\Lambda_{low}}{\|\mathbf{\Sigma}\|_2}\epsilon^2\right) \to 0,$$ and finally 
\[
\PP\left(\|\hat{f}_n - f_0\|_n \geq \epsilon\right) \to 0.
\]
The proof is completed.
\end{proof}

\subsection{Proof of Propositions \ref{prop-main-2} and \ref{prop-main-1}}
\begin{proof}[Proof of Proposition \ref{prop-main-2}]
\blue{We first provide the expression of the Mat\'ern covariance function. 
\begin{definition}[Mat\'{e}rn process]\label{def-Matern}
The Mat\'{e}rn class of covariance functions with parameters $\btheta = (\sigma^2, \phi, \nu)$ is given by
\begin{equation}\label{eq-Matern}
C(s_i, s_j) = C(\|s_i - s_j\|_2) = \frac{2^{1-\nu}\sigma^2(\sqrt{2}\phi\|s_i - s_j\|)^{\nu}}{\Gamma(\nu)}\mathcal{K}_{\nu}(\sqrt{2}\phi\|s_i - s_j\|),
\end{equation}
and $\mathcal{K}_{\nu}(\cdot)$ is the modified Bessel function of second kind. Note that for simplicity, we sometimes describe a Mat\'{e}rn covariance with parameters $\btheta = (\sigma^2, \phi, \nu, \tau^2)$, where the nuggets $\tau^2$ is additionally involved, and the covariance becomes $C(h | \sigma^2, \phi, \nu) + \tau^2 I(h=0)$.
\end{definition}}

We only have to prove that $\mathbf{\Sigma}$ satisfies Assumption \ref{Asmp-1} and \blue{$\Eb = \mathbf{\Sigma}^{\frac{\top}{2}}\Qb\mathbf{\Sigma}^{\frac{1}{2}}$ } satisfies Assumption \ref{Asmp-3}. \blue{The latter is immediate for Proposition \ref{prop-main-2} as $\Eb$ becomes the identity matrix on account of using the true covariance matrix $\bSigma$ as the working covariance matrix $\Qb^{-1}$. Hence, only need to prove that the eigenvalues of $\bSigma$ are uniformly bounded (in $n$) from above and below. 

As $\bSigma=\Cb + \tau^2 \Ib$ where $\Cb=(C(\|s_i-s_j\|_2))$, if the nugget variance $\tau^2 >0$, then eigenvalues of $\bSigma$ are uniformly bounded from below by $\tau^2$. Even if there is no nugget, i.e., $\tau^2=0$, when the covariance is Mat\'{e}rn, a uniform lower bound for the eigenvalues can be obtained using the minimum distance condition and applying the results of \cite{wendland2004scattered}. We provide this lower bound formally in Lemma \ref{lemma-min-eigen} below. The uniform upper bound is provided in Lemma \ref{lem:upper}.} Together these prove Assumption \ref{Asmp-1}, implying consistency of NN-GLS for Mat\'ern GP with a non-linear mean.
\end{proof}

\blue{The following two lemmas provide uniform (in the number of locations) lower and upper bounds for the eigen-values of Mat\'ern covariance matrix when the locations are separated by a minimum distance. These type of results have been commonly used previously in the literature in different forms. We provide the references but also add proofs for our specific setting here for the ease of reading.}

\blue{
\begin{lemma}[Uniform lower bound for eigenvalues of Mat\'{e}rn covariance matrix, \cite{wendland2004scattered}]\label{lemma-min-eigen}
Assume that $\Cb$ is a covariance matrix derived from a Mat\'{e}rn GP with parameters $\mathbf{\theta}_0 = (\sigma^2_0, \phi_0, \nu_0)$, sampled at locations $s_1, \cdots, s_n \in \mathbb{R}^2$,  The locations are separated by a minimum distance $h > 0$. Then, for any $n$, the minimum eigenvalue of $\Cb$ has bound:
\begin{equation}\label{eq:maternlow}
\lambda_{\min}(\Cb) \geq \Psi_{low}= \frac{\pi^2\sigma_0^2\Gamma(\nu_0+1)(\phi_0^2)^{\nu_0}}{\Gamma(\nu_0)(\phi_0^2h^2 + 128\pi^3)^{\nu_0+1}}h^{2\nu_0}.
\end{equation}
\end{lemma}
\begin{proof}[Proof of Lemma \ref{lemma-min-eigen}]
For a Mat\'{e}rn process on $\mathbb{R}^2$ with parameters $\mathbf{\theta}_0 = (\sigma_0^2, \phi_0, \nu_0)$, its Fourier's transform can be expressed as (formula (4.15) in \cite{williams2006gaussian}):
\[
\mathbf{\Phi}(f) = \frac{4\pi\sigma_0^2\Gamma(\nu_0+1)(\phi_0^2)^{\nu_0}}{2\Gamma(\nu_0)}\left(\phi_0^2 + 2\pi^2\|f\|^2\right)^{-(\nu_0+1)}.
\]
Obviously, $$\inf_{\|f\|_2 \leq 2M} \mathbf{\Phi}(f) = \mathbf{\Phi}(2M) = \frac{2\pi\sigma_0^2\Gamma(\nu_0+1)(\phi_0^2)^{\nu_0}}{\Gamma(\nu_0)}\left(\phi_0^2 + 8\pi^2M^2\right)^{-(\nu_0+1)}.$$ Then using Theorem 12.3 in \cite{wendland2004scattered}, we know that 
\[
\lambda_{\min}(\Cb) \geq  \frac{M^2\mathbf{\Phi}(2M)}{16} = \frac{\pi\sigma_0^2\Gamma(\nu_0+1)(\phi_0^2)^{\nu_0} M^2}{8\Gamma(\nu_0)}\left(\phi_0^2 + 8\pi^2M^2\right)^{-(\nu_0+1)}
\]
for any $M \geq \frac{4\sqrt{\pi}}{h}$, i.e.
$
\lambda_{\min}(\Cb) \geq \frac{\pi^2\sigma_0^2\Gamma(\nu_0+1)(\phi_0^2)^{\nu_0}}{\Gamma(\nu_0)(\phi_0^2h^2 + 128\pi^3)^{\nu_0+1}}h^{2\nu_0}
$.
\end{proof}
}

\begin{lemma}[Uniform upper bound for eigenvalues of Mat\'{e}rn covariance matrix]\label{lem:upper}
Let $\Cb$ be a covariance matrix derived from a Mat\'{e}rn GP with parameters $\sigma^2_0, \phi_0, \mbox{ and } \nu_0$, sampled at locations $s_1, \cdots, s_n \in \mathbb{R}^2$,  The locations are separated by a minimum distance $h > 0$. Then the maximum eigenvalue of 
$\bSigma=\Cb + \tau_0^2 \Ib$ has an uniform (in $n$) upper bound \blue{$\Psi_{high}$.}
\end{lemma}

\begin{proof}
\blue{For simplicity, we drop the subscript $_0$ indicating true parameter values, as the result holds for any choices of these parameters. As the Mat\'ern covariance is strictly positive, using the Perron-Frobenius Theorem, we have}
\[
\|\mathbf{\Sigma}\|_2 = \sup\{\lambda(\mathbf{\Sigma})\} \leq \sup_{i = 1,\cdots, n}\sum_{j = 1}^n \mathbf{\Sigma}_{ij}
\]
\blue{Let $d_{ij}=\|s_i-s_j\|$.} For the Mat\'{e}rn Gaussian process, there exists universal constants $C_1, C_2>0$, such that $|\Cb_{ij}| \leq C_1\cdot d_{ij}^{\nu}\exp(-C_2d_{ij})$, where $d_{ij}$ is the spatial distance between $s_i$ and $s_j$ (see \cite{abramowitz1948handbook}, example 9.7.2). 
Then
\begin{equation}\label{prop-mat-2norm}
\begin{split}
\sum_{j = 1}^n \mathbf{\Sigma}_{ij} &\leq C_1\sum_{j = 1}^n d_{ij}^{\nu}\exp(-C_2d_{ij}) + \tau^2\\
&\leq C_1\sum_{j = 1}^{\infty} d_{ij}^{\nu}\exp(-C_2d_{ij})+ \tau^2\\
&\leq C_0 + C_1\sum_{k = k_0}^{\infty} n_k (2kh)^{\nu}\exp(-2C_2kh)+ \tau^2
\end{split}
\end{equation}
where $k_0$ is large enough so that $C_2kh \geq 1$ and $x^{\nu}\exp(-C_2x)$ decreases for $x \geq 2k_0h$, $C_0$ is the part of the summation for $d_{ij} < 2k_0h$, $n_k$ is the number of $s_j$'s located in the ring between sphere $S(s_i, 2kh)$ and  $S(s_i, 2(k+1)h)$. Since $h$ is the minimum distance between any two locations, \blue{$C_0$ comprises of atmost a fixed number of terms which does not grow with $n$ and hence is uniformly bounded. The minimum distance $h$} guarantees that spheres $\mathcal{S}(s_j, h/2)$, centered at $s_j$'s and have radius $h/2$, have no intersections. $n_k$ is thus upper bounded by the area between $S(s_i, 2kh)$ and  $S(s_i, 2(k+1)h)$ divided by the area of the spheres centered at the locations. 
\[
\blue{n_k \leq  \frac{4\pi\Big(\big((k+1)h\big)^2 - (kh)^2\Big)}{\pi h^2/4} = 16(2k+1)}
\]
Plugging this into (\ref{prop-mat-2norm}), we get, 
\begin{equation}
\begin{split}
\sum_{j = 1}^n \mathbf{\Sigma}_{ij} 
&\leq C_0 + C_1\sum_{k = k_0}^{\infty} \blue{16}(2k+1) (2kh)^{\nu}\exp(-2C_2kh)+ \tau^2\\
&\leq C_0 + 32\frac{C_1}{C_2^{\nu}\blue{h}}\int_0^{\infty}x^{\nu+1}\exp(-x)dx\\
&:= \Psi_{high} < \infty,
\end{split}
\end{equation}

Since $\Psi_{high}$ is a constant free of $i$ \blue{and $n$}, it is an uniform upper bound for the spectral norm of $\|\mathbf{\Sigma}\|_2$. 
\end{proof}

\begin{proof}[Proof of Proposition \ref{prop-main-1}]
\blue{Proposition \ref{prop-main-1} has the same data generation assumptions as Proposition \ref{prop-main-2}, so Assumption \ref{Asmp-1} is satisfied by Lemmas \ref{lemma-min-eigen} and \ref{lem:upper}. We only need to prove that Assumption \ref{Asmp-3} is satisfied when using an NNGP working covariance matrix and with spatial parameter values different from the truth.}

\blue{To verify Assumption \ref{Asmp-3}, we first focus on the eigenbounds of $\mathbf{\Sigma}^{-\frac{\top}{2}}\Eb\mathbf{\Sigma}^{-\frac{1}{2}}$, and then derive the ones for $\Eb$. The reason is that $\mathbf{\Sigma}^{-\frac{\top}{2}}\Eb\mathbf{\Sigma}^{-\frac{1}{2}} = \Qb$ has a clearer structure than $\Eb$ with the condition we assumed. Lemmas \ref{lem:nngp_prec_upper} and \ref{lem:nngp_prec_lower} show that the eigenvalues of the precision matrix $\Qb$ are uniformly bounded in $n$. Combining these uniform bounds on $\Qb$ with those on $\bSigma$ from Lemmas \ref{lemma-min-eigen} and \ref{lem:upper}, we have uniform upper and lower bounds for the eigenvalues of $\Eb$ as 
$\lambda_{\min}(\Eb) \geq \lambda_{\min}(\mathbf{\Sigma})\lambda_{\min}(\Qb)$ and $\lambda_{\max}(\Eb) \leq \lambda_{\max}(\mathbf{\Sigma})\lambda_{\max}(\Qb)$.}
\end{proof}


\blue{\begin{lemma}[Uniform upper bound for eigenvalues of NNGP precision matrix]\label{lem:nngp_prec_upper}
Let $\Qb$ denote the NNGP precision matrix derived from a Mat\'ern GP with parameters $\mathbf{\theta} = (\sigma^2, \phi, \nu,\tau^2)$, using $m$-nearest neighbors, and  for $n$ locations separated by a minimum distance $h > 0$. Let the neighbor sets for NNGP be chosen such that each location $s_i$ appears in neighbor sets of at most $M$ locations. Then the maximum eigenvalue of $\Qb$ has a uniform (in $n$) upper bound.
\end{lemma}}

\begin{proof} \blue{Let $\bar\bSigma$ denote the Mat\'ern GP covariance matrix from which the NNGP working precision matrix is derived. We use $\mathbf{\bar\Sigma}$ to differentiate this from $\bSigma$, the Mat\'ern GP covariance of the data, which has different parameters.} Recall from Section \ref{sec-splmm} and Equation (\ref{eq:nngp_bf}) that the NNGP precision matrix $\Qb$ is given by $\Qb = (\Ib - \Bb)^\top\mathbf{F}^{-1}(\Ib - \Bb)$ where $\Bb$ and $\Fb$ are constructed using \[
\begin{split}
\bb_i &= \mathbf{\bar\Sigma}_{N(i)N(i)}^{-1}\mathbf{\bar\Sigma}_{N(i)i}\\
f_i &=  \mathbf{\bar\Sigma}_{ii} - \mathbf{\bar\Sigma}_{iN(i)}\bb_i.
\end{split}
\]
Here $\bb_i$'s fill into the non-zero elements of the $i$th row of $\Bb$ at positions indexing the $m$-nearest neighbor of $s_i$,  and $f_i$'s are the diagonal elements of $\Fb$. \blue{So, for any $\bx \in \mathbb R^n$, the NNGP quadratic form $\bx^\top\Qb\bx$ can be written as 
\begin{align}\label{eq:nngp_qf}
    \bx^\top\Qb\bx =& \bx_{1:m+1}^\top
    \bar\bSigma_{1:m+1,1:m+1}^{-1}\bx_{1:m+1} + \sum_{i=m+2}^n (x_i - \bb_i^\top\bx_{N(i)})^2/f_i \nonumber \\
    =& \bx_{1:m+1}^\top
    \bar\bSigma_{1:m+1,1:m+1}^{-1}\bx_{1:m+1} + \sum_{i=m+2}^n \left[\bx_{N(i)}^\top
    \bar\bSigma_{N(i),N(i)}^{-1}\bx_{N(i)} + (x_i - \bb_i^\top\bx_{N(i)})^2/f_i  \right.\nonumber \\
    & \qquad \left.  - \bx_{N(i)}^\top   \bar\bSigma_{N(i),N(i)}^{-1}\bx_{N(i)}\right] \\
    =& \bx_{1:m+1}^\top
    \bar\bSigma_{1:m+1,1:m+1}^{-1}\bx_{1:m+1} + \sum_{i=m+2}^n \left[\bx_{N^*[i]}^\top   \bar\bSigma_{N^*[i],N^*[i]}^{-1}\bx_{N^*[i]} -  \bx_{N(i)}^\top   \bar\bSigma_{N(i),N(i)}^{-1}\bx_{N(i)}\right]. \nonumber
\end{align}

The last equality follows from the fact that NNGP is derived using the actual conditional distributions from Mat\'ern GP (see the forms of $\bb_i$ and $f_i$ above). Hence the $i^{th}$ term in the NNGP quadratic form $(x_i - \bb_i^\top\bx_{N(i)})^2/f_i$ is the conditional negative log-likelihood of $x_i \given \bx_{N(i)}$ from a Mat\'ern GP and when 
added to the negative log-likelihood  $\bx_{N(i)}^\top   \bar\bSigma_{N(i),N(i)}^{-1}\bx_{N(i)}$ for the neighbors of $s_i$ yields the joint negative log-likelihood $\bx_{N^*[i]}^\top \bar\bSigma_{N^*[i],N^*[i]}^{-1}\bx_{N^*[i]}$ for the total set of locations $N^*[i]=N(i) \cup \{i\}$. 

From (\ref{eq:nngp_qf}), we have 
\begin{align*}
\bx^\top\Qb\bx \leq & \bx_{1:m+1}^\top
    \bar\bSigma_{1:m+1,1:m+1}^{-1}\bx_{1:m+1} + \sum_{i=m+2}^n \bx_{N^*[i]}^\top   \bar\bSigma_{N^*[i],N^*[i]}^{-1}\bx_{N^*[i]}\\
    \leq & \Psi_{low}^{-1} \left[\|\bx_{1:m+1}\|^2 + \sum_{i=m+2}^n \|\bx_{N^*[i]}\|^2 \right] \mbox{( by Lemma \ref{lemma-min-eigen})} \\
    \leq & \Psi_{low}^{-1} \sum_{i=1}^n x_i^2 (1 + M).
\end{align*}
The last inequality follows as each location can be in the neighbor sets of at most $M$ locations. This implies $\lambda_{\max}(\Qb) \leq (M+1) \Psi_{low}^{-1}$ providing the uniform upper bound.}
\end{proof}

\blue{\begin{lemma}[Uniform lower bound for eigenvalues of NNGP precision matrix]\label{lem:nngp_prec_lower}
Let $\Qb$ denote the NNGP precision matrix derived from a Mat\'ern GP with parameters $\mathbf{\theta} = (\sigma^2, \phi, \nu,\tau^2)$, using $m$-nearest neighbors, and  for $n$ locations separated by a minimum distance $h > 0$. Let the neighbor sets for NNGP be chosen such that each location $s_i$ appears in neighbor sets of at most $M$ locations. Then the maximum eigenvalue of $\Qb$ has a uniform (in $n$) upper bound.
\end{lemma}}

\begin{proof}
We first give a lower eigenvalue bound for $(\Ib-\Bb)^\top(\Ib-\Bb)$.
\[
\begin{split}
\|(\Ib-\Bb)\bx\|_2^2 & = \sum_{i=1}^n(x_i - \bb_i^{\top}\bx)^2\\
&= \sum_{i=1}^n\Big(x_i - \sum_{j=1}^n\Bb_{ij}x_j\Big)^2\\
&= \sum_{i=1}^nx_i^2 - 2\sum_{i,j=1}^n\Bb_{ij}x_ix_j+\sum_{i=1}^n\Big(\sum_{j=1}^n\Bb_{ij}x_j\Big)^2\\
&\geq \sum_{i=1}^nx_i^2 -\sum_{i,j=1}^n\Bb_{ij}(x_i^2+x_j^2)\\
& = \sum_{i=1}^n\Big(1 - \sum_{j\in N(i)}|\Bb_{ij}| - \sum_{i \in N(j)}|\Bb_{ij}|\Big)x_i^2.
\end{split}
\]
By the definition of eigenvalue, 
\[
\lambda_{\min}\big((\Ib-\Bb)^\top(\Ib-\Bb)\big) \geq \min_{i=1, \cdots, n}\Big(1 - \sum_{j\in N(i)}|\Bb_{ij}| - \sum_{i \in N(j)}|\Bb_{ij}|\Big).
\]
\blue{We now obtain the following bounds for the absolute row- and column sums of $\Qb$. Note that the Mat\'ern covariance (Definition \ref{def-Matern}) can be written as $C(h)=\sigma^2 C^*(\phi h)$ where $C^*$ is the Mat\'ern covariance with unit variance, unit spatial decay, and smoothness $\nu$. All the off-diagonal elements of the Mat\'ern covariance is less than $C(h)$ and for any $h$, and $C(h) = \sigma^2 C^*(\phi h) \to 0$ as $\phi \to \infty$. So, for large enough $\phi$ (which we specify later), by Gershgorin's circle theorem, for all $i$, $\lambda_{\min}(\mathbf{\Sigma}_{N(i)N(i)}) \geq \sigma^2+\tau^2 - mC(h) > 0$, and $\big\|\mathbf{\Sigma}_{N(i)N(i)}^{-1}\big\|_2 \leq \frac{1}{\sigma^2+\tau^2 - mC(h)}$. So we have}
\begin{equation}\label{prop-mat-eq-1}
  \sum_{j\in N(i)}|\Bb_{ij}| = \|\bb_i\|_1 \leq \sqrt m\|\bb_i\|_2 = \blue{\sqrt m \|\bar\bSigma_{N(i)N(i)}^{-1}\bar\bSigma_{iN(i)}\|_2 \leq \frac{m C(h)}{\sigma^2+\tau^2 - mC(h)},} 
\end{equation}
\begin{equation}\label{prop-mat-eq-2}
\sum_{j=1}^n|\Bb_{ji}| \leq \sum_{j: i \in N(j)}\|\bb_j\|_{\infty} \leq \sum_{j: i \in N(j)}\|\bb_j\|_2 \leq \blue{ 
M \frac{\sqrt m C(h)} {\sigma^2+\tau^2 - mC(h)}.}
\end{equation}

Using these bounds we have 
\begin{equation}\label{prop-mat-eq-lambda_min}
\lambda_{\min}\big((\Ib-\Bb)^\top(\Ib-\Bb)\big) \geq \blue{\Big(1 - \frac{(m + M\sqrt m)C(h)}{\blue{\sigma^2+\tau^2 - mC(h)}}\Big).}
\end{equation}
\blue{In order to keep a positive lower bound, $C(h)$ must satisfy \[ 
\sigma^2+\tau^2 - mC(h) > 0 \mbox{ and } 1 - \frac{(m + M\sqrt m)C(h)}{\sigma^2+\tau^2 - mC(h)} > 0. 
\]
Both inequalities are satisfied when $$\phi > \frac 1h C^{*-1}\left(\min(1,\frac{\sigma^2 + \tau^2}{(2m+M\sqrt m)\sigma^2})\right).$$
}



According to Ostrowski’s Inequality (see \cite{horn2012matrix}, Theorem 4.5.9), since $\Fb^{-1}$ is obviously Hermitian, 
\begin{equation}\label{prop-mat-lowerbound}
\lambda_{\min}\big((\Ib-\Bb)^\top\Fb^{-1}(\Ib-\Bb)\big) \geq \lambda_{\min}\big((\Ib-\Bb)^\top(\Ib-\Bb)\big)\lambda_{\min}(\Fb^{-1}) > \blue{\frac{\Big(1 - \frac{(m + M\sqrt m)C(h)}{\blue{\sigma^2+\tau^2 - mC(h)}}\Big)}{\sigma^2+\tau^2}},
\end{equation}
\blue{where the last inequality comes from the fact that $f_i$ are the conditional covariance of a Mat\'ern GP at location $s_i$ given the realizations of the GP at $N(i)$, and is less than the unconditional covariance $\sigma^2 + \tau^2$. Equation (\ref{prop-mat-lowerbound}) provides the uniform lower bound for the eigenvalues of the NNGP precision matrix, proving the result.} 

\end{proof}

\subsection{Proof of convergence rate}
\blue{
To prove the convergence result (Theorem \ref{thm-rate}), a key step is to study the discrepancy between the empirical loss $\mathcal L_n$ and the expected loss $L_n$ (minimized by the estimation $\hat{f}_n$ and the truth $f_0$ respectively), and obtain the rate of $\hat{f}_n - f_0$ consequently. A relevant result is Theorem 3.2.5 in \cite{van1996weak}. However, it only considered a fixed class of stochastic processes. This cannot be directly matched our setting where the expected GLS loss $L_n(f)$ and its seive minimizer $\pi_{r_n}f_0$, changes with $n$ due to use of the working precision matrix $\Qb^{-1}$. Moreover, the working precision matrix introduces new scaling factors (for example $C_n$ and $\delta_n$ in the theorem), which explicitly reflect its effect on the convergence. We state and prove an extended version of the theorem as the following lemma.
\begin{lemma}[Extension of Theorem 3.2.5 \cite{van1996weak}]\label{lemma-rate}
    For each $n$, let $M_n$ and $\mathcal{M}_n$ be stochastic processes indexed by a set $\mathbf{\Theta}$. Let $\theta_n \in \mathbf{\Theta}$ and $d_n: \theta \to d_n(\theta, \theta_n)$ be an arbitrary map from $\mathbf{\Theta}$ to $[0, \infty)$. 
    Also suppose there exists a two series of numbers $\delta_n \in [0, \eta)$, $C_n \in (0, \infty)$, for $n=1,2,\ldots$ and constants $\rho \in (0,1)$ and $\eta >0$ such that for any $\delta \in [\delta_n, \eta]$, the stochastic processes satisfies:
    \begin{equation}\label{lemma-rate-cond-1}
    \inf_{\theta:d_n(\theta, \theta_n) \in [\rho\delta, \delta]} M_n(\theta) - M_n(\theta_n) \geq C_n\delta^2,
    \end{equation}
    \begin{equation}\label{lemma-rate-cond-2}
    \mathbb{E}\sup_{\theta:d_n(\theta, \theta_n) \in [\rho\delta, \delta]}\Big|(\mathcal{M}_n - M_n)(\theta_n)-(\mathcal{M}_n - M_n)(\theta)\Big| \lesssim \frac{\phi_n(\delta)}{\sqrt{n}}
    \end{equation}
    where function $\phi_n$ must satisfy $(\phi_n(t)/t^{\alpha})' < 0$ on $ t \in [\rho\delta, \delta]$ for some $\alpha < 2$.

    Then, when a sequence $\{\hat{\theta}_n\}_{n = 1, 2, \cdots} \in \mathbf{\Theta}$ satisfies 
    \begin{equation}\label{lemma-rate-cond-5}
        \mathcal{M}_n(\hat{\theta}_n) \leq \mathcal{M}_n(\theta_n),
    \end{equation}
    converge consistently to $\theta_n$, i.e. $d_n(\hat{\theta}_n, \theta_n) \to 0$ in probability, the convergence rate of $d_n(\hat{\theta}_n, \theta_n)$ can be depicted by $\{a_n\}_{n = 1, 2, \cdots}$ in terms of 
    \[
    a_nd_n(\hat{\theta}_n, \theta_n) = O_p(1),
    \]
    for any $a_n$ satisfies:
    \begin{equation}\label{lemma-rate-cond-3}
        a_n \lesssim \delta_n^{-1}
    \end{equation}
    \begin{equation}\label{lemma-rate-cond-4}
        \frac{a_n^2}{C_n}\phi_n\left(\frac{1}{a_n}\right) \leq \sqrt{n}
    \end{equation}
\end{lemma}
\begin{proof}[Proof of Lemma \ref{lemma-rate}]
For each $n$, we partition the parameter space into the rings $R_{n, j} := \{\theta: a_nd_n(\theta, \theta_n) \in (\rho^{-j+1}, \rho^{-j}]\}$ where $j \in {1, 2, \cdots}$. Then for a fixed $K$, by condition (\ref{lemma-rate-cond-5})
\[
\begin{split}
\Pb\Big(a_nd_n(\hat{\theta}_n, \theta_n) \geq \rho^{-K}\Big)
&\leq \sum_{j \geq K, \rho^{-j} \leq \eta a_n}\Pb\big(\hat{\theta}_n \in R_{n,j}\big) + \Pb(\hat{\theta}_n \notin \cup_{j \geq K, \rho^{-j} \leq \eta a_n}R_{n,j})\\
&\leq \sum_{j \geq K, \rho^{-j} \leq \eta a_n} \Pb\Big(\inf_{\theta\in R_{n, j}}\big(\mathcal{M}_n(\theta)- \mathcal{M}_n(\theta_n)\big) < 0 \Big)\\
& + \Pb(d_n(\hat{\theta}_n, \theta_n) \geq \rho\eta)
\end{split}
\]
Since the second term converges to $0$ by consistency condition, we only have to consider the first term. By the first condition on $M_n$, on $R_{n, j}$, 
\[
\inf_{\theta \in R_{n,j}} \big\{M_n(\theta) - M_n(\theta_n)\big\}) \geq C_n(a_n\rho^j)^{-2},
\]
we obtain that
\[
\begin{split}
\Pb\Big(\inf_{\theta\in R_{n, j}}\big(\mathcal{M}_n(\theta)- \mathcal{M}_n(\theta_n)\big) < 0 \Big) &\leq \Pb\Big(\sup_{\theta \in R_n}\big|(\mathcal{M}_n - M_n)(\theta_n) - (\mathcal{M}_n - M_n)(\theta)\big| \geq C_n(a_n\rho^j)^{-2}\Big)\\
&\lesssim \frac{(a_n\rho^j)^{2}\phi((a_n\rho^j)^{-1})}{\sqrt{n}C_n} 
\lesssim \frac{(a_n\rho^j)^{2}\rho^{-\alpha j}\phi(a_n^{-1})}
{\sqrt{n}C_n} \\
&\lesssim \rho^{j(2-\alpha)}.
\end{split}
\]
by using Markov's inequality, conditions (\ref{lemma-rate-cond-2}) and (\ref{lemma-rate-cond-4}), and the monotonicity of $\phi(t)/t^{\alpha}$. Now the first term becomes
\[
\sum_{j \geq K, \rho^{-j} \leq \eta a_n} \Pb\Big(\inf_{\theta\in R_{n, j}}\big(\mathcal{M}_n(\theta)- \mathcal{M}_n(\theta_n)\big) < 0 \Big) \leq \sum_{j \geq K, \rho^{-j} \leq \eta a_n} \rho^{j(2-\alpha)},
\]
and, as $\rho \in (0,1)$, there exists $K_1>0$ such that the first term can be upper-bounded by a given $\epsilon > 0$ when $K>K_1$. Additionally, when $K$ is large enough, for conditions (\ref{lemma-rate-cond-1}) and (\ref{lemma-rate-cond-2}) to be applicable for $\theta \in R_{n,j}$, $\theta$ should satisfy $d_n(\theta, \theta_n)\in [\rho\delta, \delta]$. Thus, we need $R_{n,j} \subseteq \{\theta : d_n(\theta, \theta_n)\in [\rho\delta, \delta]\}$ for $j\geq K$. Combining this with $\delta \in [\delta_n, \eta]$ and the definition of $R_{n,j}$, one of the requirement is that there exists $K_2>0$, such that $a_n^{-1}\rho^{-K} \geq \delta_n$ for all $K\geq K_2$, which is guaranteed by condition (\ref{lemma-rate-cond-3}). Now for any $\epsilon >0$, we proved that for $K > \max\{K_1, K_2\}$ , $\Pb\Big(a_nd_n(\hat{\theta}_n, \theta_n) \geq \rho^{-K}\Big) \leq \epsilon$. Since $\epsilon$ is arbitrary,  $a_nd_n(\hat{\theta}_n, \theta_n) = O_p(1)$, and the proof completes.
\end{proof}

\begin{lemma}\label{lemma-rate-2}
    Under Assumptions \ref{Asmp-1} and \ref{Asmp-3}, for any $n$ and $\delta \geq 6\frac{\Psi_{high}\Lambda_{high}}{\Psi_{low}\Lambda_{low}}\|f_0 - \pi_{r_n}f_0\|_n$, 
    \[
    \inf_{\frac{\delta}{2}<\|f - \pi_{r_n}f_0\|_n < \delta}L_n(f) - L_n(\pi_{r_n}f_0) \geq  \frac{\Lambda_{low}}{12\Psi_{high}}\delta^2
    \]
\end{lemma}
\begin{proof}[Proof of Lemma \ref{lemma-rate-2}]
    By definition of the expected GLS loss $L_n$, we have
    \[
    \begin{split}
        L_n(f) - L_n(\pi_{r_n}f_0) &=  \frac{1}{n}\Big((\bm{f}_0(x) - \bm{f}(x))^\top\mathbf{\Sigma}^{-\frac{\top}{2}}\Eb\mathbf{\Sigma}^{-\frac{1}{2}}(\bm{f}_0(x) - \bm{f}(x)) \\
        &- (\bm{f}_0(x) - \pi_{r_n}\bm{f}_0(x))^\top\mathbf{\Sigma}^{-\frac{\top}{2}}\Eb\mathbf{\Sigma}^{-\frac{1}{2}}(\bm{f}_0(x) - \pi_{r_n}\bm{f}_0(x))\Big)\\
        &=  \frac{1}{n}\Big((\bm{f}(x) - \pi_{r_n}\bm{f}_0(x))^\top\mathbf{\Sigma}^{-\frac{\top}{2}}\Eb\mathbf{\Sigma}^{-\frac{1}{2}}(\bm{f}(x) - \pi_{r_n}\bm{f}_0(x)) \\
        &+ 2(\pi_{r_n}\bm{f}_0(x) - \bm{f}_0(x))^\top\mathbf{\Sigma}^{-\frac{\top}{2}}\Eb\mathbf{\Sigma}^{-\frac{1}{2}}(\bm{f}(x) - \pi_{r_n}\bm{f}_0(x))\Big).\\
    \end{split}
    \]
    A lower bound is given by applying matrix norm inequalities
    \[
    \begin{split}
        L_n(f) - L_n(\pi_{r_n}f_0) &\geq \frac{\Lambda_{low}}{\Psi_{high}}\|f - \pi_{r_n}f_0\|^2_n - 2\frac{\Lambda_{high}}{\Psi_{low}}\|f - \pi_{r_n}f_0\|_n\|f_0 - \pi_{r_n}f_0\|_n
    \end{split}
    \]
    When $\|f - \pi_{r_n}f_0\|_n \geq \frac{\delta}{2} \geq 3\frac{\Psi_{high}\Lambda_{high}}{\Psi_{low}\Lambda_{low}}\|f_0 - \pi_{r_n}f_0\|_n$, $\|f_0 - \pi_{r_n}f_0\|_n \leq \frac{\Psi_{low}\Lambda_{low}}{3\Psi_{high}\Lambda_{high}}\|f - \pi_{r_n}f_0\|_n$ and we plug this into the lower bound of $L_n(f) - L_n(\pi_{r_n}f_0)$:
    \[
    \begin{split}
    L_n(f) - L_n(\pi_{r_n}f_0) &\geq  \frac{\Lambda_{low}}{\Psi_{high}}\|f - \pi_{r_n}f_0\|^2_n - \frac{2\Lambda_{low}}{3\Psi_{high}}\|f - \pi_{r_n}f_0\|^2_n\\
    &= \frac{\Lambda_{low}}{3\Psi_{high}}\|f - \pi_{r_n}f_0\|^2_n \geq \frac{\Lambda_{low}}{12\Psi_{high}}\delta^2
    \end{split}
    \]
\end{proof}

\begin{lemma}\label{lemma-rate-3}
    Under Assumptions \ref{Asmp-1} and \ref{Asmp-3}, for any $n$ and $\delta \geq 0$, there exists a function $\phi_n(\cdot)$,
    \begin{equation}\label{lemma-rate-3-eq-1}
    \mathbb{E}\left[\sup_{\frac{\delta}{2}<\|f - \pi_{r_n}f_0\|_n < \delta}\Big|(\mathcal{L}_n - L_n)(\pi_{r_n}f_0) - (\mathcal{L}_n - L_n)(f)\Big|\right] \lesssim \frac{\phi_n(\delta)}{\sqrt{n}}
    \end{equation}
    where $\frac{d(\delta^{-\alpha}\phi(\delta))}{d\delta} < 0$ on $(0, \eta)$ for some $\alpha < 2$, $\eta > 0$.
\end{lemma}
\begin{proof}
With the calculation in the first equation of (\ref{lemma-main-eq-1}), 
\[
(\mathcal{L}_n - L_n)(f) = \frac{1}{n}\big(\bm{\epsilon}^\top\blue{\mathbf{\Sigma}^{-\frac{\top}{2}}\Eb\mathbf{\Sigma}^{-\frac{1}{2}}}\bm{\epsilon} - \EE[\bm{\epsilon}^\top\blue{\mathbf{\Sigma}^{-\frac{\top}{2}}\Eb\mathbf{\Sigma}^{-\frac{1}{2}}}\bm{\epsilon}]\big) - 2\frac{1}{n}\bm{\epsilon}^\top\blue{\mathbf{\Sigma}^{-\frac{\top}{2}}\Eb\mathbf{\Sigma}^{-\frac{1}{2}}}(\bm{f} - \bm{f}_0),
\]
and 
\[
(\mathcal{L}_n - L_n)(\pi_{r_n}f_0)  - (\mathcal{L}_n - L_n)(f) = \frac{2}{n}\bm{\epsilon}\Qb(\bm{f} - \pi_{r_n}\bm{f}_0) = \frac{2}{\sqrt{n}}(X_{f} - X_{\pi_{r_n}f_0}).
\]
Similar to the proof of Lemma \ref{lemma-main-1} in (\ref{lemma-main-eq-CN}), with the same $K, C$, using Theorem 2.2.8 in \cite{van1996weak}, the left hand side of (\ref{lemma-rate-3-eq-1}) equals:
\[
\begin{split}
\frac{2}{\sqrt{n}}\EE\left[\sup_{\frac{\delta}{2}<\|\bm{f} - \pi_{r_n}f_0\|_n < \delta}\big|X_{f} - X_{\pi_{r_n}f_0}\big|\right]  &\leq \frac{2}{\sqrt{n}}KC_{\Lambda}\int_{0}^{\delta}\sqrt{\log N\Big(\frac{\eta}{2}, \mathcal{F}_n, \|\cdot\|_{n}\Big)} d\eta.\\
&\leq \frac{2}{\sqrt{n}}KC_{\Lambda}\int_{0}^{\delta}\sqrt{\log N\Big(\frac{\eta}{2}, \mathcal{F}_n, \|\cdot\|_{\infty}\Big)} d\eta.
\end{split}
\]
The bound of the covering number is given in (\ref{lemma-main-eq-coveringnum}), i.e. 
\[
N\left(\frac{\eta}{2}, \mathcal{F}_n, \|\cdot\|_{\infty}\right) \leq \left(\frac{2e\big[r_n(d+2)+1\big]V_n^2}{\eta(V_n-4)}\right)^{r_n(d+2)+1} := \left(\frac{C(d, r_n, V_n)}{\eta}\right)^{r_n(d+2)+1},
\]
When $\delta <1$, we can introduce $\phi(\cdot)$ by 
\begin{equation}\label{eq-def-phi}
\begin{split}
\frac{2}{\sqrt{n}}\EE\left[\sup_{\frac{\delta}{2}<\|\bm{f} - \pi_{r_n}f_0\|_n < \delta}\big(X_{f} - X_{\pi_{r_n}f_0}\big)\right]  &\leq \frac{2KC_{\Lambda}\sqrt{r_n(d+2)+1}}{\sqrt{n}}\int_{0}^{\delta}\sqrt{\log\left(\frac{C(d, r_n, V_n)}{\eta}\right)}d\eta \\
&\lesssim \frac{4KC_{\Lambda}\sqrt{r_n(d+2)+1}}{\sqrt{n}}\delta\sqrt{\log\left(\frac{eC(d, r_n, V_n)}{\delta}\right)} := \frac{\phi(\delta)}{\sqrt{n}}
\end{split}
\end{equation}
where the second inequality is similar to Lemma 3.8 in \cite{mendelson2003few}. Specifically, 
\[
\frac{d}{d\delta}\int_{0}^{\delta}\sqrt{\log\left(\frac{C(d, r_n, V_n)}{\eta}\right)}d\eta = \sqrt{\log\left(\frac{C(d, r_n, V_n)}{\delta}\right)};
\]
and 
\[
\frac{d}{d\delta}2\delta\sqrt{\log\left(\frac{eC(d, r_n, V_n)}{\delta}\right)} = 2\sqrt{\log\left(\frac{eC(d, r_n, V_n)}{\delta}\right)} - \frac{1}{\sqrt{\log\left(\frac{eC(d, r_n, V_n)}{\delta}\right)}}.
\]
Note that $C(d, r_n, V_n) > 2$ holds uniformly by its definition. When $\delta \in (0,1)$, since $\log\left(\frac{eC(d, r_n, V_n)}{\delta}\right) \geq 1$,  
\[
\frac{d}{d\delta}\int_{0}^{\delta}\sqrt{\log\left(\frac{C(d, r_n, V_n)}{\eta}\right)}d\eta \leq \frac{d}{d\delta}2\delta\sqrt{\log\left(\frac{eC(d, r_n, V_n)}{\delta}\right)}.
\]
Since the two functions are both $0$ when $\delta = 0$, we have that
\[
\int_{0}^{\delta}\sqrt{\log\left(\frac{C(d, r_n, V_n)}{\eta}\right)}d\eta \leq 2\delta\sqrt{\log\left(\frac{eC(d, r_n, V_n)}{\delta}\right)} \ \text{for } \delta \in (0,1),
\]
and the second inequality in the definition of $\phi(\cdot)$ (\ref{eq-def-phi}) holds.
Now we take the derivative of $\delta^{-\alpha}\phi(\delta)$ and get
\[
\begin{split}
\frac{d(\delta^{-\alpha}\phi(\delta))}{d\delta} &= \frac{d}{d\delta}\left(4KC_{\Lambda}\sqrt{r_n(d+2)+1}\delta^{1-\alpha}\sqrt{\log\left(\frac{eC(d, r_n, V_n)}{\delta}\right)}\right)\\
&= 4KC_{\Lambda}\sqrt{r_n(d+2)+1}(1-\alpha)\delta^{-\alpha}\sqrt{\log\left(\frac{eC(d, r_n, V_n)}{\delta}\right)} \\
& + 2KC_{\Lambda}\sqrt{r_n(d+2)+1}\delta^{1-\alpha}\log^{-\frac{1}{2}}\left(\frac{eC(d, r_n, V_n)}{\delta}\right)\frac{d}{d\delta}\log\left(\frac{eC(d, r_n, V_n)}{\delta}\right)\\
&= \frac{2KC_{\Lambda}\sqrt{r_n(d+2)+1}}{\sqrt{\log\left(\frac{C(d, r_n, V_n)}{\delta}\right)}+1}\delta^{-\alpha}\left[2(1-\alpha)\log\left(\frac{eC(d, r_n, V_n)}{\delta}\right) - 1\right].
\end{split}
\]
 When $\delta \in (0, 1)$ (i.e. $\eta = 1$) and $\alpha \in (1,2)$, $(1-\alpha)$ is negative and $\log\left(\frac{eC(d, r_n, V_n)}{\delta}\right)$ is positive, the whole term in the parenthesis is negative so that $\frac{d(\delta^{-\alpha}\phi(\delta))}{d\delta} < 0$. In conclusion, when $\phi(\cdot)$ is as defined, $\eta = 1$ and $\alpha \in (1,2)$, the derivative on $(0, \eta)$ is negative and the proof completes. 
\end{proof}

\begin{proof}[Proof of Theorem \ref{thm-rate}]
In order to find the convergence rate for $\|\hat{f}_n - f_0\|_n$, we apply Lemma \ref{lemma-rate} on the stochastic processes $L_n$, $\mathcal{L}_n$, and the set $\mathcal{F}_0$. Specifically, plug $L_n$ and $\mathcal{L}_n$ into $M_n$ and $\mathcal{M}_n$ respectively, $\pi_{r_n}f_0$ and $\hat{f}_n$ into $\theta_n$ and $\hat{\theta}_n$ respectively in Lemma \ref{lemma-rate}. By Lemma \ref{lemma-rate-2}, condition (\ref{lemma-rate-cond-1}) is satisfied when $C_n = \frac{\Lambda_{low}}{12\Psi_{high}}$, $\rho = \frac{1}{2}$ and $\delta_n = 6\frac{\Psi_{high}\Lambda_{high}}{\Psi_{low}\Lambda_{low}}\|f_0 - \pi_{r_n}f_0\|_n$. By Lemma \ref{lemma-rate-3}, condition (\ref{lemma-rate-cond-2}) is satisfied with 
\[
\phi_n(\delta) = 4KC_{\Lambda}\sqrt{r_n(d+2)+1}\delta\sqrt{\log\left(\frac{eC(d, r_n, V_n)}{\delta}\right)},
\]
and $\eta = 1$. Since $\hat{f}_n = \argmin_{f\in \mathcal{F}_{n}}\mathcal{L}_n(f)$, condition (\ref{lemma-rate-cond-5}) is naturally satisfied. We only need to find a supremum rate for $a_n$ from conditions (\ref{lemma-rate-cond-3}) and (\ref{lemma-rate-cond-4}).
Note that condition (\ref{lemma-rate-cond-4}) implies:
\[
KC_{\Lambda}\sqrt{r_n(d+2)+1}a_n\sqrt{\log C(d, r_n, V_n) + \log a_n + 1} \lesssim \sqrt{n}C_n,
\]
recall that $K$ is uniformly constant and $C = 2(\sqrt{2}+1)^{\frac{1}{2}}\Lambda_{high}\Psi^{-\frac{1}{2}}_{low}$ from Proposition \ref{prop-thm-main-1}, it is sufficient to find $a_n$ such that,
\[
C_{\Lambda}(r_n(d+2)+1)\log C(d, r_n, V_n)a^2_n \lesssim nC_n^2, \quad C_{\Lambda}(r_n(d+2)+1)a^2_n\log a_n \lesssim nC_n^2.
\]
With $C(d, r_n, V_n) = \frac{2e[r_n(d+2)+1]V_n^2}{(V_n-4)} \sim r_nV_n$ and $r_n \to \infty$, from the first restriction, we obtain that $a_n^2 \lesssim \frac{n}{r_n\log(r_nV_n)}$. By the condition on $r_n, V_n$ in Assumption \ref{Asmp-2}, $r_nV_n = o(n)$, and $\log a_n \lesssim \log n$, the second restriction can be relaxed to $C(r_n(d+2)+1)a^2_n\log n \lesssim nC_n^2$. Now it suffices to find $a_n$ satisfies,
\[
a^2_n \lesssim \left(\frac{nC_n^2}{C_{\Lambda}^2r_n\log(r_nV_n)} \wedge \frac{nC_n^2}{C_{\Lambda}^2r_n\log n}\right) \lesssim \frac{nC_n^2}{C_{\Lambda}^2r_n\log n}.
\]
Combine with Condition (\ref{lemma-rate-cond-4}), by letting $\{a_n\}$ to have the maximum rate, the result of Lemma \ref{lemma-rate} implies that
\[
d_n(\hat{f}_n, \pi_{r_n}f_0) = \|\hat{f}_n - \pi_{r_n}f_0\|_n =  O_p\left(\frac{\Psi_{high}\Lambda_{high}}{\Psi_{low}\Lambda_{low}}\|\pi_{r_n}f_0 - f_0\|_n  \vee \frac{\Psi_{high}\Lambda_{high}}{\Psi^{\frac{1}{2}}_{low}\Lambda_{low}}\sqrt{\frac{r_n\log n}{n}}\right).
\]
The same rate also applies to $\|\hat{f}_n - f_0\|_n$ since $\frac{\Psi_{high}\Lambda_{high}}{\Psi_{low}\Lambda_{low}} > 1$, and we have that:
\[
d_n(\hat{f}_n, f_0) = \|\hat{f}_n - \pi_{r_n}f_0\|_n + \|\pi_{r_n}f_0 - f_0\|_n =  O_p\left(\frac{\Psi_{high}\Lambda_{high}}{\Psi_{low}\Lambda_{low}}\|\pi_{r_n}f_0 - f_0\|_n \vee \frac{\Psi_{high}\Lambda_{high}}{\Psi^{\frac{1}{2}}_{low}\Lambda_{low}}\sqrt{\frac{r_n\log n}{n}}\right).
\]
\end{proof}

}

\newpage
\blue{\section{Additional theoretical insights guiding implementation choices}\label{Append-add}

\subsection{Consistency of spatial parameter estimation}\label{Append-parameter}
The theory of consistency of NN-GLS in estimating the the mean function assumed that the spatial parameters $\btheta$ specifying $\Qb$ are held fixed (at values possibly different from the true values). This is the norm in theoretical studies of non-parametric mean estimators for spatial data. For example, the theoretical properties of GAM-GLS were developed in \cite{nandy2017additive} assuming fixed spatial parameter values, the estimation of which is not discussed (only cross-validation is done in practice). Similarly, the theory of RF-GLS \citep{saha2023random} assumes a known parameter value, although in practice they pre-estimate these values using a crude method. 

The practical implementation of NN-GLS offers benefits regarding this aspect. Using the representation of NN-GLS as a graph neural network (GNN), the spatial parameters of the working covariance matrix simply become parameters specifying the graph convolution laters, and hence are updated similarly to the NN parameters. Equation (12) in Section \ref{sec:backprop} provides the updated equation for the spatial parameters. 

We now provide some informal theory and empirical evidence on how the spatial parameters can be consistently updated subsequent to estimation of the mean function. 

Let $\hat f_n$ denote the NN-GLS estimate of $f_0$ obtained using a working covariance matrix with fixed spatial parameters. Using (12), we can then maximize the following likelihood based on the residual estimation:
\[
\begin{split}
    \log\text{-Likelihood}(\btheta) &=  -\frac{1}{n}\big(\by - \hat{\bm{f}_n}\big)'\Qb\big(\by - \hat{\bm{f}_n}\big) + \log|\Qb|\\
     &=  -\frac{1}{n}\big(\bm{f_0} + \bm{\epsilon} - \hat{\bm{f}_n}\big)'\Qb\big(\bm{f_0} + \bm{\epsilon} - \hat{\bm{f}_n}\big) + \log|\Qb|\\
    & = -\frac{1}{n}\big(\bm{f_0} - \hat{\bm{f}_n}\big)\Qb'\big(\bm{f_0} - \hat{\bm{f}_n}\big)' - \frac{2}{n}\bm{\epsilon}'\Qb\big(\bm{f_0} - \hat{\bm{f}_n}\big)\\
    & + \frac{1}{n}\big( - \bm{\epsilon}'\Qb\bm{\epsilon} +\log|\Qb|\big),
\end{split}
\]
where $\boldsymbol\epsilon \sim \mathcal{N}(\mathbf{0}, \mathbf{\Sigma})$. As $\Qb =\bSigma^{-\frac\top 2}\Eb\bSigma^{\frac 12}$. Under Assumpions \ref{Asmp-1} and \ref{Asmp-3}, the first term is $O_p(\|\hat{f}_n - f_0\|^2_n)$ and the second term is $O_p(\|\hat{f}_n - f_0\|_n)$. By the consistency result of Theorem \ref{Thm-main-1}, both these terms converge to $0$ in probability. The terms left are $ \frac{1}{n}\big( - \bm{\epsilon}'\Qb\bm{\epsilon} + \log|\Qb|\big)$. As was shown in Proposition \ref{prop-E}, $\frac{1}{n}\mathbb{E}(\bm{\epsilon}'\Qb\bm{\epsilon}) = 1$, which guarantees that the log-likelihood is dominated by these terms when $n$ is large. We consider the simple case $\Qb=\mathbf{\Sigma}^{-1}$, a full GP precision matrix. Then the log-likelihood for $\btheta$ is dominated by the term 
\[
\frac{1}{n}\big( - \bm{\epsilon}'\mathbf{\Sigma}^{-1}\bm{\epsilon} - \log|\mathbf{\Sigma}|\big).
\]
This is actually the log-likelihood of a spatial GP. Though we didn't find a consistency result for the MLE estimator of the spatial GP exactly under our setting, \cite{mardia1984maximum} shows that for the Mat\'ern GP, $\btheta$ can be estimated consistently by optimizing this log-likelihood, under the increasing domain (same assumption as ours). This provides a theoretical intuition on the consistency of the estimate of $\btheta$ in NN-GLS algorithm. We do not pursue this more formally here, as the primary goal of the theory was to demonstrate consistent estimation of the mean function using NN-GLS. 

\begin{figure}[!t]
\centering
\begin{subfigure}{.5\textwidth}
  \centering
  \includegraphics[width=0.9\linewidth]{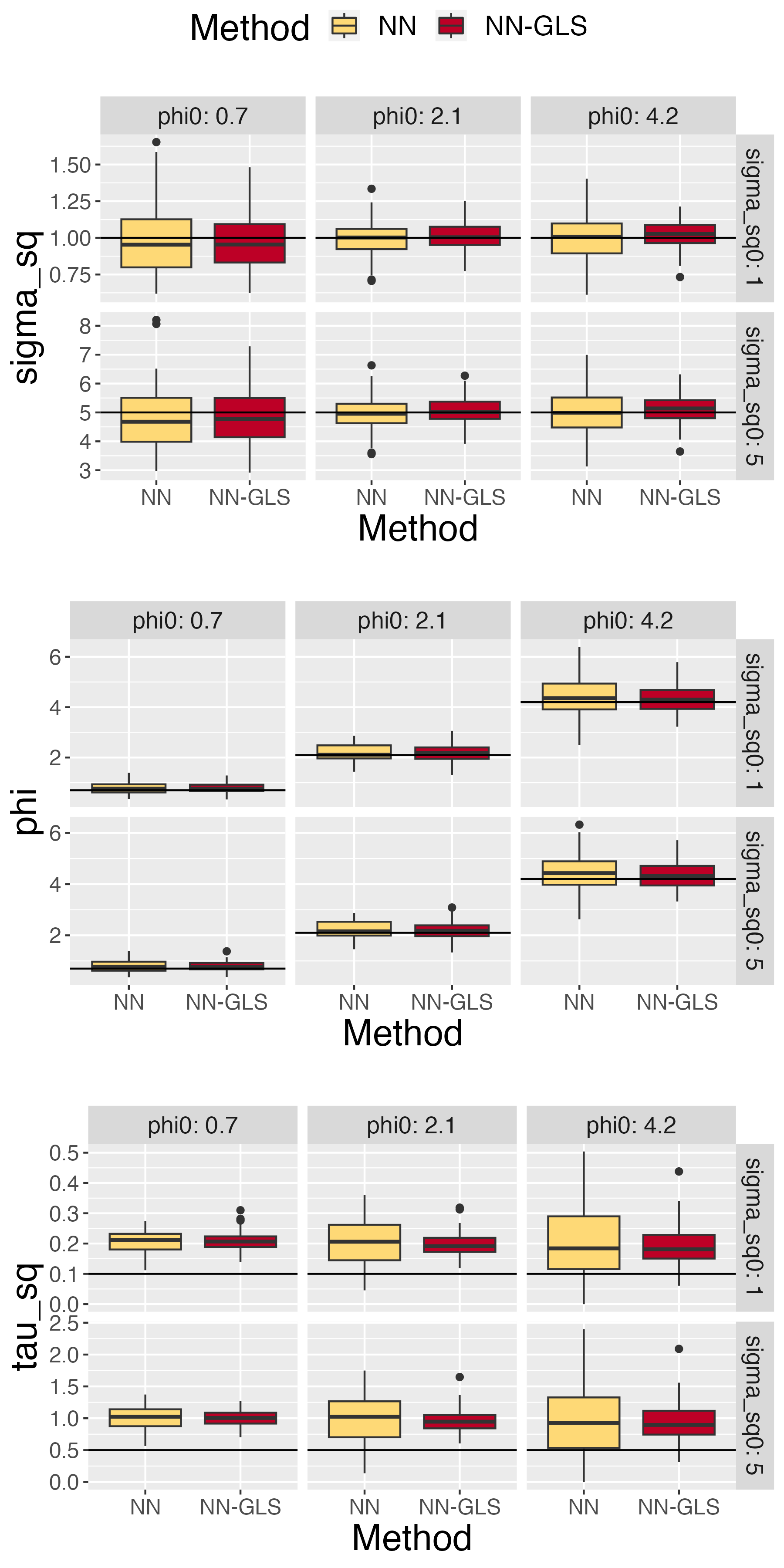}
  \caption{Estimation when $f_0 = f_1$}
\end{subfigure}%
\begin{subfigure}{.5\textwidth}
  \centering
  \includegraphics[width=0.9\linewidth]{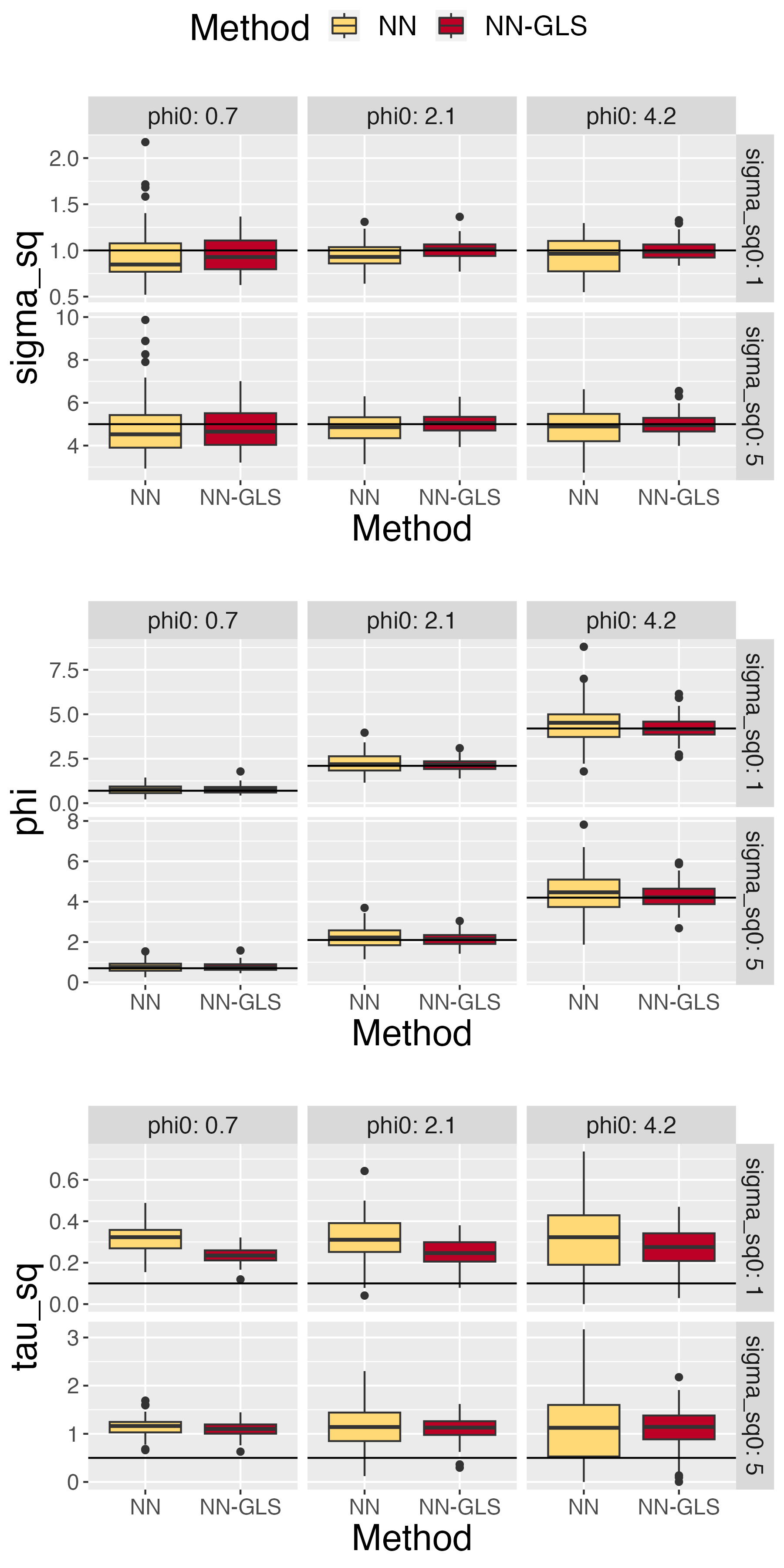}
  \caption{Estimation when $f_0 = f_2$}
\end{subfigure}
\caption{\blue{Parameter estimations in NN-GLS. We denote NN-nonspatial by NN here to save space for axis text.}}
\label{fig-par}
\end{figure}

An empirical illustration of this comes from the simulation study (Figure \ref{fig-par}). We present two sets of estimates of the spatial parameters $\sigma^2$, $\phi$ and $\tau^2$ along with their true values. The first estimate corresponds to the MLE of $\btheta$ based on initial residuals non-spatial NN estimate of $f_0$ (see algorithm \ref{alg-NN-GLS-main}). The second estimate is from NN-GLS where $\btheta$ is updated using (12). The estimates are compared under settings similar to Section \ref{sec-sim}, as the true spatial parameters $\sigma^2 \in \{1, 5\}$, $\phi \in \{1/\sqrt{2}, 3/\sqrt{2}, 6/\sqrt{2}\}$ and $\tau^2/\sigma^2 = 0.1$. 

For $\sigma^2$ and $\phi$, both approaches provide accurate estimations. This is not surprising as even the non-spatial NN estimate is consistent for $f_0$, so the argument given above applies to the first estimate of $\btheta$ as well, in addition to the NN-GLS estimate. For $\tau^2/\sigma^2$, bias exists for both approaches. However, the bias is on the conservative side (overestimation of the nugget variance) and the bias is less when using the NN-GLS estimate for $\hat f_n$ than from a non-spatial NN.  
Overall, NN-GLS estimates of $\btheta$ have smaller bias and variance across all settings, which shows the advantage of updating the spatial parameters as part of the estimation process. Note that this comparison is just for the spatial parameter estimates. The main benefits of NN-GLS lie in more precise estimates of the mean function and spatially informed prediction. These advantages are illustrated via the extensive experiments of Sections \ref{sec-sim} and \ref{sec-real}. Additionally,  Section \ref{sec-sim-orc} demonstrates that there is minimal impact of spatial parameter estimation on mean function estimation and predictions. 

\subsection{Benefits of using GLS loss over OLS loss in neural networks for spatial data}\label{Append-E}
Theorem \ref{thm-rate} shows that the error rates of NN-GLS get scaled up by a factor $\Lambda_{high}/\Lambda_{low}$. This ratio of the upper and lower eigen-bounds of $E$ is always $\geq 1$ and is smallest when it is close to one. The ratio thus measures the distance of the discrepancy matrix $\Eb$ and the identity matrix (for which the ratio is $1$), or equivalently the distance between the true covariance matrix $\bSigma$ and the working covariance matrix $\Qb^{-1}$. As was stated in Section \ref{sec:gnn}, we use the NNGP precision matrix $\Qb=\tilde{\mathbf{\Sigma}}^{-1}$ as an approximation of the true precision matrix $\mathbf{\Sigma}^{-1}$ to offer scalable implementation of NN-GLS using connections to Graph Neural Networks. The NNGP approximation gets better when the neighborhood size $m$ increases, in the following sense:
\begin{proposition}\label{prop-E}
    When using NNGP precision matrix with a neighbor size $m$, i.e. $\Qb(m) := \big(\Ib - \Bb(m)\big)^T\mathbf{F}(m)^{-1}\big(\Ib - \Bb(m)\big)$ as the working precision matrix $\Qb$. The {\em Kullback–Leibler divergences} (KLD) between identity matrix $\Ib$ and the discrepancy matrix $\Eb(m)  = \mathbf{\Sigma}^{\frac{\top}{2}}\Qb(m)\mathbf{\Sigma}^{\frac{1}{2}}$, defined as $\text{KLD}\big(\Ib, \Eb(m)\big)$, decrease monotically with $m$.
\end{proposition}
\begin{proof}[Proof of Proposition \ref{prop-E}]
    We first study the trace of $\Eb:=\bE(m)$:
    \[
    \text{tr}(\Eb) = \text{tr}(\Qb\mathbf{\Sigma}) = \mathbb{E}(\epsilon'\Qb\epsilon),
    \]
    where $\epsilon \sim \mathcal{N}(\mathbf{0}, \mathbf{\Sigma})$.
    Since the NNGP precision matrix $\Qb$ can be decomposed as $\Qb = \big(\Ib - \Bb\big)^T\mathbf{F}^{-1}\big(\Ib - \Bb\big)$, by definition of $\Bb$ and $\mathbf{F}$ in (\ref{eq:nngp_bf}), 
    \[
    \begin{split}
        \mathbb{E}(\epsilon'\Qb\epsilon) &= \sum_{i=1}^n\frac{\mathbb{E}\Big[(\epsilon_i - \mathbf{\Sigma}_{i, N(i)}\mathbf{\Sigma}^{-1}_{N(i), N(i)}\epsilon_{N(i)})^2\Big]}{\mathbf{\Sigma}_{ii} - \mathbf{\Sigma}\big(i, N(i)\big)\mathbf{\Sigma}\big(N(i), N(i)\big)^{-1}\mathbf{\Sigma}\big(N(i), i\big)} \\
        &= \sum_{i=1}^n\frac{\mathbf{\Sigma}_{ii} - \mathbf{\Sigma}\big(i, N(i)\big)\mathbf{\Sigma}\big(N(i), N(i)\big)^{-1}\mathbf{\Sigma}\big(N(i), i\big)}{\mathbf{\Sigma}_{ii} - \mathbf{\Sigma}\big(i, N(i)\big)\mathbf{\Sigma}\big(N(i), N(i)\big)^{-1}\mathbf{\Sigma}\big(N(i), i\big)}\\
        &= n.
    \end{split}
    \]

    Thus the trace of $\Eb(m)$ is a constant free of neighbor size $m$. By definition of KL divergence, 
    \[
    \begin{split}
        \text{KLD}\big(\Ib, \Eb(m)\big) &= \text{tr}(\Ib\Eb(m)) - \log|\Eb(m)| = n + \log\left(\frac{|\mathbf{\Sigma}^{-1}|}{|\Qb(m)|}\right)\\
        &= n + \log|\mathbf{\Sigma}^{-1}| + \log \prod_{i = 1}^n\Big[\mathbf{\Sigma}_{ii} - \mathbf{\Sigma}\big(i, N(i)\big)\mathbf{\Sigma}\big(N(i), N(i)\big)^{-1}\mathbf{\Sigma}\big(N(i), i\big)\Big]
    \end{split}
    \]
    Since the Gaussian conditional variances $\Big[\mathbf{\Sigma}_{ii} - \mathbf{\Sigma}\big(i, N(i)\big)\mathbf{\Sigma}\big(N(i), N(i)\big)^{-1}\mathbf{\Sigma}\big(N(i), i\big)\Big]$ decrease with increasing conditioning sets $N(i)$, the second term decreases monotonically with $m$, and so does the whole KL divergence.
\end{proof}

This proposition describes the closeness between $\Ib$ and $\Eb$, which intuitively also represents the closeness between $\mathbf{\Sigma}^{-1}$ and $\Qb$. In a similar sense, Assumption \ref{Asmp-3}'s goal is to bound the distance between $\mathbf{\Sigma}^{-1}$ and $\Qb$ so that the NN-GLS's estimation obtains nice properties including consistency. However, instead of using KL-divergence, Assumption \ref{Asmp-3} restricts the spectrum of $\Eb$. A spectral interval close to $[1,1]$ represents a high similarity between $\Ib$ and $\Eb$, and thus between $\mathbf{\Sigma}^{-1}$ and $\Qb$. 

\begin{figure}[!t]
\centering
\includegraphics[width=0.9\linewidth]{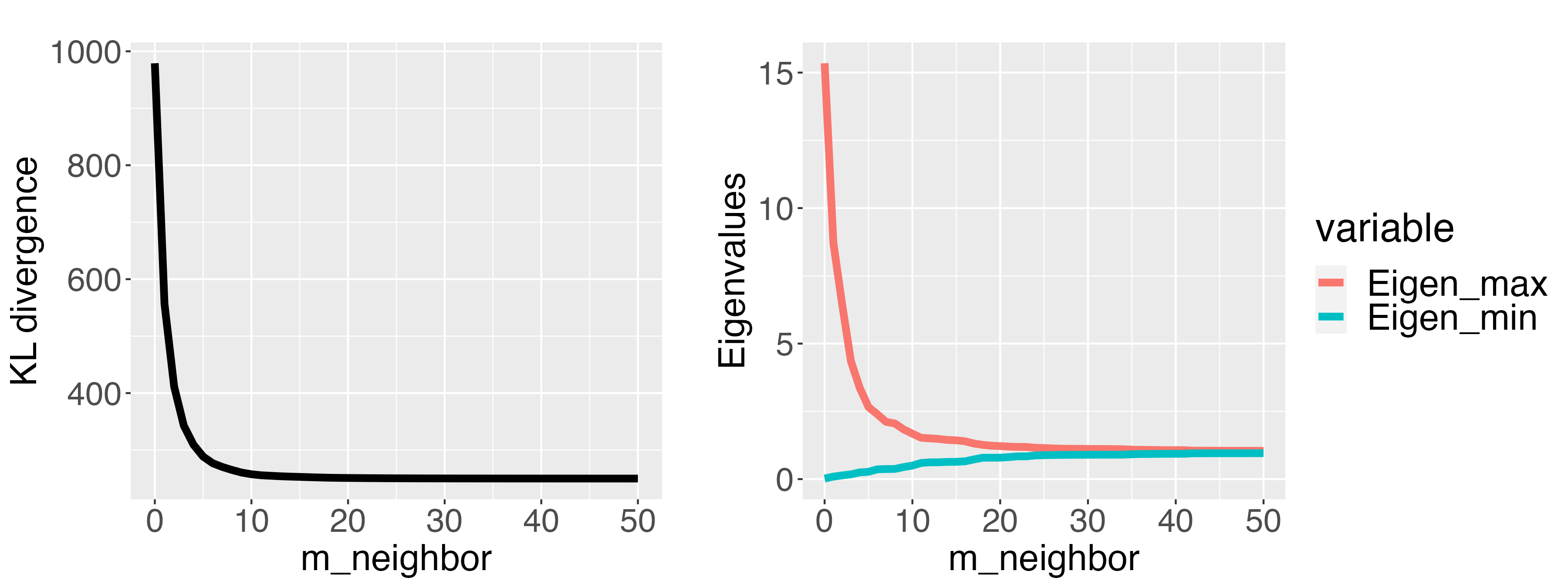}
\caption{Different metrics measuring the distance between the identity matrix and the discrepancy matrix $\Eb$ when using NNGP working precision matrix with $m$ nearest neighbors in the GLS loss. $m=0$ corresponds to the identity covariance matrix, i.e., using OLS loss.}
\label{fig-E}
\end{figure}

Figure \ref{fig-E} empirically illustrates the equivalence between KL-divergence similarity and spectrum similarity under the NNGP-context, where the covariance matrix is generated from an exponential GP with parameters $(\sigma^2, \phi, \tau) = (1, 3/\sqrt{2}, 0.01)$ and 1000 spatial location randomly selected from a $10\times 10$ square. The KL divergence between $\Eb$ and $\Ib$ decreases monotonically with $m$. At the same time, the spectrum of $\Eb$ tightens, converging to $[1, 1]$. NNGP with neighbor sets size of $20$ or larger generally yields very low KLD and spectral interval very close to $[1,1]$. Using the OLS loss is equivalent to NNGP with $m=0$ neighbors which leads to very large KLD and very wide spectral interval which in turn leads to large error rates according to Theorem \ref{thm-rate}. 
}

\subsection{Guidance on choice of number of nearest neighbors}\label{sec:choosem}

\cyan{
The number of nearest neighbors $m$ in the NNGP working covariance matrix used in NN-GLS is a tuning parameter that carries an accuracy-computation tradeoff. We see from Theorem \ref{thm-rate} that the error rates of NNGLS depend on the distance between the the NNGP precision matrix $\Qb$ and the true data precision matrix $\bSigma^{-1}$. Proposition \ref{prop-E} and Figure \ref{fig-E} shows that this distance decreases with increasing $m$. However, increasing $m$ also leads to increase in computation, and beyond a certain point, will lead to diminishing returns (negligible increase in accuracy at the expense of large computation times). 

In all of our experiments for this manuscript, we chose $m$ to be $20$. It empirically works well as seen in all the prediction and estimation performance evaluations, while providing reasonable run times (less than 1 hour for $n=200,000$). We also show in Figure \ref{fig-E} in Section \ref{Append-E} that for a sample size of $n=1000$, when $m = 20$, the KL divergence between $\Eb$ and $\Ib$ almost converges to $0$, and the max-min eigen ratio, which determines the theoretical error bounds, is very close to the optimal value of one. Although the results of Figure \ref{fig-E} may not generalize to other sample sizes, spatial design, or strength and nature of spatial dependence,  extensive empirical results in \citep{datta2016hierarchical} 
and in the many subsequent applications of NNGP in the literature attest to a choice of $m=15$ or $20$ offering excellent approximation to the full GP. 
This ability of NNGP to approximate inference from full GP using very small $m$ is one of the reasons for its success and popularity in the last decade. This also guided our choice of $m=20$ as the default for NN-GLS.  

In scenarios, where there is a possibility of very long-range spatial correlation, a larger $m$ might be warranted. If computing resources permit, one can run NN-GLS (possibly, in parallel) for a set of increasing values of $m$, stopping when there is no noticeable improvement in accuracy (of predictions on a test set). Alternatively, such long-range spatial structures can also be modeled through the mean. A common practice in spatial linear models in such cases is to add lat-lon or lower-order polynomials of them as covariates while letting the GP to model the finer scale spatial dependence. A similar strategy can be easily used with NN-GLS.}

\pagebreak
\section{Additional Simulation Results}\label{Append-sim}
\subsection{Implementation details of data simulation and competing methods}\label{sec:sim_methods}

\subsubsection{Data generation process}\label{sec:dgp}
\blue{For generating data for the simulation experiments,} the error process is generated as the sum of a GP with the exponential covariance function, and a nugget process. 
  We consider \blue{several} choices for each spatial parameter --- variance $\sigma^2$, spatial decay $\phi$, and error variance (nuggets) ratio $\tau^2/\sigma^2$, i.e. \blue{$\sigma^2  = 1, 5$}, $\phi \in \{1/\sqrt{2}, 3/\sqrt{2}, 6/\sqrt{2}\}$ and $\tau^2/\sigma^2 \in \{0.01, 0.1, 0.25\}$, thereby providing \blue{$18$} combinations in total (all settings could be found in Section \ref{Append-sim}). For \blue{these} settings, we keep the training and testing sample size $1000$ and perform $100$ independent experiments, including data generation, model fitting, and evaluation. \blue{Larger sample sizes are considered in Section \ref{sec-sim-large}.} The covariates $\bX$ and coordinates $s$ are independently sampled from $\text{Unif}\big([0,1]^d\big)$ and $\text{Unif}\big([0,10]^2\big)$. \blue{Other scenarios consider various forms of misspecification, and are presented in Sections \ref{sec-sim-mis1} and \ref{sec-sim-mis2}.}

\blue{For the linear spatial model, we use the {\em BRISC} R package developed by \cite{saha2018brisc}. For a spatial linear GP model, given the $(\bX, \bY, \bs)$ as the input, BRISC returns the MLE of the linear regression coefficients as well as the spatial parameters. Thus, we also use BRISC to do the residual kriging for the estimation of spatial covariance parameters in NN-GLS using (12). For the random forests family approaches, given the target sample size, we tune the number of ensemble trees by comparing the estimation error on an independently generated sample and making $60$ the final choice. The maximum size of leaf nodes is set to be $20$, aligned with the setting in the RF-GLS paper \cite{saha2023random}. For methods in the GAM family implemented through the {\em pyGAM} python package, each covariate will be expanded to $20$ B-splines and apply a penalized regression. GAM-latlon simply includes spatial coordinates as the additional covariates, and GAM-GLS \citep{nandy2017additive} is implemented by decorrelating both the covariates matrix $\bX$ and response $\bY$ using the Cholesky factor of the estimated precision matrix. 

For the methods in the NN family, we build the framework using {\em PyTorch} \citep{paszke2019pytorch}. NN-latlon simply added normalized spatial coordinates to the covariates, and NN-splines \citep{chen2020deepkriging} expanded the coordinates into $56$ B-spline coefficients. For the NN-GLS algorithm, we choose the neighbor size to be $m=20$. A similar neighbor set size was recommended in \cite{datta2016hierarchical} who demonstrated diminishing returns for increasing $m$ further. Also, this choice works well empirically by achieving a close enough approximation between $\Qb$ and $\mathbf{\Sigma}^{-1}$ (see Section \ref{Append-E}). However, if desired one can also choose $m$ by cross-validation on hold-out data. 
For all the methods here, we use the same NN architecture of one single layer, 50 hidden units, and a sigmoid activation function. Empirically, this architecture works well for both NN and NN-GLS in most of the settings. In terms of training details, for a sample size of $2000$, the training-validation-testing split is $40\% - 10\% - 50\%$. For the estimation and prediction experiment, we use a batch size of $50$ for mini-batching and a learning rate of $0.1$ for the {\em ADAM} \citep{kingma2014adam} optimization algorithm. The training is terminated after $20$ epochs of non-decreasing validation loss. The experiments are conducted with Intel Xeon CPU, and 8 GB RAM (The running time experiment takes up to 100 GB). We notice the optimal \blue{number of batches} for NN-GLS is relatively invariant against the sample size, while NN (together with NN-latlon and NN-spline) requires a larger number of batches to achieve the best performance. Even though both methods have linear running times with a fixed batch number, the previous observation makes the NN's running time increase at a higher rate.

The following table summarizes the features of each of the methods considered. }

\begin{table}[h]
  \centering
  \blue{
  \begin{tabular}{|c|c|c|c|}
    \hline
    Methods & Estimation of mean & Spatial predictions & Large sample scalability \\
    \hline
    NN-GLS &  Yes & Yes & Yes \\
    \hline
    NN-non-spatial\tablefootnote{\blue{In the experients, we keep the non-spatial predictions from NN-non-spatial in the comparisons as a baseline for performance of NN when completely ignoring the spatial information.}}  & Yes & No & Yes \\
    \hline
    NN-latlon &  No & Yes & Yes \\
    \hline
    NN-splines &  No & Yes & Yes \\
    \hline
    GAM  & Yes (no interaction) & No & Yes \\
    \hline
    GAM-GLS\tablefootnote{\blue{The GAM-GLS, while similar in spirit to RF-GLS and NN-GLS, is presented in \cite{nandy2017additive} as a standalone method for mean function estimation and do not offer spatial predictions. Conceptually, it can be extended to offer predictions by embedding the GAM in a GP model, where predictions will leverage kriging. Similarly, using NNGP covariance instead of full GP will improve scalability of GAM-GLS but it is not currently implemented in the method presented by \cite{nandy2017additive}.}} &  Yes (no interaction) & No & No \\
    \hline
    GAM-latlon &  No & Yes & Yes \\
    \hline
    RF & Yes & No & Yes \\
    \hline
    RF-GLS &  Yes & Yes & No \\
    \hline
    Linear-spatial &  Yes (linear) & Yes & Yes \\
    \hline
  \end{tabular}}
  \caption{\blue{Different methods considered for the numerical experiments.}}
  \label{tab:methods}
\end{table}

\subsection{Performance comparison for Friedman and 1-dimensional functions}\label{sec-sim-5}
This subsection follows up the demonstration in Section \ref{sec-sim}, presenting estimation and prediction comparisons for all parameter scenarios. 

\begin{figure}[!h]
\centering
\begin{subfigure}{1\textwidth}
  \centering
  \includegraphics[width=0.9\linewidth]{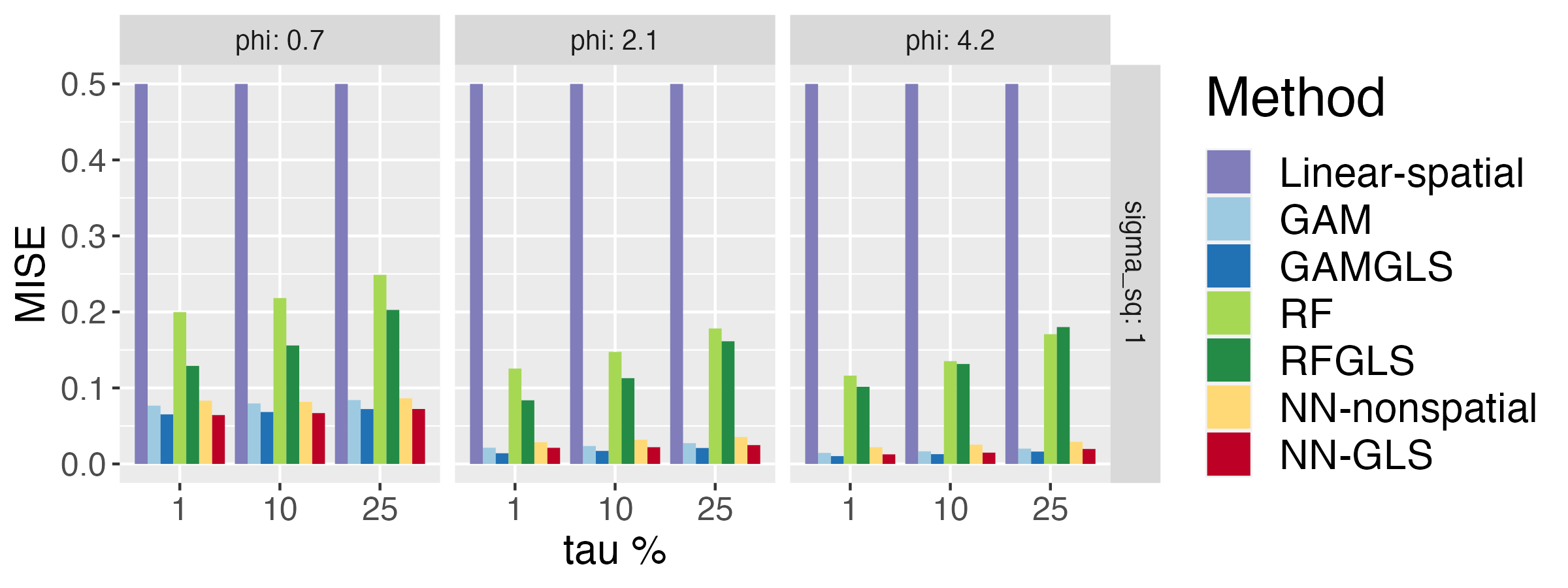}
  \caption{Estimation performance for $f_1$*}
\end{subfigure}%
\\
\centering
\begin{subfigure}{1\textwidth}
  \centering
  \includegraphics[width=0.9\linewidth]{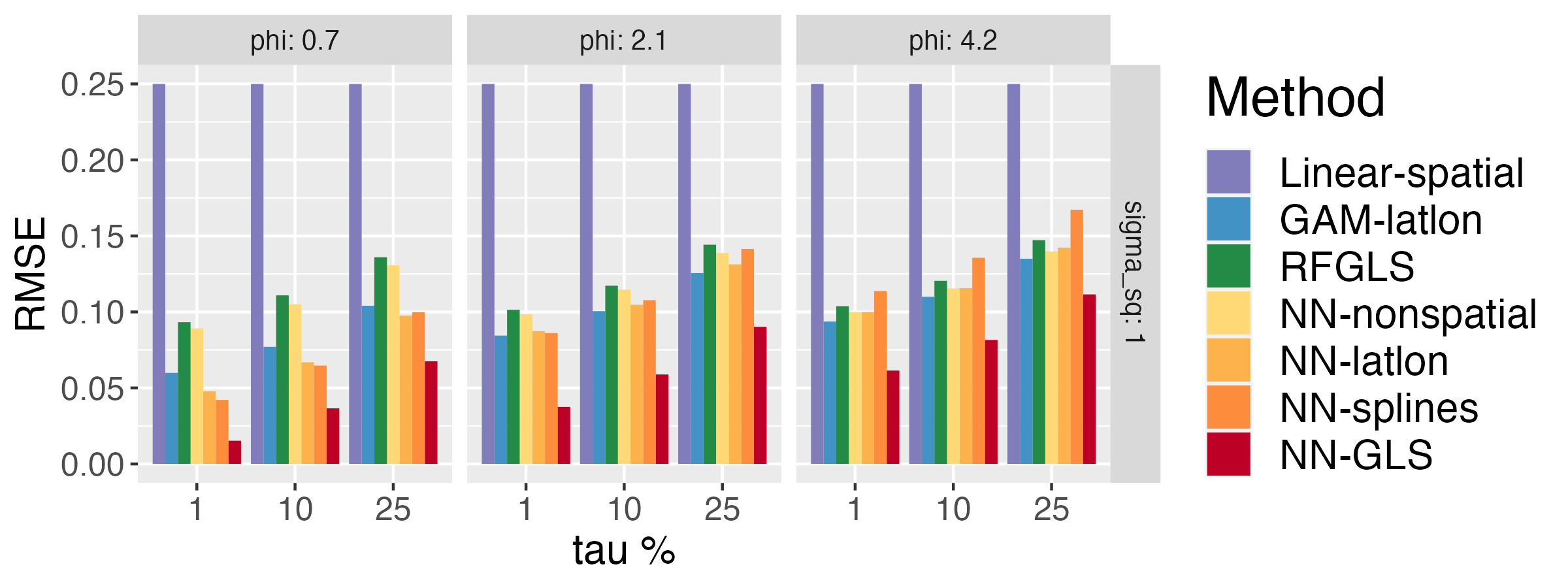}
  \caption{Prediction performance for $f_1$*}
\end{subfigure}
\caption{\blue{(Section \ref{sec-sim-5}) Comparison between competing methods on (a) estimation and (b) spatial prediction
when the mean function is $f_0 = f_1$.} \blue{We add * in the caption as the MISE and RMSE for the linear-spatial model (which were very large) had to be truncated for better illustration of the performance of the other methods.}}
\label{fig-sim-2}
\end{figure}
In Figure \ref{fig-sim-2}, we compare the estimation and prediction errors with different spatial parameter combinations for the true mean function $f_1(x)$ which is simply a one-dimensional sinusoidal function. \blue{NN-GLS outperforms the other competing methods for both estimation and prediction in almost all the scenarios. The most notable aspect of the results is the very large errors for the linear model. This is expected as the true mean function is highly non-linear (sinusoidal).}

\begin{figure}[htbp]
\centering
\begin{subfigure}{1\textwidth}
  \centering
  \includegraphics[width=0.9\linewidth]{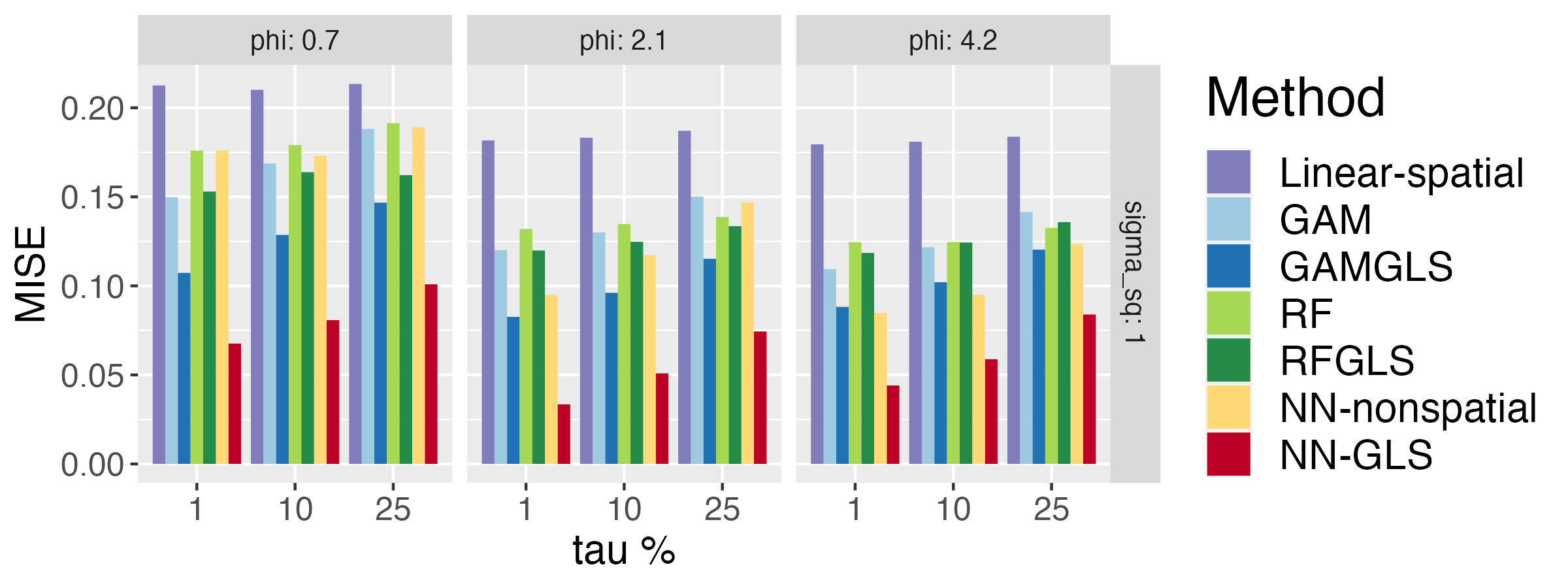}
  \caption{Estimation performance}
\end{subfigure}%
\\
\centering
\begin{subfigure}{1\textwidth}
  \centering
  \includegraphics[width=0.9\linewidth]{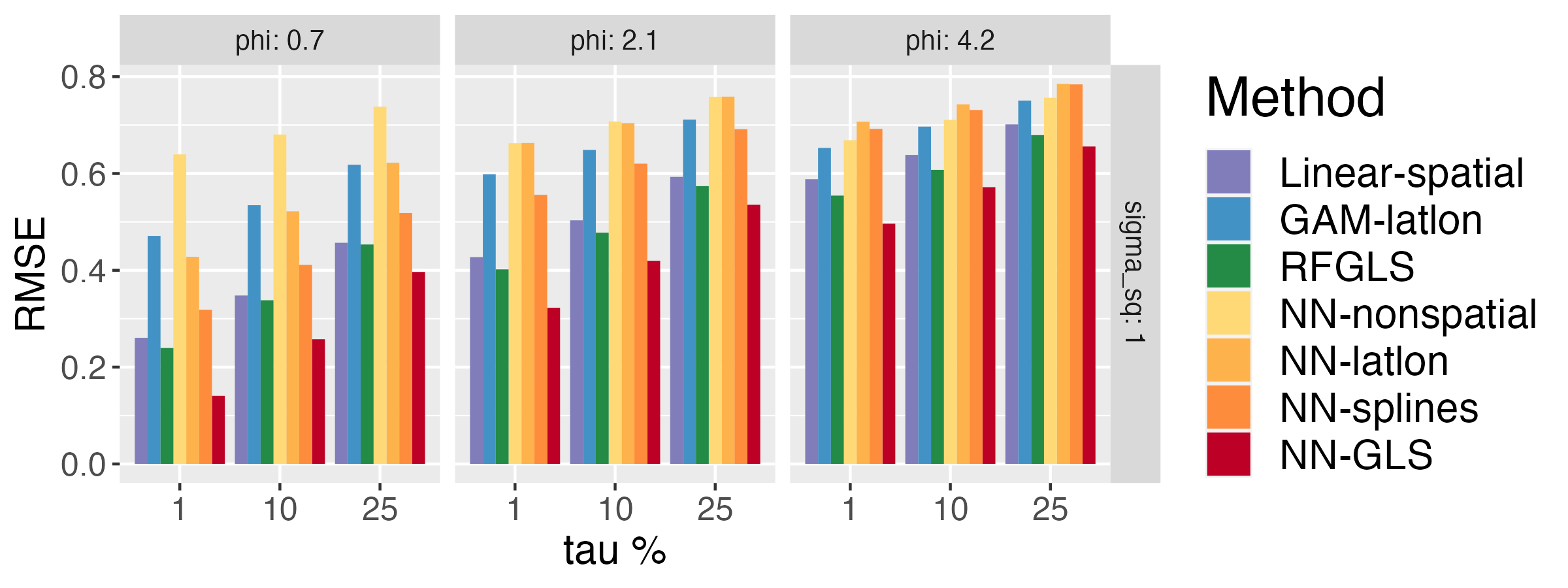}
  \caption{Prediction performance}
\end{subfigure}
\caption{\blue{(Section \ref{sec-sim-5}) Comparison between competing methods on (a) estimation and (b) spatial prediction
when the mean function is $f_0 = f_2$.}}
\label{fig-sim-1}
\end{figure}
The results for $f_0=f_2$ are presented in Figure \ref{fig-sim-1}. \blue{NN-GLS once again uniformly outperforms all other methods.} For the estimation, as we expected, MISE for the methods increases with an increasing $\tau^2$, which represent larger overall variations. \blue{GAM and GAM-GLS perform worse in this case than when $f_0=f_1$. This is expected due to the existence of the interaction term in $f_2$ that cannot be modeled using GAMs (see Section \ref{sec-sim-friedman} for more on this).} 

\begin{figure}[htbp]
\centering
\includegraphics[width=0.7\linewidth]{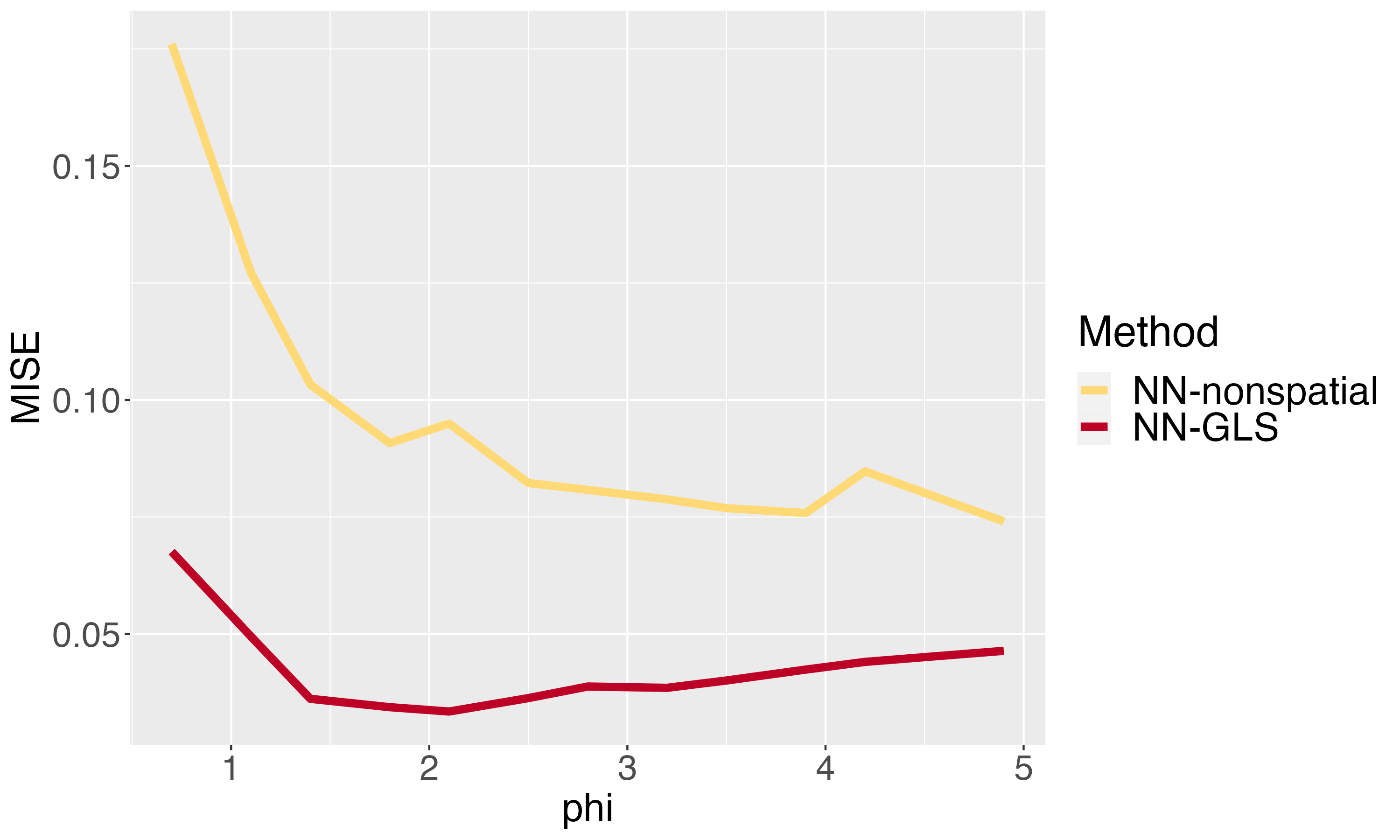}
\caption{\blue{Estimation performance against $\phi$
when the mean function is $f_0 = f_2$.}}
\label{fig-sim-5-phi}
\end{figure}

\blue{In Figure \ref{fig-sim-5-phi}, we study how the improvement of NN-GLS over the non-spatial NN (in terms of MISE) relates to the strength of spatial correlation in the data. We fix $(\sigma^2, \tau^2) = (1, 0.01)$ and vary the spatial decay parameter $\phi$ from $1/\sqrt{2}$ to $7/\sqrt{2}$. We see that small values of $\phi$ (stronger spatial correlation) lead to more significant advantage of NN-GLS over NN. This is expected as NN ignoring the strong spatial correlation suffers in estimation performance. The difference between the two methods is less when $\phi$ is large (weaker spatial correlation), although NN-GLS still does considerably better.}

For the predictions, we see from Figures \ref{fig-sim-2} (b) and \ref{fig-sim-1}  (b) that NN-GLS is more favorable than other methods with smaller $\phi$ and $\tau^2$, i.e., relatively stronger spatial correlation and low noise. If both of these parameters are large, all the other methods become comparable to NN-GLS. However, in contrast to NN-GLS's stable performance, NN-latlon and NN-splines have \blue{much more severely} degrading performance when $\phi$ is large, \blue{producing worse prediction performance than even the non-spatial NN. Also, when $f_0=f_2$, the linear-spatial model outperforms all methods except NN-GLS in terms of prediction performance (Figure \ref{fig-sim-1} (b)). The reason is the partial-linearity of Friedman function and the power of the GP in the linear model, NN-GLS offers the lowest RMSE. In fact, all the GP-based approaches (spatial linear model, RF-GLS, NN-GLS) outperform the added-spatial-features approaches (GAM-latlon,NN-latlon,NN-splines). These results show the shortcomings of the {\em added-spatial-features} approach.} 
If the spatial correlation is weak, these \blue{large number of} added spatial features become redundant and adversely affect the performance of these methods.  



\subsection{Performance comparison under large noise}\label{sec-sim-noise}
\blue{In the previous section, the scale of the spatial effect used was relatively weak ($\sigma^2=1$) compared with the true signal. Under a large variance of the spatial effect, together with a strong spatial correlation, the vector of errors can become approximately like a non-zero intercept term. This intercept is not separable from the true underlying function. Overfitting happens easily under this setting and the mean function can then be only identified up to an intercept shift. 
To illustrate the robustness of NN-GLS in this scenario, we show both the estimation and prediction performance for $f_1$ and $f_2$ with $\sigma^2 = 5$ in Figures \ref{fig-sim-large-noise-dim1} and \ref{fig-sim-large-noise-dim5}.
\begin{figure}[htbp]
\centering
\begin{subfigure}{1\textwidth}
  \centering
  \includegraphics[width=0.9\linewidth]{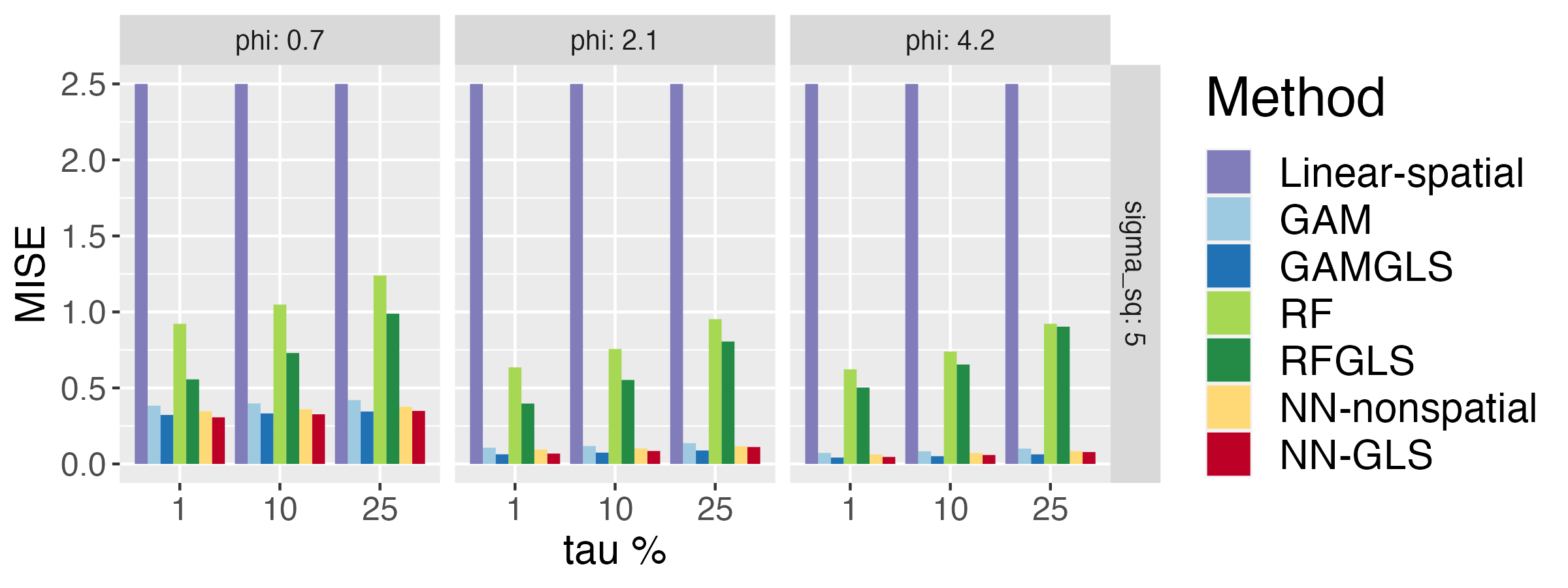}
  \caption{Estimation performance for $f_1$*}
\end{subfigure}%
\\
\centering
\begin{subfigure}{1\textwidth}
  \centering
  \includegraphics[width=0.9\linewidth]{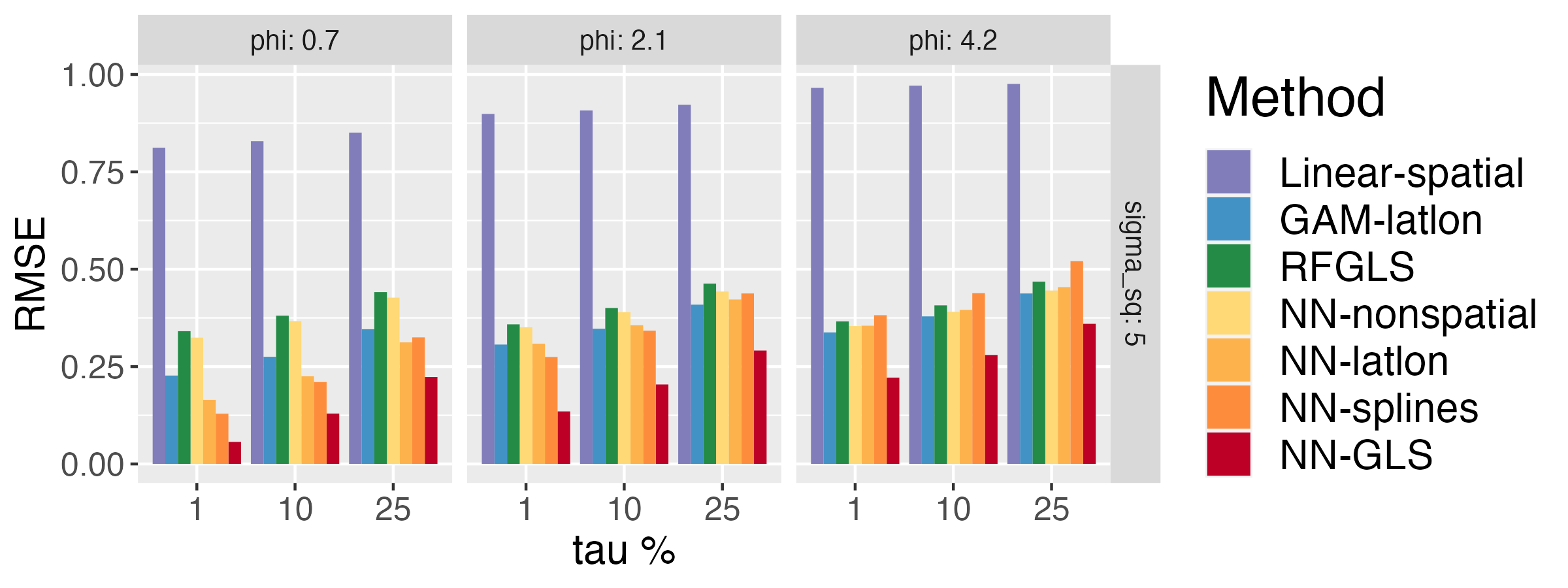}
  \caption{Prediction performance for $f_1$}
\end{subfigure}
\caption{\blue{(Section \ref{sec-sim-noise}) Comparison between competing methods on (a) estimation and (b) spatial prediction
when the mean function is $f_0 = f_1$ under large noise. \blue{We add a * in figure (a) as the MISE for the linear-spatial model (which was very large) had to be truncated for better illustration of the performance of the other methods.}}}
\label{fig-sim-large-noise-dim1}
\end{figure}
\begin{figure}[!t]
\centering
\begin{subfigure}{1\textwidth}
  \centering
  \includegraphics[width=0.9\linewidth]{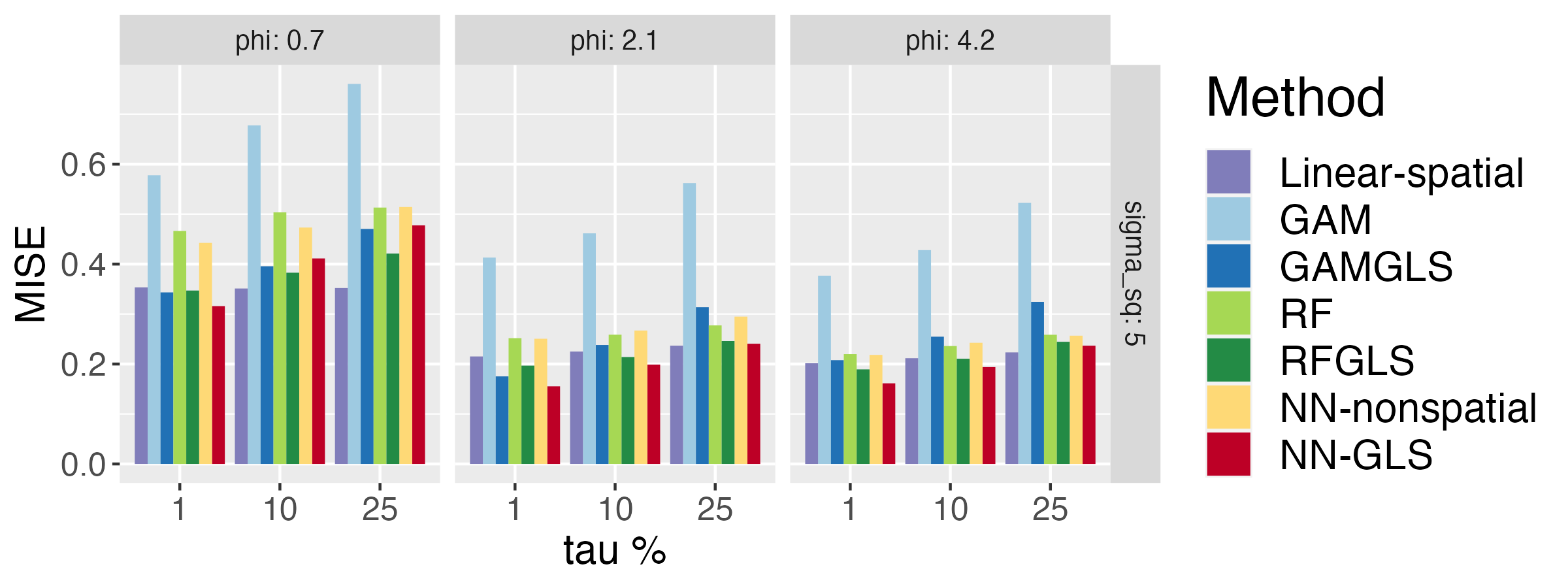}
  \caption{Estimation performance for $f_2$}
\end{subfigure}%
\\
\centering
\begin{subfigure}{1\textwidth}
  \centering
  \includegraphics[width=0.9\linewidth]{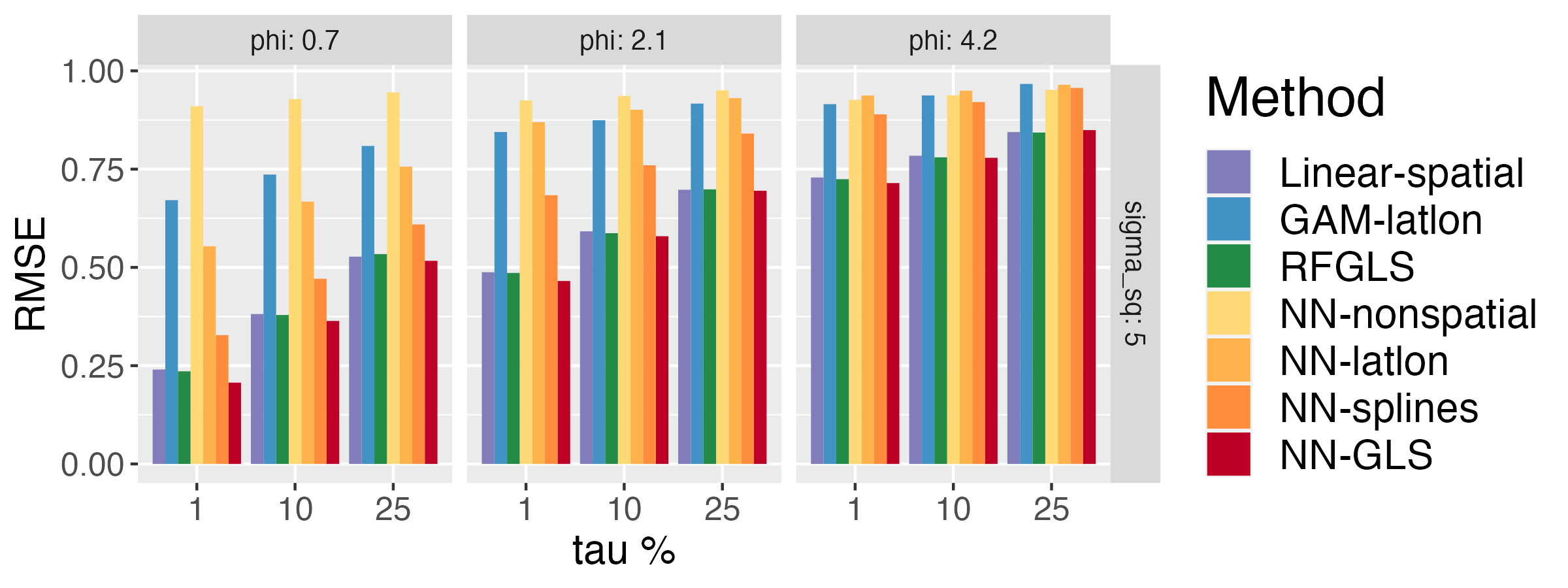}
  \caption{Prediction performance for $f_2$}
\end{subfigure}
\caption{\blue{(Section \ref{sec-sim-noise}) Comparison between competing methods on (a) estimation and (b) spatial prediction
when the mean function is $f_0 = f_2$ under large noise.}}
\label{fig-sim-large-noise-dim5}
\end{figure}

With a large noise, NN-GLS still performs well consistently in both estimation and prediction tasks. For the prediction part, the advantage of NN-GLS is consistent across all the settings. The advantage becomes more significant when the spatial correlation gets stronger (small $\phi$). For estimation, when $f_0 = f_1$, NN-GLS, together with GAM-GLS, have the best performance. Methods in the RF family struggle here due to overfitting, as the sample size of 1000 is not sufficient here due to the large noise. In other words, for these learning approaches, the large noise is equivalent to the insufficient sample size issue, which means that the error here is not merely a function family's approximation error but contains a substantial amount of estimation error. 

When $f_0 = f_2$, estimation from the linear-spatial model performs comparably well since, relative to the large error variance, the non-linear variations of the mean (Friedman function) in the $[0, 1]^5$ hyperunit square are small. When the error variance was smaller, we saw in Figure \ref{fig-sim-1}(a) that the linear model estimation suffers even for the Friedman function. 

In the following sections, both the normal case ($\sigma^2 = 1$) and the large noise case ($\sigma^2 = 5$) will be presented together for ease of comparisons.
}

\subsection{Performance comparison between GAM and NN}\label{sec-sim-friedman}
\blue{In the above sections, an observation is that GAM and GAM-GLS tend to have similar performance respectively with NN-nonspatial and NN-GLS, when the truth $f_0 = f_1$, the 1-dimensional sine function, while the GAM methods perform worse for the Friedman function $f_0=f_2$. In this subsection, we investigate this in more detail. We utilize the interaction term in the Friedman function to illustrate the robustness of NN-based methods against GAM-based methods. We design a family of `modified' Friedman functions parametrized by $\rho$ as:
\[
f_2(\cdot, \rho) = \rho\cdot \frac{10}{3}\sin(\pi x_1x_2)+ (1-\rho)\cdot\frac{1}{3}(20(x_3 - 0.5)^2+10x_4+5x_5),
\]
where $\rho \in [0,1]$ measures the weight of the interaction term. When $\rho = \frac{1}{2}$, $f_2(\cdot, \rho)$ is exactly the classical Friedman function $f_2$. 
\begin{figure}[!h]
\centering
\includegraphics[scale=0.4]{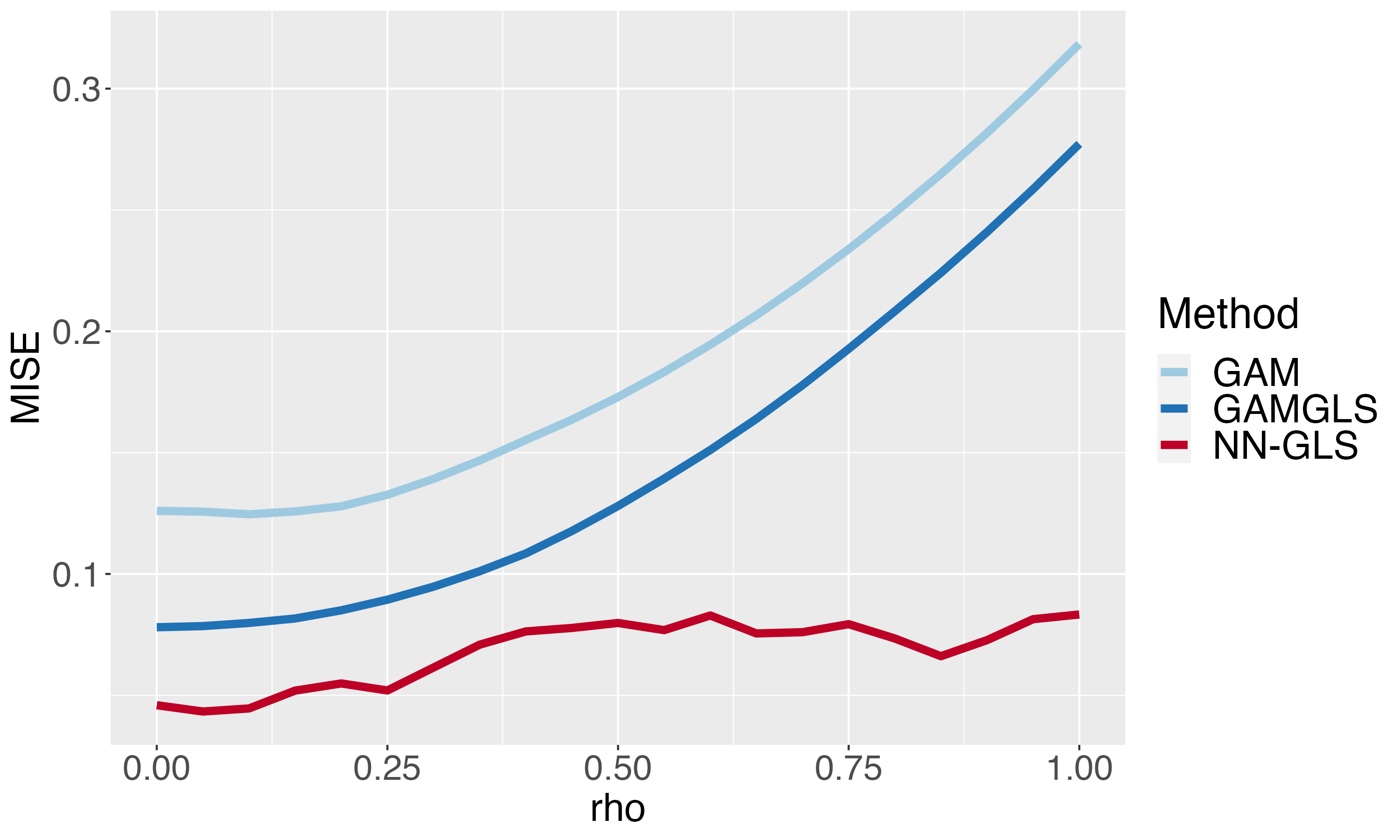}
\caption{(Section \ref{sec-sim-friedman}) MISE of competing methods vs the weight of the interaction term in the true function. This figure is the same as Figure \ref{fig-sim-main}(b)}
\centering
\label{fig-sim-friedman}
\end{figure}
Result in Figure \ref{fig-sim-friedman} is obtained from experiment with $\sigma^2, \phi, \tau^2 = (1, 3/\sqrt{2}, 0.01)$ and $n = 1000$. We see that the errors of both GAM and GAM-GLS  increase significantly with the increasing weight $\rho$ of the interaction term, while the trend for NN-GLS is much weaker. We use this example to show that NN models are more robust and generalizable compared with GAM models. Even though GAM models are simpler to formulate, NN-GLS will be preferred if there is no additional knowledge about the true mean function.}

\subsection{\blue{Confidence Interval}}\label{sec-sim-CI}
\blue{We now evaluate the confidence intervals (CI) produced by the procedure in Section \ref{sec-mtd-CI} under different settings. Among the other methods considered for the simulation studies that provide an estimate of the mean function (see Table \ref{tab:methods}), we can obtain confidence intervals for NN-non-spatial using na\"ive bootstrap. The other two spatial methods that offer a non-linear estimate of $\hat f$, i.e., RF-GLS \citep{saha2023random} and GAM-GLS \citep{nandy2017additive} do not offer any method for inference on the mean function. Hence, we only compare NN-GLS and NN-non-spatial for pointwise inference on the mean function.

For each simulation, we generate $100$ bootstrap samples and evaluate the confidence interval from $100$ independent simulations. The spatial parameters used are shown in the figure, i.e. $\sigma^2 \in \{1, 5\}$, $\phi \in \{1/\sqrt{2}, 3/\sqrt{2}, 6/\sqrt{2}\}$ and $\tau^2/\sigma^2 \in \{0.01, 0.25\}$. In order to avoid the identifiability issue, when simulating the sample, we center the spatial effects. In all subfigures, NN-GLS consistently offers CI's with around $95\%$ coverage while non-spatial NN's CI fails to achieve it. Along with the coverage, we also use interval score as an additional metric \citep{gneiting2007strictly}. For the true function $f_0(\cdot)$ and its upper(lower) level-$\alpha$ interval bound $f_u(\cdot)$($f_l(\cdot)$), the interval score at $x$ is defined as:
\begin{equation}\label{def-interval_score}
s(x) = f_u(x) - f_l(x) + \frac{2}{\alpha}\Big((f_l(x) - f_0(x))\mathbb{I}(f_l(x) > f_0(x)) + (f_0(x) - f_u(x))\mathbb{I}(f_u(x) < f_0(x))\Big).
\end{equation}
The interval score considers the interval width and penalizes any miss of the interval, making it a more comprehensive metric than simply coverage. A better interval shall have a lower interval score. 
\begin{figure}[!h]
\centering
\begin{subfigure}{.5\textwidth}
  \centering
  \includegraphics[width=0.9\linewidth]{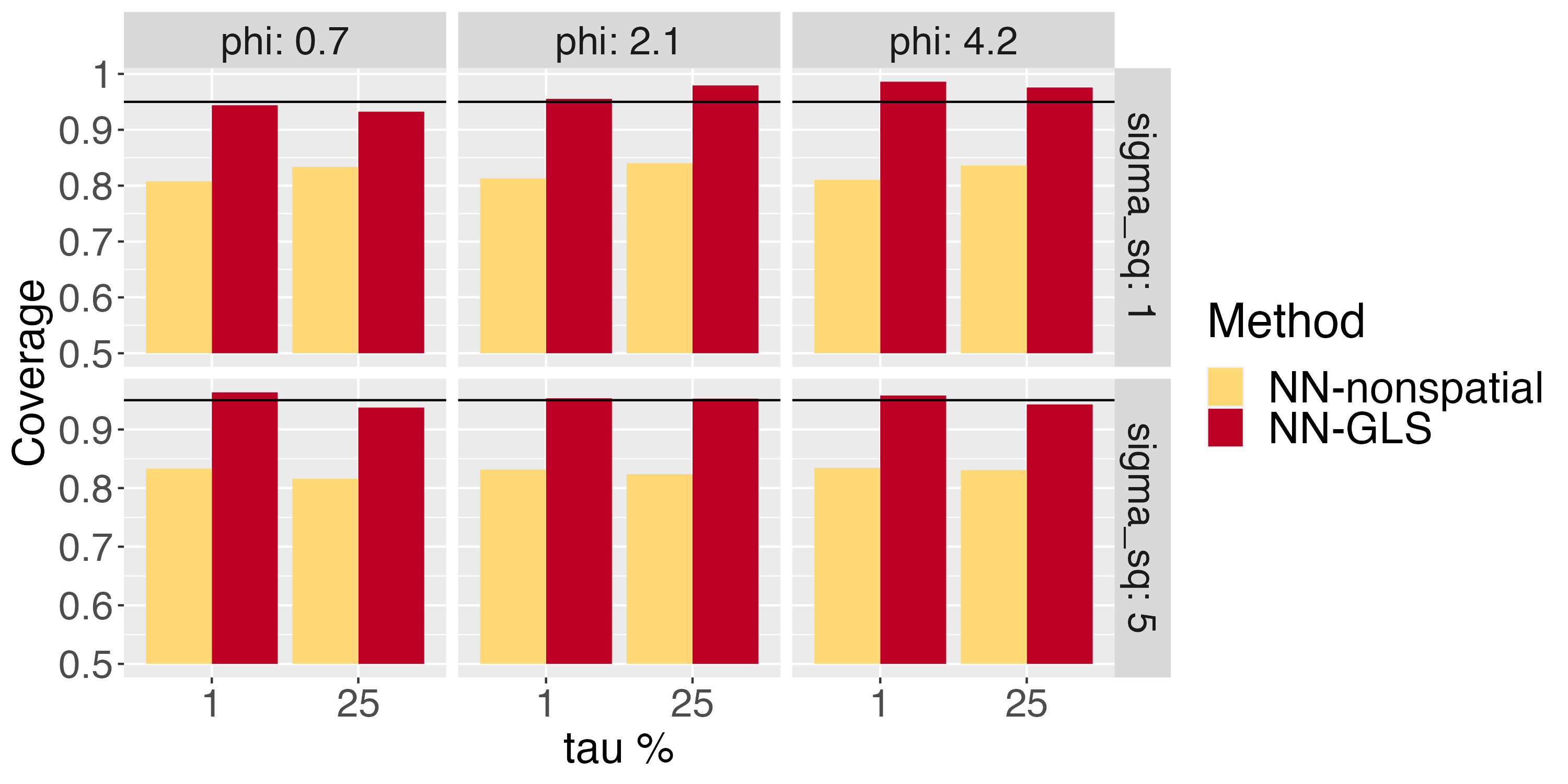}
  \caption{Confidence coverage for $f_0 = f_1$}
\end{subfigure}%
\begin{subfigure}{.5\textwidth}
  \centering
  \includegraphics[width=0.9\linewidth]{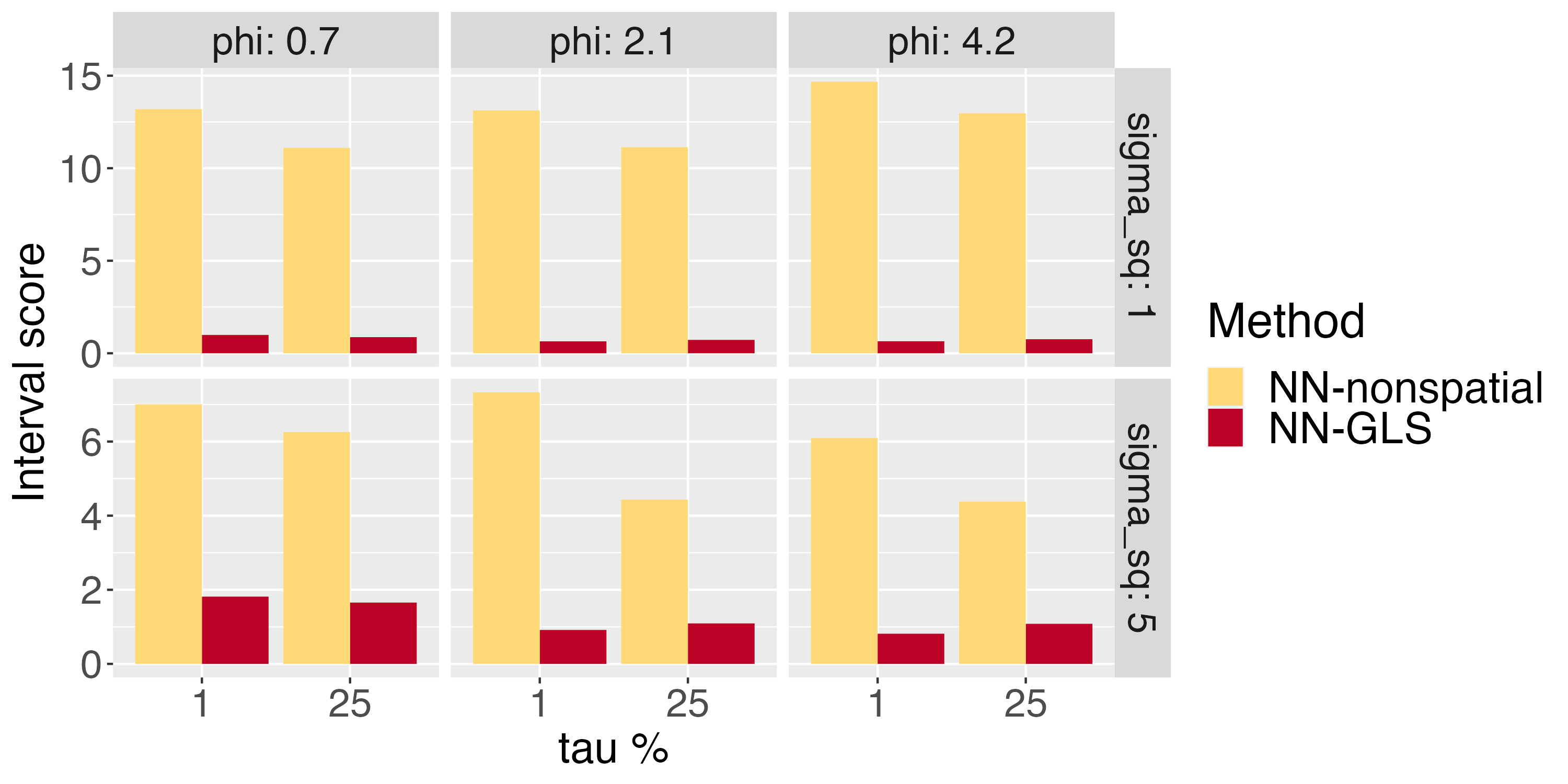}
  \caption{Interval score for $f_0 = f_1$}
\end{subfigure}
\\
\centering
\begin{subfigure}{.5\textwidth}
  \centering
  \includegraphics[width=0.9\linewidth]{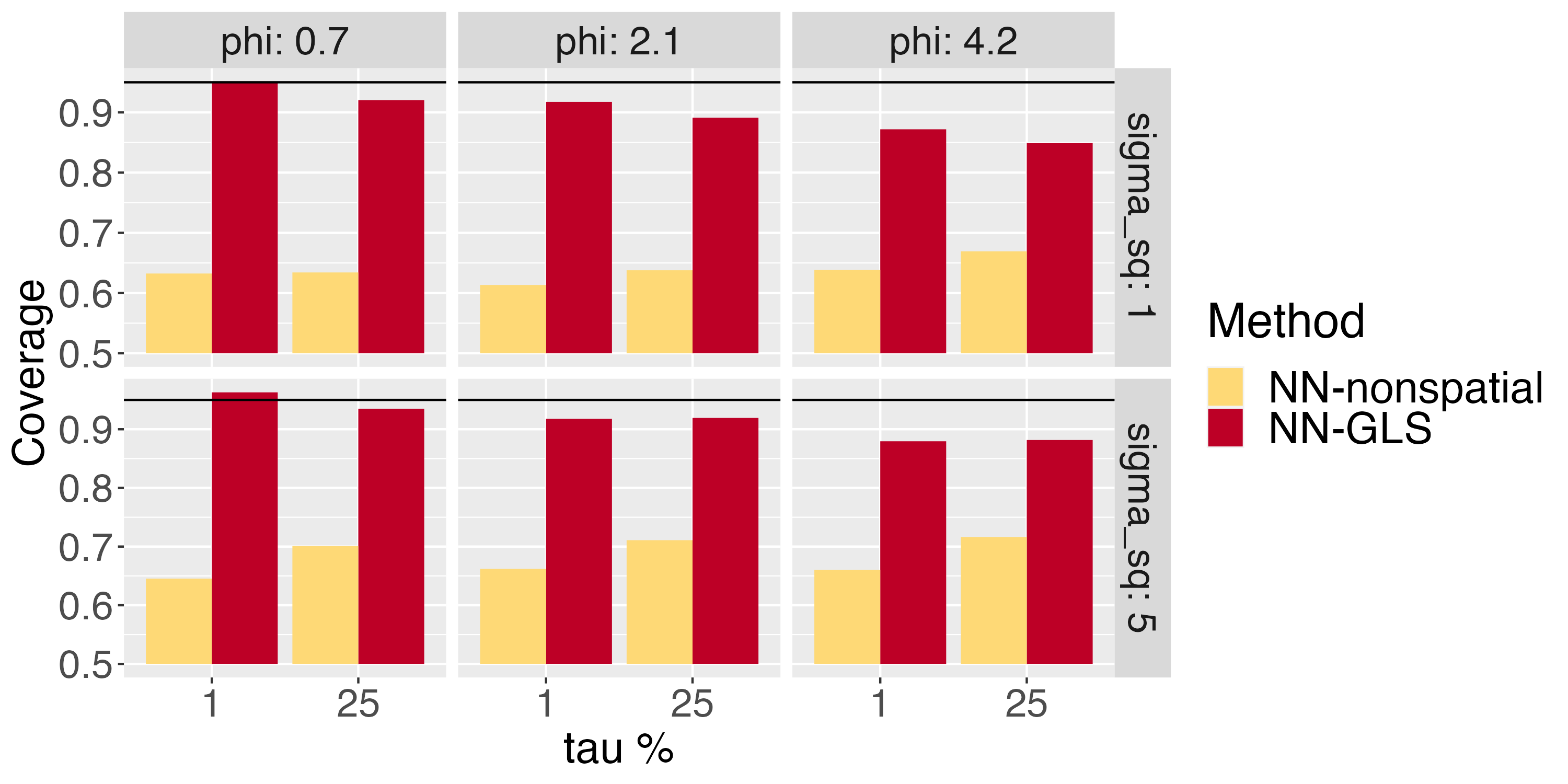}
  \caption{Confidence coverage for $f_0 = f_2$}
\end{subfigure}%
\begin{subfigure}{.5\textwidth}
  \centering
  \includegraphics[width=0.9\linewidth]{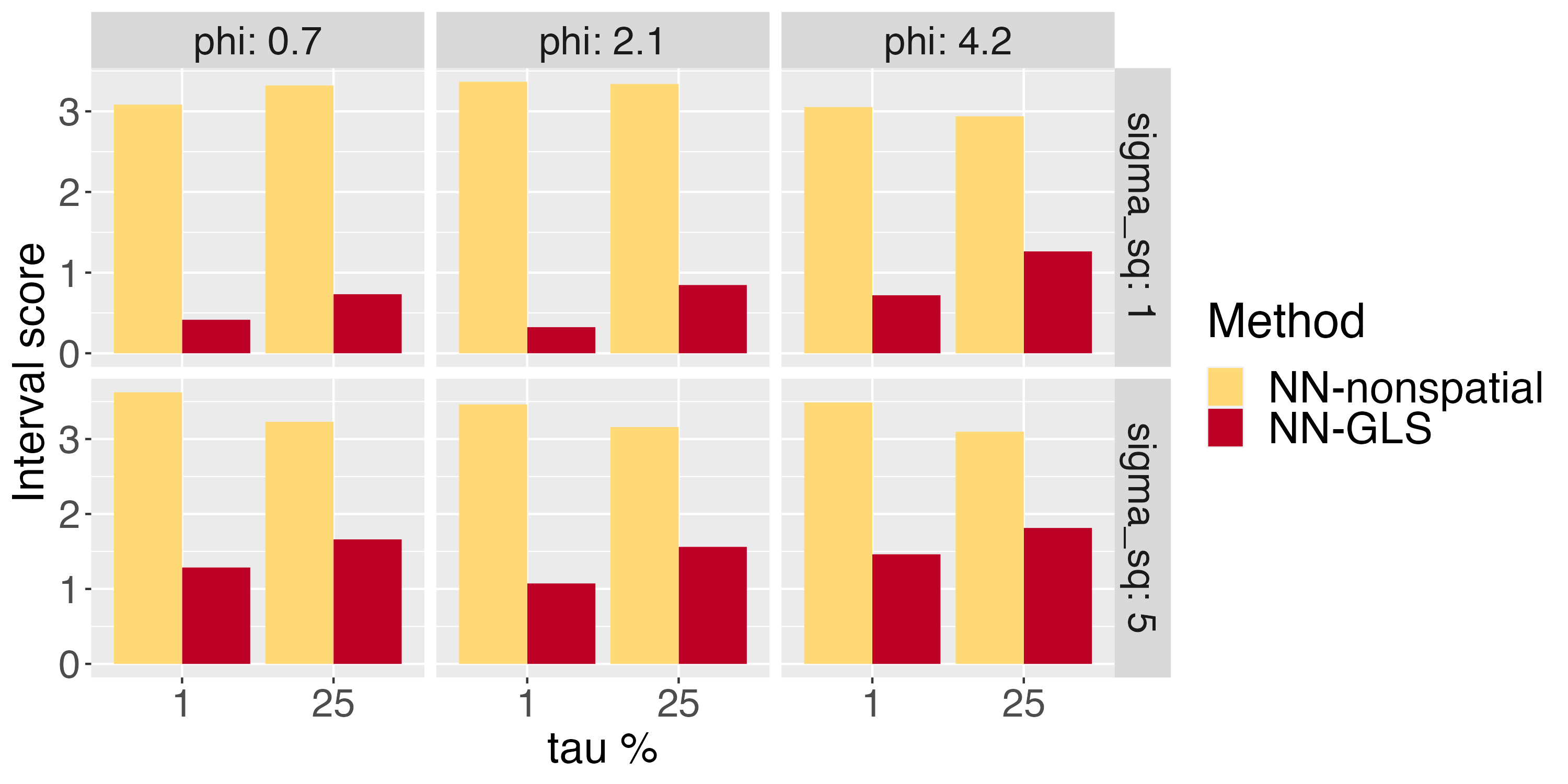}
  \caption{Interval score for $f_0 = f_2$}
\end{subfigure}
\caption{(Section \ref{sec-sim-CI}) Comparing confidence intervals produced by NN and NN-GLS 
with the mean function being $f_1$ or $f_2$. 
}
\label{fig-sim-CI}
\end{figure}
Figure \ref{fig-sim-CI} compares the results between these NN-GLS and the non-spatial-NN. For $f_0=f_1$ the coverages of the two methods are generally close to nominal coverage, 
but overall NN-GLS provides much tighter confidence intervals  reflected int he considerably better interval score. In the Friedman function's case ($f_0=f_2$), across all settings, NN-GLS's CIs have near closer to nominal coverage and much lower interval scores. The NN-non-spatial suffers from significant under-coverage leading to very poor interval scores. 

In general, NN-GLS is more accurate and its bootstrap estimations concentrate more closely on the truth by accounting for the spatial correlation, thus giving a narrower CI. }

\subsection{Prediction Interval}\label{sec-sim-PI}
\blue{NN-GLS algorithm not only provides point predictions at new locations but also computes prediction intervals using kriging variance (see Section \ref{sec-mtd-krig}). Ideally, if the estimation $\hat{f}(\cdot)$ and parameter estimation $\hat{\bm{\theta}}$ are both accurate enough and the spatial effect and noise distributions are correctly modeled, the coverage of the Gaussian prediction interval should be around $95\%$. To illustrate the performance of NN-GLS's prediction intervals, experiments are conducted under different settings (same as the ones for sections \ref{sec-sim-5} and \ref{sec-sim-noise}). For comparison of prediction intervals, we consider the spatial linear method and RF-GLS \citep{saha2023random}, since both of them involve similar steps of spatial parameter estimation and kriging prediction. The only difference among them is the $\hat{f}(\cdot)$, being in linear form, random forest class, or NN class respectively. 

\begin{figure}[!h]
\centering
\begin{subfigure}{.5\textwidth}
  \centering
  \includegraphics[width=0.9\linewidth]{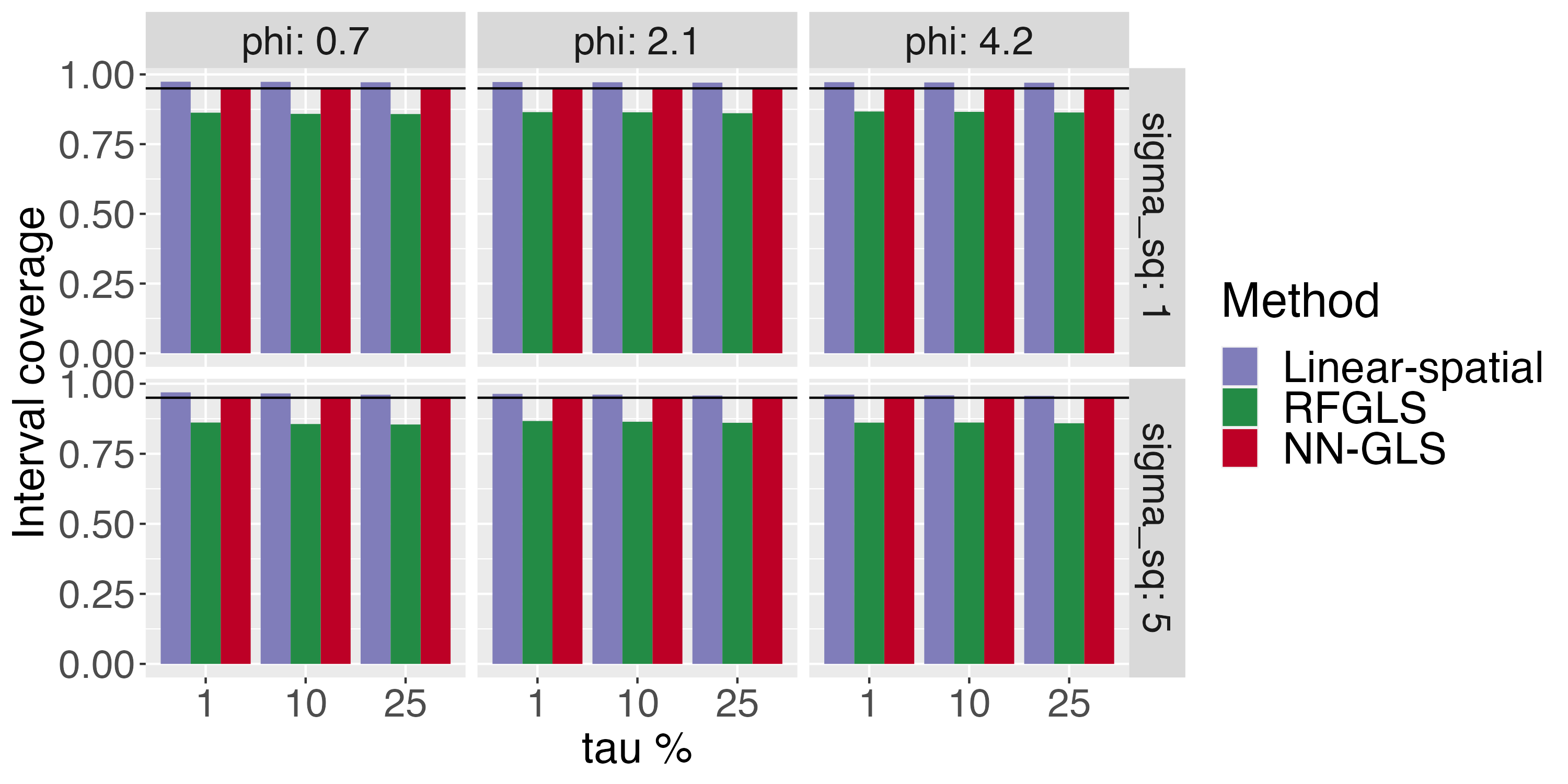}
  \caption{Prediction coverage for $f_0 = f_1$}
\end{subfigure}%
\begin{subfigure}{.5\textwidth}
  \centering
  \includegraphics[width=0.9\linewidth]{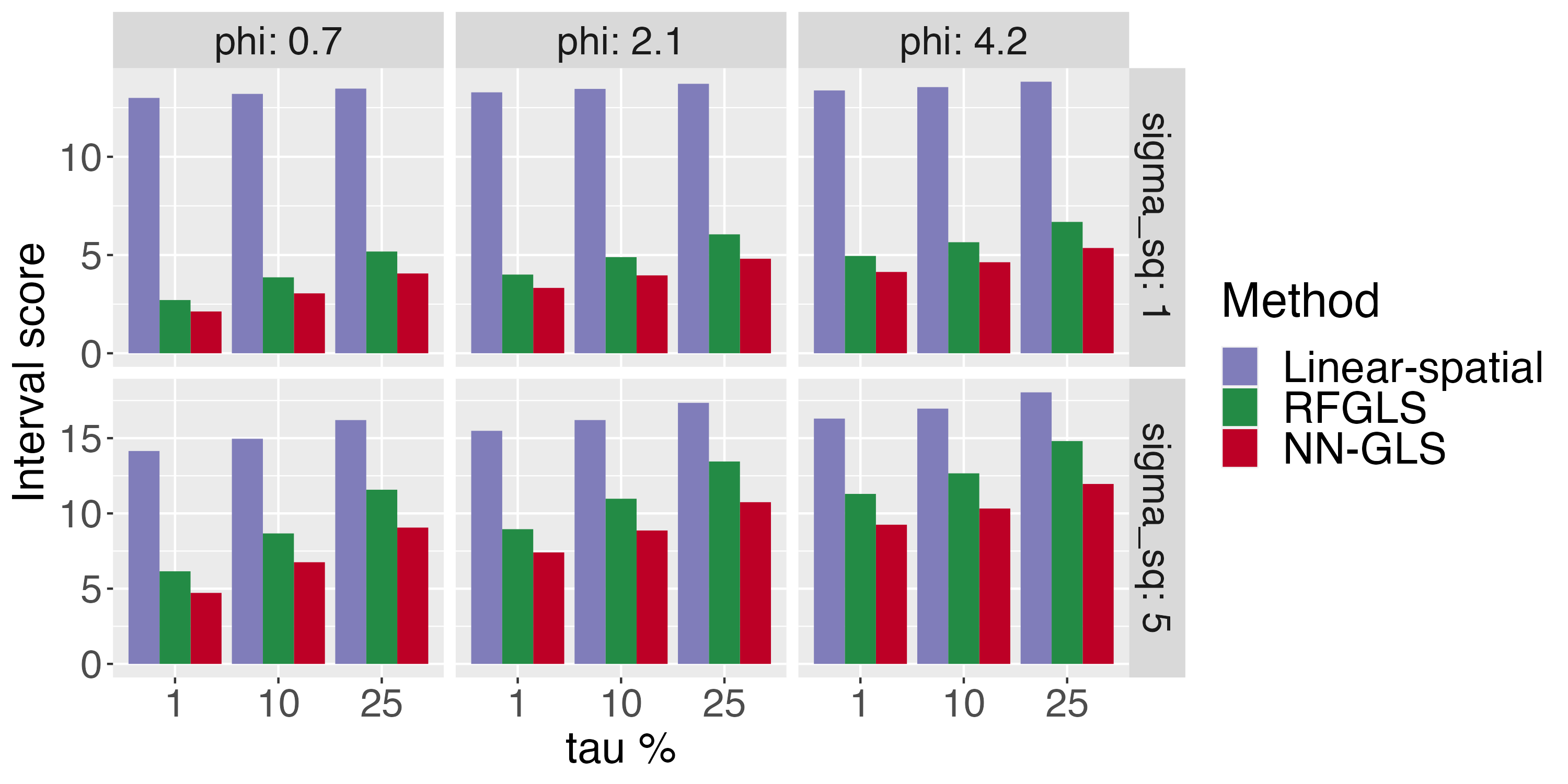}
  \caption{Interval score for $f_0 = f_1$}
\end{subfigure}
\\
\centering
\begin{subfigure}{.5\textwidth}
  \centering
  \includegraphics[width=0.9\linewidth]{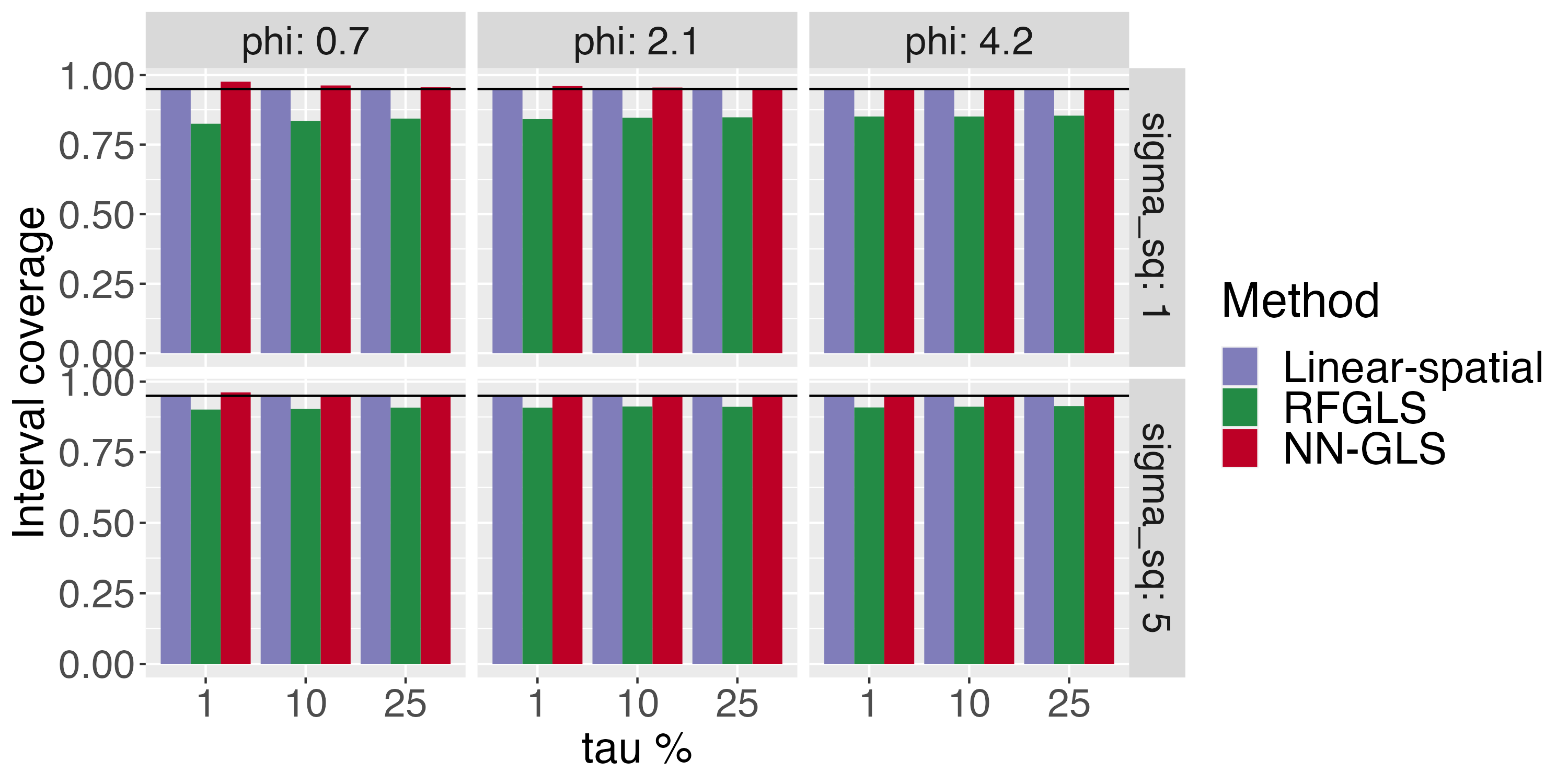}
  \caption{Prediction coverage for $f_0 = f_2$}
\end{subfigure}%
\begin{subfigure}{.5\textwidth}
  \centering
  \includegraphics[width=0.9\linewidth]{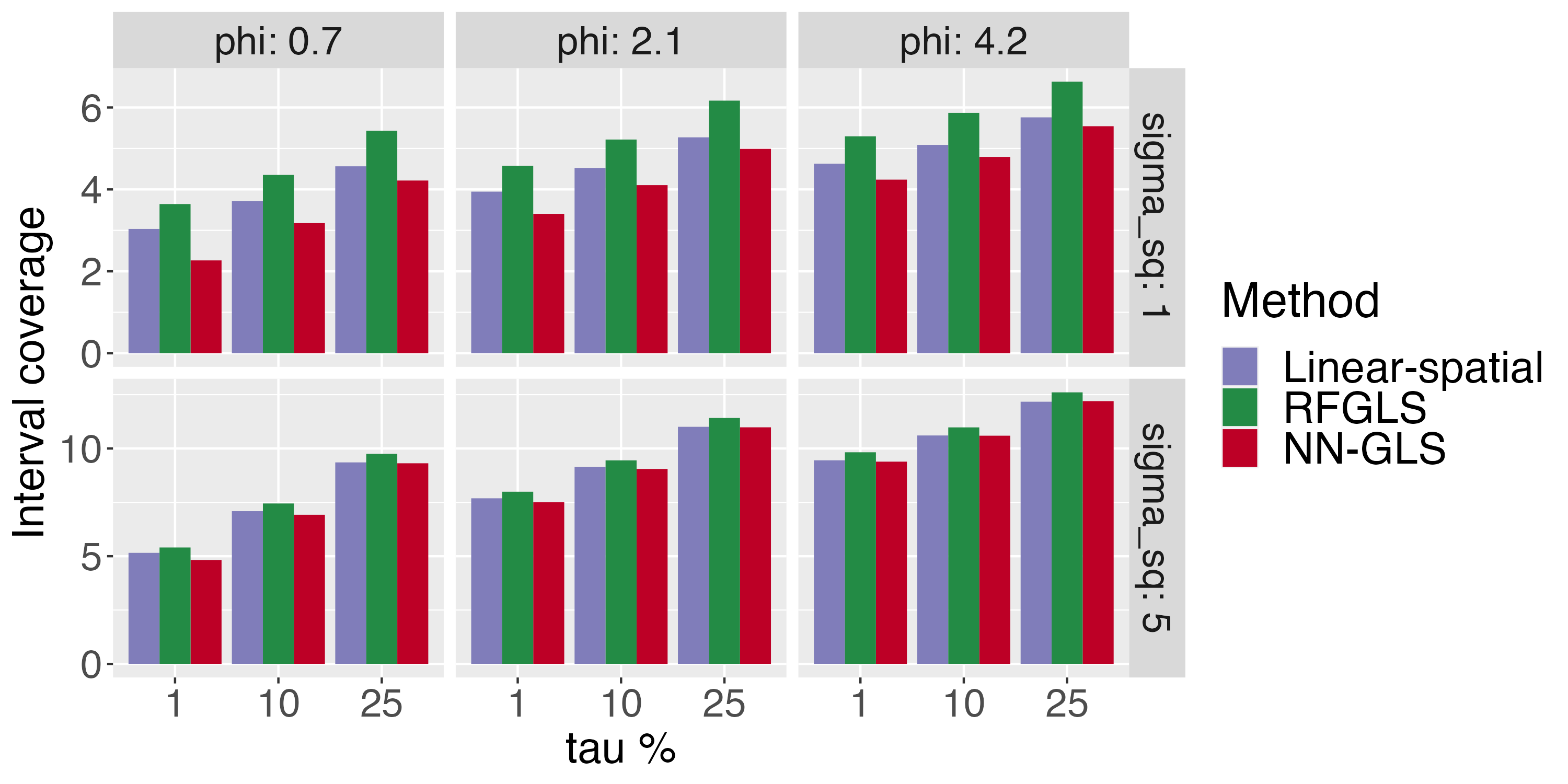}
  \caption{Interval score for $f_0 = f_2$}
\end{subfigure}
\caption{(Section \ref{sec-sim-PI}) Comparing prediction intervals produced by spatial linear prediction, and NN-GLS 
with the mean function being $f_1$ or $f_2$.}
\label{fig-sim-PI}
\end{figure}

In Figure \ref{fig-sim-PI}, among the three methods, RF-GLS consistently suffers from slight under-coverage. This is due to its estimation being worse compared to NN-GLS, as shown in Sections \ref{sec-sim-5} and \ref{sec-sim-noise}). The linear-spatial model and NN-GLS both approach a $95\%$ prediction coverage. However, when $f_0 = f_1$, which is extremely non-linear, the linear-spatial model sacrifices the accuracy, estimating a very large error variance, leading to unacceptably wide intervals and very poor interval scores (\ref{def-interval_score}) (recall that the interval score involves a width term). NN-GLS consistently has near nominal coverage and the best interval scores.} 

\subsection{Large sample performance}\label{sec-sim-large}
\blue{Theoretically, NN-GLS has linear running time by using NNGP covariance and the GNN framework (see Section \ref{sec:gnn}). In this section, we verify this empirically as well as compare different methods in terms of run times and performance as a function of sample sizes going up to $500,000$. 

We first present in Figure \ref{fig-time} the running times of different components of the NN-GLS algorithm against sample sizes in a log-log scale. The training sample size used in the experiments ranges from $100$ to $500000$ with spatial parameters being $(\sigma^2, \phi, \tau^2) = (1, 3/\sqrt{2}, 0.01)$ and the underlying function being the Friedman function. All experiments in this subsection are under an increasing-domain design.

\begin{figure}[!t]
\centering
\begin{subfigure}{.5\textwidth}
  \centering
  \includegraphics[width=0.85\linewidth]{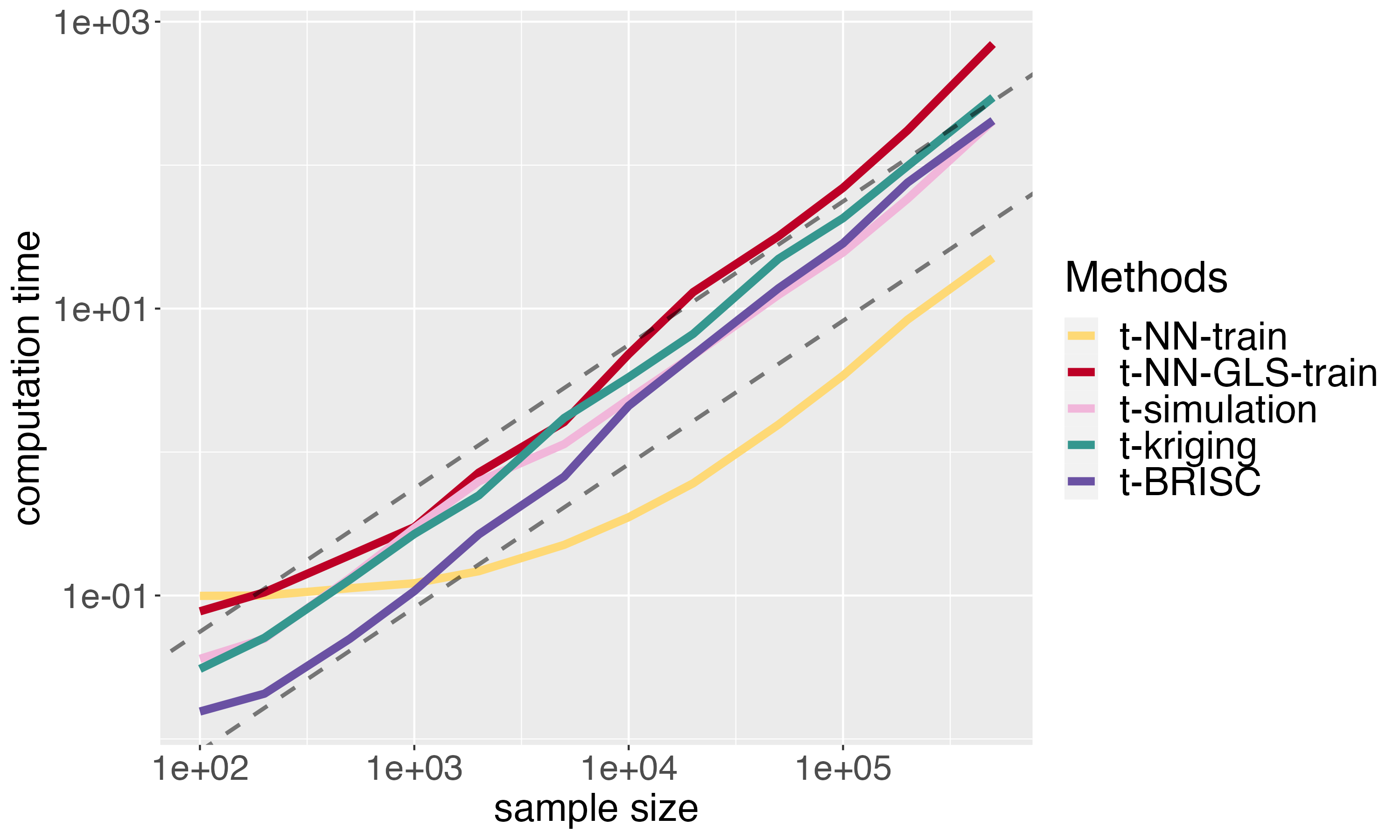}
  \caption{Running time for NN-GLS steps}
\end{subfigure}%
\begin{subfigure}{.5\textwidth}
  \centering
  \includegraphics[width=0.85\linewidth]{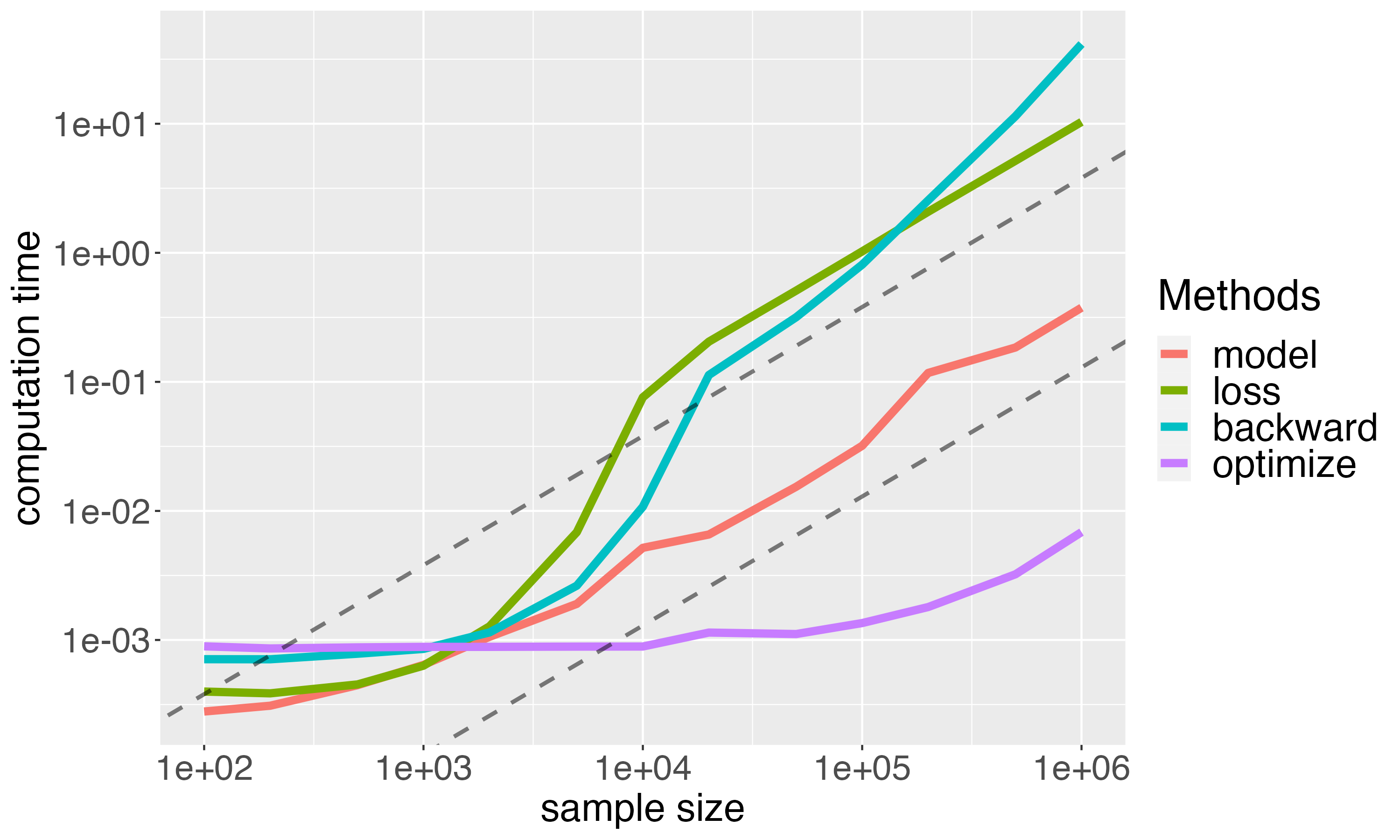}
  \caption{Running time for NN-GLS model training}
\end{subfigure}
\caption{(Section \ref{sec-sim-large}) Log running times for different components of the NN-GLS algorithm as functions of the log 
 sample size. The dashed auxiliary lines represent linear running time.}
\label{fig-time}
\end{figure}

The different components of the NN-GLS run times studied in Figure \ref{fig-time}(a) are: 'Simulation' represents the sample generation; 'NN' represents the step of initial non-spatial training of the NN parameters; 'BRISC' represents the spatial parameter estimation from the initial residual using the BRISC package; 'NN-GLS-train' represents the estimation of parameters in NN-GLS; 'Kriging' stands for the kriging predictions from NN-GLS. We see that run times for most components grow approximately linearly in the sample size $n$, except for the training component for extremely large sample sizes.
Figure \ref{fig-time}(b) breaks this down further, presenting the decomposition of one epoch in NN-GLS training, where 'loss' represents the computation of the GLS loss function; 'backward' stands for the back-propagation step; 'model' represents an evaluation of model's value $\hat{f}(X_{\text{train}})$; and 'optimize' represents a gradient-descent step. We see that the non-linearity comes primarily from the back-propagation step, which is inherently implemented by the Pytorch module and beyond our control as we only specify the GLS loss function and computation of that part is indeed linear. 

Overall, we saw that NN-GLS achieves a linear computational cost up to a sample size of around $100000$, and while there is slight super-linear complexity for larger sizes due to the Python back-propagation module slowing down at these sizes. However, even with its current implementation which is not fully optimized, \blue{NN-GLS can analyze a hundreds-of-thousands-sized sample within an hour without any parallelization}. These results verify the theoretical linear computational complexity and demonstrate the potential of NN-GLS to be scaled for large datasets in the future.}

\begin{figure}[!h]
\centering
\begin{subfigure}{.5\textwidth}
  \centering
  \includegraphics[width=0.85\linewidth]{Simulation/Running_time/runningtime_compare.png}
  \caption{Running time for different methods}
\end{subfigure}%
\begin{subfigure}{.5\textwidth}
  \centering
  \includegraphics[width=0.85\linewidth]{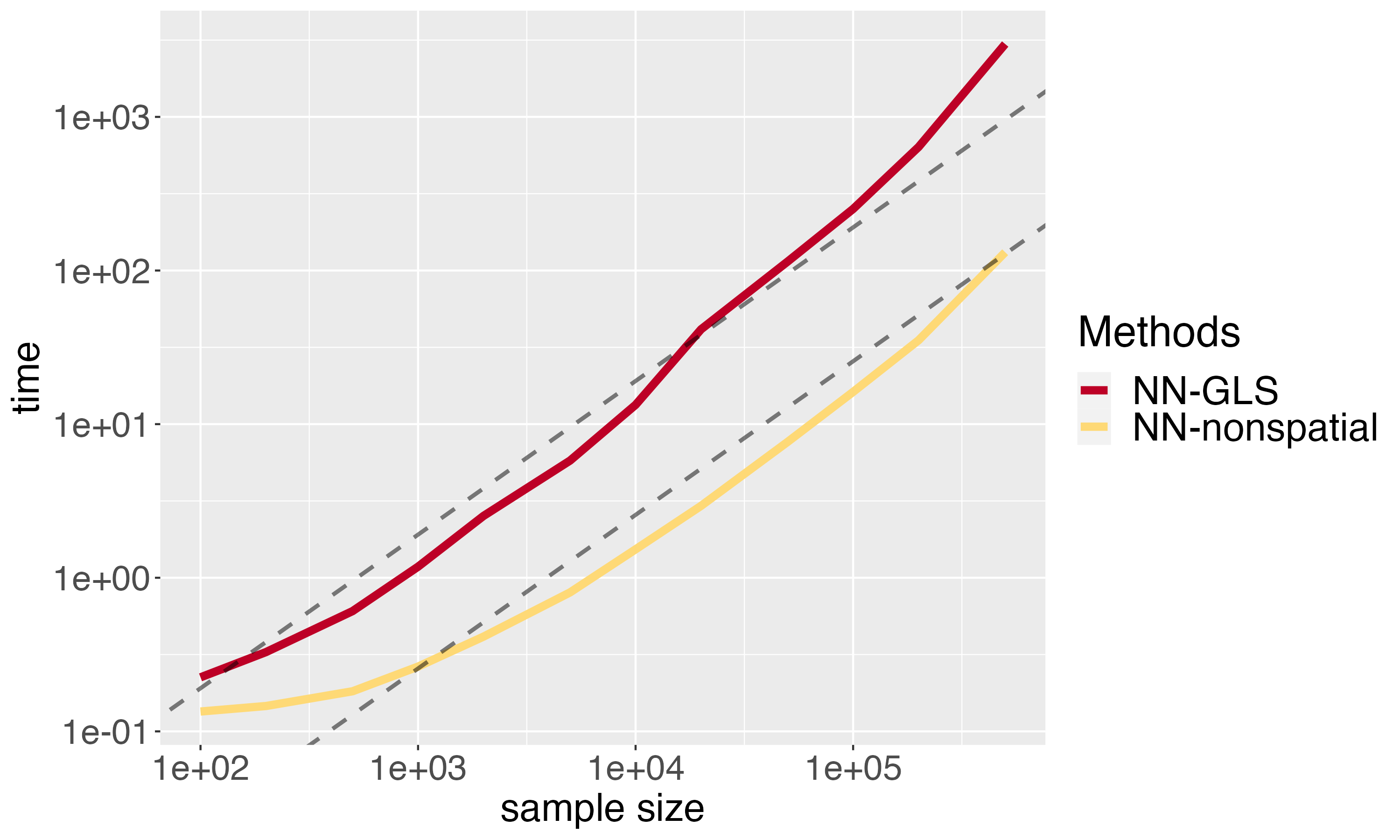}
  \caption{Running time for non-spatial NN and NN-GLS with fixed number (10) of epochs}
\end{subfigure}
\\
\centering
\begin{subfigure}{.5\textwidth}
  \centering
  \includegraphics[width=0.85\linewidth]{Simulation/Running_time/MISE_compare.png}
  \caption{MISE for different methods}
\end{subfigure}
\caption{(Section \ref{sec-sim-large}) Log running times and estimation loss (MISE) as functions of the log sample size. The dashed auxiliary lines represent linear running time. \blue{Both RF-GLS and GAM-GLS's curves are truncated due to the time and memory limit. Subfigure (a) and (c) are the same as Figure \ref{fig-sim-main} (e) and (f).}}
\label{fig-time2}
\end{figure}

\blue{Additionally in Figure \ref{fig-time2}, we compare the running times and MISE as a function of the sample size for NN-GLS and all the other methods mentioned in Table \ref{tab:methods} that offer an estimate of the mean function -- GAM, GAM-GLS, RF, RF-GLS, and linear spatial model. Unsurprisingly, among the non-linear methods, the run times for the OLS based approaches (RF, GAM, NN-non-spatial) are the smallest as they use the OLS loss. Among the GLS based methods, NN-GLS is the fastest while RF-GLS and GAM-GLS do not scale linearly and could not be used for the larger sample sizes. This is because GAM-GLS currently do not use NNGP which can potentially speed it up in the future, and RF-GLS, despite using NNGP, requires brute force evaluation of the GLS loss function for each value of the covariate, leading to super-linear run-times. 

We note that NN-GLS's running time seems to have a non-linear trend at around $n = 10000$ because the number of epochs also increases at that interval. 
Figure \ref{fig-time2} (b) shows the running time of both non-spatial NN and NN-GLS with a fixed number of epochs, where a clear linear trend is observed. NN-GLS is of course slower than the non-spatial-NN which uses the OLS loss. However, from our experience, compared with NN-GLS, non-spatial NN usually takes more epochs before converging.

In terms of estimation, the MISE plot in Figure \ref{fig-time2} (c) shows that NN has the lowest approximation error and NN-GLS uniformly dominates all the other methods even at very large sample sizes. The GAM methods converge to non-zero estimation error on account of not being able to model the interaction term in the Fridman function. The linear model, unsurprisingly, converges to the highest MISE.} 

\subsection{Performance comparison with denser observations}\label{sec-sim-5000}
\blue{
We also assess the impact of sample size under denser sampling within a fixed domain. This study supplements the theoretical results and previous numerical experiments which all correspond to increasing domain asymptotics where the sampling density does not increase. In Figures  \ref{fig-sim-5000-dim1} and \ref{fig-sim-5000-dim5}, the performance comparison among the competing methods is demonstrated under a sample size of $5000$ in the same spatial domain as considered earlier for lower sample sizes. 
The result of RF-GLS could not be shown in the figures due to the scalability issue (see Section \ref{sec-sim-large}). 
\begin{figure}[htbp]
\centering
\begin{subfigure}{1\textwidth}
  \centering
  \includegraphics[width=0.9\linewidth]{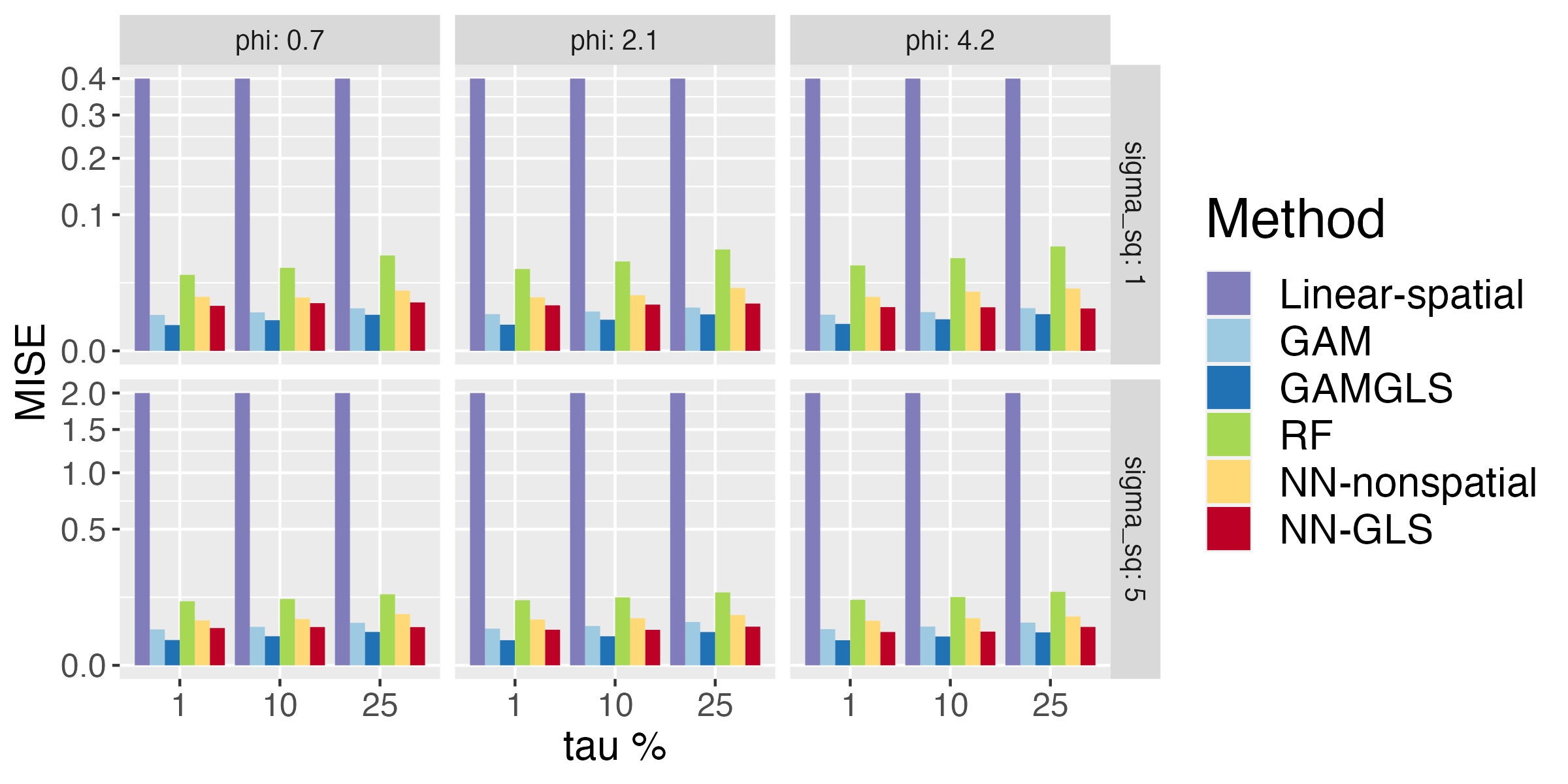}
  \caption{Estimation performance for $f_1$*}
\end{subfigure}%
\\
\centering
\begin{subfigure}{1\textwidth}
  \centering
  \includegraphics[width=0.9\linewidth]{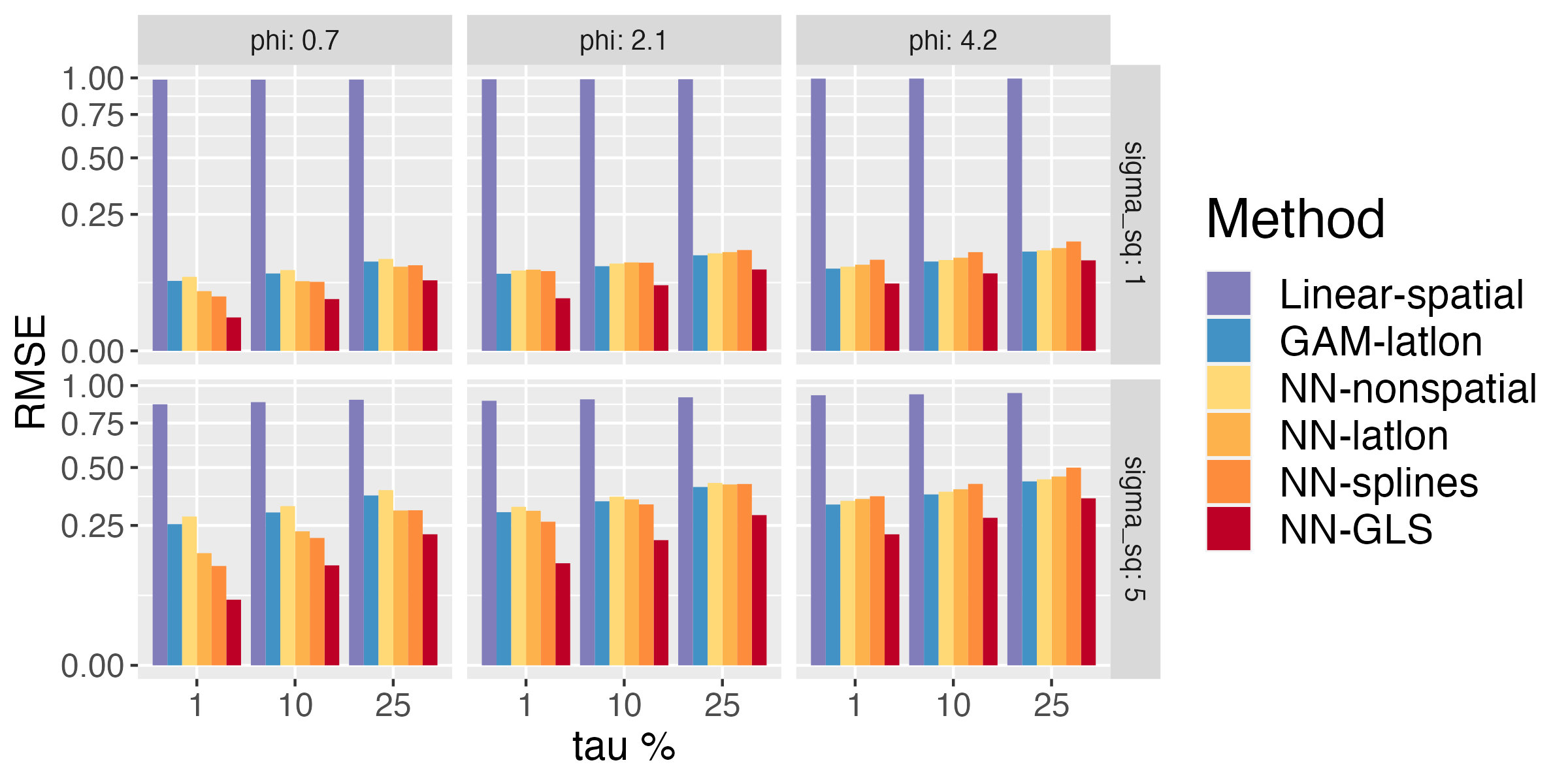}
  \caption{Prediction performance for $f_1$}
\end{subfigure}
\caption{\blue{(Section \ref{sec-sim-5000}) Comparison between competing methods on (a) estimation and (b) spatial prediction
when the mean function is $f_0 = f_1$ under denser sampling.} \blue{We add a * in figure (a) as the MISE for the linear-spatial model (which was very large) had to be truncated for better illustration of the performance of the other methods.}}
\label{fig-sim-5000-dim1}
\end{figure}

\begin{figure}[!t]
\centering
\begin{subfigure}{1\textwidth}
  \centering
  \includegraphics[width=0.9\linewidth]{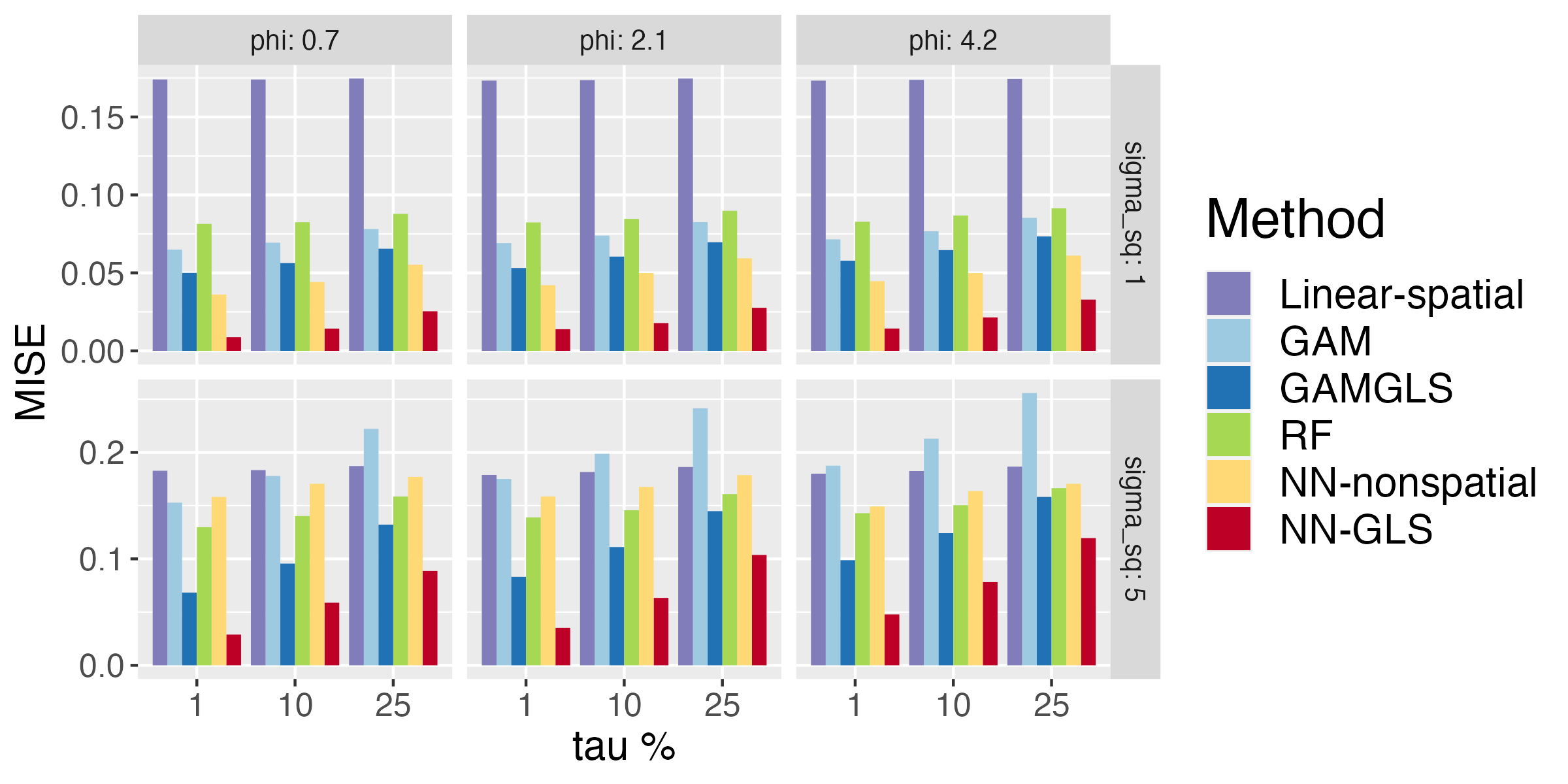}
  \caption{Estimation performance for $f_2$}
\end{subfigure}%
\\
\centering
\begin{subfigure}{1\textwidth}
  \centering
  \includegraphics[width=0.9\linewidth]{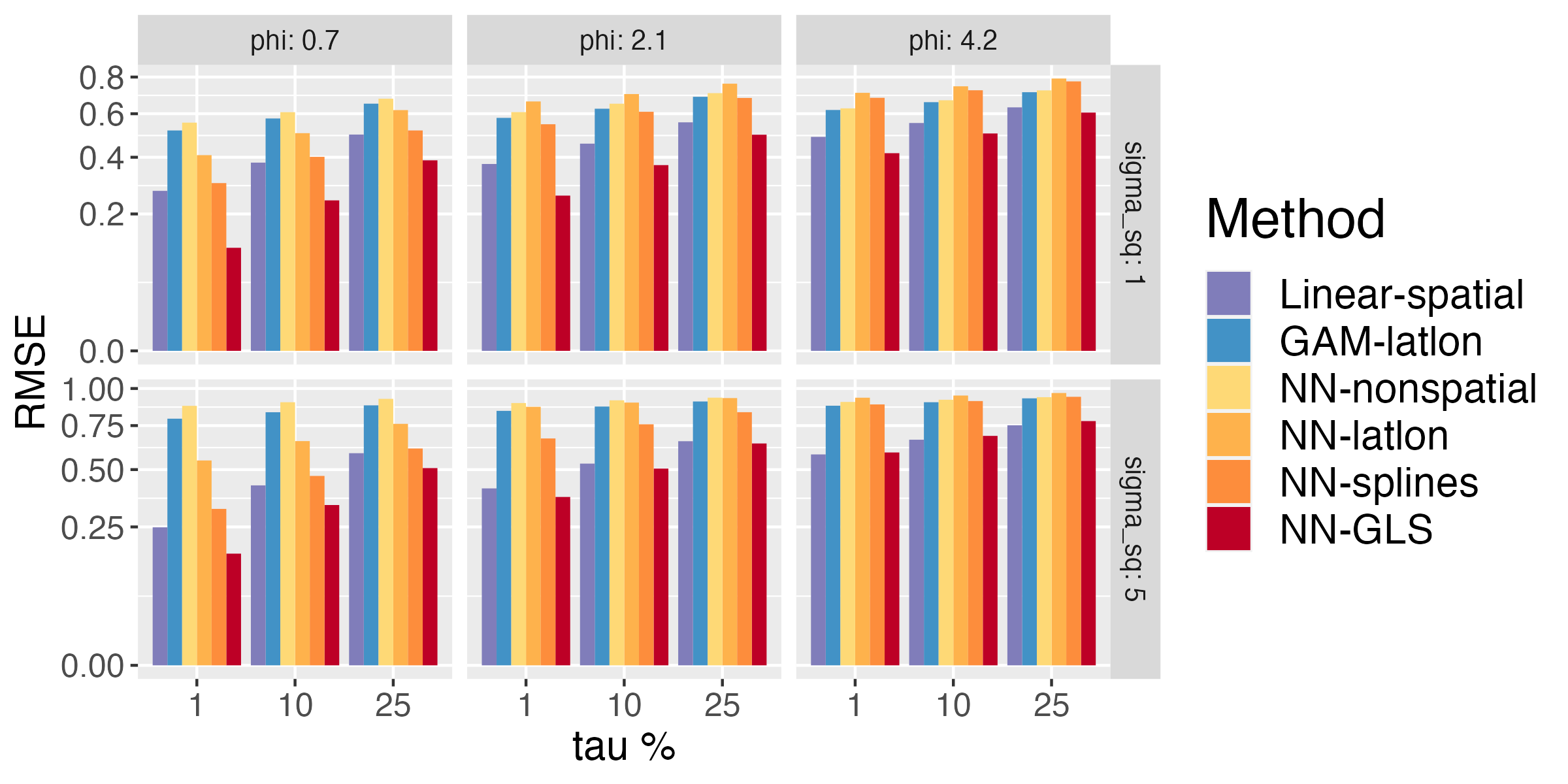}
  \caption{Prediction performance for $f_2$}
\end{subfigure}
\caption{\blue{(Section \ref{sec-sim-5000}) Comparison between competing methods on (a) estimation and (b) spatial prediction
when the mean function is $f_0 = f_2$ under denser sampling.}}
\label{fig-sim-5000-dim5}
\end{figure}

According to the result, we conclude that denser observation will reduce the estimation error of all the methods. In comparing Figure \ref{fig-sim-5000-dim5} (a) with Figure \ref{fig-sim-1}(a) and Figure \ref{fig-sim-large-noise-dim5}(a), we see that the MISE converges to the approximation error when the sample density grows. For example, it is known that GAM is misspecified in this case for the Friedman function due to the interaction term, which implies its larger approximation error compared with the NN family. However, the relationship is not fully reflected in Figure \ref{fig-sim-large-noise-dim5} (a). As was discussed previously, the approximation error is masked by the estimation error when the sample size is small. The increased sample density here uncovers the former by reducing the latter and demonstrates that GAM-GLS performs worse than the NN-GLS for estimation of the mean. The same thing happens for $f_1$, for the simple sine function, all non-linear methods achieve an almost ignorable MISE while the linear method produces extremely large errors.  

Thus, even though our theory is based on the increasing-domain settings, NN-GLS provides the best or comparable to the best performance under fixed domain sampling.}

\newpage
\subsection{Performance comparison between NN-splines and NN-GLS}\label{sec-sim-NN}
\blue{
This section shows the consistent advantages of NN-GLS over NN-splines under different settings. We consider $f_0 = f_2$ and increase sample sizes in a fixed domain. For the latent spatial effect, we consider two scenarios: first, when the spatial effect is generated from an exponential Gaussian process, and when spatial effects are sampled from one fixed smooth surface (we use a single realization of a Gaussian process surface with parameters $(\sigma^2, \phi, \tau^2) = (1, 3/\sqrt{2}, 0.01)$). In the first case, NN-GLS is correctly specified whereas NN-slines is not, in the second case it is the other way around. 


\begin{figure}[!h]
\centering
\begin{subfigure}{.5\textwidth}
  \centering
  \includegraphics[width=0.9\linewidth]{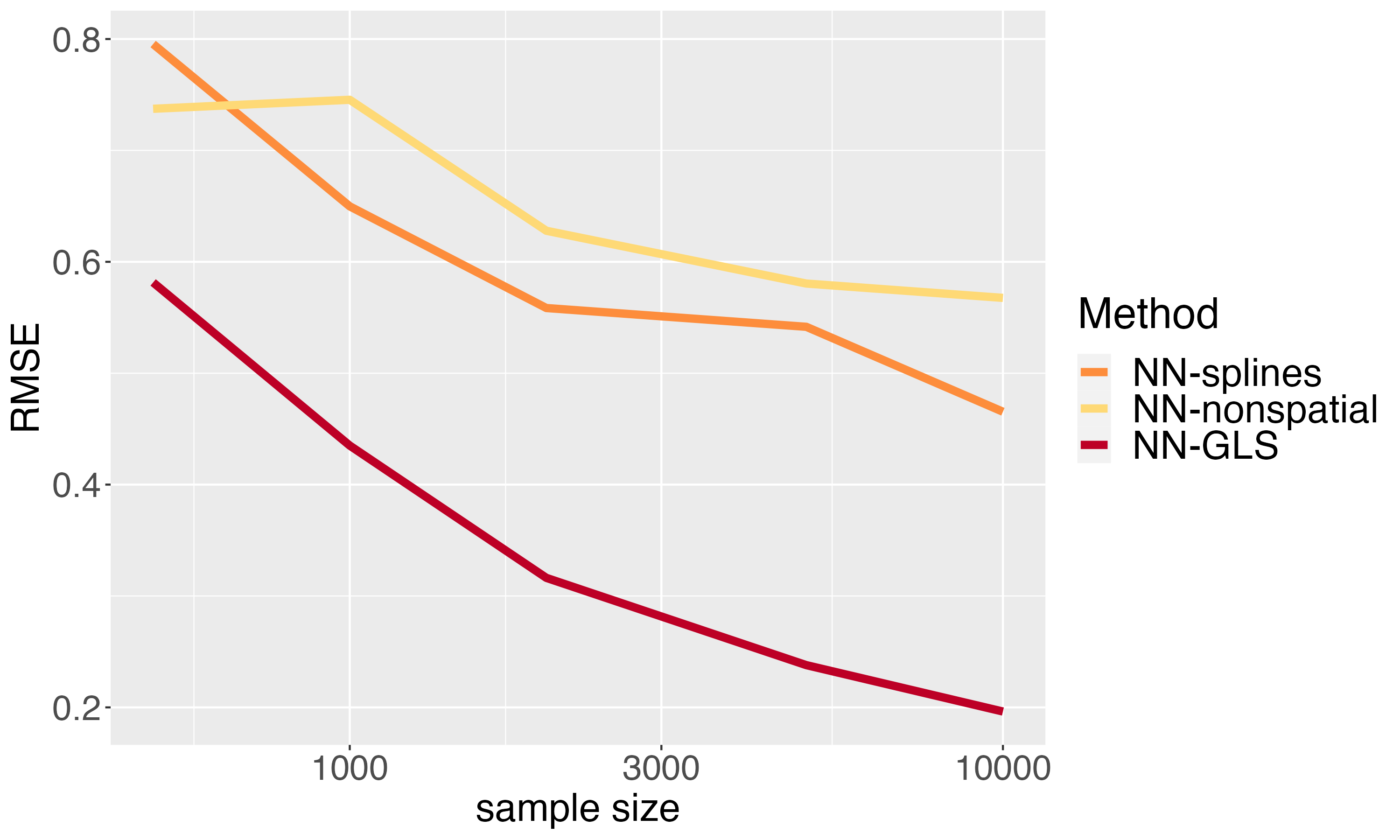}
  \caption{Prediction under normal simulation}
\end{subfigure}%
\begin{subfigure}{.5\textwidth}
  \centering
  \includegraphics[width=0.9\linewidth]{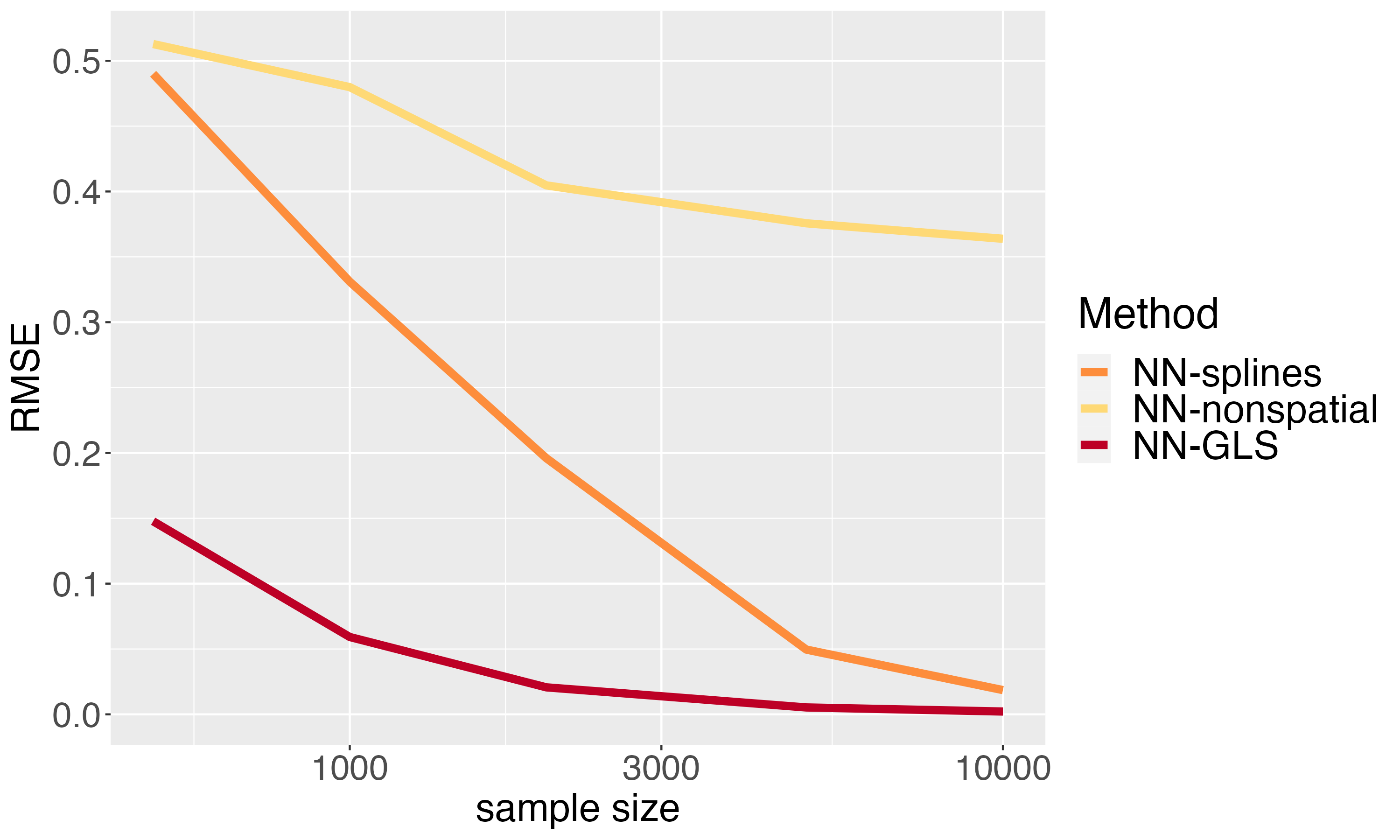}
  \caption{Prediction under fixed-surface simulation}
\end{subfigure}
\caption{(Section \ref{sec-sim-NN}) Comparing prediction performance for NN-splines and NN-GLS. Subfigure (b) is the same as Figure \ref{fig-sim-main} (b)}
\label{fig-sim-NN}
\end{figure}

According to Figure \ref{fig-sim-NN}, when the sample size increases within a fixed domain, the prediction performance of both NN-GLS and NN-splines improves since more spatial information can be aggregated from closer neighboring observations. NN-GLS significantly benefits from its parsimony when the sample size is small having considerably smaller RMSE than NN-splines for both scenarios. 
With an increasing sample size, in the first scenario (Figure \ref{fig-sim-NN} (a)) where the spatial effect is sampled randomly as a Gaussian process,  NN-splines fails to capture all the spatial patterns. Thus while NN-splines are very non-parametric in theory and should be able to capture any spatial trend given a large number of basis functions and enough data, in practice we see its performance falls short of NN-GLS even for large sample sizes. In the second case (Figure \ref{fig-sim-NN} (a)), where NN-splines is correctly specified, it catches up to NN-GLS when the sample size increases. 
However, even then the RMSE of NN-GLS is uniformly smaller than NN-splines, despite NN-GLS being misspecified. We postulate that for NN-splines to improve, 
both the complexity of splines or basis functions used and the network architecture design should increase with the sample size. In both cases, NN-GLS benefits from the parsimony of using GP, specified using a fixed and small number of spatial covariance parameters and without requiring any additional design change (like number of basis functions) to accommodate the increasing sample size.
}

\newpage
\subsection{Performance comparison between NN-GLS and NN-GLS (Oracle)}\label{sec-sim-orc}
NN-GLS estimates the spatial parameters as part of its algorithm, since in practice we don't know the true spatial parameters. 
In order to see how the estimation of the spatial parameters affects the function estimation of NN-GLS, we fix spatial parameters to the true values before the training, skip the parameter updating, and calculate the MISE (we call this approach {\em NN-GLS oracle}). Figure \ref{fig-sim-oracle} shows the MISE comparison between standard NN-GLS and its oracle version.  
A trivial expectation is that NN-GLS `oracle' will generally perform better since the model got more precise model information. This is indeed the case. However, the difference is minimal across the parameter combinations, with no systematic trend. In conclusion, the additional task of spatial parameter estimation does not contribute significantly to the total estimation error.

\begin{figure}[!h]
\centering
\begin{subfigure}{.5\textwidth}
  \centering
  \includegraphics[width=0.9\linewidth]{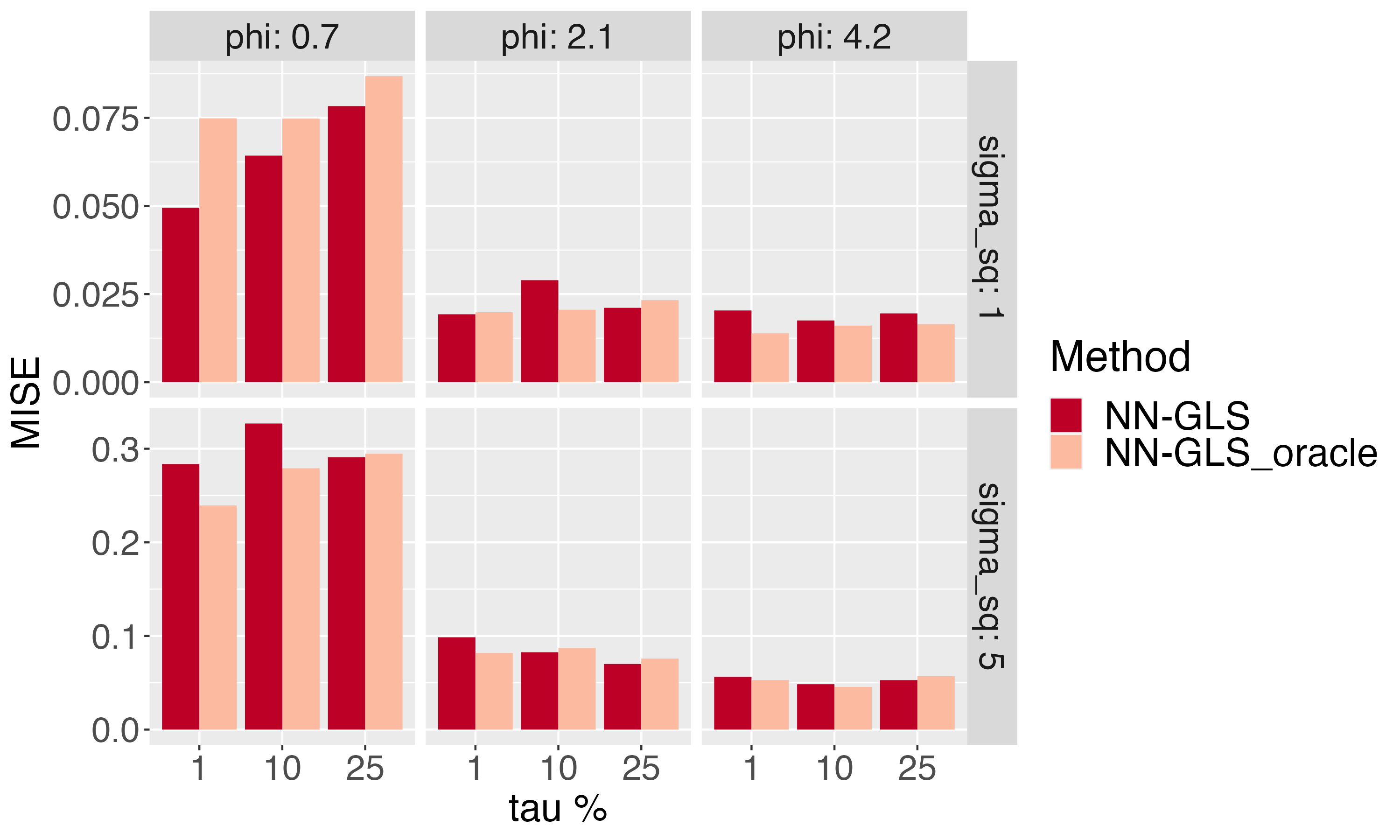}
  \caption{$f_0 = f_1$}
\end{subfigure}%
\begin{subfigure}{.5\textwidth}
  \centering
  \includegraphics[width=0.9\linewidth]{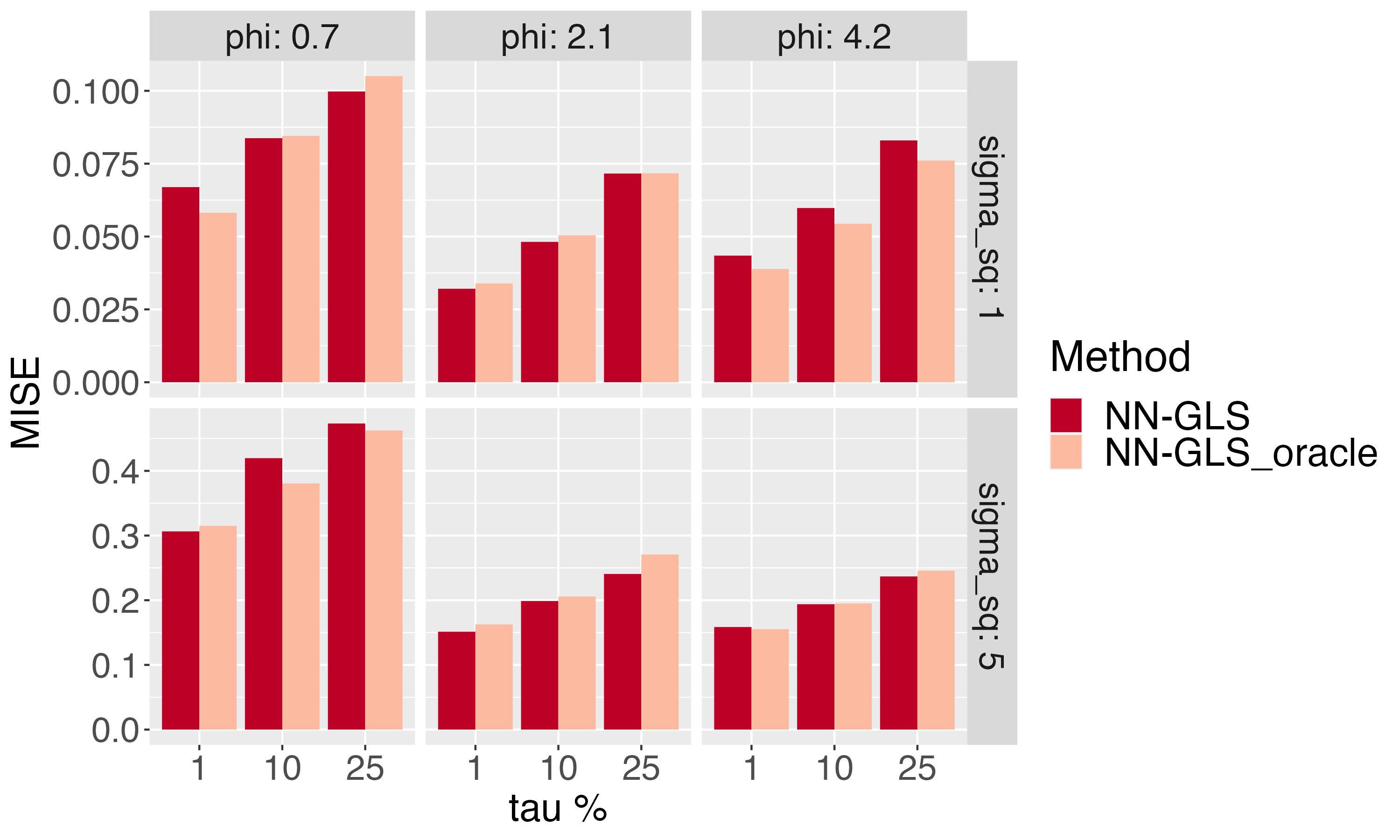}
  \caption{$f_0 = f_2$}
\end{subfigure}
\caption{(Section \ref{sec-sim-orc}) Comparison between NN-GLS and NN-GLS (Oracle) on estimation when the mean function
is (a) $f_0 = f_1$ and (b) $f_0 = f_2$.}
\label{fig-sim-oracle}
\end{figure}

\subsection{Higher dimensional function}\label{sec-sim-hd}
We also replicate the simulation steps for $f_0 = f_1$ and $f_0 = f_2$ on a $15$-dimensional function:
\[
\begin{split}
f_3(x) &= \bigg(10\sin(\pi x_1x_2)+20(x_3 - 0.5)^2+10x_4+5x_5 + \frac{3}{(x_6+1)(x_7+1)} + 4\exp(x_8^2)\\ 
&+ 30x_9^2+x_{10} + 5\big(\exp(x_{11})\cdot \sin(\pi x_{12}) + \exp(x_{12})\cdot \sin(\pi x_{11})\big)\\ 
&+ 10x_{13}^2\cdot\cos(\pi x_{14}) + 20 x_{15}^4\bigg)/6.
\end{split}
\]
As is shown in Figure \ref{fig-sim-dim15}, NN-GLS still has an advantage over other methods in terms of both estimation and prediction performances. However, for estimation, the decrease in MISE for this $15$-dimensional function is not as significant for higher spatial correlations (low $\phi$, top left panel of Figure \ref{fig-sim-dim15} (a)) as for the $5$-dimensional Friedman function (top left panel of Figure \ref{fig-sim-1} (a)). This possibly suggests that a higher dimensional function is more challenging to estimate in moderate-sized datasets, and the estimation error can often dominate the improvement of NN-GLS over the non-spatial NN.

\begin{figure}[htbp]
\centering
\begin{subfigure}{1\textwidth}
  \centering
  \includegraphics[width=0.9\linewidth]{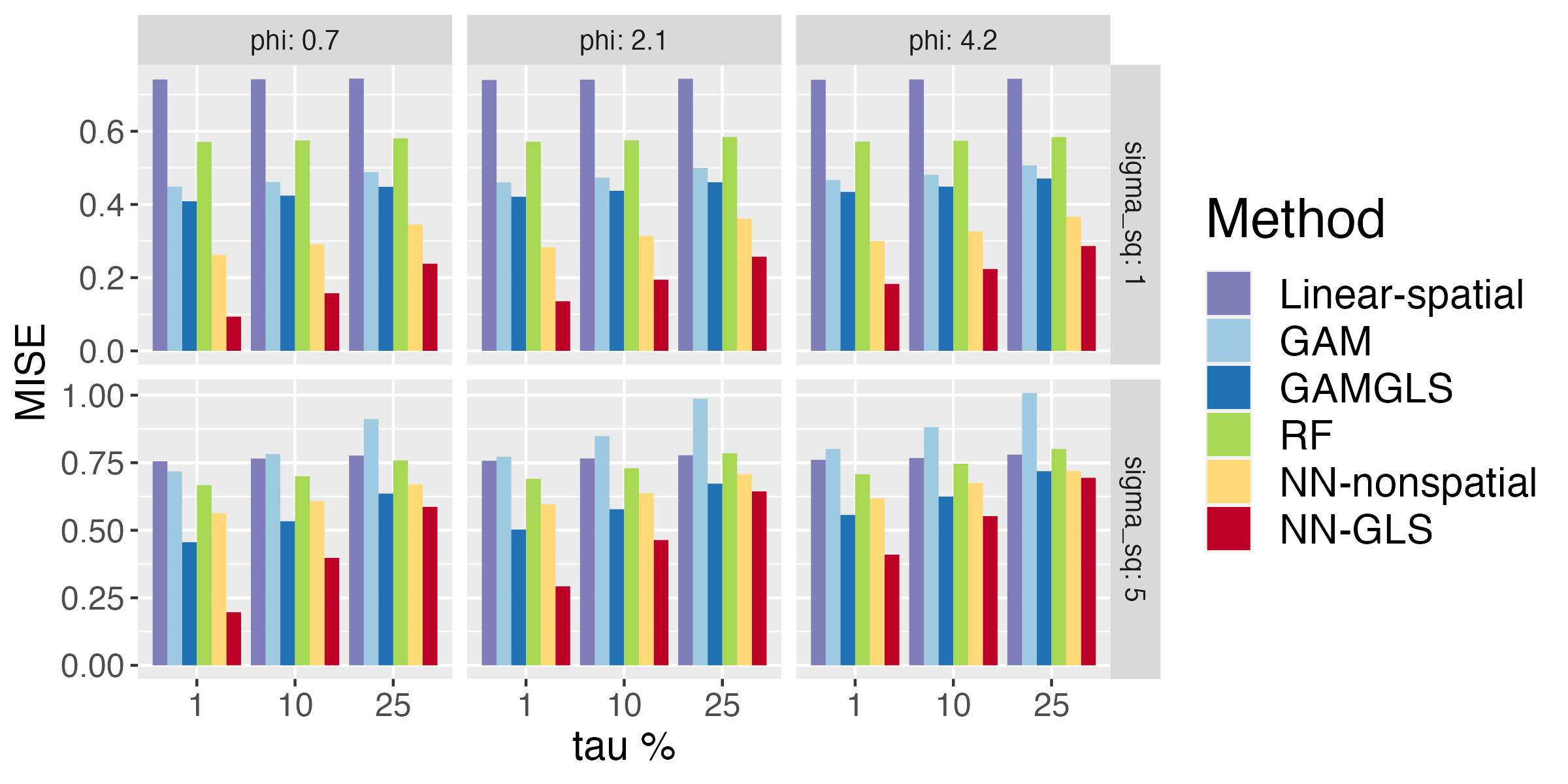}
  \caption{Estimation performance}
\end{subfigure}%
\\
\centering
\begin{subfigure}{1\textwidth}
  \centering
  \includegraphics[width=0.9\linewidth]{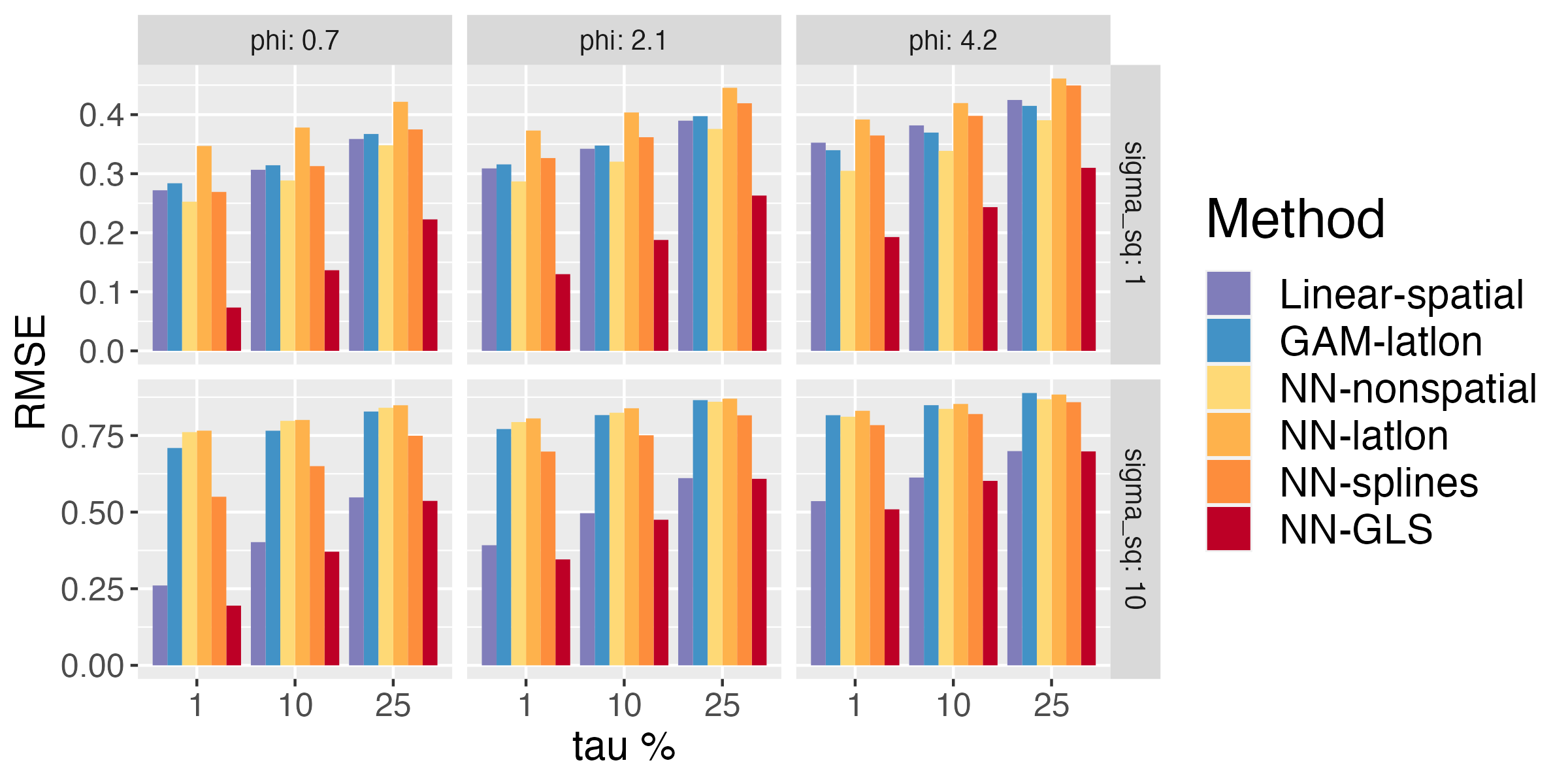}
  \caption{Prediction performance}
\end{subfigure}
\caption{(Section \ref{sec-sim-hd}) Comparison between competing methods on (a) estimation and (b) spatial prediction
when the mean function is $f_0 = f_3$ (higher dimensional setting).}
\label{fig-sim-dim15}
\end{figure}

\newpage
\subsection{Model misspecification: misspecified Mat\'{e}rn covariance}\label{sec-sim-mis1}
In the previous simulation setups, we assumed that data were generated from the model where the parametric spatial covariance family was correctly specified. Such an assumption plays a key role in the estimation as we use the likelihood from using this covariance family for updating the parameters, which involves the covariance matrix in an explicit form. In this subsection, we run NN-GLS on the data under misspecification of the GP covariance family.

We simulate the data from a Gaussian process with a Mat\'{e}rn covariance with $\nu = 1.5$ and fit NN-GLS by assuming an exponential covariance structure. Note that the exponential covariance is a special case of Mat\'{e}rn family by taking $\nu = 0.5$. The values of $\nu$ imply that we are expecting less smoothness in the model compared with the truth.
In the implementation, we replicate the setups in Section \ref{sec-sim}, with all the $27$ combinations of spatial parameters, and NN-GLS is modeled with an exponential covariance structure. 

The results are presented in Figure \ref{fig-matern-1} (for $f_0=f_1$) and Figure \ref{fig-matern-2} (for $f_0=f_2$). In terms of estimation, for the true function $f_1$, NN-nonspatial gives a comparable performance to NN-GLS because of $f_1$'s simple structure. The improvement earned by considering spatial covariance is small here because the parameterization is also incorrect here. However, for $f_2$, the advantage of NN-GLS is evident, especially for low spatial variance (low $\sigma^2$). For prediction, NN-GLS outperforms other methods significantly for small noise-to-signal ratio $\tau \%$ and shows convincing advantages with all the parameter combinations.  
\begin{figure}[htbp]
\centering
\begin{subfigure}{1\textwidth}
  \centering
  \includegraphics[width=0.9\linewidth]{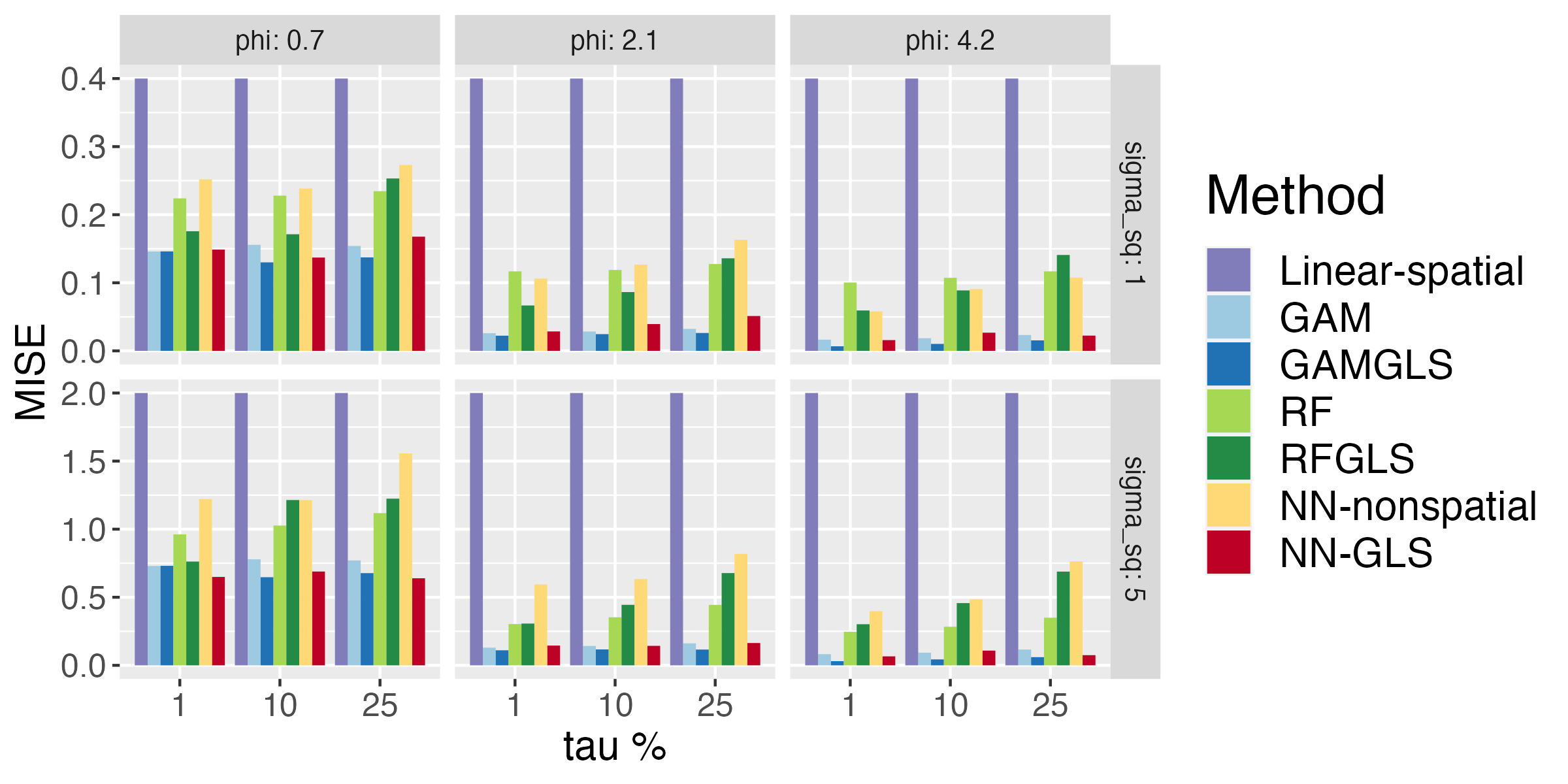}
  \caption{Estimation performance*}
\end{subfigure}%
\\
\centering
\begin{subfigure}{1\textwidth}
  \centering
  \includegraphics[width=0.9\linewidth]{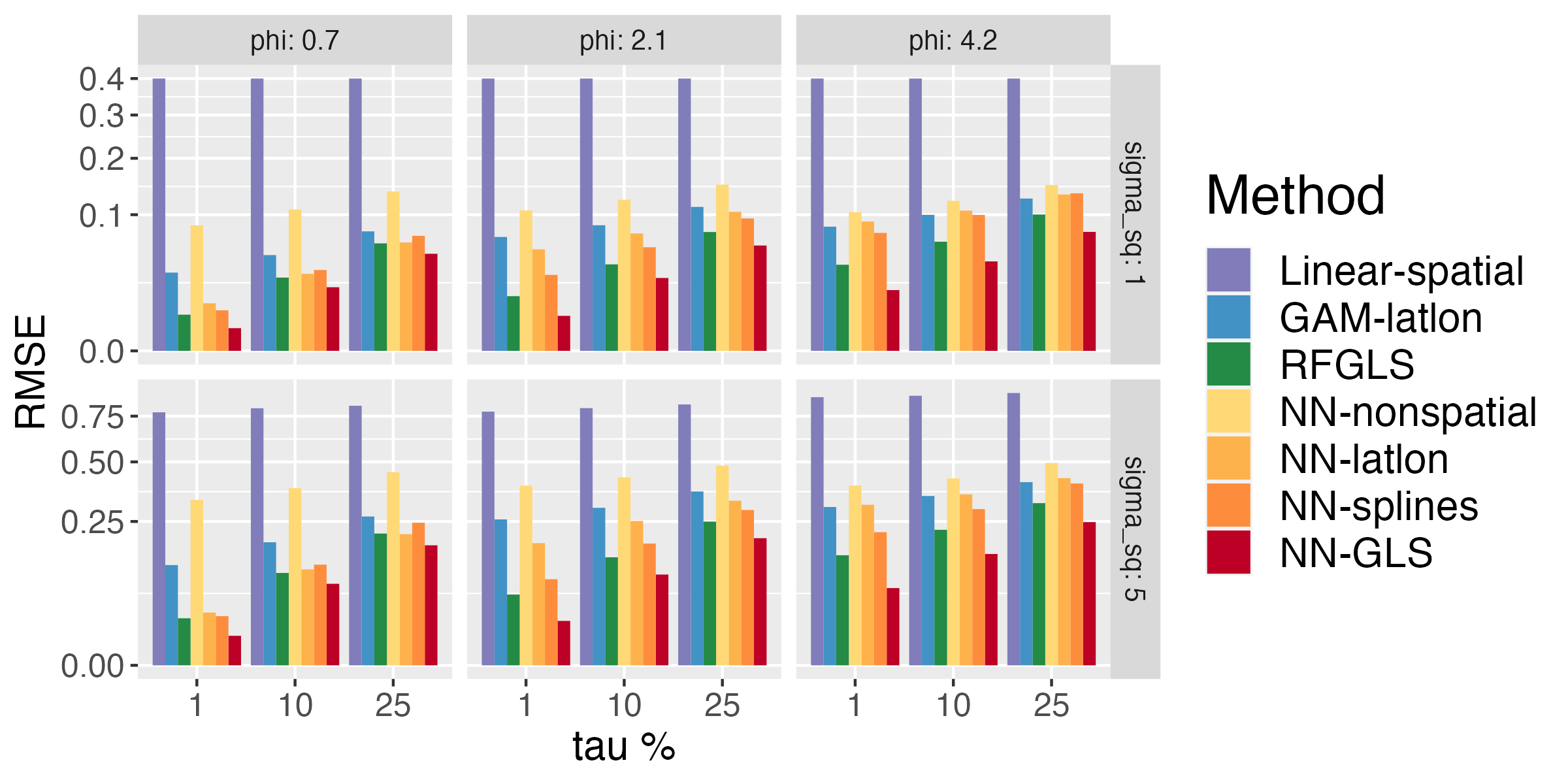}
  \caption{Prediction performance*}
\end{subfigure}
\caption{(Section \ref{sec-sim-mis1}) Comparison between competing methods on (a) estimation and (b) spatial prediction
when the spatial surface $\omega(s)$ is generated from Mat\'{e}rn covariance with smoothness $\nu=1.5$ and mean function is $f_0 = f_1$. \blue{We add a * in figure (a) as the MISE and RMSE for the linear-spatial model (which was very large) had to be truncated for better illustration of the performance of the other methods.}}
\label{fig-matern-1}
\end{figure}
\begin{figure}[htbp]
\centering
\begin{subfigure}{1\textwidth}
  \centering
  \includegraphics[width=0.9\linewidth]{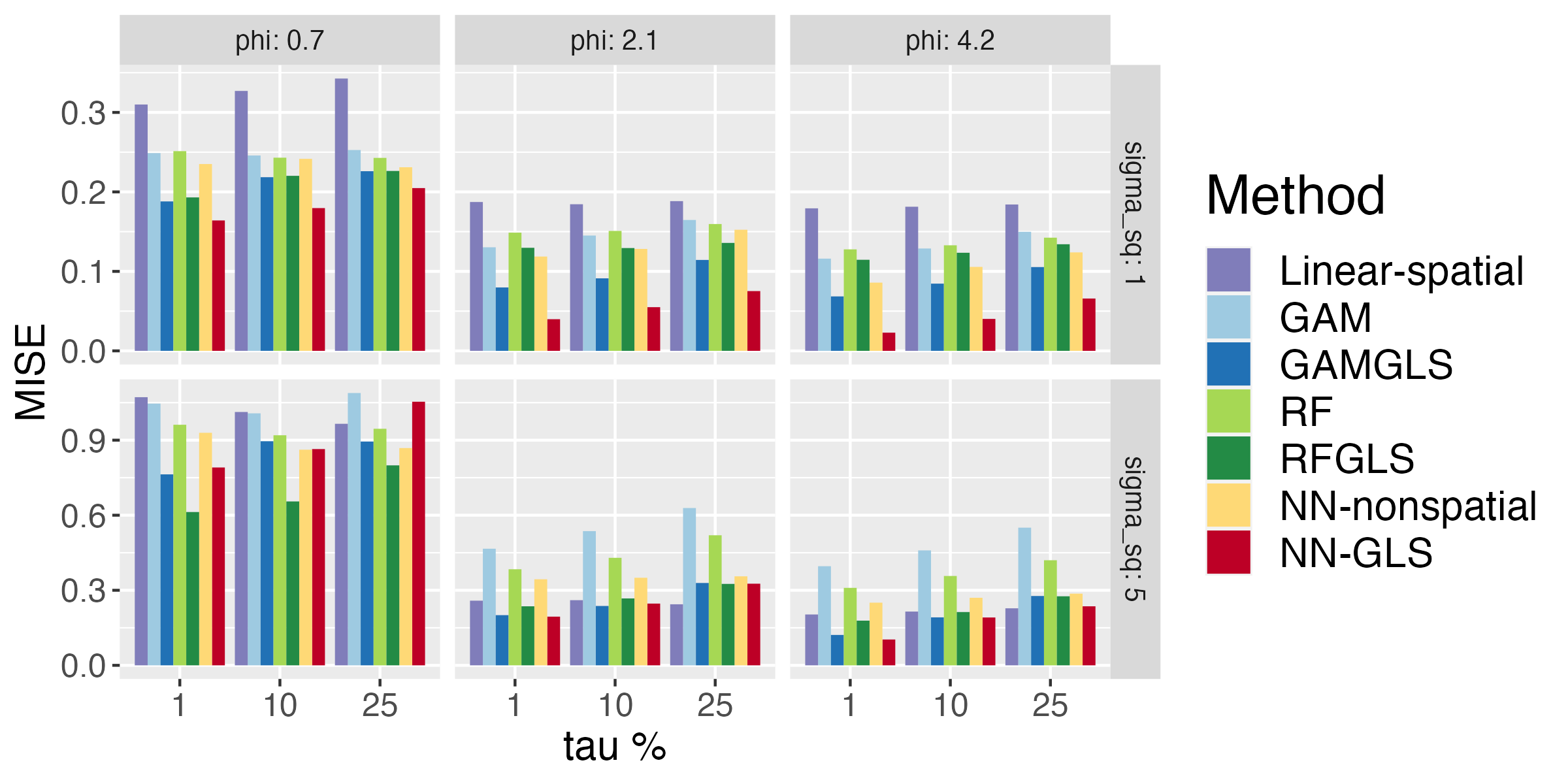}
  \caption{Estimation performance}
\end{subfigure}%
\\
\centering
\begin{subfigure}{1\textwidth}
  \centering
  \includegraphics[width=0.9\linewidth]{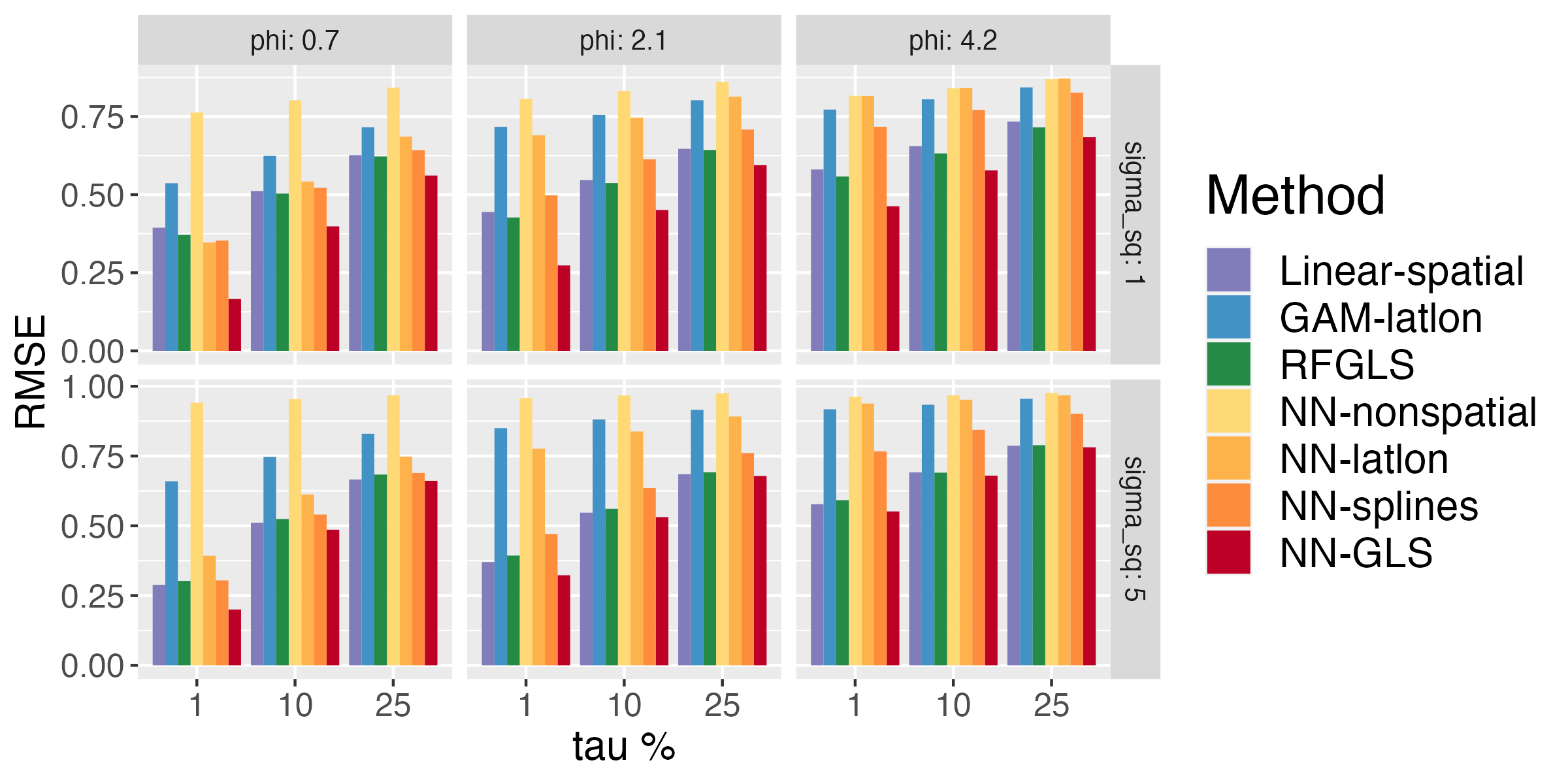}
  \caption{Prediction performance}
\end{subfigure}
\caption{(Section \ref{sec-sim-mis1}) Comparison between competing methods on (a) estimation and (b) spatial prediction
when the spatial surface $\omega(s)$ is generated from Mat\'{e}rn covariance and mean function is $f_0 = f_2$.}
\label{fig-matern-2}
\end{figure}

\newpage
\subsection{Model misspecification: misspecified spatial effect}\label{sec-sim-mis2}
We can further consider the scenario where the spatial effect is generated from a fixed spatial surface instead of a Gaussian process. 
In particular, the surface is generated as follows, we first define a Gaussian process $w(\cdot)$ following the exponential model called parent process and realize it at $100$ randomly selected locations $s = (s_1, \cdots, s_{100})$. Then the conditional expectation (kriged prediction) given $w(s)$ and the spatial covariance function gives a spatial surface. Equivalently, this surface is a realization of a Gaussian predictive process \citep{banerjee2008gaussian} and is shown in Figure \ref{fig-fix-shape} for one set of parameter values.
 \begin{figure}[!h]
 \centering
\includegraphics[scale=0.6]{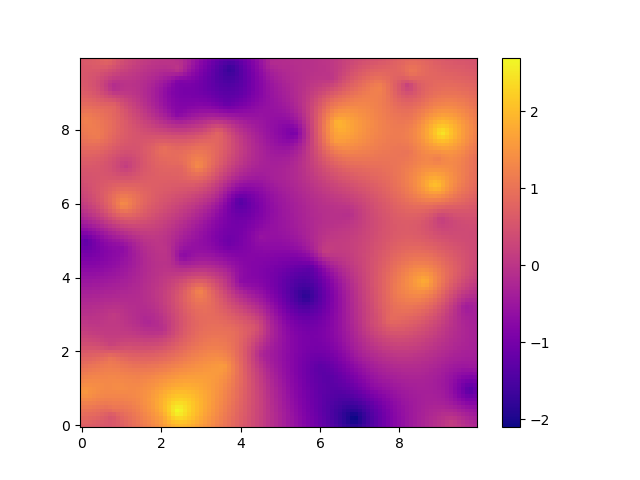}
\caption{(Section \ref{sec-sim-mis2}) Spatial surface generated when $\sigma^2, \phi, \tau^2 = (1, 1/\sqrt{2}, 0.01)$}
\centering
\label{fig-fix-shape}
\end{figure}

We experimented on both $f_1$ and $f_2$, both with the parameters \blue{$\sigma^2 \in \{1, 5\}$, $\phi \in \{1/\sqrt{2}, 3/\sqrt{2}, 6/\sqrt{2}\}$ and $\tau^2/\sigma^2 \in \{0.01, 0.1, 0.25\}$, i.e. 18 combinations in total}, which here are parametrizing the 'parent processes' in the Gaussian predictive process context. Exponential covariance structure is assumed in NN-GLS's implementation. The results are shown in Figure \ref{fig-mis-fix-dim1} to \ref{fig-mis-fix-dim5}. As is clearly shown in the plots, NN-GLS earns considerable advantages over NN. The fixed surface is equivalent to having a mean function with two additional covariates (the spatial coordinates). This information is missed by NN-nonspatial and leads to serious bias in estimation. The advantage is more significant for $f_0=f_2$ and for a smaller noise variance  $\tau^2$. In prediction, NN-nonspatial performs extremely poorly. 
{\em Added-spatial-features} approaches like NN-latlon and NN-splines mitigate the situation by involving spatial predictors in the neural network, but still, struggle in scenarios with strong spatial signals. NN-GLS  outperforms the others consistently, and the gain is more substantial than the scenarios in the previous sections. 

\begin{figure}[htbp]
\centering
\begin{subfigure}{1\textwidth}
  \centering
  \includegraphics[width=0.9\linewidth]{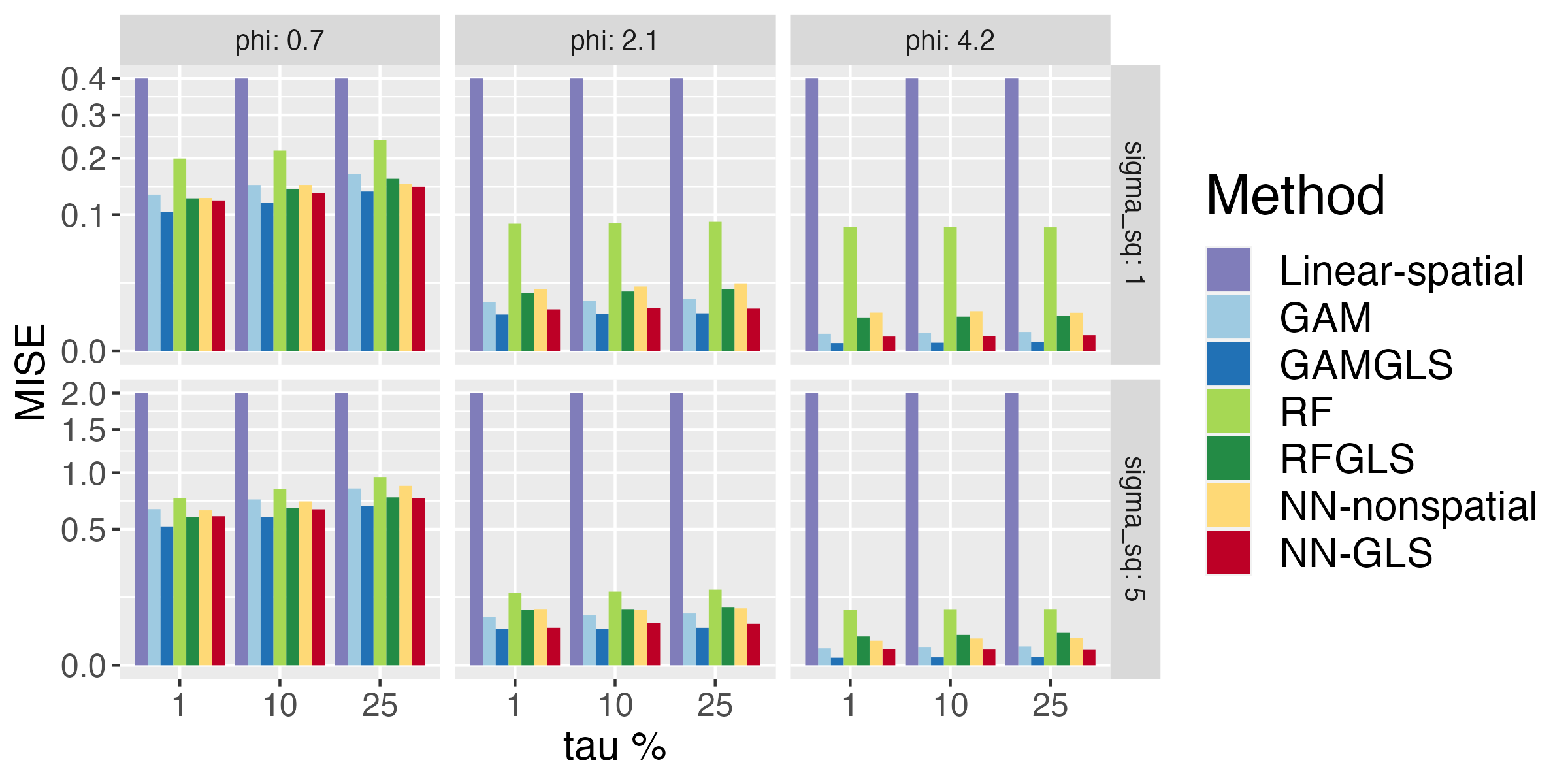}
  \caption{Estimation performance for $f_1$*}
\end{subfigure}%
\\
\centering
\begin{subfigure}{1\textwidth}
  \centering
  \includegraphics[width=0.9\linewidth]{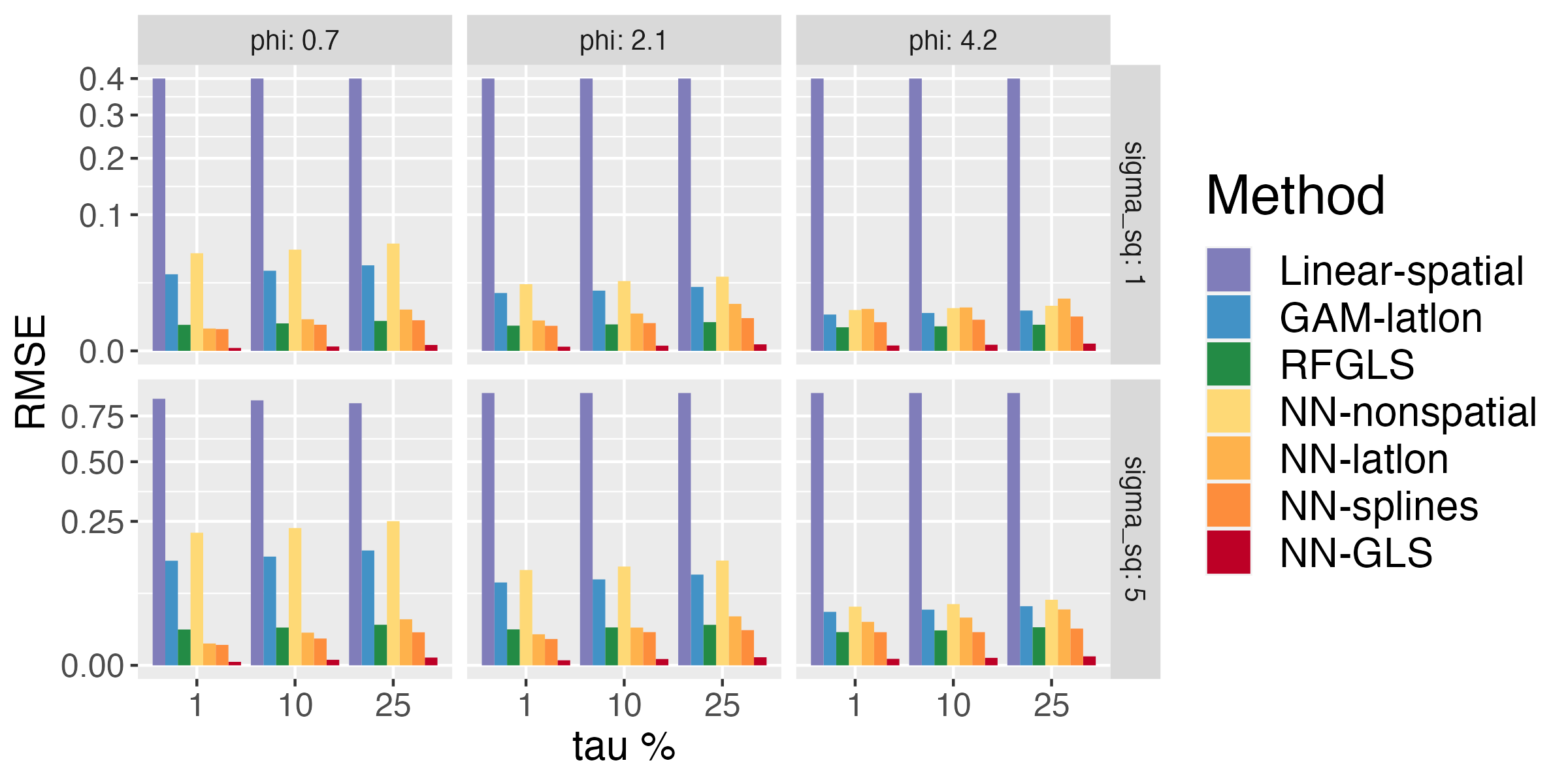}
  \caption{Prediction performance for $f_1$*}
\end{subfigure}
\caption{(Section \ref{sec-sim-mis2}) Comparison between competing methods on (a) estimation and (b) spatial prediction
when the spatial surface $\omega(s)$ is simulated from the smooth function given by  by Figure \ref{fig-fix-shape} and the mean functions are $f_0 = f_1$. \blue{We add a * in figure (a) as the MISE and RMSE for the linear-spatial model (which was very large) had to be truncated for better illustration of the performance of the other methods.}}
\label{fig-mis-fix-dim1}
\end{figure}
\begin{figure}[htbp]
\centering
\begin{subfigure}{1\textwidth}
  \centering
  \includegraphics[width=0.9\linewidth]{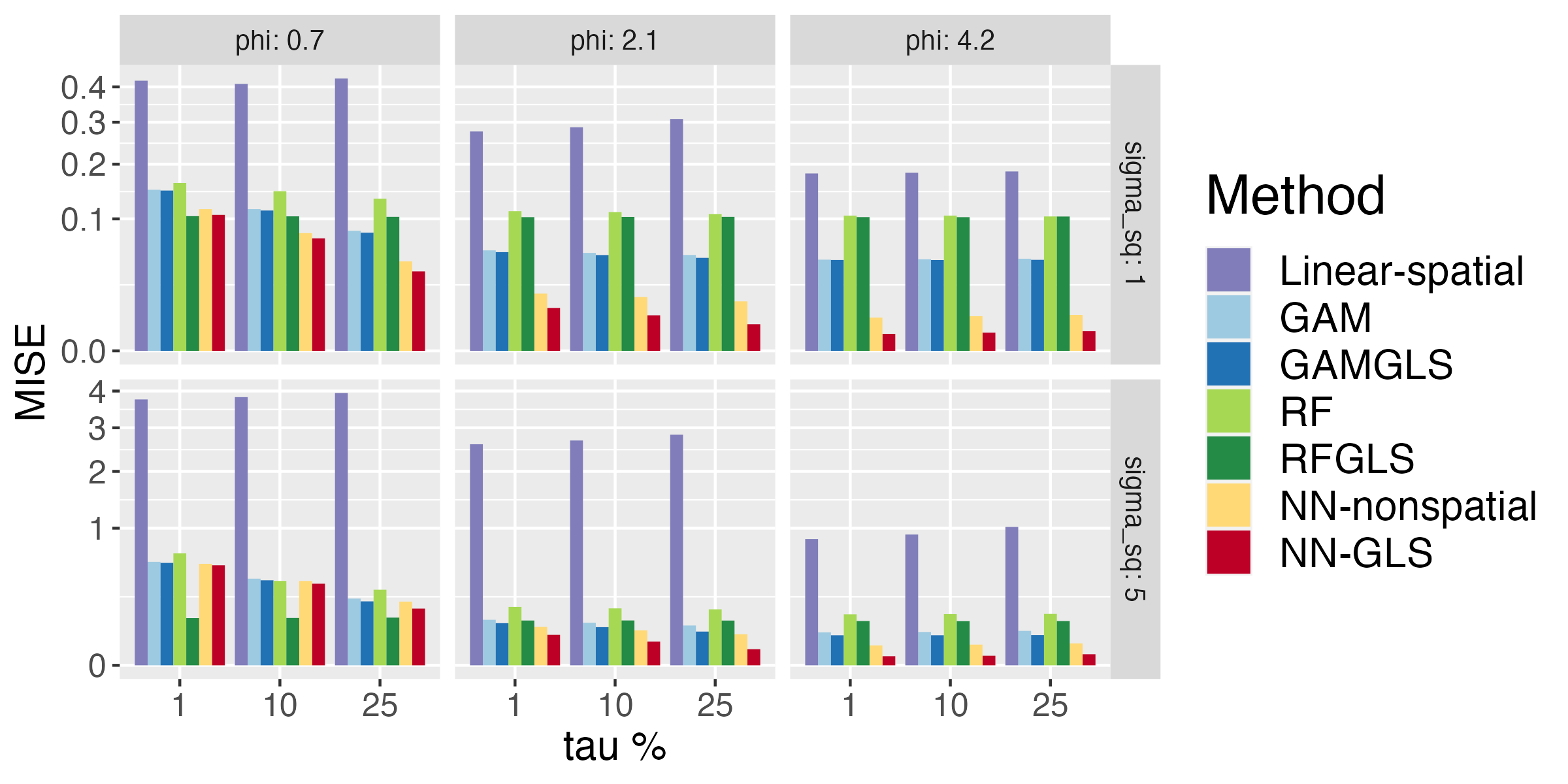}
  \caption{Estimation performance for $f_2$}
\end{subfigure}%
\\
\centering
\begin{subfigure}{1\textwidth}
  \centering
  \includegraphics[width=0.9\linewidth]{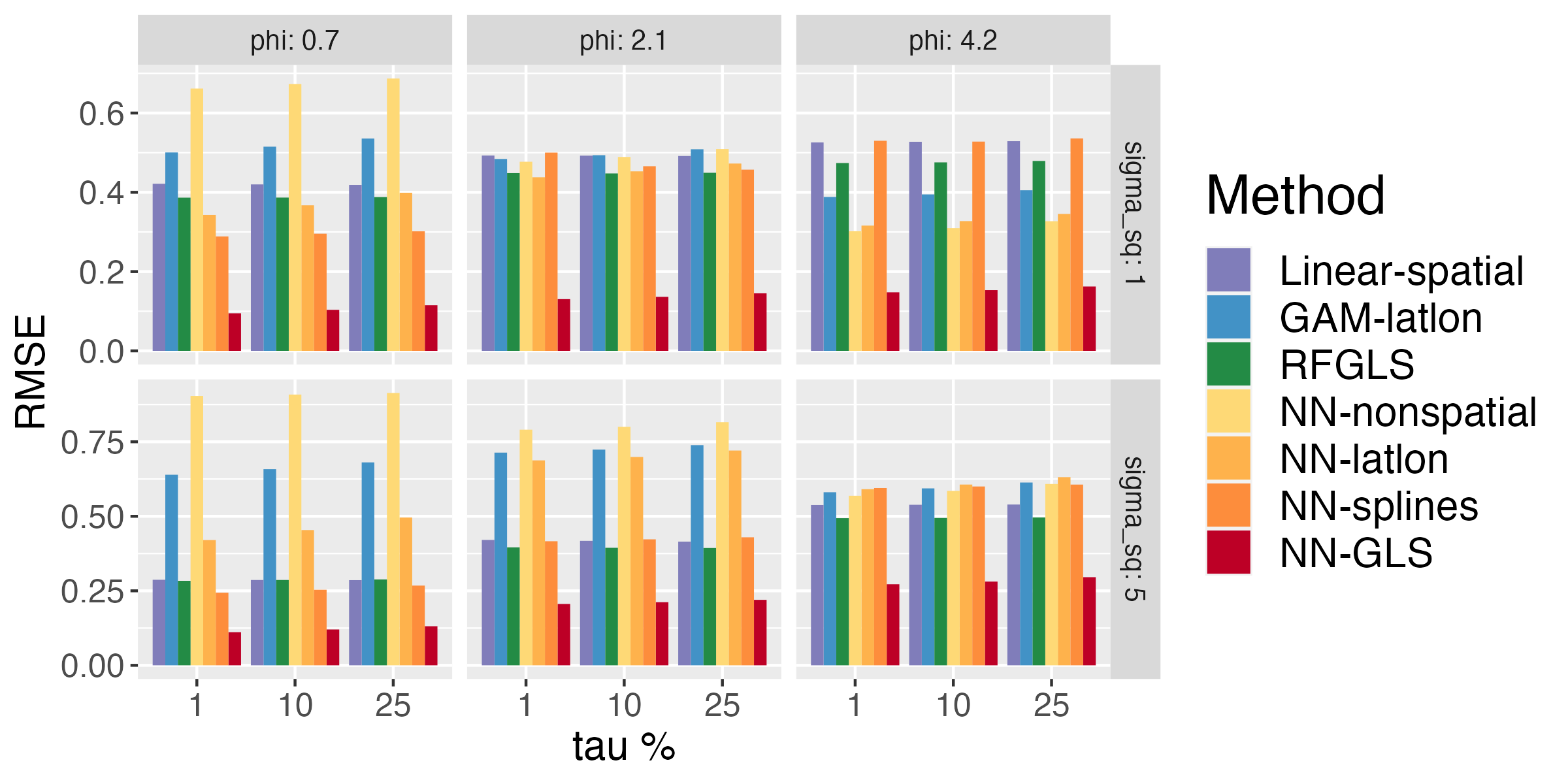}
  \caption{Prediction performance for $f_2$}
\end{subfigure}
\caption{(Section \ref{sec-sim-mis2}) Comparison between competing methods on (a) estimation and (b) spatial prediction
when the spatial surface $\omega(s)$ is simulated from the smooth function given by Figure \ref{fig-fix-shape} and the mean functions are $f_0 = f_2$.}
\label{fig-mis-fix-dim5}
\end{figure}

\pagebreak

\newpage
\section{Additional real data example}\label{Append-real}
\subsection{Data details}\label{sec:data}
The regulatory PM$_{2.5}$ data comes from the US Environmental Protection Agency's (EPA) air quality system (AQS). The AQS collects daily air quality data through widely distributed monitors and validates it through a quality assurance procedure. Since PM$_{2.5}$ is known to have prominent spatial patterns, we use NN-GLS, modeling the PM$_{2.5}$ concentrations across the U.S. as the correlated response, and using other meteorological variables as covariates. The meteorological data are obtained from the National Centers for Environmental Prediction's (NCEP) North American Regional Reanalysis (NARR) product, which generates reanalyzed data for temperature, wind, moisture, soil, and dozens of other parameters with a spatial resolution of about $32 \times 32$ km. Following a similar analysis in \cite{chen2020deepkriging}, we consider the daily PM$_{2.5}$ data on \blue{June 18th, 2022} for this analysis (other dates are considered in the Supplementary materials). We have PM$_{2.5}$ concentration ($\mu g/m^3$) from \blue{$1078$} stations across the states and six meteorological variables provided at \blue{$13612$} grid cells from NARR. The six variables are precipitation accumulation, air temperature, pressure, relative humidity, west wind (U-wind) speed, and north wind (V-wind) speed. Since the coordinates of the two data sets are different, the spatial resolution of NARR data is retained and the PM$_{2.5}$ data are matched onto the grid by averaging in the grid cell. Grid cells without any EPA monitor are removed and there are \blue{$719$} data points left for the downstream analysis.

\subsection{Block-random data splitting}\label{sec-block}
We consider the prediction performance of methods under  block-random splits of different block sizes. Here the domain is split into $k\times k$ small blocks, and we randomly select $k$ among these $k^2$ blocks to be the testing set and ensure that there is only one test block in each row and column. Such kind of split is more realistic since in practice, often there is only the data from some of the regions, and we want to make predictions (interpolation) in a region with no data, based on the data from the neighboring ones. 

\begin{figure}[!h]
\centering
\begin{subfigure}{.5\textwidth}
  \centering
  \includegraphics[width=0.9\linewidth]{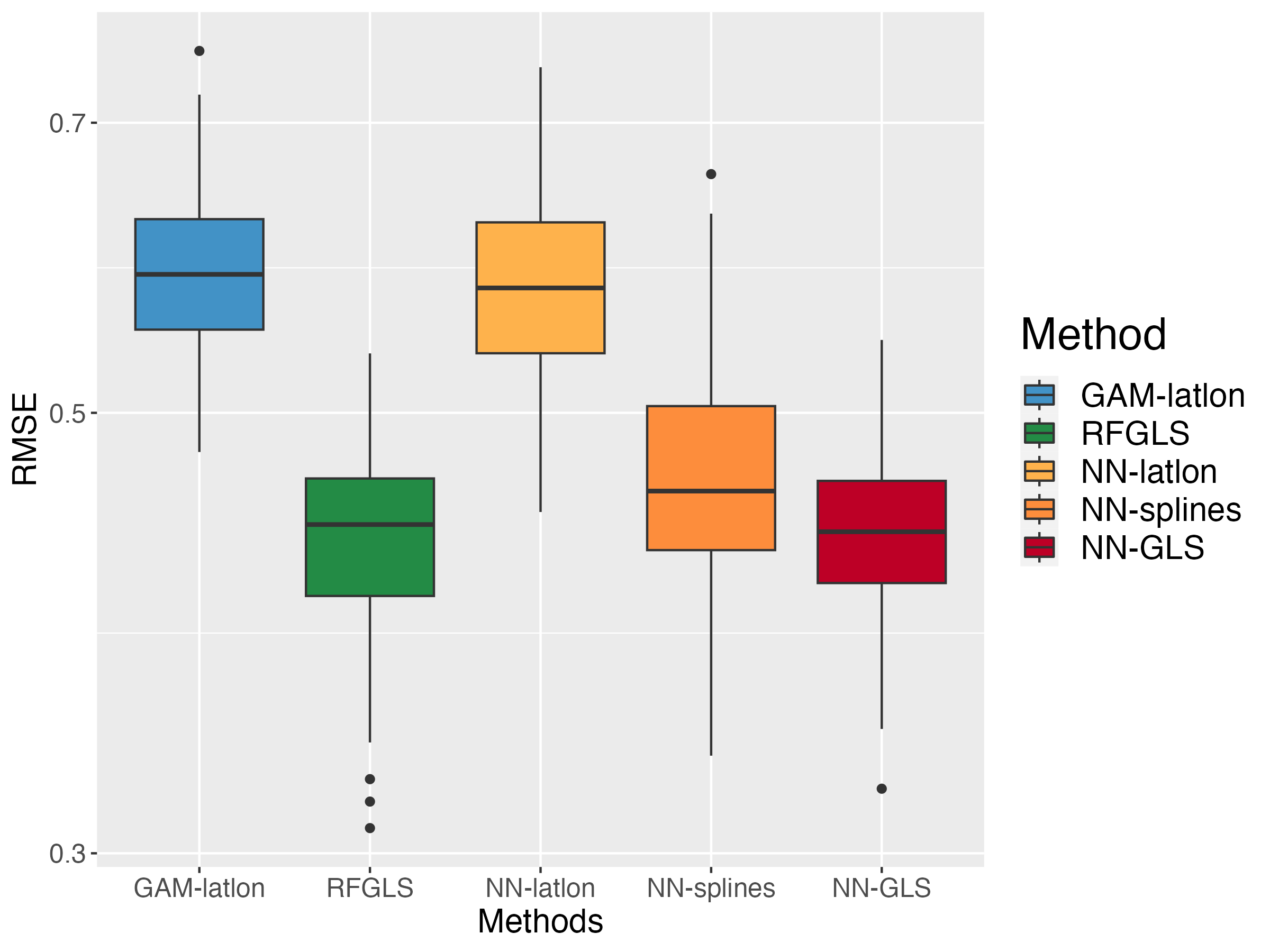}
  \caption{Purely random test set}
\end{subfigure}%
\begin{subfigure}{.5\textwidth}
  \centering
  \includegraphics[width=0.9\linewidth]{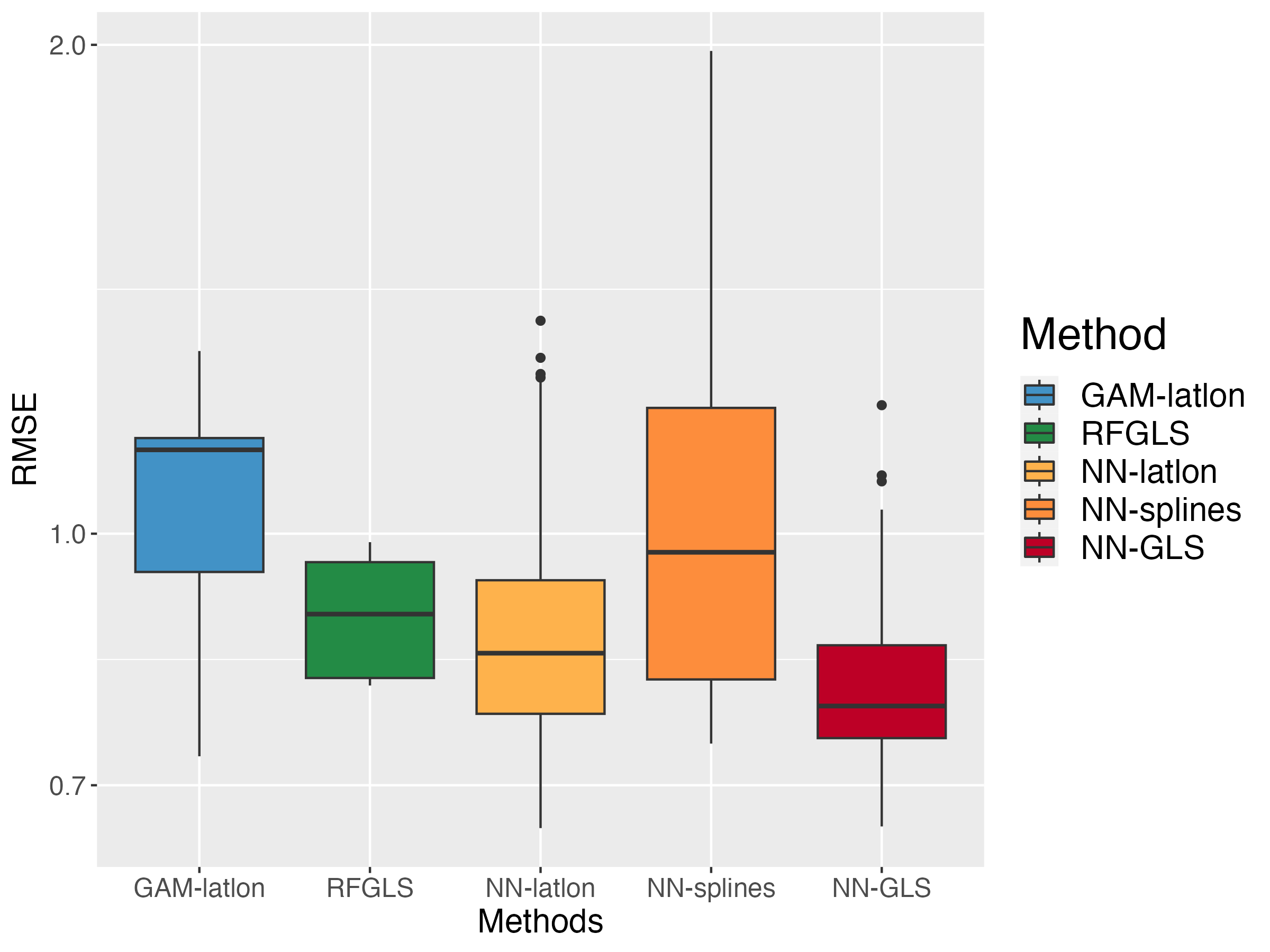}
  \caption{$3 \times 3$ block-random test set}
\end{subfigure}%
\caption{Results with different random test sets}
\label{fig-real-block}
\end{figure}

\blue{Figure \ref{fig-real-block} extend the $6 \times 6$ random block setting in Section \ref{sec-real} to two extremes, a purely random test set (i.e., each location is a considered a block) and large $3 \times 3$ blocks. In the purely random case, NN-GLS is doing comparably well as NN-splines and RF-GLS. The reason for NN-splines' catch-up is that when the test set is randomly selected, there's enough local information for the spatial effect at any location. While in the $3 \times 3$ block-random case, spatial effect at other blocks won't help the local prediction. The performance relies on capturing the global trend $f_0(x)$ and estimating the latent spatial correlation structure. As a result, NN-GLS outperforms the others and NN-splines produce larger errors.}

\subsection{Performance on other days}\label{sec:real_days}
The PM$_{2.5}$ distribution varies on different days (see Figure \ref{fig-real-4}). To assess the robustness of the prediction performance of the methods to choice of the day, we analyze PM$_{2.5}$ data on two other days with considerably different spatial structure in PM$_{2.5}$. \blue{We analyze data from July 4th, 2022, and June 5th, 2019, both with random training-testing splitting (we select June 5th, 2019 since it was used in \cite{wang2019nearest}). On these two days, 
there is data with PM$_{2.5}$ concentration (ug/m3) from $1257(842)$ stations across the states. The six meteorological variables are provided at $13612(7706)$ grid cells from NARR, which leads to $729(605)$ data points after preprocessing and aligning the two data sources (see Section \ref{sec-real}).} 

\begin{figure}[!h]
\centering
\begin{subfigure}{.5\textwidth}
  \centering
  \includegraphics[width=0.9\linewidth]{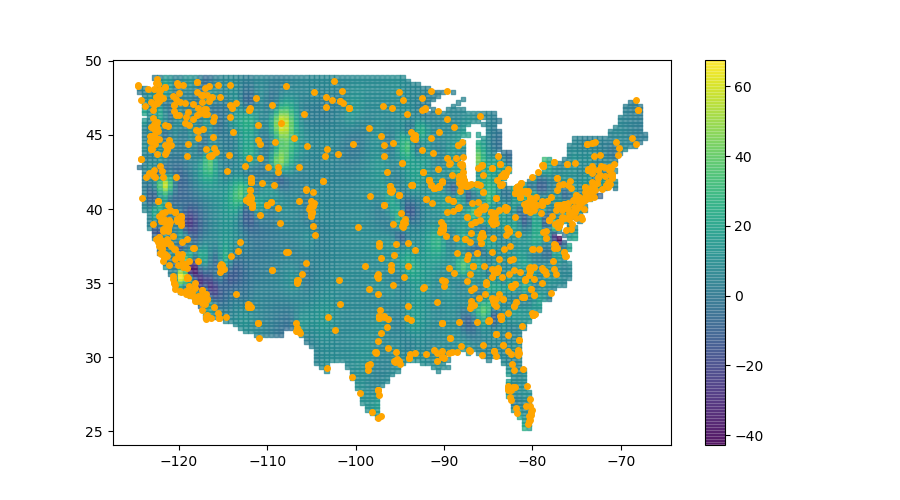}
  \caption{July 4th, 2022}
\end{subfigure}%
\begin{subfigure}{.5\textwidth}
  \centering
  \includegraphics[width=0.9\linewidth]{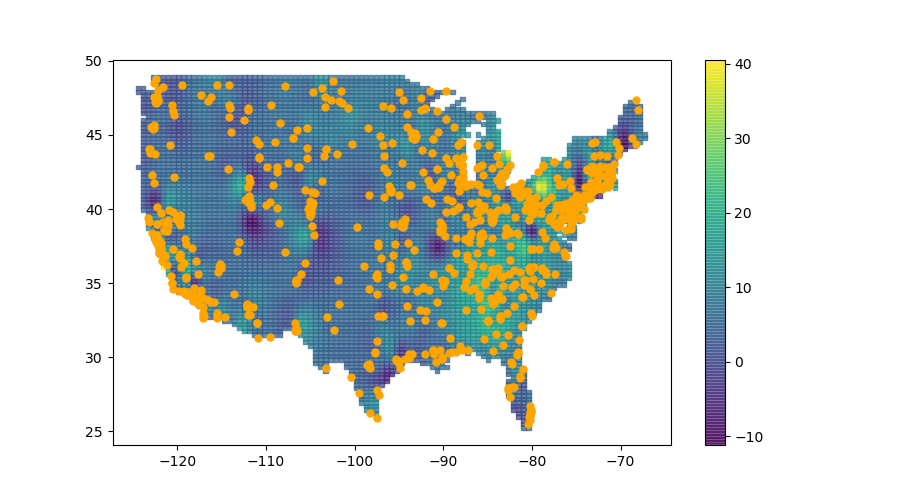}
  \caption{June 5th, 2019}
\end{subfigure}
\caption{Interpolated PM$_{2.5}$ level in the U.S. on the analyzed days.}
\label{fig-real-4}
\end{figure}

 \blue{The prediction results for these additional days are given in Figure \ref{fig-real-3}, among which June 5th, 2019 is the date used in \cite{chen2020deepkriging}. 
 On both dates, GAM-latlon and NN-latlon are significantly worse than our method, RF-GLS, and NN-splines. NN-GLS has a slight advantage over the latter two methods.}
\begin{figure}[!h]
\centering
\begin{subfigure}{.5\textwidth}
  \centering
  \includegraphics[width=0.9\linewidth]{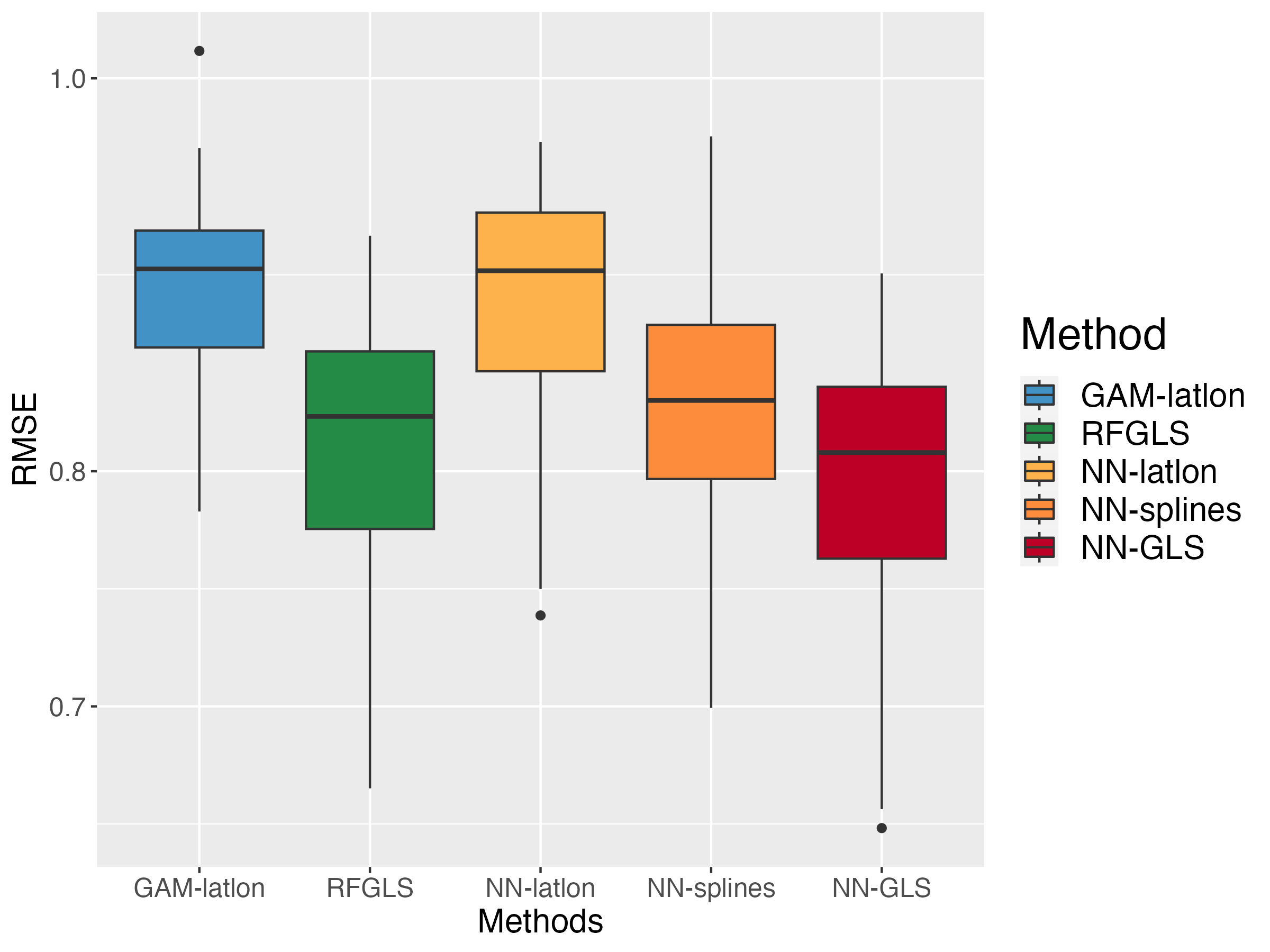}
  \caption{Prediction performance on July 4th 2022}
\end{subfigure}%
\begin{subfigure}{.5\textwidth}
  \centering
  \includegraphics[width=0.9\linewidth]{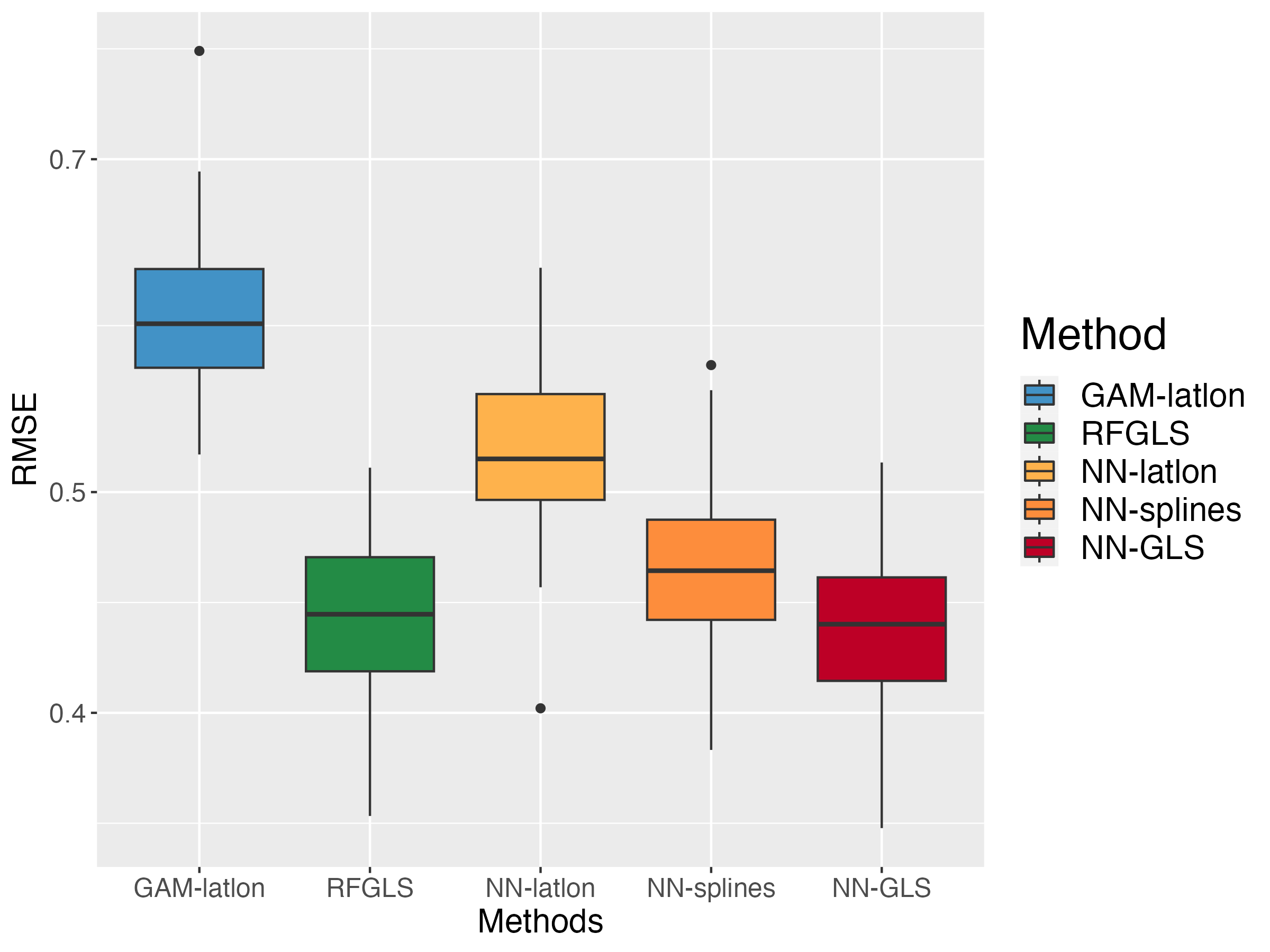}
  \caption{Prediction performance on June 5th 2019}
\end{subfigure}
\caption{Additional prediction performance comparison.}
\label{fig-real-3}
\end{figure}

\subsection{Prediction interval}\label{sec-real-PI}

\blue{In Figure \ref{fig-real-PI}, we present the prediction interval produced by NN-GLS and RF-GLS, as these are the two best methods.

\begin{figure}[!h]
\centering
\begin{subfigure}{.5\textwidth}
  \centering
  \includegraphics[width=0.9\linewidth]{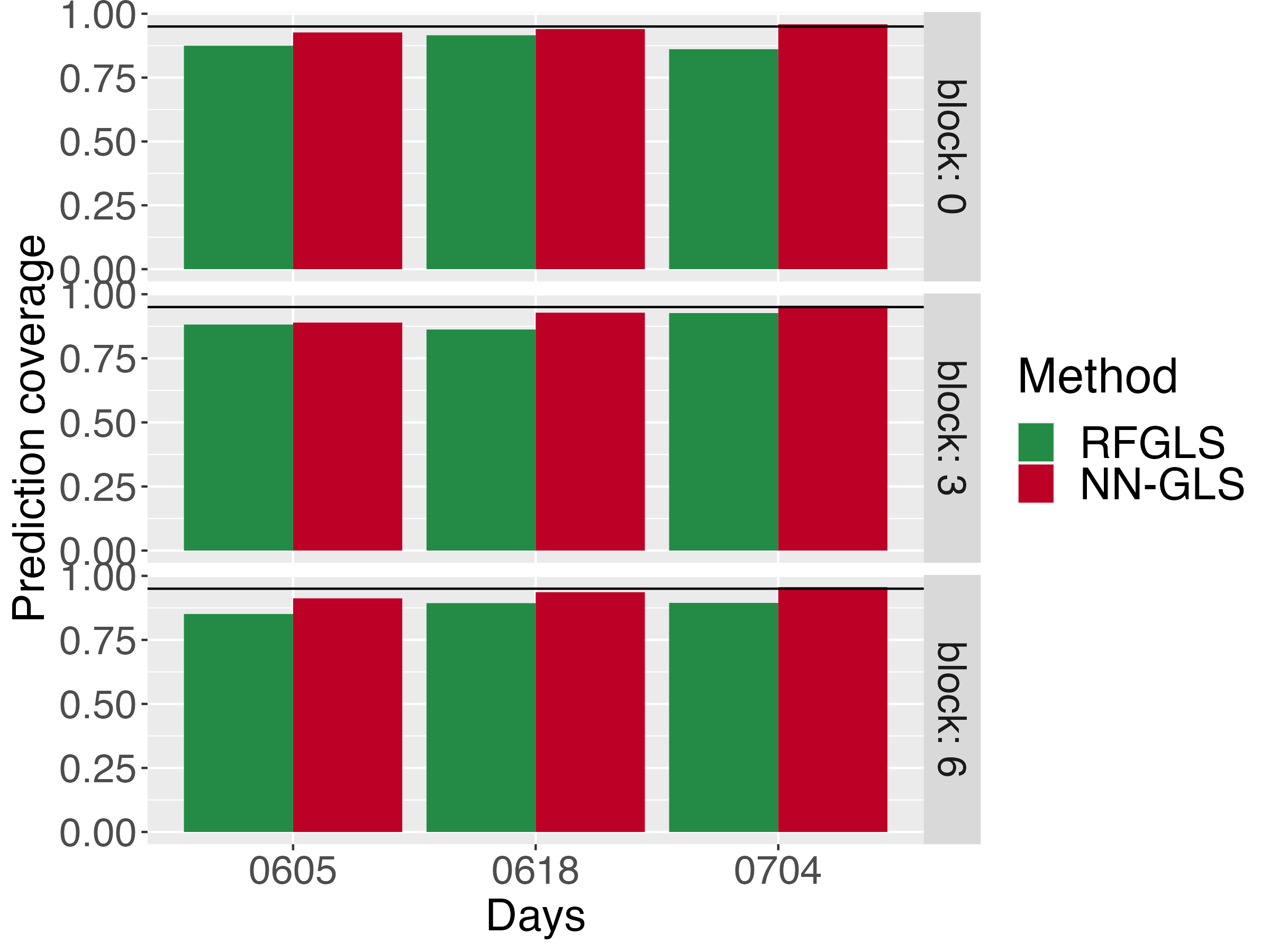}
  \caption{Prediction coverage}
\end{subfigure}%
\begin{subfigure}{.5\textwidth}
  \centering
  \includegraphics[width=0.9\linewidth]{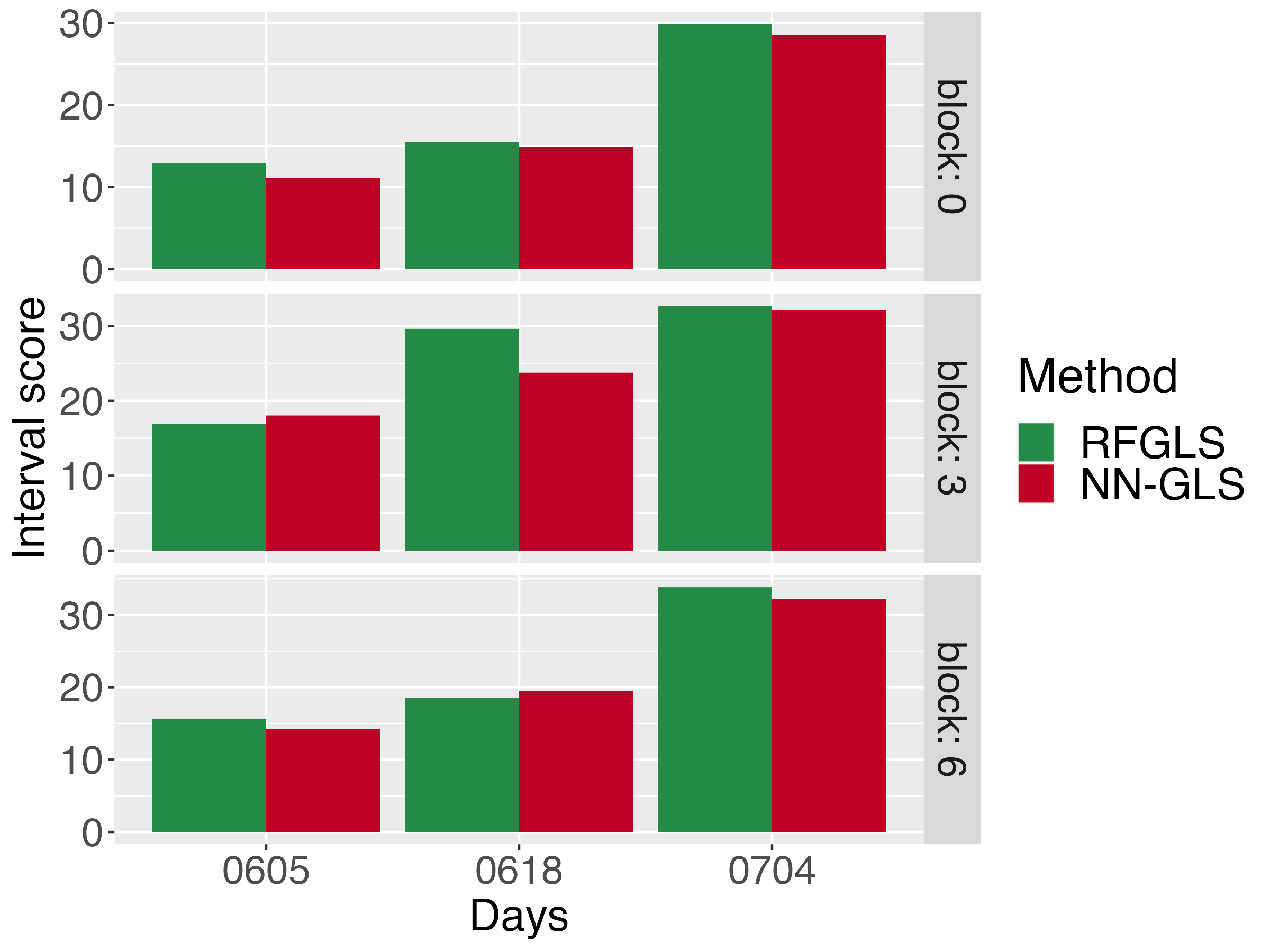}
  \caption{Prediction interval score}
\end{subfigure}
\caption{\blue{Prediction interval's performance on different days, with different splitting strategies. }}
\label{fig-real-PI}
\end{figure}
Similar to the findings for simulated data in Figures \ref{fig-sim-PI} (c-d), we see that RF-GLS is slightly undercovering while NN-GLS produces near nominal coverage. The interval scores are similar for the two methods.

We note that it is unclear how to produce spatially informed prediction intervals for NN-splines type added-spatial-fearures approaches as simply using na\"ive bootstrapping or estimate of the standard error to construct intervals would ignore the spatial correlation. 
} 

\subsection{Assumptions check}\label{sec:asmp-check}
\blue{We perform diagnostic checks to see if the data conforms to normality and spatial dependency structures used in the analysis models. For the normality assumption, we plot a histogram of the prediction residuals from NN-GLS, corresponding to the random noise process. For the spatial dependency, we check the semi-variogram plot of the estimation residual, corresponding to the spatial process plus the random noise process. Figure \ref{fig-real-asmp} shows that the prediction residuals are close to normal, being slightly more spiked (under-dispersed) for June 18th and July 4th in 2022. Thus the normality assumption is reasonable (and conservative). In Figure \ref{fig-real-asmp}, the semivariance generally shows the spatial structure in the data with an increasing pattern as a function of distance, except at the larger distances where there is little data. Among the three days, June 5th, 2019's data is the closest to both the modeling assumptions of normality and spatial dependence captured by an exponential covariance, while for other days the modeling assumptions are slightly misspecified with the error process being underdispersed than Gaussian, and the variogram showing some variability at larger distances. Despite this misspecification, NN-GLS performs consistently well both for point and interval predictions. Also in simulation studies in Sections \ref{sec-sim-mis1} and \ref{sec-sim-mis2}, we have concluded that NN-GLS is robust under various forms of model misspecifications, providing confidence in the results for NN-GLS for the PM$_{2.5}$ data.}
\begin{figure}[!h]
\centering
\begin{subfigure}{.5\textwidth}
  \centering
  \includegraphics[width=0.9\linewidth]{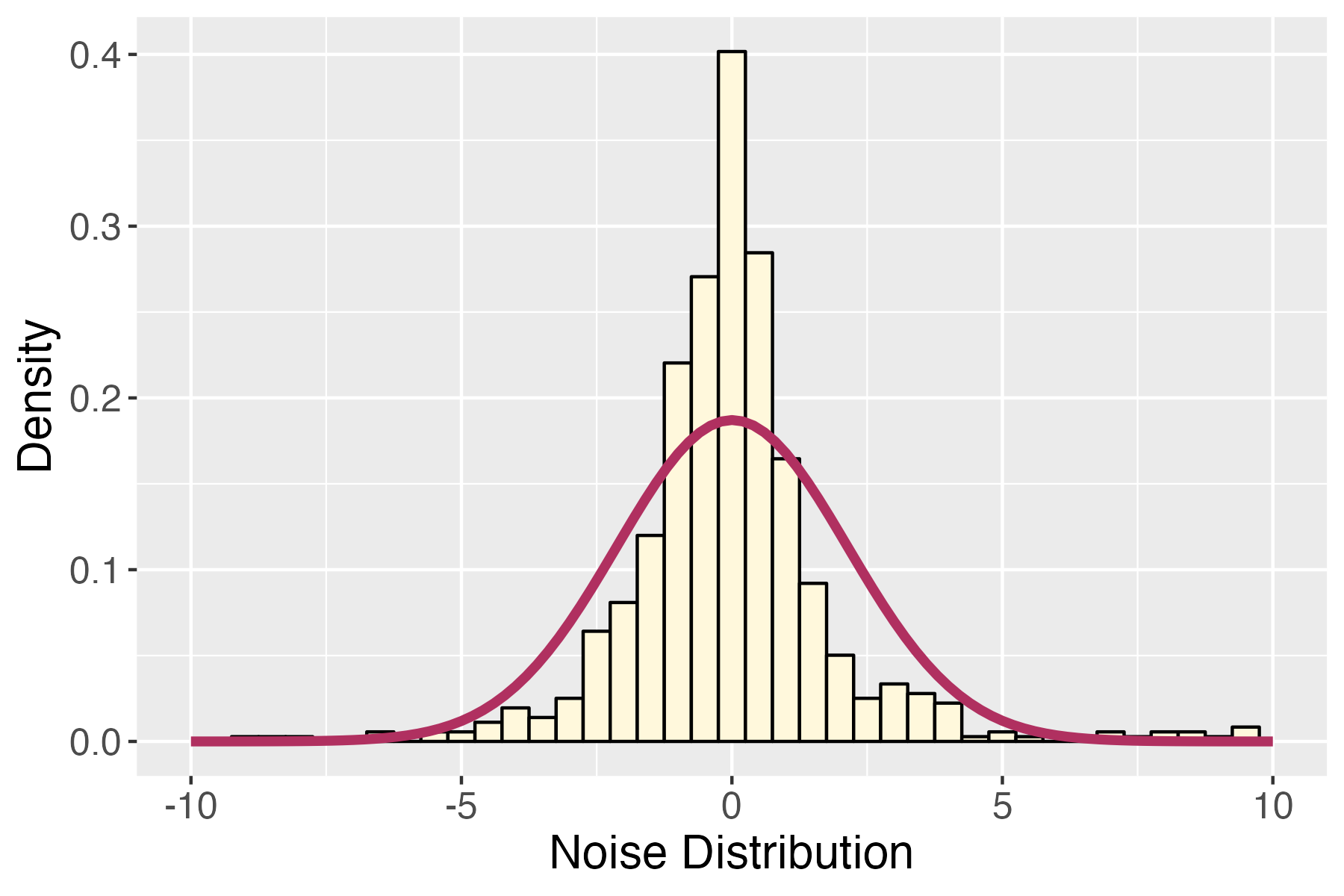}
  \caption{Noise distribution on June 18th 2022}
\end{subfigure}%
\begin{subfigure}{.5\textwidth}
  \centering
  \includegraphics[width=0.9\linewidth]{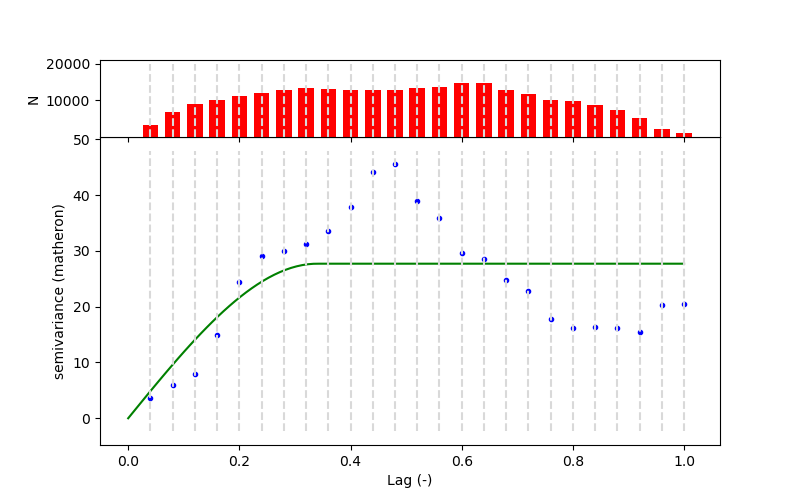}
  \caption{Semi-variogram on June 18th 2022}
\end{subfigure}%
\\
\centering
\begin{subfigure}{.5\textwidth}
  \centering
  \includegraphics[width=0.9\linewidth]{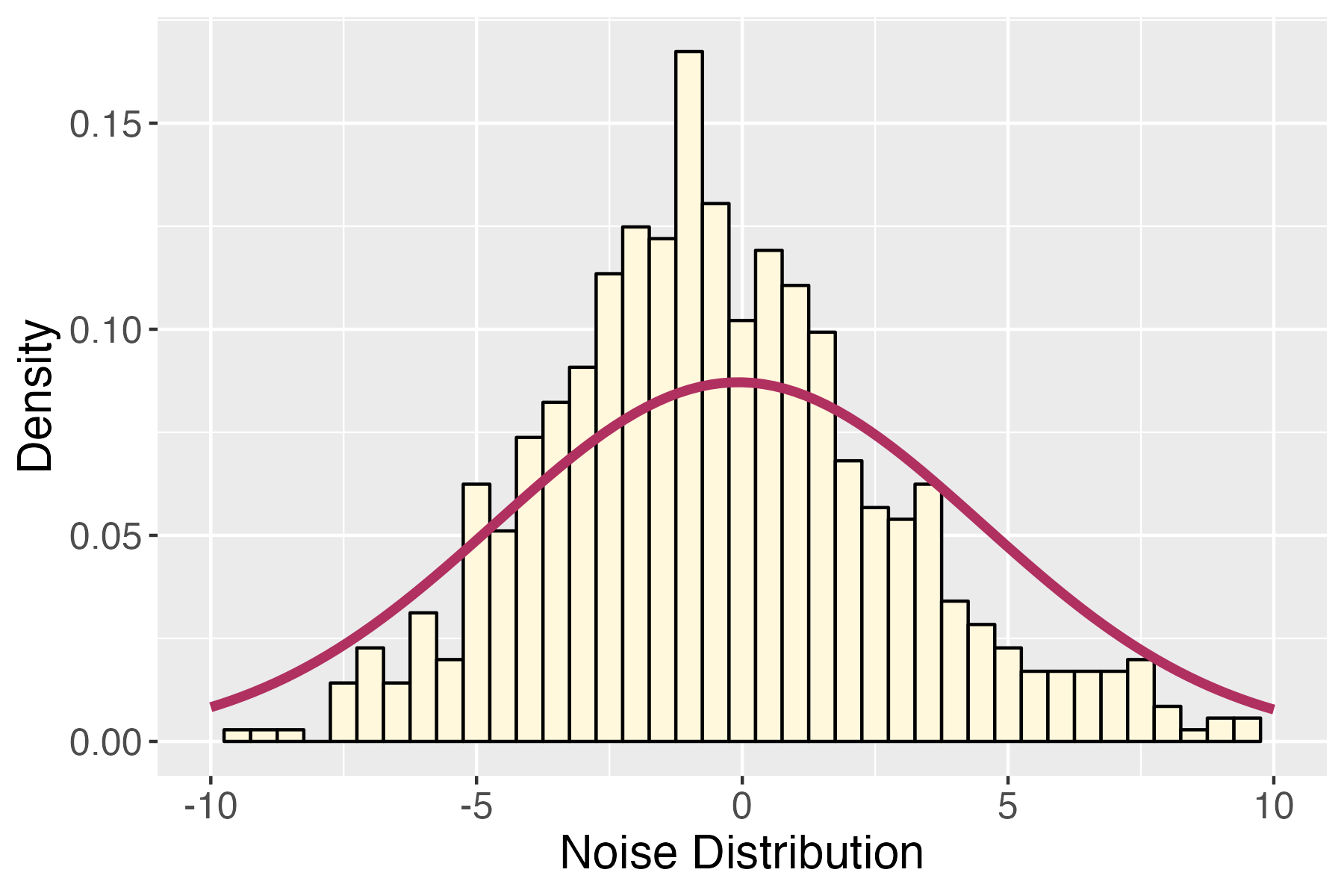}
  \caption{Noise distribution on July 4th 2022}
\end{subfigure}%
\begin{subfigure}{.5\textwidth}
  \centering
  \includegraphics[width=0.9\linewidth]{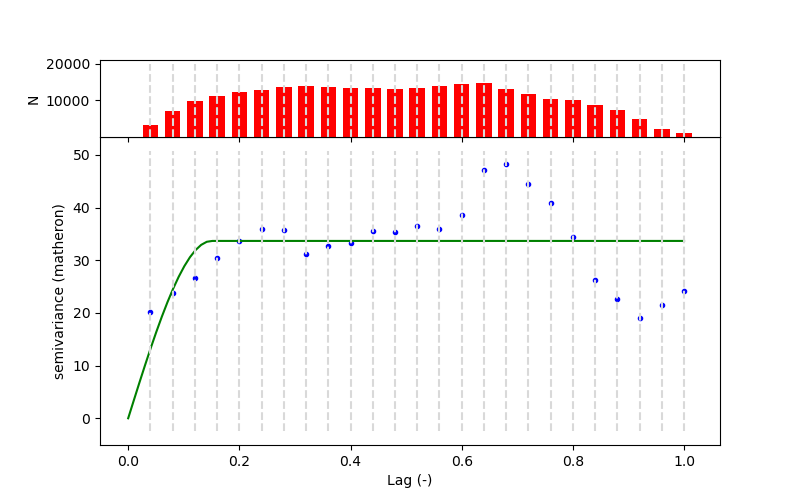}
  \caption{Semi-variogram on July 4th 2022}
\end{subfigure}%
\\
\centering
\begin{subfigure}{.5\textwidth}
  \centering
  \includegraphics[width=0.9\linewidth]{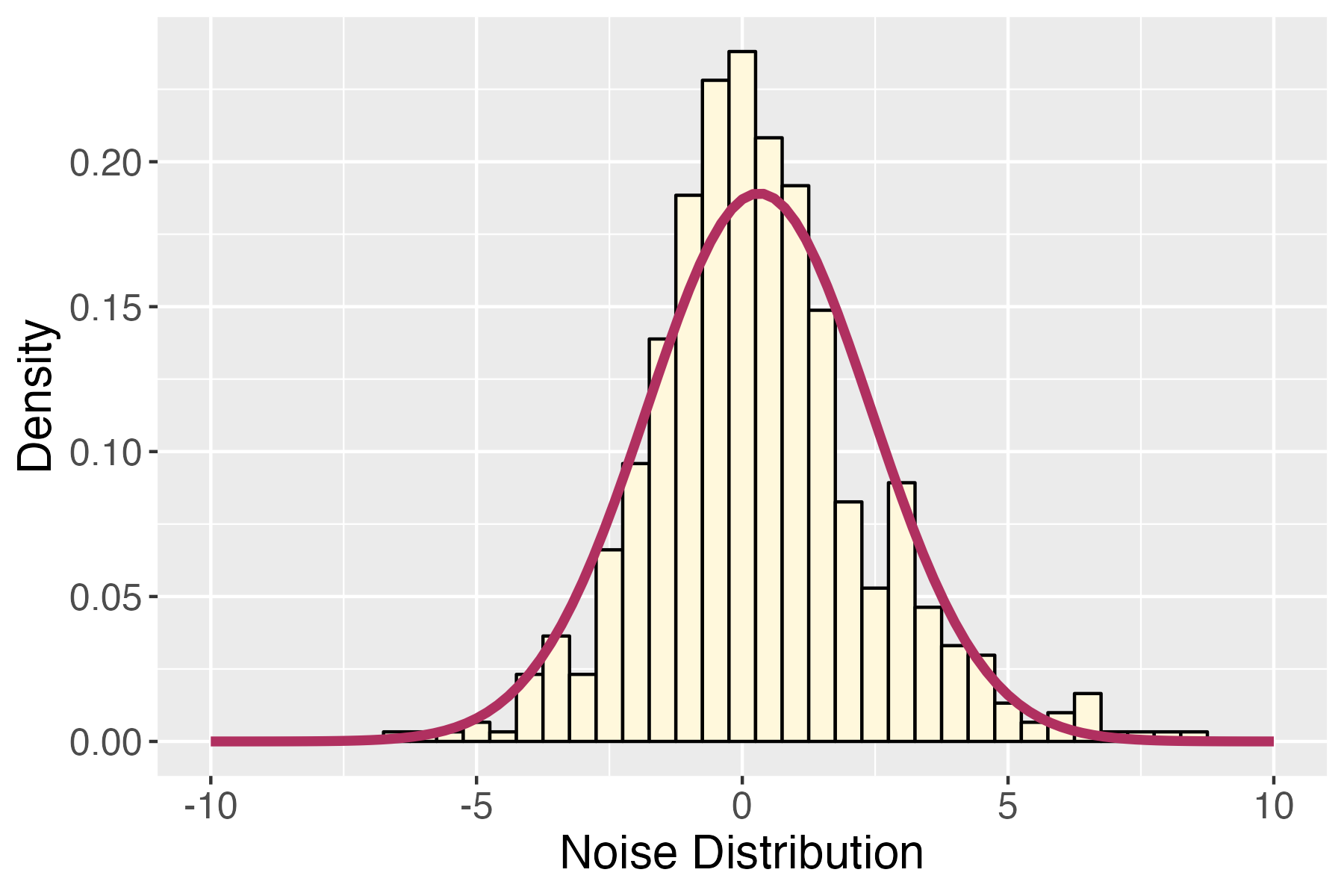}
  \caption{Noise distribution on June 5th 2019}
\end{subfigure}%
\begin{subfigure}{.5\textwidth}
  \centering
  \includegraphics[width=0.9\linewidth]{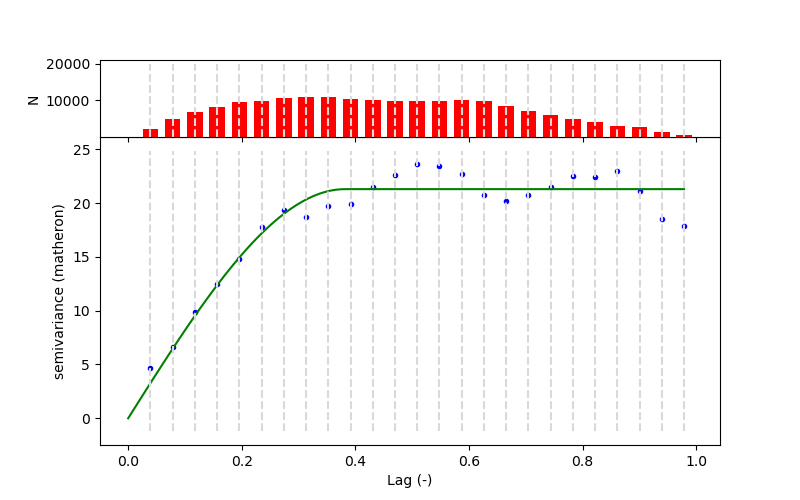}
  \caption{Semi-variogram on June 5th 2019}
\end{subfigure}
\caption{\blue{Assumption check for data on different days}}
\label{fig-real-asmp}
\end{figure}

\subsection{Partial dependence plot (PDP) illustration}\label{sec:PDP}
\blue{We briefly introduce the partial dependence plot (PDP) as an auxiliary tool for our algorithm. PDP shows the marginal effect one or two features have on the predicted response of a machine learning model \citep{friedman2001greedy}, and the idea behind it is simply integrating the nuisance features. If $\hat{f}(\cdot)$ is the estimated function from the model, and $(\bX_{1}, \bX_{2})$ is the feature space, $\bX_1$ are the features that we want to see their effect on the prediction, while $\bX_{2}$ are the features whose effects are not interested.
\begin{definition}\label{def-PDP}
The partial dependence function (PDF) is defined as:
\begin{equation}
\hat{f}_1(\bX_1) = \EE_{\bX_2}\big[\hat{f}(\bX_1, \bX_2)\big] \approx \frac{1}{n}\sum_{i=1}^n\hat{f}(\bX_1, \bX_2^{(i)}).
\end{equation}
\end{definition}

Figure \ref{fig-PDP} contains the PDPs for all 6 covariates as a complement to the ones in the main content. As is shown in the plots, NNGLS captures considerable non-linear effects from at least two of the variables including relative humidity, and U-wind. These non-linear effects can not be captured by the traditional linear geostatistical models (\ref{eq:stoc}) and demonstrate the need to expand the geostatistical modeling machinery to hybrid approaches like NN-GLS for spatial data with non-linear relationships.} 
\begin{figure}[!h]
\centering
\begin{subfigure}{.3\textwidth}
  \centering
  \includegraphics[width=0.9\linewidth]{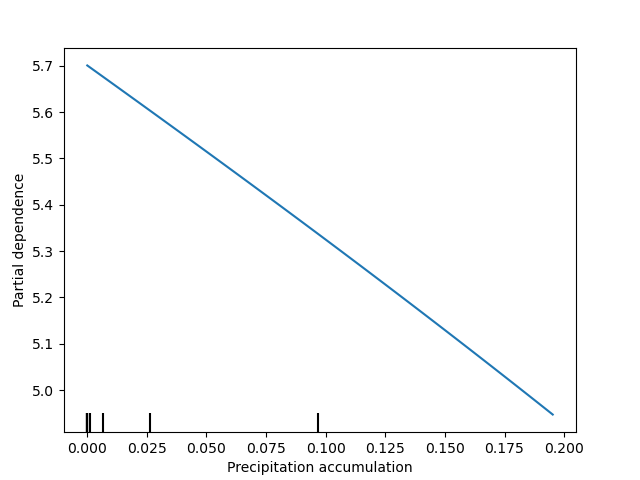}
  \caption{Precipitation accumulation}
\end{subfigure}
\begin{subfigure}{.3\textwidth}
  \centering
  \includegraphics[width=0.9\linewidth]{Simulation/Realdata/Air_temperature_NNGLS.png}
  \caption{Air temperature}
\end{subfigure}%
\begin{subfigure}{.3\textwidth}
  \centering
  \includegraphics[width=0.9\linewidth]{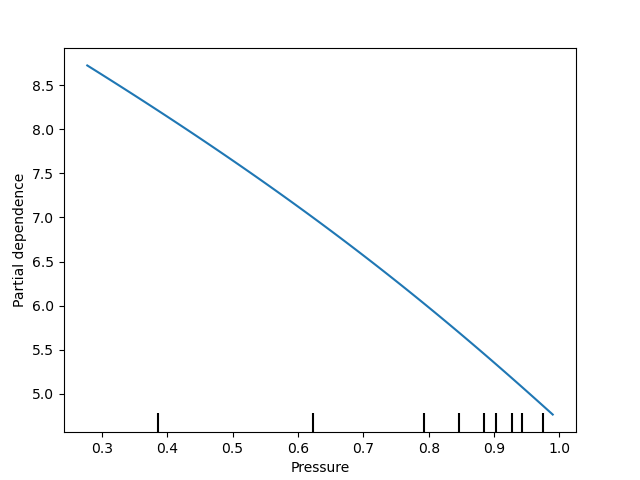}
  \caption{Pressure}
\end{subfigure}%
\\
\begin{subfigure}{.3\textwidth}
  \centering
  \includegraphics[width=0.9\linewidth]{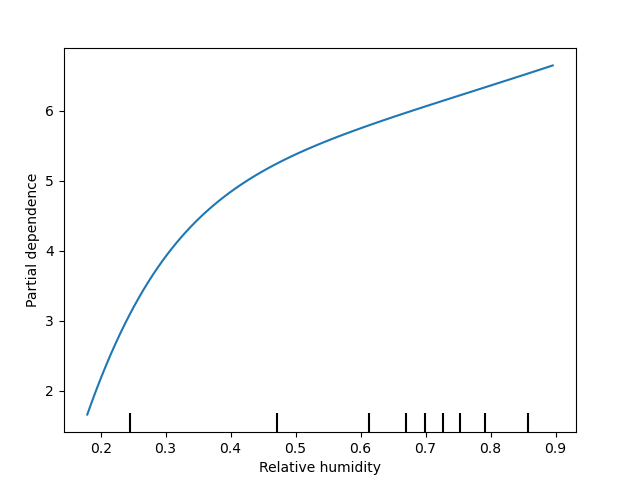}
  \caption{Relative humidity}
\end{subfigure}%
\begin{subfigure}{.3\textwidth}
  \centering
  \includegraphics[width=0.9\linewidth]{Simulation/Realdata/U-wind_NNGLS.png}
  \caption{U-wind}
\end{subfigure}%
\begin{subfigure}{.3\textwidth}
  \centering
  \includegraphics[width=0.9\linewidth]{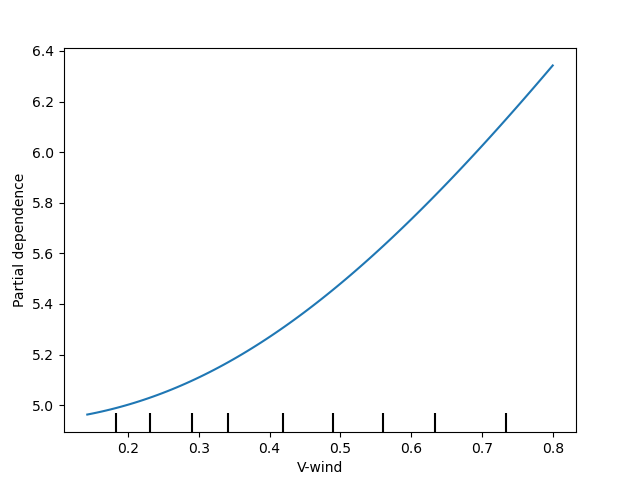}
  \caption{V-wind}
\end{subfigure}%
\caption{PDPs showing the marginal effects.}
\label{fig-PDP}
\end{figure}

\end{document}